\documentclass[a4paper, 11pt, openright, twoside]{reportPhD}

\usepackage[T1]{fontenc}
\usepackage{amssymb}
\usepackage{amsthm}
\usepackage[pdftex]{graphicx}
\usepackage{epstopdf}
\usepackage{color}
\usepackage{amsmath}
\usepackage{txfonts}
\usepackage{amsfonts}
\usepackage{amssymb}
\usepackage{caption}
\usepackage{graphicx}
\usepackage{subfigure}
\usepackage{float}
\usepackage{multirow}
\usepackage{lmodern}
\usepackage[final]{microtype}
\usepackage[utf8]{inputenc}
\usepackage{array,mathtools}
\usepackage[bookmarks]{hyperref}
\usepackage{xcolor}
\usepackage[inner=3cm, outer=2cm, bottom=3cm]{geometry}
\usepackage{enumerate}
\usepackage[inline,shortlabels]{enumitem}
\usepackage{xfrac}
\usepackage{booktabs}
\usepackage{multirow}

\usepackage[toc,page]{appendix}
\usepackage{courier}
\usepackage{listings}
\usepackage{lipsum}

\usepackage{algpseudocode}
\usepackage{algorithm}

\usepackage{color,soul}
\usepackage[usestackEOL]{stackengine}[2013-09-11]

\setstackgap{L}{1.4\baselineskip}
\usepackage{booktabs}
\usepackage[load-configurations=version-1]{siunitx} 
\usepackage{tikz}
\usetikzlibrary{shapes,snakes}
\usetikzlibrary{bayesnetID}

\usepackage{bm}
\usepackage{arydshln}
\usepackage{tabularx}

\usepackage[draft]{minted}
\usepackage{fancyvrb}
\usetikzlibrary{external}

\graphicspath{{./img/}}

\newcommand{\dd}[1]{\text{d}#1}

\usepackage{multicol}
\usepackage{multirow}
\usepackage{tabularx}
\usepackage{booktabs}
\usepackage{setspace}
\usepackage{hyperref}

\algblockdefx[Train]{Train}{EndTrain}[1][]{\textbf{Training} #1}{\textbf{End Training}}
\algblockdefx[Preprocessing]{Preprocessing}{EndPreprocessing}[1][]{\textbf{Preprocessing} #1}{\textbf{End Preprocessing}}
\algblockdefx[Operation]{Operation}{EndOperation}[1][]{\textbf{Operation} #1}{\textbf{End Operation}}

\usepackage[capitalize]{cleveref}
\crefname{appcha}{Appendix}{Appendices}

\usepackage[backend=bibtex8, 
            style=authoryear, 
            citestyle=authoryear,
            maxbibnames=15,
            maxcitenames=2, 
            doi=false,
            giveninits=true,
            hyperref]{biblatex}
            
\DeclareNameAlias{sortname}{last-first}

\DeclareBibliographyCategory{mypapers}

\usepackage{phd}

\hypersetup{colorlinks=false, pdfborder={0 0 0}}

\numberwithin{figure}{chapter}
\numberwithin{table}{chapter}

\title{Contributions to Large Scale Bayesian Inference and Adversarial Machine Learning \\
\vspace{0.3cm}
{\Large (Contribuciones a la Inferencia Bayesiana a Gran Escala y al \\ Aprendizaje Automático Adversario) }}
\author{Víctor Gallego Alcalá}
\newcommand{\director}{David Ríos Insua and David Gómez-Ullate Oteiza}

\usepackage{chngcntr}
\counterwithin{equation}{chapter}
\numberwithin{equation}{chapter}

\AtBeginEnvironment{subappendices}{%
\chapter*{Appendix}
\addcontentsline{toc}{chapter}{Appendices}
\counterwithin{figure}{chapter}
\counterwithin{table}{chapter}
\counterwithin{algorithm}{chapter}
}

\DeclareCiteCommand{\citejournal}
  {\usebibmacro{prenote}}
  {\usebibmacro{citeindex}%
    \usebibmacro{journal}}
  {\multicitedelim}
  {\usebibmacro{postnote}}

\begin{document}

\titulo

\cleardoublepage
\begin{description}[labelwidth=\widthof{\textbf{Department:}}, leftmargin=!, labelsep=2em]
\item[Department:] Estadística e Investigación Operativa \\ Facultad de Ciencias Matemáticas \\ Universidad Complutense de Madrid (UCM) \\ Spain
\item[Title:] Contributions to Large Scale Bayesian Inference and Adversarial Machine Learning
\item[Author:] Víctor Gallego Alcalá
\item[Advisors:] David Ríos Insua and David Gómez-Ullate Oteiza
\item[Date:] June 2021
\end{description}

\chapter*{Abstract}

The field of machine learning (ML) has experienced a major boom in the past years, both in theoretical developments and application areas. However, the rampant adoption of ML methodologies has revealed that models are usually adopted to make decisions without taking into account the uncertainties in their predictions. More critically, they can be vulnerable to adversarial examples, strategic manipulations of the data with the goal of fooling those systems. For instance, in retailing, a model may predict very high expected sales for the next week, given a certain advertisement budget. However, the predicted variance may also be quite big, thus making the prediction almost useless depending on the risk tolerance of the company. Similarly, in the case of spam detection, an attacker may insert additional words in a given spam email to evade being classified as spam by making it to appear more legit.
Thus, we believe that developing ML systems that take into account predictive uncertainties and are robust against adversarial examples is a must for critical, real-world tasks. This thesis is a step towards achieving this goal. 

In Chapter 1, we start with a case study in retailing. We propose a robust implementation of the Nerlove–Arrow model using a Bayesian structural time series model to explain the relationship between advertising expenditures
of a country-wide fast-food franchise network with its weekly sales. Its Bayesian nature facilitates incorporating prior information reflecting the manager’s views, which can be updated with relevant data. However, this case study adopted classical Bayesian techniques, such as the Gibbs sampler. Nowadays, the ML landscape is pervaded with complex models, huge in the number of parameters. This is the realm of neural networks and this chapter also surveys current developments in this sub-field. In doing this, three challenges that constitute the core of this thesis are identified.

Chapter 2 is devoted to the first challenge. In it, we tackle the problem of scaling Bayesian inference to complex models and large data regimes. In the first part, we propose a unifying view of two different Bayesian inference algorithms, Stochastic Gradient Markov Chain Monte Carlo (SG-MCMC) and Stein Variational Gradient Descent (SVGD), leading to improved and efficient novel sampling schemes. In the second part, we develop a framework to boost the efficiency of Bayesian inference in probabilistic models by embedding a Markov chain sampler within a variational posterior approximation. We call this framework “variationally inferred sampling”. This framework has several benefits, such as its ease of implementation and the automatic tuning of sampler parameters, leading to a faster mixing time through automatic differentiation. Experiments show the superior performance of both developments compared to baselines. In addition, both could be combined to further improve the results.

In Chapter 3, we address the challenge of protecting ML classifiers from adversarial examples. So far, most approaches to adversarial classification have followed a classical game-theoretic framework. This requires unrealistic common knowledge conditions untenable in the security settings typical of the adversarial ML realm. After reviewing such approaches, we present an alternative perspective on AC based on adversarial risk analysis, and leveraging the scalable Bayesian approaches from chapter 2.

In Chapter 4, we turn our attention form supervised learning to reinforcement learning (RL), addressing the challenge of supporting an agent in a sequential decision making setting where there can be adversaries, specifically modelled as other players. We introduce Threatened Markov Decision Processes (TMDPs) as an extension of the classical Markov Decision Process framework for RL. We also propose a level-$k$ thinking scheme resulting in a novel learning approach to deal with TMDPs. After introducing our framework and deriving theoretical results, relevant empirical evidence is given via extensive experiments, showing the benefits of accounting for adversaries in RL while the agent learns.

Finally, Chapter 5 sums up with several conclusions and avenues for further work.
\vspace{1cm}

The following papers derived from this thesis have been published or already accepted:
\begin{itemize}
    \item \cite{gallego2019dlms}. \citetitle{gallego2019dlms}. In \citejournal{gallego2019dlms}.
    \item \cite{gallego2019reinforcement}. \citetitle{gallego2019reinforcement}. In \citejournal{gallego2019reinforcement}.
     \item \cite{gallego2019vis}. \citetitle{gallego2019vis}. In \citejournal{gallego2019vis}.   
     \item \cite{math8111957}. \citetitle{math8111957}. In \citejournal{math8111957}.
     \item \cite{nn2022}. \citetitle{nn2022}. In \citejournal{nn2022} (to appear).
\end{itemize}

The following papers derived from this thesis are under submission:
\begin{itemize}
    \item \citeauthor{gallego2019opponent}. \citetitle{gallego2019opponent}. 
        \item \citeauthor{gallego2021data}. \citetitle{gallego2021data}. 
    \item \citeauthor{gallego2018stochastic}. \citetitle{gallego2018stochastic}. 
    \item \citeauthor{AMLARA}. \citetitle{AMLARA}. 
     \item \citeauthor{gallego2020protecting}. \citetitle{gallego2020protecting}. 
\end{itemize}

The following papers related to the contents of this thesis were also published:
\begin{itemize}
    \item \cite{angulo2018bayesian}. \citetitle{angulo2018bayesian}. In \citejournal{angulo2018bayesian}.
    \item \cite{banks2020adversarial}. \citetitle{banks2020adversarial}. In \citejournal{banks2020adversarial}.
\end{itemize}

\chapter*{Resumen}

El campo del aprendizaje automático (AA) ha experimentado un auge espectacular en los últimos años, tanto en desarrollos teóricos como en áreas de aplicación. Sin embargo, la rápida adopción  de las metodologías del AA ha mostrado que los modelos que habitualmente se emplean para toma de decisiones no tienen en cuenta la incertidumbre en sus predicciones o, más crucialmente, pueden ser vulnerables a ejemplos adversarios, datos manipulados estratégicamente con el objetivo de engañar estos sistemas de AA. Por ejemplo, en el sector de la hostelería, un modelo puede predecir unas ventas esperadas muy altas para la semana que viene, fijado cierto plan de inversión en publicidad. Sin embargo, la varianza predictiva también puede ser muy grande, haciendo la predicción escasamente útil según el nivel de riesgo que el negocio pueda tolerar. O, en el caso de la detección de spam, un atacante puede introducir palabras adicionales en un correo de spam  para evadir el ser clasificado como tal y aparecer  legítimo.
Por tanto, creemos que desarrollar sistemas de AA que puedan tener en cuenta también las incertidumbres en las predicciones y ser más robustos frente a ejemplos adversarios es una necesidad para tareas críticas en el mundo real. Esta tesis es un paso hasta alcanzar este objetivo.

En el capítulo 1, empezamos con un caso de estudio en el sector de la hostelería. Proponemos una implementación robusta del modelo de Nerlove-Arrow usando un modelo  estructural bayesiano de series temporales para explicar la relación entre las inversiones en publicidad con las ventas semanales de una cadena nacional de restaurantes de comida rápida. Su naturaleza bayesiana facilita  incorporar conocimiento a priori que refleje las creencias del gestor, y pueden  actualizarse con datos observados. Sin embargo, este caso de estudio emplea técnicas bayesianas ya clásicas, como el muestreador de Gibbs. Hoy en día, el panorama del AA está repleto de modelos complejos, enormes en cuanto a número de parámetros. Este es caso de las redes neuronales, así que en este capítulo también resumimos los avances recientes en este subcampo. Tres desafíos constituyen el cuerpo de esta tesis.

El capítulo 2 va dedicado al primer desafío. En él, atacamos el problema de escalar la inferencia Bayesiana a modelos complejos o regímenes de grandes datos. En la primera parte, proponemos una visión unificadora de dos algoritmos de inferencia Bayesiana, Monte Carlo mediante cadenas de Markov con Gradientes Estocásticos y Descenso por el Gradiente Variacional Stein, llegando a  esquemas mejorados y eficientes de inferencio. En la segunda parte, desarrollamos una metodología para mejorar la eficiencia de la inferencia bayesiana mediante el anidamiento de un muestreador basado en cadenas de Markov dentro de una aproximación variacional. A esta metodología la llamamos "aproximación variacional refinada". La metodología conlleva varios beneficios, como su facilidad de implementación y el ajuste automático de los hiperparámetros del muestreador, logrando tiempos de convergencia más rápidos gracias a la diferenciación automática. Los experimentos muestran el rendimiento superior de ambos desarollos comparado con algunas alternativas.

En el capítulo 3, nos centramos en el desafío de proteger clasificadores de AA de los ejemplos adversarios. Hasta ahora, la mayoría de enfoques en clasificación adversaria han seguido el paradigma clásico de teoría de juegos. Esto requiere condiciones poco realistas de conocimiento común, que no son admisibles en entornos típicos en seguridad del aprendizaje automático adversario. Tras revisar estos enfoques, presentamos una perspectiva alternativa basada en análisis de riesgos adversarios y aprovechamos las técnicas bayesianas escalables del capítulo 3.

En el capítulo 4, pasamos nuestra atención del aprendizaje supervisado al aprendizaje por refuerzo (AR), incidiendo en el desafío de apoyar un agente en un escenario de toma de decisiones secuenciales en el que puede haber adversarios, modelizados como otros jugadores. Introducimos los Procesos de Decisión de Markov Amenazados como una extensión del paradigma clásico de los procesos de decisión Markovianos. También proponemos un esquema basado en pensamiento de nivel-$k$ resultando en un nuevo algoritmo de aprendizaje. Tras introducir la metodología y derivar algunos resultados teóricos, damos evidencia empírica relevante mediante experimentos extensos, mostrando los beneficios de modelizar oponentes en AR mientras el agente aprende.

Finalmente, en el capítulo 5 terminamos con varias conclusiones y posibles extensiones para trabajo futuro.

\vspace{1cm}
Los siguientes artículos derivados de esta tesis ya han sido publicados o están aceptados:
\begin{itemize}
    \item \cite{gallego2019dlms}. \citetitle{gallego2019dlms}. En \citejournal{gallego2019dlms}.
    \item \cite{gallego2019reinforcement}. \citetitle{gallego2019reinforcement}. En \citejournal{gallego2019reinforcement}.
     \item \cite{gallego2019vis}. \citetitle{gallego2019vis}. En \citejournal{gallego2019vis}.   
     \item \cite{math8111957}. \citetitle{math8111957}. En \citejournal{math8111957}.
     \item \cite{nn2022}. \citetitle{nn2022}. En \citejournal{nn2022} (por aparecer).
\end{itemize}

Los siguientes artículos derivados de esta tesis se encuentran bajo revisión:
\begin{itemize}
    \item \citeauthor{gallego2019opponent}. \citetitle{gallego2019opponent}. 
        \item \citeauthor{gallego2021data}. \citetitle{gallego2021data}. 
    \item \citeauthor{gallego2018stochastic}. \citetitle{gallego2018stochastic}. 
    \item \citeauthor{AMLARA}. \citetitle{AMLARA}.
     \item \citeauthor{gallego2020protecting}. \citetitle{gallego2020protecting}. 
\end{itemize}

Los siguientes artículos relacionados con el contenido de la tesis también han sido publicados:
\begin{itemize}
    \item \cite{angulo2018bayesian}. \citetitle{angulo2018bayesian}. En \citejournal{angulo2018bayesian}.
    \item \cite{banks2020adversarial}. \citetitle{banks2020adversarial}. En \citejournal{banks2020adversarial}.
\end{itemize}

\addcontentsline{toc}{chapter}{Abstract}
\addcontentsline{toc}{chapter}{Resumen}

\chapter*{Agradecimientos}

Estos cuatro años de tesis se han pasado volando. Aunque se hayan materializado en parte en este documento, no puedo olvidarme de todas las personas que, de un modo u otro, han contribuido a este trabajo. \\

En primer lugar, debo agradecer enormemente la labor y apoyo de mis directores durante este período. A David Ríos, especialmente por su incansable atención y paciencia, sobre todo al leer mis textos y dudas; y a David Gómez-Ullate, por darme la oportunidad de empezar en este mundo de los datos hace ya cinco años. Gracias a los dos por todas las oportunidades. Les considero verdaderos mentores de los que he podido aprender muchísimo en estos años, no solo acerca de los temas de esta tesis, así que espero seguir manteniendo nuestra relación en el futuro. 
Asimismo, debo agradecer la labor del Prof. David Banks, quien me dio la oportunidad de hacer una estancia de investigación en Duke y SAMSI, resultando en una experiencia muy enriquecedora académica y vitalmente. También agradezco al Ministerio por la beca FPU16-05034, la cátedra AXA-ICMAT, el programa “Severo Ochoa”
para Centros de Excelencia en I+D, y la Fundación BBVA, entre otros. \\

También quería mencionar a los compañeros y amigos hechos en el grupo formado en el Datalab del ICMAT y el vecino IFT (y algunos por extensión ya, en Komorebi). Roi, David, Alberto R., Alberto T., Simón, Jorge, Bruno, April, Aitor, Alex, Nadir, Christian... 
Lamento que por la pandemia apenas nos veamos ya por el campus, pero siempre me acordaré de tantos buenos momentos. También quiero mostrar mi agradecimiento a los investigadores Pablo Angulo y Pablo Suárez, por haber tenido el placer de trabajar con ellos al principio de mi carrera. Y a Marta Sanz, por estar siempre dispuesta a resolver cualquier trámite. \\

Por último, mención especial merecen mis padres, Sotero y María, siempre dándome su apoyo y cuidado incondicional a pesar de todo. E Irene, por acompañarme siempre ahí y quien me hizo ver que todo es posible, con toda su ilusión, ingenio y magia. Gracias, os quiero.

\addcontentsline{toc}{chapter}{Agradecimientos}

\indice
\indicetablas
\indicefiguras
\indicealgoritmos


\chapter{Introduction}\label{cha:intro}
\setcounter{page}{1}
\pagenumbering{arabic}

\section{A motivation for Bayesian methods}

Statistical decision theory \parencite{french} is a fundamental block of modern machine learning and statistics research and practice,
providing fundamental tools for analyzing the vast amount of data that have become available in science, government, industry, and everyday life. Over the last century, many problems have been solved (at least partially) with probabilistic models \parencite{bishop2006pattern}, such as classifying email as spam, identifying patterns in genetic sequences, recommending similar movies or performing automatic translation between two languages. For each of these applications, a statistical model was fitted to the data, typically solving the task at hand. Since the previous applications were incredibly diverse,  models came from  different research communities. But there was a common theme shared amongst them: \emph{the need to support a decision using the model}. Thus, statistical decision theory  emerged as a powerful framework to reason between different scientific areas. 

We believe that building and using probabilistic models is not a single-step task, but rather an iterative process, closely resembling the scientific method \parencite{conant1959understanding}. First, propose a simple model based on the latent structure that your prior knowledge make you believe exists in the data. Then, given a data set, use an inference
method to approximate the posterior distribution of the  parameters
given the data. Lastly,
use that posterior to test the model against new, test data. If not satisfied, we criticise and revise it, and we modify the model, iterating all over again. This is called the Box's loop \parencite{doi:10.1080/00401706.1962.10490015,10.2307/1266125}. It focuses on the scientific method, producing new knowledge  through iterative experimental design, data collection,
model formulation, and model criticism.
Amongst all the different paradigms in statistics, the one that is closest to the previous approach to knowledge discovery is Bayesian analysis \parencite{gelman2013bayesian,insua2012bayesian}, and this will be the one adopted in this thesis.

In the past years, we have witnessed the success of neural networks and deep learning, achieving state of the art results in many different tasks \parencite{lecun2015deep}. A brief introduction to deep learning will appear in Section \ref{sec:deep_intro}, but let us advance a few things here. Bayesian analysis is specially compelling for this kind of models, mostly because the object of interest  is now the predictive distribution,
$$
p(y|x, \mathcal{D}) = \int p(y|x,\theta) p (\theta| \mathcal{D}) d\theta,
$$
with $x$ being a sample to be classified or regressed, $y$ the predicted target, $\mathcal{D}$ the training dataset, and $\theta$ the parameters of a neural network. The previous marginalization equation expresses epistemic uncertainty, that is, uncertainty over which configuration of parameters (hypothesis) is correct, given limited data. If the posterior $p(\theta| \mathcal{D})$ is sharply peaked, it makes no difference using a Bayesian approach versus the standard maximum likelihood or maximum a posteriori estimations, since just a single point mass may be a reasonable approximation to the posterior. However, neural networks are typically very underspecified by the training data available (a state of the art NN might have millions of parameters for just thousands of data points), and will thus have diffuse likelihood distributions $p(\mathcal{D}|\theta)$, leading to flatter posteriors which are also multi-modal.

Indeed, there are large valleys in the loss landscape of
neural networks \parencite{garipov2018loss}, over which parameters incur very little loss, but give rise to different high performing
functions which make meaningfully different predictions on the test dataset. The work of \parencite{ZOLNA2020102969}
also demonstrates the variety of well-performing solutions that can be expressed by a neural network
posterior, as it is highly flexible. In these cases, we desire to perform  Bayesian model averaging, since it leads to an ensemble of diverse yet good models, achieving better generalization  and performance stability capabilities than classical training.

In this thesis, we provide several contributions to large scale Bayesian machine learning, motivated by the following case study. On the one hand, it serves us to introduce notation and key concepts, and, on the other hand, the case will motivate our three main problems of interest.

\section{A case study in retailing}\label{sec:dlms}

\subsection{Context}

It is widely acknowledged that a firm's expenditure  on advertising has a positive effect on sales \parencite{assmus1984advertising, tellis2007advertising, luo2012does, wiesel2011practice}. However, the exact relationship between them remains a moot point, see \parencite{tellis2009generalizations} for a broad survey. Since \textcite{dorfman1954optimal} seminal work\, several models have been proposed to pinpoint this relationship, although consensus on the best approach has not been reached yet. Two diverging model-building schools seem to dominate the marketing literature \parencite{little1979aggregate}: \emph{a priori} models rely heavily on intuition and are derived from general principles, although usually with a practical implementation on mind (\textcite{nerlove1962optimal}, or \textcite{vidale1957operations} and \textcite{little1975brandaid}, \emph{inter alia}); and  \emph{statistical} or \emph{econometric} models, which usually start from a specific dataset to be modelled, e.g. \textcite{assmus1984advertising}). Here we will mostly rely on the first type of models, \emph{viz.} that of \textcite{nerlove1962optimal}, which extends  \textcite{dorfman1954optimal} to a dynamic setting \parencite{bagwell2007economic} and adapts seamlessly to the \emph{state-space} or \emph{structural} time series approach.

Bayesian structural time series models \parencite{scott2014predicting}, in turn, have positioned themselves in the past few years as very effective tools  not only for analysing marketing time-series, but also to throw light into more uncertain terrains like  causal impacts, incorporating \emph{a priori} information into the model, accommodating multiple sources of variations or supporting variable selection. Although the origins of this formalism can be traced back to the 1950's in the engineering problems of filtering, smoothing and forecasting, first with \textcite{wiener1949extrapolation} and specially with \textcite{kalman1960new}, these problems can also be understood from the perspective of estimation in which a vector valued time series $\{ X_0, X_1, X_2, \ldots\}$ that we wish to estimate (the \emph{latent} or \emph{hidden} states) is observed through a series of noisy measurements $\{ Y_0, Y_1, Y_2, \ldots\}$. This  means that, in the Bayesian sense, we want to compute the joint posterior distribution of all  states given all the measurements \parencite{sarkka2013bayesian}. The ever-growing computing power and release of several programming libraries  in the last few years like \parencite{petris2010r, scott2016bsts} have in part alleviated the difficulties in the implementation  that this formalism suffers, making these methods broadly known and used. This family of models have been used successfully to, e.g., model financial time series data \parencite{doi:10.1002/asmb.428}, infer causal impact of marketing campaigns \parencite{brodersen2015inferring}, select variables and nowcast consumer sentiment  \parencite{scott2015bayesian}, or for predicting other economic time series models like unemployment \parencite{scott2014predicting}.

We use the formalism of Bayesian structural time-series models to formulate a robust model that links advertising expenditures with weekly sales. Due to the flexibility and modularity of the model, it will be well suited to generalization to various markets and scenarios. Its Bayesian nature also adapts smoothly to the issue of introducing prior information. The formulation of the model allows for non-gaussian innovations of the process, which will take care of the heavy-tailedness of the distribution of sales increases. We also discuss how the forecasts produced by this model can help the manager in allocating the advertising budget. The decision space is reduced to a one dimensional curve of Pareto optimal strategies for  two moments of the forecast distribution (expected return and variance).


\subsection{Theoretical background and model definition}\label{sec:nerlove}

\paragraph{The Nerlove-Arrow model.}

Numerous formulations of aggregate advertising response models exist in the  literature, e.g. \parencite{little1979aggregate}. The model of  \textcite{nerlove1962optimal} extends the Dorfman-Steiner model to cover the situation in which current advertising expenditures affect future product demand; it is parsimonious and is considered a standard in the quantitative marketing community. We use it as our starting point.

In this model, advertising expenditures are considered similar in many ways to investments in durable plant and equipment, in the sense that they affect the present and future character of output and, hence, the present and future net revenue of the investing firm. The idea is to define an ``advertising stock'' called \emph{goodwill}  $A(t)$ which seemingly summarizes the effects of current and past advertising expenditures over demand. Then, the following dynamics is defined for the goodwill
\begin{equation}\label{eq:NA}
\frac{dA}{dt} = qu(t) - \delta A(t),
\end{equation}
where $u(t)$ is the advertising spending rate (e.g., euros or gross rating points per week), $q$ is a parameter that reflects the advertising quality (an effectiveness coefficient) and $\delta$ is a decay or forgetting rate. \emph{Goodwill} then increases linearly with  advertisement expenditure but decreases also linearly due to forgetting. 

Several extensions and modifications have been proposed to this simple model: it can include a limit for potential costumers \parencite{vidale1957operations}, a non-linear response function to advertise expenditures \parencite{little1975brandaid}, wear-in and wear-out effects of advertising \parencite{naik1998planning},  interactions between different advertising channels \parencite{bass2007wearout}, among others. Still, for most  tasks, the  Nerlove-Arrow model remains as a simple and solid starting point. 

\paragraph{Bayesian structural time series models.}

\emph{Structural time series models} or \emph{state-space models} provide a general formulation that  allows a unified treatment of virtually any linear time series model through the Kalman filter and the associated smoother. Several handbooks \parencite{durbin2012time, petris2009dynamic, sarkka2013bayesian, west1998bayesian} discuss this topic in depth. We will  present a few salient features that concern our modelling problem. For further details, the reader may check the aforementioned handbooks. 

The state-space formulation of a time series consists of two different equations: the \emph{state} or \emph{evolution equation} which determines the dynamics of the state of the system as a first-order Markov process ---\,usually parametrized through \emph{state} variables\,--- and an \emph{observation} or \emph{measurement  equation} which links the latent state with the observed state. Both equations are also affected by noise.
\begin{equation}\label{eq:st}
\theta_{t} = G_t \theta_{t-1} + \epsilon_t \qquad \epsilon_t \sim \mathcal{N}(0, W_t).
\end{equation}
The states ($\theta_t$) are not generally observable, but are linked to the \emph{observation variables} $Y_t$ through the \emph{observation equation}:
\begin{equation}\label{eq:obs}
Y_t = F_t \theta_t + \epsilon'_t \qquad \epsilon'_t \sim \mathcal{N}(0, V_t).
\end{equation}
We shall point out that the noise terms $\epsilon_t$ and $\epsilon'_t$ are uncorrelated. We denote by $\mathbf{\theta}_t$ the $m\times1$ \emph{state vector} describing the inner state of the system, by $G_t$ the $m\times m$ matrix that generates the dynamics, and by $ \epsilon_t $ a $g\times 1$ vector of serially uncorrelated disturbances with mean zero and covariance matrix $W_t$. $F_t$ is the $1\times m$ matrix that links the inner state to the observable, and $V_t \in \mathbb{R}^+$ is the variance of $\epsilon'_t $, the random disturbances of the observations.



The specification of the state-space system is completed by assuming that the initial state vector $\theta_0$ has mean $\mu_0$ and a covariance matrix $\Sigma_0$ and it is uncorrelated with the noise. The problem then consists of \emph{estimating the sequence of states $\{\theta_1, \theta_2, \ldots\}$ for a given series of observations $\{y_1, y_2, \ldots\}$} and whichever other structural parameters of the transition and observation matrices. State estimation is readily performed via the \emph{Kalman filter};  different alternatives however arise  when structural parameters are unknown. In the classical setting, these are estimated using maximum likelihood. In the Bayesian approach, the probability distribution about the unknown parameters is updated via Bayes Theorem. If exact computation through conjugate priors is not possible, the probability distributions before each measurement are updated by approximate procedures such as Markov chain Monte Carlo (MCMC) \parencite{scott2014predicting}. 

The Bayesian approach offers several advantages compared to classical methods. For instance, it is natural to incorporate external information through the prior distributions. In particular, this will be materialized in Section \ref{sec:s_s} where expert information is incorporated through the spike and slab prior. Another useful advantage is that, due to the Bayesian nature of the model, it is straightforward to obtain predictive intervals through the predictive distribution (see Section \ref{s:MCMC}).

\paragraph{Model specification.}\label{sec:model}


The continuous-time Nerlove-Arrow model must be first cast in discrete time so as to formulate our model in state-space. From equation (\ref{eq:NA}), we get
$$
A_t =  (1-\delta) A_{t-1} + q u_{t-1} + \epsilon_t 
$$
where $A_t$ is the \emph{goodwill} stock, $u_t$ is the advertising spending rate, $q$ is the effectiveness coefficient, the random disturbance $\epsilon_t$  captures the net effects of the variables that affect the goodwill  but cannot be modelled explicitly, and with $|\delta| < 1$.  This discrete counterpart of Nerlove-Arrow is a distributed-lag structure with geometrically declining weights, i.e., a \emph{Koyck model} \parencite{clarke1976econometric, koyck1954distributed}. Since in our setting the model includes the effect of $k$ different channels in the goodwill, we modify the previous equations to:
$$
A_t =  (1-\delta) A_{t-1} + \sum_{i=1}^k q_i u_{i(t-1)} + \epsilon_t.
$$
Now, following  (\ref{eq:st}) and (\ref{eq:obs}), the discrete-time Nerlove-Arrow model in state-space form will read:
\begin{description}

\item[Evolution equation:]

\begin{equation} \label{eq:NA_st}
\theta_{t} = G_t \theta_{t-1} +  \epsilon_t \qquad \epsilon_t \sim \mathcal{N}(0, W_t),   
\end{equation}
\noindent where
\begin{equation*} 
\theta_{t}  = \begin{bmatrix} A_t  \\ q_1 \\  \vdots \\q_k \end{bmatrix}
, \quad
G_t =  \begin{bmatrix}
   (1- \delta) & u_{1(t-1)} &  \ldots & u_{k(t-1)} \\
   0 & 1 &   \ldots & 0 \\
   \vdots  &   \vdots & \ddots & \vdots \\
   0 & 0 &  \ldots & 1 \\
   \end{bmatrix}.
\end{equation*}

Note that the $q_i$ are constant over time and that the matrix $G_t$ depends on the known inversion levels at time $t-1$ and on an unknown parameter ($\delta$) to be estimated from the data.

 \item[Observation equation:]
 
 \begin{equation} \label{eq:NA_obs}
 Y_t = F_t \theta_t + \epsilon'_t \qquad \epsilon'_t \sim \mathcal{N}(0, V_t)  
 \end{equation}

 where $Y_t$ are the observed sales at time $t$ and $F_t = \begin{bmatrix} 1,  & 0 , & \ldots, & 0 \end{bmatrix}$.
\end{description}

\paragraph{Modularity and additional structure.}

The above model is very flexible in the sense that it can be defined \emph{modularly},  in as much as different hidden states evolve independently of the others (\emph{i.e.} the evolution matrix can be cast in block-diagonal form). This greatly simplifies their implementation and allows for simple building-blocks with characteristic behavior. Typical blocks specify trend and seasonal components ---\,which can be helpful to discover additional patterns in the time series--- or explanatory variables that can be added to further reduce  uncertainty in the model and bridge the gap between time series and regression models. Via the \emph{superposition principle} \parencite[Chapter 3]{petris2009dynamic} we could include additional blocks in our model

$$
Y_t = Y_{NA, t} + Y_{R, t} + Y_{T, t} + Y_{S, t}
$$
where $Y_{NA, t}$ corresponds to the discretized Nerlove-Arrow equation, defined in (\ref{eq:NA_st}) and (\ref{eq:NA_obs}); $Y_{R, t}$ is a regression component; containing the effects of  external explanatory variables $X_t$; $Y_{T, t}$ is a trend component or a simpler local level component; and $Y_{S, t}$ is a seasonal component.

\paragraph{Regression components. Spike and slab variable selection.}\label{sec:s_s}

To take into account the effects of external explanatory variables such as the weather or sport events, a  regression component can be easily incorporated into the model through

$$
Y_{R,t} = X_t \beta + \epsilon_t,
$$
where the state $\beta$ is constant over time to favor parsimony.

A spike and slab prior \parencite{mitchell1988bayesian} is used for the  regression component, since it can incorporate \emph{prior information} and also facilitate variable selection. This is specially useful for models with a large number of regressors, a typical setting in business scenarios.
Let $\gamma$ denote a binary vector that indicates whether the regressors are included in the regression. Specifically, $\gamma_i = 1$ if and only if $\beta_i \neq 0$. The subset of $\beta$ for which $\gamma_i = 1$ will be denoted $\beta_{\gamma}$. Let $\sigma^2_{\epsilon}$ be the residual variance from the regression part. The spike and slab prior \parencite{george1997approaches} can be expressed as
$$
p(\beta, \gamma, \sigma^2_{\epsilon}) = p(\beta_{\gamma} | \gamma, \sigma^2_{\epsilon})p(\sigma^2_{\epsilon} | \gamma)p(\gamma).
$$
A usual choice for the $\gamma$ prior is a product of Bernoulli distributions:

$$
\gamma \sim \Pi_i \pi_i^{\gamma_i}(1-\pi_i)^{1-\gamma_i}.
$$
The manager of the firm may elicit these $\pi_i$ in various ways. A reasonable choice when detailed prior information is unavailable is to set all $\pi_i = \pi$. Then, we may specify an expected number of non-zero coefficients by setting $\pi = k/p$, where $p$ is the total number of regressors. Another possibility is to set $\pi_i = 1$ if the manager believes that the $i-$th regressor is crucial for the model. 

\paragraph{Model estimation and forecasting}\label{s:MCMC}

Model parameters can be estimated using Markov Chain Monte Carlo simulation, as described in Chapter 4 of \parencite{petris2009dynamic} or \parencite{scott2014predicting}. We follow the same scheme.

Let $\Theta$ be the set of model parameters other than $\beta$ and $\sigma^2_{\epsilon}$. The posterior distribution can be simulated with the following Gibbs sampler 

\begin{enumerate}
\item Simulate $\theta \sim p(\theta | y, \Theta, \beta, \sigma^2_{\epsilon})$.
\item Simulate $\Theta \sim p(\Theta | y, \theta, \beta, \sigma^2_{\epsilon})$.
\item Simulate $\beta, \sigma^2_{\epsilon} \sim p(\beta, \sigma^2_{\epsilon} | y, \theta, \Theta)$.
\end{enumerate}
Defining $\rho = (\Theta, \beta, \sigma^2_{\epsilon}, \theta)$ and repeatedly iterating the above steps gives a sequence of draws $\rho^{(1)}, \rho^{(2)}, \ldots, \rho^{(K)}$ $\sim$ $p(\rho)$. In our experiments, we set $K = 4000$ and discard the first 2000 draws to avoid burn-in issues. \\

In order to sample from the predictive distribution, we follow the usual Bayesian approach summarized by the following predictive equation, in which $y_{1:t}$ denotes the sequence of observed values, and $\bar{y}$ denotes the set of values to the forecast

\[
p(\bar{y} | y_{1:t}) = \int p(\bar{y}| \rho)p(\rho | y_{1:t}) d\rho,
\]
assuming conditional independence of $\bar{y}$ and $y_{1:t}$ given $\rho$. Thus, it is sufficient to sample from $p(\bar{y}| \rho^{(i)})$, which can be achieved by iterating through equations (\ref{eq:obs}) and (\ref{eq:st}). With these predictive samples $\bar{y}^{(i)}$ we can compute statistics of interest regarding the predictive distribution   $p(\bar{y} | y_{1:t})$ such as the mean or variance (MC estimates of $E[ \bar{y} | y_{1:t}]$ and $Var[\bar{y} | y_{1:t}]$, respectively) or quantiles of interest.\\

\paragraph{Robustness.}

We can replace the assumption of Gaussian errors with student-$t$ errors in the observation equation, thus leading to the model
$$ 
Y_{t} = F_t \theta_t + \epsilon'_t \qquad \epsilon'_t \sim \mathcal{T}_\nu(0, \tau^2).
$$
Typically, in these settings we set $\nu > 2$ to impose a finite variance, and this variance parameter can be estimated from data using empirical Bayes methods, for instance.
In this manner, we allow the model to predict occasional larger deviations, which is reasonable in the context of forecasting sales. For instance, a special event not taken into account through the predictor variables may lead to an increase in the sales for that week.


\subsection{Case Study. Data and parameter estimation}

\paragraph{Data analysis.}

\begin{figure}[h]
\centering
\includegraphics[scale=0.6]{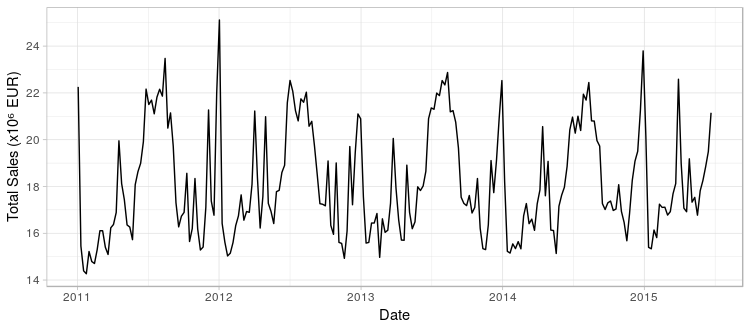}
\caption{Total weekly sales. Jan-2011 to Jun-2015}\label{fig:sales}
\end{figure}

The time series analyzed in this case study contains the total weekly sales of a country-wide franchise of fast food restaurants, Figure \ref{fig:sales}, covering the period January 2011 - June 2015, thus comprising 234 observations. The total weekly sales is in fact the aggregated sum from the individual sales of the whole country network of 426 franchises. Along with the sales figures, the series includes the investment levels $\{u_{it}\}$  in advertising during this period  for seven different channels \emph{viz.} \texttt{OOH} (Out-of-home, \emph{i.e.} billboards), \texttt{Radio}, \texttt{TV}, \texttt{Online}, \texttt{Search}, \texttt{Press} and \texttt{Cinema}, Figure \ref{fig:canales}, $i=1,\ldots,7$.

\begin{figure}[h]
\centering
\includegraphics[scale=0.5]{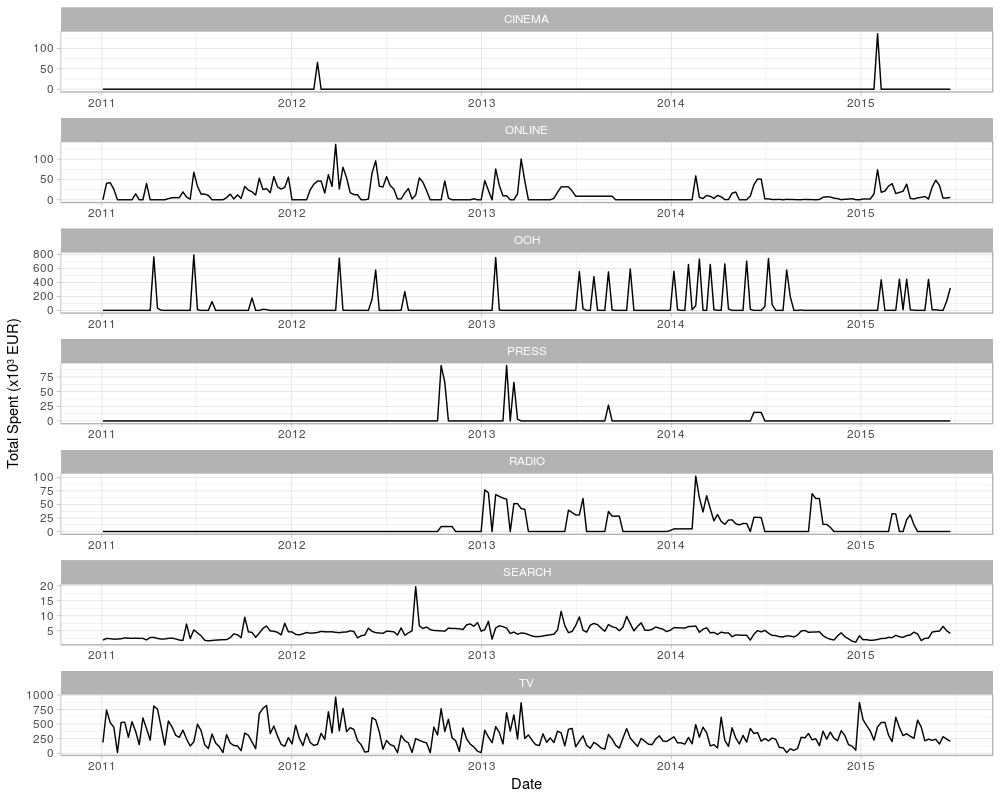}
\caption{Advertising expenditure. Jan-2011-Jul-2015. \\ Note  different scales on  y-axis.}\label{fig:canales}
\end{figure}

From such graphs we observe that:
\begin{itemize}
\item In Figure \ref{fig:sales}, the series has a characteristic seasonal pattern that shows peaks in sales coinciding with Christmas, Easter and summer vacations. 
\item In Figure \ref{fig:canales} we observe that the investment levels at each channel vary largely in scale, with investments in \texttt{TV}, and \texttt{OOH} dominating the other channels.
\item The investment strategies adopted by the firm at each channel are also qualitatively different. Some of them show spikes while others depict a relatively even spread investment across time.
\end{itemize}
A handful of other predictors $X_t$ which are also known to affect sales will be used in the model, all of them weekly sampled

\begin{itemize}
\item Global economic indicators: unemployment rate (\texttt{Unemp\_IX}), price index (\texttt{Price\_IX}) and consumer confidence index (\texttt{CC\_IX}).
\item Climate data: average weekly temperature (\texttt{AVG\_Temp}) and weekly rainfall (\texttt{AVG\_Rain}) along the country.
\item Special events: holidays (\texttt{Hols}) and important sporting events (\texttt{Sport\_EV}).
\end{itemize}

\paragraph{Experimental setup.}

Following the notation in Section \ref{sec:model}, we consider three model variants for the particular dataset in increasing order of complexity


\begin{itemize}
\item \textbf{Baseline model}, which makes use of no external variables $$Y^{\text{B}}_t = Y_{NA,t} + Y_{T, t}.$$
\item \textbf{Auto-regression}: this model (we will refer to it as RA) incorporates the external ambient and investment variables, so the equation of the model becomes 
\begin{equation}\label{eq:autoreg}
Y_t^{\text{RA}} = Y_{NA,t} + Y_{T, t} + Y_{R,t}.
\end{equation}
We select an expected model size\footnote{Defined as the average number of selected regression features.} of 5 in the spike and slab prior, letting all variables to be treated equally.
\item \textbf{Regression (forcing)}: the model has same equation as Eq. (\ref{eq:autoreg}) (we will refer to it as \text{RF}).
However, in the prior we force investment variables to be used by setting their corresponding $\pi_i = 1$, and imposing an expected model size of 5 for the rest of the variables.
\end{itemize}
In all cases, only the five main advertising channels (\texttt{TV}, \texttt{OOH}, \texttt{ONLINE}, \texttt{SEARCH} and \texttt{RADIO}) will be used; the remaining two (\texttt{CINEMA} and \texttt{PRESS}) are sensibly lower both in magnitude and frequency than the others so we can safely disregard them in a first approximation.

As customary in a supervised learning setting with time series data, we perform the following split of our dataset: since it comprises four years of sales, we take the first two years of observations as training set, and the rest as holdout, in which we assess several predictive performance criteria. Before fitting the data, we scale the series to have zero mean and unit variance as this increases MCMC stability. Reported sales forecasts are transformed back to the original scale for easy interpretation. The models were implemented in R  using the \texttt{bsts} package \parencite{scott2016bsts}.


\subsection{Discussion of results}

It is customary to aim at models achieving good predictive performance. 
For this reason, we test  our three models using two metrics

\begin{itemize}
\item Mean Absolute Percentage Error $$ \text{MAPE} = \frac{100\%}{T}\sum_{t=1}^T \frac{|y_t - \hat{y}_t|}{y_t},$$
where $y_t$ denotes the actual value; $\hat{y}_t$, the mean one-step-ahead prediction; and, $T$ is the length of the hold-out period.
\item Cumulative Predicted Sales over a year $Y$
$$
\text{CPS}_{Y} = \sum_{t \in \mathcal{T}(Y)} \hat{y}_t
$$
where $\mathcal{T}(Y)$ denotes the set of time-steps $t$ contained in year $Y$.
\end{itemize}
These scores are reported in Table \ref{tab:mapes} with sales in million EUR. Note that the models which include external information (RA and RF) achieve better accuracy than the baseline. In addition, 
predictions are unbiased, since cumulative predictions are extremely close to their observed counterparts. Overall, we found the predictive performance of our models to be successful for a business scenario, as we achieve under 5\% relative absolute error using the variants augmented with external information. This is clearly useful for a decision maker who wants to forecast their weekly sales one week ahead to within a 5\% error in the estimation.

\begin{table}[h]
\centering
\begin{tabular}{|l|c|c|c|}
\hline
Model & MAPE & $\text{CPS}_{2013}$ & $\text{CPS}_{2014}$ \\
\hline
B &  5.85\% &   9660 & 9627 \\
RA &  4.62\% &  9680 & 9613 \\
RF &  4.59\% &  9665 & 9582 \\
\hline
\multicolumn{2}{|c|}{Cumulative True Sales:} & 9666 & 9610 \\
\hline
\end{tabular}
\caption{Accuracy measures for each model variant.} \label{tab:mapes}
\end{table}



Figure \ref{fig:forecasts} displays the predictive ability of model RF over the hold-out period. The model seems sufficiently flexible to adapt to fluctuations such as Christmas peaks. Predictive intervals also adjust their width with respect to the time to reflect varying uncertainty, yet in the worst cases they are sufficiently narrow. Further information can be tracked in Figure \ref{fig:residuals}, where mean standardized residuals are plotted for each model variant. Notice that the residuals for models RA and RF are roughly comparable, being both sensibly smaller than those of the baseline. This means that the simpler Nerlove-Arrow model benefits from the addition of the ambient variables $X_i$, as suggested in  Table \ref{tab:mapes}.
\begin{figure}[h]
\centering
\includegraphics[scale=0.55]{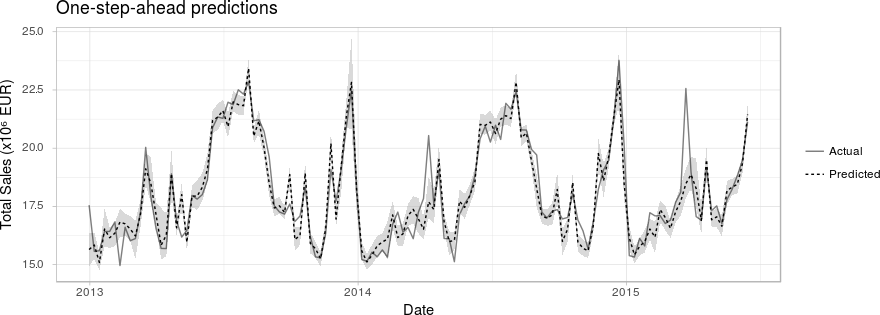}
\caption{One-step-ahead forecasts of  RF model versus actual data during hold-out period. 95\% predictive intervals depicted in light gray.}\label{fig:forecasts}
\end{figure}

\begin{figure}[h]
\centering
\includegraphics[scale=0.6]{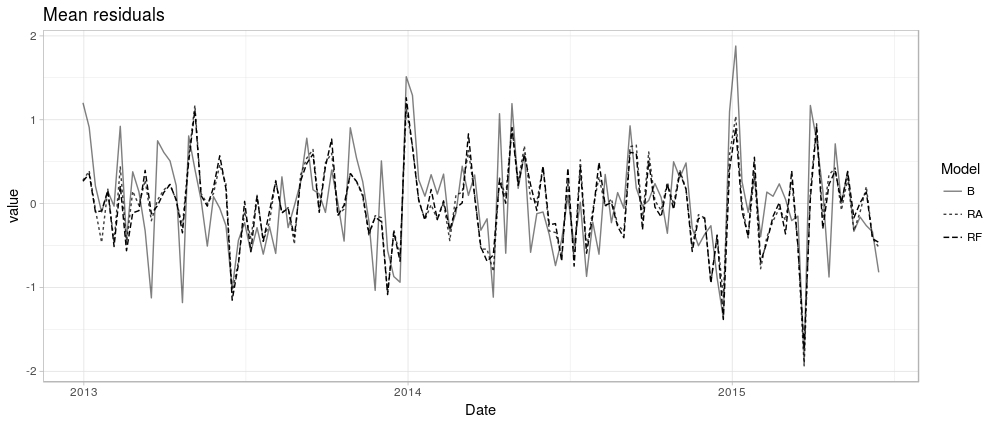}
\caption{Mean residuals for each model}\label{fig:residuals}
\end{figure}

Having built good predictive models, we inspect them more closely with the aim of performing valuable inferences in our business setting. The average estimated parameter and expected standard deviations of the $q_i$ coefficients for the different advertising channels are displayed in Table \ref{tab:1}. We also show the weights of the ambient variables $X_i$ for the augmented models RA and RF in Table \ref{tab:11}, as well as the probability of a variable being selected in the MCMC simulation for a given model in Figure \ref{fig:inc_probas}. Convergence diagnostics of the MCMC scheme are reported in Appendix \ref{app:GR}.


\begin{table}[h]
\centering
\begin{tabular}{ |l|c|c|c|c|c|c|c| }
  \hline
  & \multicolumn{2}{|c|}{Model B} & \multicolumn{2}{|c|}{Model RA} & \multicolumn{2}{|c|}{Model RF}\\
  \hline
  Channel & mean & sd & mean & sd & mean & sd\\
  \hline
  \texttt{y\_AR} & \textbf{8.00e-01} & \textbf{5.10e-02} &  \textbf{5.17e-01} &  \textbf{4.50e-02} &  \textbf{5.07e-01} &  \textbf{4.55e-02}   \\
  \texttt{OOH} & 9.30e-03 & 3.11e-02 & 1.15e-03 & 9.54e-03  & \textbf{6.10e-02} &  \textbf{3.54e-02}\\
  \texttt{ONLINE} & 1.10e-03 & 9.25e-03  & 1.95e-03 &  1.23e-02 & \textbf{9.04e-02} & \textbf{4.01e-02}\\
  \texttt{RADIO} & -5.34e-04 & 6.80e-03  &-2.61e-05 & 2.30e-03 &   -2.57e-02 & 3.42e-02\\
  \texttt{TV} & -2.85e-04 & 5.12e-03 & -5.53e-05 & 3.71e-03  & -6.14e-02 & 4.21e-02\\
  \texttt{SEARCH} & 1.51e-05  & 2.74e-03  & 8.58e-05 & 2.49e-03  & 1.38e-02 & 3.50e-02 \\
    \hline
\end{tabular} \caption{Expected value and standard error of  $q_i$. Statistically significant coefficients  in bold.}\label{tab:1}
\end{table}

\begin{table}[h]
\centering
\begin{tabular}{ |l|c|c|c|c|c| }
  \hline
  &  \multicolumn{2}{|c|}{Model RA} & \multicolumn{2}{|c|}{Model RF}\\
  \hline
  Channel & mean & sd & mean & sd\\
  \hline
  \texttt{Sport\_EV} & \textbf{-2.06e-01} & \textbf{3.80e-02} & \textbf{-2.00e-01} &\textbf{ 3.71e-02} \\
  \texttt{AVG\_Temp} &  \textbf{2.57e-01} & \textbf{4.43e-02} & \textbf{2.43e-01} & \textbf{4.65e-02}   \\
  \texttt{Hols} & \textbf{3.10e-01} &\textbf{ 3.70e-02}   & \textbf{3.15e-01} & \textbf{3.72e-02}  \\
  \texttt{AVG\_Rain} &  -2.86e-02 & 5.01e-02 & -1.96e-02 & 4.14e-02     \\
  \texttt{Price\_IX} &  1.38e-03 &  1.15e-02 & 3.60e-03 & 1.92e-02  \\
  \texttt{Unemp\_IX} &  6.34e-05 & 3.32e-03 & 2.36e-04 & 7.79e-03  \\
  \texttt{CC\_IX} & 6.27e-07 & 2.94e-03 & 2.85e-04 & 5.50e-03 \\
    \hline
\end{tabular} \caption{Expected value and standard error of  $\beta_i$. Statistically significant coefficients in  bold.}\label{tab:11}
\end{table}


\begin{figure}[h]
\centering
\includegraphics[scale=0.6]{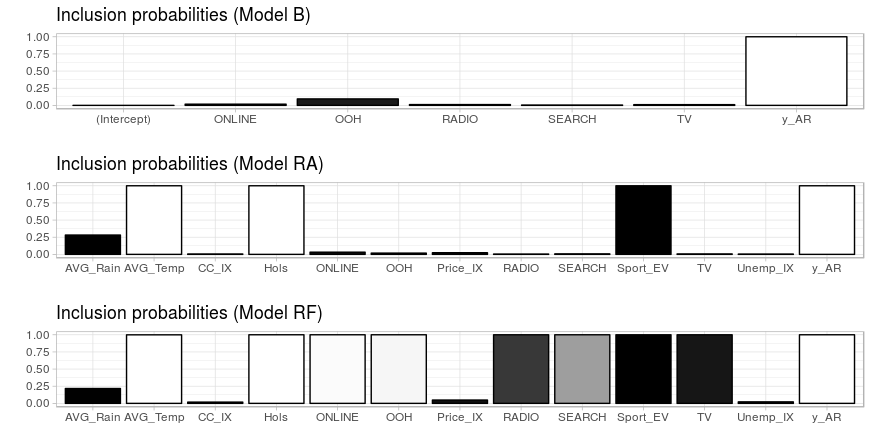}
\caption{Selection probabilities in  MCMC simulation for each predictor variable in the  models. Color code shows variables positively (white) or negatively (black) correlated with sales. }\label{fig:inc_probas}
\end{figure}


Looking at the ambient variables, the following comments seem in order:
\begin{itemize}

\item From Table \ref{tab:11} we see that the socio-economic indicators (unemployment rate, inflation and consumer confidence) do not seem statistically relevant for this problem. 

\item Looking at the sign of the coefficients of the most significant regressors $X_i$ (\texttt{Hols}, \texttt{Sport\_EV}, \texttt{AVG\_Temp} and \texttt{AVG\_Rain}) we see that they are as we would naturally expect. Moreover, their absolute value is well above the error in both models RA and RF, a strong indicator of their influence in the expected weekly sales (\emph{cf.} Table \ref{tab:11}, Figure \ref{fig:inc_probas}). 

\item We see, for instance, that sporting events are negatively correlated with sales. This can be interpreted as follows: major sporting events in the country where the data have been recorded receive a large media coverage and are followed by a significant fraction of the population. The chain in this study has no TVs broadcasting in their restaurants, so customers probably choose alternative places to spend their time on a day when \texttt{Sport\_EV} = 1. 

\item The sign of \texttt{Hols} and $\texttt{AVG\_Temp}$ is positive, showing strong evidence for the fact that sales increase in periods of the year where potential customers have more leisure time, like national holidays or  summer vacation. 

\item One would expect \texttt{AVG\_Rain} to be negatively correlated with restaurant sales, but in our study (despite having negative sign) it is not statistically significant. A possible explanation is that \texttt{AVG\_Rain} records average rainfall over a large country. 

\end{itemize}

Next, we turn our attention to the investment variables across different advertising channels.

\begin{itemize}

\item The advertising channels $u_i$ are almost never selected in model RA, and their $q_i$ coefficients are not significantly higher than their errors to be considered influential in the model. In model RF, however, there is strong evidence that their effect is more than a random fluctuation (\emph{cf.} Table \ref{tab:1}, Figure \ref{fig:inc_probas}). 
\item The negative sign in both \texttt{RADIO} and \texttt{TV} in all three models suggests that (at least locally) part of the expenditures in these two channels should be diverted towards other channels with positive sign on their $q_i$ coefficients, specially to the channel with the strongest positive coefficient (\texttt{ONLINE} and \texttt{OOH}).
\item It is interesting that the trend that shows the year-to-year advertising budget of this firm (\emph{cf.} figure \ref{fig:yearly}) has a significant reduction in \texttt{TV} expenditures and a big increase in \texttt{OOH}. \texttt{RADIO} however is not reduced accordingly but increased, and \texttt{ONLINE} ---\,which our model considers the best local inversion alternative\,--- follows the inverse path\footnote{It has to be noted, however, that \texttt{ONLINE} expenditures are typically correlated with discounts, coupons and offers, and this information was not available to us.}.
\item The autoregressive term is close to $0.5$, which means that the immediate effect of advertising is roughly half to the \textit{long run accumulated effects}.
\end{itemize}

\begin{figure}[h]
\centering
\includegraphics[scale=0.75]{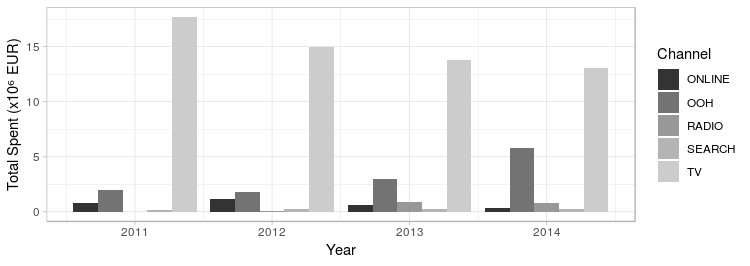}
\caption{Yearly total expenditures on advertising campaigns for the restaurant chain, per advertising channel.}\label{fig:yearly}
\end{figure}



\paragraph{Budget allocation. Model-based solutions.}\label{section: budget allocation}

Based on the above, we propose a model which can be used as a decision support system by the manager, helping her in adopting the advertising investment strategy. The company is interested in maximizing the expected sales for the next period, subject to a budget constraint for the advertising channels and also a risk constraint, i.e., the variance of the predicted sales must be under a certain threshold.
This optimization problem based on one-step-ahead forecasts can be formulated as a non linear, but convex, problem that depends on  parameter $\sigma^2$:

\begin{equation*}
\begin{aligned}
& \underset{u_{(t+1),1}...u_{(t+1),k}}{\text{maximize}}
& &  E[ \bar{y}_{t+1} | y_{1:t}, x_{t+1}, u_{t+1}] \\
& \text{subject to}
& & \sum_{i=1}^k  u_{(t+1),i} \leq b_{t+1} \\
& & & Var[ \bar{y}_{t+1} | y_{1:t}, x_{t+1}, u_{t+1}] \leq \sigma^2, \\
\end{aligned}
\end{equation*}
where $b_t$ is the total advertising budget for week $t$ and $\sigma$ is a parameter that controls the risk of the sales. We have made explicit the dependence on the regressor variables $x_t$ and advertisement investments $u_t$ in the mean and variance expressions.
Solving for different values of $\sigma$, we obtain a continuum of Pareto optimal investment strategies that we can present to the manager, each one representing a different trade-off between risk and expected sales that we can plot in a risk-return diagram.
This approach greatly reduces the decision space for the manager.

A possible alternative would be to rewrite the objective function as 
$$E[ \bar{y}_{t+1} | y_{1:t}, x_{t+1}, u_{t+1}] - \lambda \sqrt{Var[ \bar{y}_{t+1} | y_{1:t}, x_{t+1}, u_{t+1}]}$$ 
which may be regarded as a lower quantile if the predictive samples $\bar{y}^{(i)}_{t+1}$ are normally distributed.
As in the previous approach, different values of $\lambda$ represent different risk-return trade-offs.

If the errors in the observation equation (\ref{eq:obs}) are normal, the computations for expected predicted sales and the above variance can be done exactly, and quickly, using conjugacy as in \parencite[3.7.1]{zbMATH06123712}.
Otherwise, the desired quantities must be computed through Monte Carlo simulation, as in Section \ref{s:MCMC}.

Note that due to the nature of the state-space model, it is straightforward to extend the previous optimization problem over $k$ timesteps, for $k>1$.
The two objectives would be the expected total sales and the total risk over the period $t+1,\dots,t+k$.
For normally distributed errors, the predicted sales on the period also follow a Normal distribution, which makes computations specially simple.

The above optimization problem may not need to be solved by exhaustively searching over the space of possible channel investments. In a typical business setting, the manager would consider a discrete set of $s$ investment strategies that are easy to interpret, so she may perform $s$ simulations of the predictive distribution and use the above strategy to discard the Pareto suboptimal strategies.

\subsection{Discussion}

We have developed a data-driven approach for the management of advertising investments of a firm. First, using the firm's investment levels in advertising, we propose a formulation of the Nerlove-Arrow model via a Bayesian structural time series to predict an economic variable (global sales)  which also incorporates information from the external environment (climate, economical situation and special events). The model thus defined offers low predictive errors while maintaining interpretability and can be built in a modular fashion, which offers great flexibility to adapt it to other business scenarios. The model performs variable selection and allows to incorporate \emph{prior} information via the \emph{spike-and-slab} prior. It can handle non-gaussian deviations and also provides hints to which of the advertising channels are having positive effects upon sales. The model can be used as a basis for a decision support system by the manager of the firm, helping with the task of allocating ad investments.

This model could be expanded in several ways. For instance, in regimes where the data is of high frequency or there are large amounts of it, the Gibbs sampler from Section \ref{sec:nerlove} could be replaced with a more scalable sampler, such as the ones developed in Chapter 2 of this thesis. We could also take into account the presence of competing retailing firms, which could impact the expected sales of the supported franchise. The next two chapters are devoted to taking the account the presence of adversaries in ML models, both in classification (Chapter 3) and reinforcement learning settings (Chapter 4). For this we require first briefly reviewing some earlier work.

\section{State of the art models and big data}\label{sec:deep_intro}

\subsection{Introduction}

The history of neural network (NN) models have gone 
through several waves of popularity. The first 
one starts with the introduction of the perceptron
by \textcite{rosenblatt1958perceptron} and its training algorithm 
for classification in linearly separable problems.
Limitations brought up by 
\textcite{minsky} somehow reduced the enthusiasm
about these models by the early 70's.
The next period of success coincides with the emergence
of results presenting NNs as universal
approximators, e.g.\ \parencite{cybenko1989approximation}. Yet 
technical issues and the emergence of other paradigms like 
support vector machines led, essentially,
to a new stalmate by the early 2000's. Finally, several of the 
technical issues were solved, coinciding with the 
availability of faster computational tools,
improved algorithms and the emergence of
large annotated datasets. These produced outstanding 
applied developments leading, over the last decade, to the current boom 
surrounding deep NNs \parencite{10.5555/3086952}. 

This section overviews  
recent advances in NNs. 
There are numerous reviews with various emphasis 
including physical \parencite{cirac}, computational \parencite{chollet}, mathematical \parencite{maths} and pedagogical \parencite{teach} pùrposes, 
notwithstanding  those concerning different  application areas, 
from autonomous driving systems (ADS) \parencite{rumanos} to
drug design \parencite{hessler}, to name but a few. 
Our emphasis is on statistical
aspects and, more specifically, on Bayesian approaches
to NN models for reasons that will become 
apparent during this thesis but include mainly:
the provision of improved uncertainty estimates, which is
of special relevance in decision support under uncertainty; their 
increased generalization capabilities; their 
enhanced robustness against adversarial attacks; 
their improved capabilities for model calibration;
and the possibility of using sparsity-inducing priors
to promote simpler NN architectures.

We first recall basic results from (the now-called) 
shallow NNs.
Section \ref{sec:dnn_examples} then covers deep NNs, including their most
relevant 
variants, as well as classical and Bayesian approaches
for their analysis. Next, Section \ref{sec:dnn_examples}
presents two examples illustrating
diverse NN architectures.

\subsection{Shallow neural networks}
This section briefly introduces key concepts
about shallow networks to support later discussions on current approaches.
Our focus will be mainly on nonlinear regression 
problems. Specifically, we aim at approximating 
an $r$-variate response (output) $y$ with respect to $p$ explanatory 
(input) variables $x=(x_1,\ldots,x_p)$ through the 
model
\begin{eqnarray}\label{kantora}
  y         & = & \sum_{j=1}^m \beta_j \psi(x' \gamma_j) +
                    \epsilon 
                    \nonumber\\
              & & \epsilon \sim N(0,\sigma^2),
                  \nonumber \\
              & & \psi(\eta) = \exp(\eta)/(1+\exp(\eta)).
                  \end{eqnarray}
This defines a neural network with one hidden 
layer with $m$ hidden neurons and logistic 
activation functions.
Figure \ref{figuradkk1} sketches 
a graphical model of a shallow NN with 10 inputs, 4 hidden nodes and 
2 outputs. 
\begin{figure}
    \centering
    \includegraphics[scale=0.5]{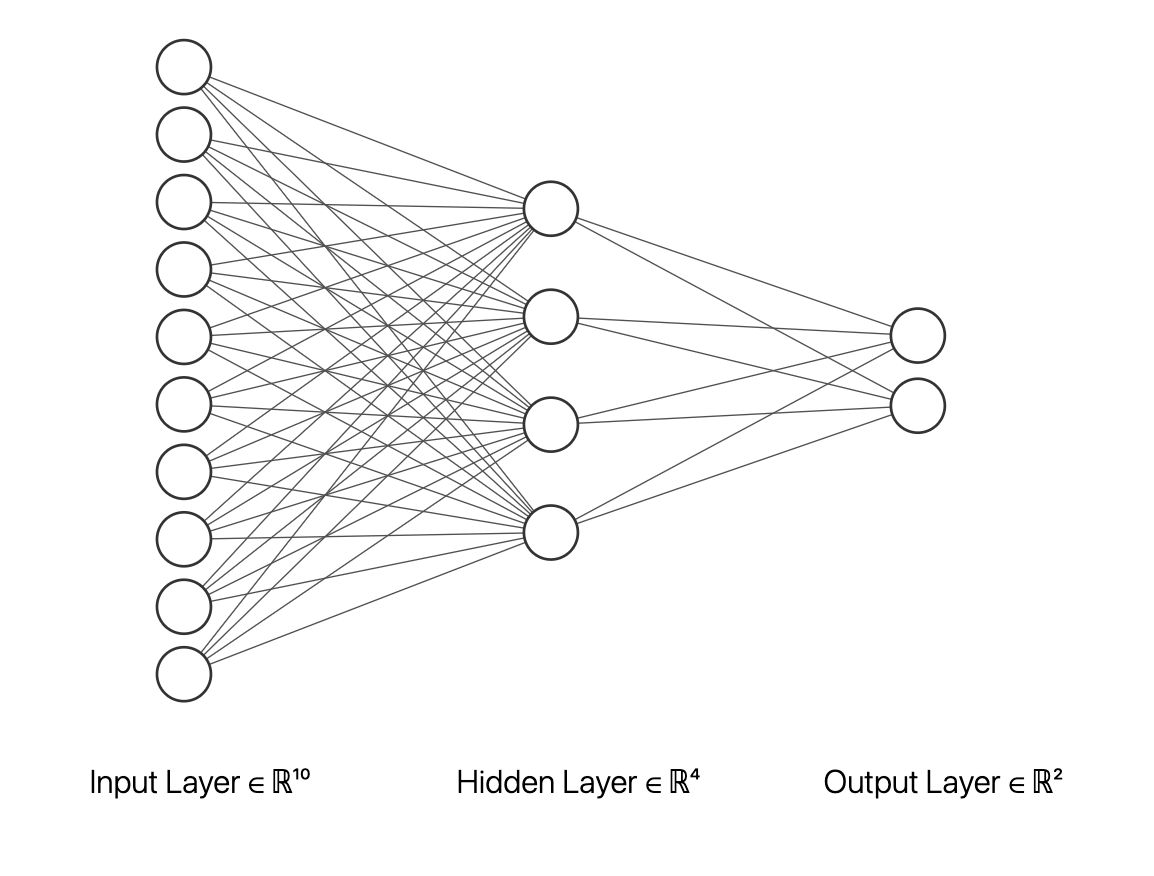}
    \caption{A shallow NN architecture with 4 hidden nodes and 2 scalar outputs}
    \label{figuradkk1}
\end{figure}

Let us designate with $\beta=(\beta_1,\ldots,\beta_m)$ and $\gamma=(\gamma_1,\ldots,\gamma_m)$ the network parameters. $\sigma$  
will be considered a hyperparameter. Clearly, the model
is linear in $\beta$ but non-linear in  
$\gamma$. Interest in this type of models stems from 
results such as those of \textcite{cybenko1989approximation}
who presents them as universal approximators:
 any continuous function in the 
$r$-dimensional unit cube 
may be approximated by models of type
$\sum_{j=1}^m \beta_j \psi(x' \gamma_j)$
when the $\psi$
functions are sigmoidal (as with the logistic) and
$m\rightarrow \infty$. Our discussion focuses on $r=1$.

\paragraph{Classical approaches.}\label{sanchez}

Given $n$ observations $D=\{ (x_i, y_i), i=1,...,n \}$,
 maximum likelihood estimation (MLE) 
 computes the log-likelihood and maximises it 
leading to the classical non-linear least squares problem
\begin{equation}\label{pdt}
 \min_{\beta , \gamma } f (\beta , \gamma) = \sum _{i=1}^n f_i(\beta, \gamma)  =\sum _{i=1}^n \left( y_i -
  \sum_{j=1}^m \beta_j \psi(x_i'\gamma_j) \right)^2 
 \end{equation}

\noindent Quite early, researchers paid attention to the introduction of regularisers, such as weight decay $\ell_2$ penalization \parencite{krogh1992simple}, so as to improve model 
generalization, leading to the modified optimisation problem
\begin{equation}\label{kkdbak}
 \min  g(\beta ,\gamma) = f (\beta ,\gamma ) +
 h (\beta ,\gamma ), \end{equation}
where $h(\beta , \gamma )$ represents the regularisation 
term. For example, in the above mentioned case, the 
additional term is  
$h(\beta , \gamma )= \lambda _1 \sum \beta_i ^2 +
\lambda _2 \sum \sum \gamma _{ji} ^2$. 

Typically problems (\ref{pdt}) and (\ref{kkdbak}) are solved via steepest gradient descent \parencite{meza} through iterations of the type
\[
   (\beta, \gamma )_{k+1}=
   (\beta, \gamma )_{k}- \eta \nabla g (  (\beta, \gamma )_{k} ),
   \]
where $\eta $ is frequently chosen as a fixed small learning rate parameter and $\nabla g $ is the gradient of function $g$, with respect to 
$(\beta ,\gamma )$.  Very importantly,  the structure of the network and the 
    chain rule of calculus 
    facilitates efficient estimation of the gradients 
    via backpropagation, e.g. \parencite{rumelhart1986learning}.
    
    A problem entailed by NN model estimation is the highly multimodal nature of
    the log-likelihood for three reasons:
    invariance with respect to arbitrary relabeling of
    parameters (these may
    be handled by means of order 
    constraints among the 
    parameters);
    inherent non-linearity; and, finally, 
    node duplicity (which may be dealt 
    with a model reflecting uncertainty 
    about the number of nodes,
    as explained below).
A way to mitigate multimodality is to use a global optimization method, like multistart, but
this is very demanding computationally in this domain.

Finally, note that the same kind of NN models 
may be used for nonlinear auto-regressions in
time series analysis \parencite{menchero} and 
non-parametric 
regression \parencite{insuamuller}. Moreover,
similar models may be used for classification purposes,
although this 
 requires modifying the likelihood
\parencite{bishop2006pattern} to e.g.\
\begin{equation}
    p(y | x, \beta, \gamma) = Mult(n=1, 
    p_1 (x, \beta, \gamma) , \ldots, p_K (x, \beta, \gamma) ),
\end{equation}
that is, a draw from a multinomial distribution with $K$ classes. 
Class probabilities
 can be computed using the softmax function,
$$
p_k = \frac{\exp{\beta_k \psi(x'\gamma_k)}}{\exp{\sum_{k=1}^K \beta_k \psi(x'\gamma_k)}}.
$$

\paragraph{Bayesian approaches.}\label{bayeshallow}
We discuss now Bayesian approaches to shallow NNs.
assuming standard priors 
in Bayesian hierarchical modeling, see e.g.\ \textcite{LavineWest}: 
$  \beta_i      \sim  N(\mu_\beta,\sigma_\beta^2)$
and 
  $\gamma_i     \sim  N(\mu_\gamma,S_\gamma^2)$,
  completed with priors over the hyperparameters
$\mu_\beta \sim N(a_\beta,A_\beta)$,
$\mu_\gamma \sim N(a_\gamma,A_\gamma)$,
$\sigma^{-2}_\beta \sim Gamma(c_b/2,c_b/2)$,
$S_\gamma^{-1} \sim Wish(c_\gamma,(c_\gamma C_\gamma)^{-1})$ and
$\sigma^{-2} \sim Gamma(s/2, s/2)$.
In this model, 
an informative prior probability model
is meaningful as parameters are interpretable. For example, the $\beta_ j$’s would reflect the
order of magnitude of the data $y_i$; typically positive and negative values for
$\beta _j$ would be equally likely, calling for a symmetric prior around 
0 with
a standard deviation reflecting the range of plausible values for $y_i$. Similarly,
a range of reasonable values for the logistic coefficients $\gamma_ j$ will be determined
mainly to address smoothness
issues.

Initial attempts to perform Bayesian analysis
of NNs, adopted
arguments based on the asymptotic normality of the posterior, 
as in  
\parencite{mckay} and \parencite{buntineweigend}. However these methods
fail if they are dominated by  
less important modes. 
 \parencite{buntineweigend} mitigate this by finding several modes and
basing inference on weighted mixtures of the corresponding normal approximations, but we return to the same issue as some
important local modes might have been left out. An alternative view was argued
by \parencite{mckay}: inference from such schemes is 
 considered as approximate
posterior inference in a submodel defined by constraining the
parameters to a neighborhood of the particular local mode. Depending on
the emphasis of the analysis, this might be reasonable, especially if in a final
implementation our aim is to set the parameters at specific values,
the usual scenario in deep learning. We prefer
though to propagate the uncertainty in parameters, since this allows
better predictions, e.g. \parencite{raftery}. 

For this, an efficient 
Markov chain Monte Carlo (MCMC) scheme
may be used \parencite{muller1998issues}. It 
 samples from the posterior conditionals when available (steps 3, 9), and use
Metropolis steps (4-8), otherwise. To fight potential inefficiencies due to
multimodality, two features are built in 
 for
fast and effective mixing over local posterior modes:
whenever
possible, the $\gamma$ weights are 
partially marginalized; second,
these weights are resampled jointly.
The key observation is that,
given
 $\gamma $,
we actually have a standard hierarchical normal linear
model \parencite{french}. This facilitates sampling from the posterior marginals of the $\beta $ weights (step 3)
 and hyperparameters (step 9) and 
 allows marginalizing the model with respect
 to the $\beta$'s
  to obtain the marginal likelihood $p(D|\gamma, \nu)$
  (step 3), where $\nu=(\mu_\beta,\sigma_\beta,\mu_\gamma,S_\gamma,\sigma^2)$ designates the hyperparameters.
The procedure runs like the described in Algorithm \ref{alg:mcmc}.

\begin{algorithm}[!ht]
\begin{algorithmic}[1]
\State Start with arbitrary $(\beta , \gamma ,\nu )$.
\While{not convergence} 
 \State Given current $(\gamma,\nu)$, draw  
    $\beta$ from  
    $p(\beta|\gamma,\nu,y)$ (a multivariate normal).
    \For{$j=1,...,m$, marginalizing in $\beta$ and given $\nu$ } 
    \State Generate a candidate $\tilde\gamma_j \sim g_j(\gamma_j)$.
    \State Compute 
    $
       a(\gamma_j,\tilde\gamma_j) =
       \min\left(1,\frac{p(D |\tilde\gamma,\nu)}
                       {p(D |\gamma,\nu)}\right)
    $
    with $\tilde\gamma = (\gamma_1,\gamma_2,\ldots,\tilde\gamma_i, ...,\gamma_m)$.
    \State With probability $a(\gamma_j,\tilde\gamma_j)$ replace $\gamma_j$
    by $\tilde\gamma_j$. If not, preserve $\gamma_j$.
    \EndFor
    \State Given $\beta$ and $\gamma$, replace $\nu$
        based on their posterior conditionals:
 $p(\mu_\beta|\beta,\sigma_\beta)$ is normal;
 $p(\mu_\gamma|\gamma,S_\gamma)$, multivariate normal;
 $p(\sigma_\beta^{-2}|\beta,\mu_\beta)$, Gamma; 
    $p(S_\gamma^{-1}|\gamma,\mu_\gamma)$, Wishart; 
    $p(\sigma^{-2}|\beta,\gamma,y)$, Gamma.
    
\EndWhile
\end{algorithmic}
 \caption{MCMC sampler}\label{alg:mcmc}
\end{algorithm}

\noindent For proposal generation distributions $g_j(\cdot)$,
normal multivariate distributions
$N(\gamma_j,c^2 C_\gamma)$ are 
adopted. 
Appropriate values for $c$ can be found by trying
a few alternative choices until acceptance rates around
 0.25 are achieved \parencite{gamerman}. 

Combined with model augmentation to a variable architecture,
 this leads to a useful scheme
for complete shallow NN analyses as it allows for 
the identification of
architectures supported by data, by  
contemplating $m$ as an additional parameter.
A random $m$ with a prior 
favoring smaller values reduces posterior multimodality.
Moreover, as marginalization over $\gamma_ j$ requires inversion of matrices
of dimension related to $m$, 
avoiding unnecessarily large
hidden layers is critical to
mitigating computational effort.
 Thus, we assume 
a maximum size $m^*$ for the network and introduce 
indicators  $d_j$ suggesting whether node
$j$ is included ($d_j=1$) or not ($d_j=0$). 
We  also 
include a linear regression term $x'a$
to favor parsimony. On the whole, the model
becomes 
\begin{eqnarray*}
  y          & = & x_i'a + \sum_{j=1}^{m^*} d_j\beta_j \psi(x '\gamma_j) +
                    \epsilon \\ 
                    & & \epsilon \sim N(0,\sigma^2),\nonumber \\
                    & &    \psi(\eta) = \exp(\eta)/(1+\exp(\eta)),
                        \nonumber \\
  Pr(d_j=0|d_{j-1}=1)   & = & 1-\alpha, \nonumber\\
  Pr(d_j=1|d_{j-1}=1)   & = & \alpha, \nonumber\\
  \beta_i    \sim  N(\mu_b,\sigma_\beta^2),& 
  a     \sim  N(\mu_a,\sigma_a^2), &   \gamma_i   \sim  N(\mu_\gamma,\Sigma_\gamma).
                \label{eq:model-var}
\end{eqnarray*}
Learning is done through a reversible jump sampler
\parencite{green} embedding our first algorithm.
As a consequence, we perform inference
about the architecture based on the distribution of 
 $p(m|D)$. 

 \textcite{neal2012bayesian} proposed using an 
algorithm merging conventional Metropolis-Hastings chains with sampling
techniques based on dynamic simulation, the currently popular
Hamiltonian Monte Carlo (HMC) approaches.
Let us designate by $\theta$ the NN 
weights, $\theta = (\beta, \gamma)$, and denote the potential energy function as
$$
U(\theta) = \tau_{\beta}\sum_{i=1}^m \beta_i^2/2 + \tau_{\gamma} \sum_{i=1}^m \gamma_i^2/2 + \tau \sum_{j=1}^n (y_j - f(x_i, \theta_i))^2/2,
$$
where $\tau_{\beta}, \tau_{\gamma}, \tau$ are hyperparameters 
controlling regularization, similarly to (\ref{kkdbak}). 
Let us also introduce the Hamiltonian 
$$
H(\theta, r) = U(\theta) + \frac{1}{2} \sum_{i=1}^m r_i^2,
$$
with momentum variables $r$ of the same dimension as $\theta$; such 
 variables serve to accelerate the walk towards posterior modes. Then, the HMC scheme would be as described in Algorithm \ref{alg:hmc}.

\begin{algorithm}[!ht]
\begin{algorithmic}[1]
\State Start with arbitrary $\theta_0 = (\beta _0, \gamma _0)$.
\While{not convergence}
  \State Given current $\theta_t$ and $r_t \sim \mathcal{N}(0, I)$, perform one or more leapfrog integration
  steps 
  \begin{align*}
  r_{t + \frac{1}{2}} &= r_t - \frac{\epsilon}{2}\nabla U(\theta_t) \\
   \theta_{t+ 1} &= \theta_t + \epsilon  r_{t + \frac{1}{2}} \\
   r_{t + 1} &= r_{t + \frac{1}{2}} - \frac{\epsilon}{2}\nabla U(\theta_{t+ 1})
    \end{align*}
    to reach  $\theta^*$ and $r^*$.
    \State Compute $\alpha(\theta_t, \theta^*) = \min \left\{ 1, \frac{\exp H(\theta^*, r^*)}{\exp H(\theta_t, r_t)} \right\} $.
   \State Accept $\theta^*$ as $\theta_{t+1}$ with probability $\alpha(\theta_t, \theta^*)$, else discard it. 
  
 \EndWhile
 \end{algorithmic}
 \caption{HMC sampler}\label{alg:hmc}
\end{algorithm}

\subsection{Deep neural networks}\label{sec:dnn_examples}

Training by backpropagation has been in use 
for many years by now.
The decade of the 2010's witnessed major developments
leading to 
the boom around deep learning \parencite{10.5555/3086952}
or inference and prediction with deep NNs. 
Such advances include:
the availability of fast GPU kernels and routines
facilitating much faster training;
the access to massive amounts of data (e.g.\ Imagenet), which prevented overfitting to smaller datasets; 
the creation of new architectures, which prevented  convergence issues, such as the vanishing gradient problem; 
and, finally, the provision of automatic differentiation libraries such as Tensorflow, Caffe or Theano.
Figure \ref{figuradkk} displays an example
of what are now designated 
deep NNs, that is NNs with more than one hidden
layer, four in the portrayed case.
\begin{figure}
    \centering
    \includegraphics[scale=0.35]{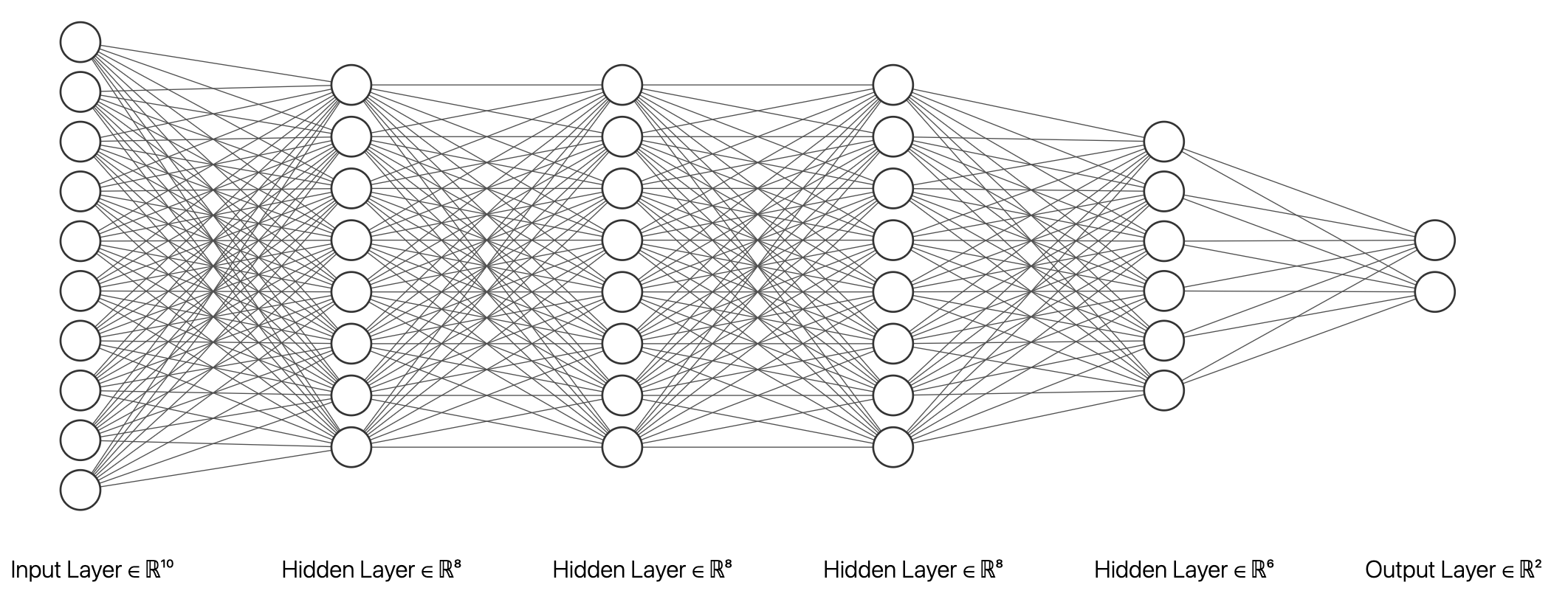}
    \caption{A deep NN architecture with four hidden layers and 2 scalar outputs}
    \label{figuradkk}
\end{figure}

A deep NN may be defined through   
a sequence of functions $\lbrace f_0, f_1, ..., f_{L-1} \rbrace$, each parametrized by some weights $\gamma_l$
of dimension $m_l$  (the corresponding number of hidden nodes) with the output of each layer being the input of the following one, as in
$$
    z_{l+1} = f_l ( z_l, \gamma_l).
$$
Lastly, we compute a prediction from the hidden activations of the last layer, as in Eq. (\ref{kantora})
\begin{eqnarray}
y         & = & \sum_{j=1}^{m_L} \beta_j z_{L,j} +
                    \epsilon 
                    \nonumber\\
              & & \epsilon \sim N(0,\sigma^2),
                  \nonumber \\
\end{eqnarray}

Modern architectures do not longer require the
$f_l$ functions to be sigmoidal (like the  
logistic functions in (\ref{kantora})) and include  
the rectified linear unit (ReLU), the leaky ReLU
or the exponential LU. In particular, these functions mitigate the vanishing gradient problem \parencite{kolen2001gradient} that 
plagued earlier attempts with deep architectures using sigmoidal
activation functions. Besides, such activation functions have other benefits likes being faster to compute both the activation and its derivative.

Beyond the above generic deep architectures 
a few important specialised models have emerged which 
are relevant in specific application domains,
as we describe now. 
\paragraph{Convolutional neural networks} 
CNNs are typically used 
 in computer vision tasks and related signal processing applications.
 Stemming from the work by Le Cun and coauthors  
  \parencite{lecun89, lecun98} and their original 
  LeNet5 design, they achieved major
 successes  in  competitions \parencite{NIPS2012_c399862d}
 leading to architectures like 
 AlexNet       \parencite{NIPS2012_c399862d}, VGGNet \parencite{simonyan2014very} or
 GoogleNet \parencite{szegedy2015going}, reaching 
 superhuman performance in 
 image recognition tasks.
 
 In CNNs, the layer transformation is taken to be a convolution with some 2D or 3D kernel; this makes the network able to recognise patterns independently of their location or scale in the input, a desirable property in computer vision tasks known as spatial equivariance. In addition, by replacing a fully-connected layer with a small kernel (typically, in the 2D case these are of shapes $3\times3, 5\times 5$ or $7\times 7$), there is weight sharing amongst the unit from the previous layer and this allows reducing 
 the number of parameters and prevents overfitting.
 Thus, the typical convolutional network layer is composed of several sub-layers:
 \begin{itemize}
     \item A convolution operation, as before, 
     serving as an affine transformation of the representation from the previous layer. A layer can apply several convolutions in parallel to produce a set of linear activations.
     \item A non-linear layer, such as the rectifier (based on the ReLU function), converting the previous activations to nonlinear ones.
     \item An optional pooling layer, which replaces the output of the net at a certain position with a summary statistic of 
     the nearby outputs (typically the mean or the maximum).  
 \end{itemize}
 
 

\paragraph{Recurrent neural networks} RNNs are typically used for sequence processing, as in natural language processing (NLP), e.g.\ \parencite{hochreiter1997long} and \parencite{chung2014empirical}. They have feedback connections which make the network aware of temporal dependencies in the input.
Let $x_t$ denote the $t$-th token (usually a word, but could even be 
a smaller part) in the input sequence. A simple recurrent layer can be described as
$$
h_t = \psi(W_x x_t + W_h h_{t-1} + b_h)
$$
with the output at that time-step given by
$$
y_t = \psi(W_y h_t + b_y),
$$
where $\psi$ is a non-linear activation function, such as the logistic or the hyperbolic tangent functions. This is the Elman network \parencite{cruse2006neural}. Note that weights are shared between different time-steps. 
For training purposes, all of the previous loops must be unrolled back in time, and then perform the usual gradient descent routine.
This is called \emph{backpropagation through time} \parencite{58337}.
Backpropagating through long sequences may lead to problems
of either vanishing or exploding gradients. Thus, 
simple architectures like Elman's cannot be applied to long inputs, as those arising in NLP. 
As a consequence, gating architectures which improve the stability have been proposed, and successfully applied in real-life tasks,
such as long 
short-term memory (LSTM) \parencite{hochreiter1997long} and gated recurrent unit (GRU) networks \parencite{cho2014learning}. 


\paragraph{Transformers} These architectures substitute the sequential processing from RNNs by a more efficient, parallel approach inspired by 
attention mechanisms \parencite{vaswani2017attention,bahdanau2014neural}. Their basic building components are scaled dot-product attention layers:
let $x_i$ be the embedding\footnote{An embedding layer is a linear layer that projects a one-hot representation of words into a lower-dimensional space.} of the $i$-th token in an input sequence;
this is multiplied by three weight matrices to obtain: 1) a query vector, $q_i = W_q x_i$; 2), 
 a key vector $k_i = W_k x_i$;
 and, 3) a value vector, $v_i = W_v x_i$. The output of the attention layer is computed, parallelizing along the input position $i$, 
 through 
$$
\mbox{softmax}\left(\frac{q k^{'}}{\sqrt{d_k}}\right) v,
$$
 a weighted average of the components of the value vector, where the average is 
computed as a normalized dot product between the key and query vectors. Thus, the attention layer produces activations for every token  that contains information not only about the token itself, but also a combination of other relevant tokens weighted by the attention weights.

Since transformer-based models are more amenable to parallelization,
 they have been trained over massive datasets in the NLP domain, leading to architectures such as Bidirectional 
Encoder Representations for Transformers (BERT) \parencite{devlin2018bert}
or the series of Generative
pre-trained Transformer (GPT) models \parencite{radford2018improving, radford2019language, brown2020language}. 

\paragraph{Generative models} 
The models from the previous paragraph belong to the discriminative family of models. Discriminative models directly learn the conditional distribution $p(y|x)$ from the data.
Generative models, as opposed to discriminative ones, take a training set, consisting of samples from a distribution $p_{data}(x)$, and learn to represent an estimation of that distribution, resulting in another probability distribution, $p_{model}(x)$. Then, one could fit a distribution to the data by performing MLE,
$$
\max_{\theta} \sum_{i=1}^n \log p_{model} (x_i | \theta),
$$
or MAP estimation if a prior over the parameters $\theta$ is also placed. Fully visible belief networks \parencite{10.5555/2998828.2998922} are a class of models that can be optimized this way. They are computationally tractable since they decompose the probability of any given $d-$dimensional input as $p(x | \theta) = \prod_{i=1}^d p(x_i | x_1 , \ldots x_{i-1}, \theta)$. Current architectures that fall into this category include WaveNet \parencite{oord2016wavenet} and pixel recurrent neural networks \parencite{pmlr-v48-oord16}. 

\paragraph{Generative adversarial networks} GANS  perform density estimation in high-dimensional spaces formulating a game between two players, a generator and a discriminator \parencite{goodfellow2014generative}. They belong to the family of generative models; however, GANs do not explicitly model a distribution $p_{model}$, i.e., they cannot evaluate it, only generate samples from it.
GANs define a probabilistic graphical model containing observed variables $x$ (the input data, like an image or text) and latent variables $z$. Then, both players can be represented as two parameterized functions via NNs. Thus, the generator will be of the form $f_G(z, \theta_G)$, i.e., a NN that takes as input a latent vector and is parameterized through weights $\theta_G$. Note that the last layer will depend on the shape and range of the data $x$. Likewise, the discriminator will be a function $f_D(x, \theta_D)$ receiving a (fake or real) sample $x$ and outputting a probability for each of these two classes. Therefore, the final activation function will be 
the sigmoid function. Each network has its own objective function
and, consequently, both networks would play a minimax game.
Now, both players update their weights sequentially, typically using SGD or any of its variants.



\paragraph{Classical approaches.}

In principle, we could think of using 
with deep NNs the approach in Section \ref{sanchez}. However,
 large scale problems bring in two major
computational issues: first, the 
evaluation of the gradient of the loss wrt the parameters
requires going through all observations
 becoming too expensive when $n$ is large;
second, estimation of the gradient component
for each point requires a much longer backpropagation recursion through the various levels of the
deep network, 
entailing again a very high computational 
expense. 

Fortunately, these computational demands are mitigated
through the use of classical stochastic gradient descent
(SGD)
methods \parencite{robbins}
to perform the estimation \parencite{bottou2010large}. SGD is the current workhorse of large-scale optimization and 
allows training deep NNs over large datasets by mini-batching: rather than going through the whole batch at each stage, just pick a small sample
(mini batch) of observations and do the corresponding
gradient estimation by backpropagation. This is reflected in the Algorithm \ref{alg:sgd}, which departs from an
initial $\theta$.

\begin{algorithm}[!ht]
\begin{algorithmic}[1]
\While{stopping criterion not met}
  \State Sample a size $l$ minibatch 
      $((x^{(1)}, y^{(1)}),..., (x^{l},y^{l})) $
              from training set. 
\State Compute a gradient estimate
      $g_k = \frac{1}{l} \sum_{i=1}^l \nabla_{\theta} f_i(\theta_k) + \nabla_{\theta} h(\theta_k)$
\State    Update $\theta_{k+1} $ = $\theta_k -\epsilon _k g_k$ 
\State    $k=k+1$ 
 \EndWhile
\end{algorithmic}
 \caption{Stochastic gradient descent}\label{alg:sgd}
\end{algorithm}

\noindent The standard Robbins-Monro conditions require
that $\sum _k \epsilon_k = \infty$ and 
$\sum _k \epsilon _k^2 < \infty $ for convergence to the optimum.

Recent work has explored ways to 
speed up convergence towards the local optimum, 
leading to several SGD variants, including the 
 addition of a momentum, as in AdaGrad, Adadelta or Adam \parencite{kingma2014adam,duchi2011adaptive,zeiler2012adadelta}. The essence of these methods is to take into account not only a moving average of the gradients, but also an estimate of its variance, so as to dynamically adapt the learning rate. 
Let us finally mention 
several techniques to improve generalization and convergence of neural networks, such as dropout \parencite{srivastava2014dropout}, batch normalization \parencite{ioffe2015batch} or weight initialization \parencite{glorot2010understanding}.


\subsection{Examples}\label{sec:nn_examples}
We illustrate learning with deep NNs with two examples portraying
different architectures and application domains. Code for the experiments was done using the \texttt{pytorch} library \parencite{paszke2017automatic} and is released at \url{https://github.com/vicgalle/nn-review}.

\paragraph{CNNs for image recognition.}\label{kkvision}

We describe an image classification task with a standard CNN, VGG-19, showcasing the superior performance when compared to non-convolutional and non-deep approaches in this application domain. As benchmark, we use the CIFAR-10 dataset \parencite{krizhevsky2014cifar},
which consists of 60000 32x32 colour images in 10 classes. As baselines, we use a linear multinomial regression model, and a three hidden layer NN (MLP) with 200 units each with ReLU activations.  {As Bayesian features for this experiment, note the use of Gaussian prior over parameters, which benefits performance compared to the vanilla MLE. Notice also that, in CNNs, we are also imposing an implicit prior, which consists in restricting the set of linear layers to be convolutions, a particular case of linear operator.}

All models are trained for 200 epochs\footnote{An epoch is defined as a pass over the full training dataset.} and minibatches of 128 samples, using SGD with learning rate of 0.1 and  momentum of 0.9.
The learning rate is adapted with the cosine annealing scheme in  \textcite{loshchilov2016sgdr}. We place  {independent} Gaussian priors over all parameters (thus equivalent to $\ell_2$ regularization) and use SWA on top of the SGD optimizer, to make predictions using an ensemble of   {100 posterior samples}, in a Bayesian way.

\begin{figure}[hbt]
\centering
  \includegraphics[width=1.\linewidth]{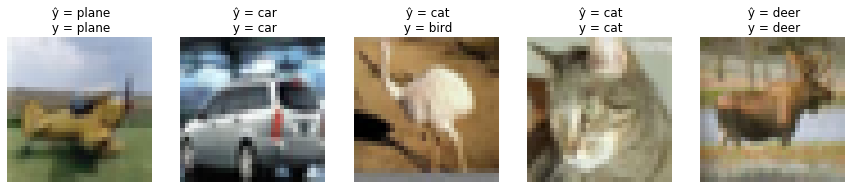}
  \caption{ {Five random samples from the CIFAR-10 dataset and their corresponding predictions and true labels using the VGG architecture. Notice how the convolutional network mistakes an ostrich (bird class) for a cat.}}
  \label{fig:cifar}
\end{figure}

Figure \ref{fig:cifar} presents five random images from the test set
and Table \ref{tab:cnn} displays results. Notice how the linear model performs poorly with just a test accuracy of 38\%, whereas increasing the flexibility of the models critically improves 
accuracy. Indeed, the MLP performs slightly better.
However, by imposing strong priors about the dataset, such as translation equivariance thanks to the convolutional layers and pooling, state-of-the-art results are achieved with VGG-19.

\begin{table}[ht]
\caption{Results over the CIFAR-10 test set.}
\centering
\begin{tabular}{ll}
Model & Test acc. \\
\hline
Linear &  $38.10\%$ \\
MLP &  $50.03\%$\\
VGG-19 &  $93.29\%$ 
\end{tabular}
\label{tab:cnn}
\end{table}

\paragraph{Transformer and recurrent models for sentiment analysis.}
This is an example with text to undertake sentiment analysis, classifying movie reviews with architectures tailored to NLP tasks. As benchmark, we use the IMBD movie review dataset 
from \parencite{maas-EtAl:2011:ACL-HLT2011}, in which a review in raw text must be classified into one of two classes: positive or negative
sentiment. 

The recurrent model consists of a LSTM network with two layers, each with hidden dimension of 256 plus a dropout of 0.5 serving as a regularizer. We also consider a simple Transformer-based model, with two encoder layers, of hidden dimension 10 plus similar dropout. For both models, the input sequence is represented as a list of one-hot encoded vectors, representing the presence of a word from a vocabulary of the 5000 most common tokens. This representation is embedded into a space of 100 dimensions
(16 in the Transformer case) by means of an affine transformation, before applying architecture-specific layers. Both models are trained using the Adam optimizer with a constant learning rate of 0.001.
We also consider a bigger, transformer-based model consisting of the recent RoBERTa architecture \parencite{liu2019roberta}, 
initially pretrained on a big corpus of unsupervised raw  English text, and then fine tuned for two epochs on the IMDB training set. 

Table \ref{tab:examples} shows two random examples from the IMDB dataset.
Results over the full test set are displayed in Table \ref{tab:nlp}. Notice how the Transformer-based models are superior to the recurrent baseline, and the extra benefits thanks to the usage of transfer learning.

{\footnotesize
\begin{table}[ht]
\caption{Two random review samples and their corresponding predictions and true labels from the RoBERTa model.}
\centering
\begin{tabularx}{\textwidth}{Xll}
Text input & Prediction & True label \\
\hline
 {\scriptsize \texttt{HOW MANY MOVIES ARE THERE GOING TO BE IN WHICH
AGAINST ALL ODDS, A RAGTAG TEAM BEATS THE BIG GUYS
WITH ALL THE MONEY?!!!!!!!! There's nothing new in
"The Big Green". If anything, you want them to
lose. Steve Guttenberg used to have such a good
resume ("The Boys from Brazil", "Police Academy",
"Cocoon"). Why, OH WHY, did he have to do these
sorts of movies during the 1990s and beyond?! So,
just avoid this movie. There are plenty of good
movies out there, so there's no reason to waste
your time and money on this junk. Obviously, the
"green" on their minds was money, because there's
no creativity here. At least in recent years,
Disney has produced some clever movies with Pixar.} } &   Negative & Negative   \\
\hline
 {\scriptsize \texttt{When I first heard that the subject matter for
Checking Out was a self orchestrated suicide
party, my first thought was how morbid, tasteless
and then a comedy on top of that. I was skeptical.
But I was dead wrong. I totally loved it. The
cast, the funny one liners and especially the
surprise ending. Suicide is a delicate issue, but
it was handled very well. Comical yes, but tender
where it needed to be. Checking Out also deals
with other common issues that I believe a lot of
families can relate with and it does with tact and
humor. I highly recommend Checking Out. A MUST
SEE. I look forward to its release to the public.} } &   Positive & Positive   \\

\end{tabularx}
\label{tab:examples}
\end{table}
}

\begin{table}[h]
\caption{Results over the IMDB test set.}
\centering
\begin{tabular}{ll}
Model & Test acc. \\
\hline
LSTM &  $81.99\%$ \\
Simple Transformer &  $87.49\%$ \\
RoBERTa &  $94.67\%$\\
\end{tabular}
\label{tab:nlp}
\end{table}

\section{Challenges}

The reflections motivated by the case study in Section 1.2, using classical Bayesian approaches, and the review in Section 1.3, have led us to enumerate the following three major challenges, object of study in this thesis.

\subsection{Challenge 1. Large scale Bayesian inference}


MCMC algorithms such as the ones presented in Section \ref{s:MCMC} or \ref{bayeshallow} 
have become standard in Bayesian inference \parencite{french}.
However, they entail a significant computational burden in large datasets. Indeed, 
computing the corresponding acceptance probabilities 
 requires iterating over the whole dataset, which often does not even
 fit into memory. Thus, they 
  do not scale well in big data settings. 
 As a consequence, several approximations have been proposed.

\paragraph{Stochastic Gradient Markov chain Monte Carlo.}\label{bayesdeep} 

SG-MCMC methods are based on the discretization of 
stochastic differential equations that have the desired target 
distribution as its limit. \textcite{ma2015complete} provide a
complete framework that encompass many earlier proposals and
facilitate such discretization, as well as a practical tool for
devising new samplers and testing the correctness of proposed samplers.
We aim at drawing samples from the
posterior $p(\theta |D) \propto \exp(-U(\theta ))$,
with potential function
$U(\theta ) = -\sum _{x\in D} \log p(x|\theta ) + \log p(\theta )$. Define also auxiliary variables $r$,
with $z=( \theta, r )$, and sample from $p(z|D) \propto  \exp(-H(z))$, with hamiltonian
$H(z) = H(\theta , r) = U(\theta ) + g(\theta , r)$, such that
$\exp(-g(\theta , r))dr = constant$. 
Marginalizing the auxiliary variables gives us the desired distribution on $\theta $.

In general, all continuous Markov processes that one might consider for sampling can be written
as a stochastic differential equation (SDE) of the form:
\begin{equation}
dz = f(z)dt +
\sqrt{ 2D(z)}dW(t),
\end{equation}
where $f(z)$ denotes the deterministic drift, frequently
related to $\nabla H(z)$,
$W(t)$  is a $d$-dimensional Wiener process, and 
$D(z)$ is a positive semidefinite diffusion matrix. Note, though,
that some care must be taken to choosing $f(z)$ and $D(z)$ 
to yield the desired stationary distribution.
\parencite{ma2015complete} propose a recipe for constructing SDEs with the correct stationary distribution through 
$f(z) = - [D(z) + Q(z)] \nabla H(z) + \Gamma (z)$, 
and $\Gamma _i (z) = \sum _{j=1}^d 
\frac{\partial  }{\partial z_j}(D_{ij} (z) + Q_{ij} (z) )
$
where $ Q(z)$ is a skew-symmetric curl matrix (representing the deterministic traversing effects seen
in HMC procedures) and the diffusion matrix $D(z)$
determines the strength of the Wiener process-driven diffusion.
When 
$D(z)$ is positive semidefinite and $Q(z)$ is skew-symmetric, 
 the convergence of the above dynamics to the desired 
distribution follows; moreover both matrices can be adjusted to attain faster convergence to
the posterior distribution. \parencite{ma2015complete} show that by properly choosing the matrices  one can recover numerous samplers such
as SGLD \parencite{welling2011bayesian} or a corrected SG-HMC \parencite{chen2014stochastic}.

In practice, simulation actually relies on a discretization of the SDE, leading to a (full-data) update rule
\begin{equation}
z_{t+1} \leftarrow z_t - \epsilon_t \left[ ( D(z_t) + Q(z_t) )
\nabla H(z_t) + \Gamma (z_t)\right]
+ N (0, 2\epsilon _t D(z_t)).
\end{equation} 
Calculating $\nabla H(z)$ entails evaluating the gradient of $U(\theta )$
which, with deep NN models, 
 becomes very intensive computationally as it relies on a sum
over all data points. Instead, 
we use a sampled data subset $S' \subset S$, with the corresponding potential for these data being 
$U_1 (\theta ) = -\frac{|S'|}{|S|} \sum _{x \in S'}
\log p(x|\theta ) + \log p(\theta )$, leading 
to the approximation 
\begin{equation}
z_{t+1} \leftarrow z_t - \epsilon_t \left[ ( D(z_t) + Q(z_t) )
\nabla H (z_t) + \Gamma (z_t) \right]
+ N (0, \epsilon _t (2 D(z_t)- \epsilon_t \hat {B_t}))
\end{equation} 
where $\hat {B_t}$ is an estimate of the variance of the error.
This provides the  stochastic gradient—or minibatch— variant of the sampler. Note also that all of these approaches could be combined with recent heuristics to improve the convergence of SG-MCMC methods, such as the adoption of cyclical step sizes to explore more efficiently the posterior  \parencite{7926641}.

\paragraph{Variational Bayes.} 

Variational inference (VI) \parencite{blei2017variational} tackles the 
 approximation of  $p(\theta | D)$ with a tractable parameterized
 distribution $q_{\phi}(\theta |D)$. The goal is to find parameters $\phi$ so that the distribution 
$q_{\phi}(\theta |D )$  (referred to as variational guide
or variational approximation)  is as close as possible to the actual posterior, with closeness typically measured through 
the Kullback-Leibler 
divergence $KL(q_{\phi } || p)$, reformulated into the ELBO
\begin{equation}\label{eq:elbo}
\mbox{ELBO}(q) = \mathbb{E}_{q_{\phi}(\theta |D)} \left[ \log p(D,\theta ) - \log q_{\phi}(\theta |D)\right],
\end{equation}
the objective to be optimized,
usually through SGD techniques. 

A standard choice
for $q_{\phi}(\theta |D )$ is a factorized Gaussian 
distribution $\mathcal{N}(\mu_{\phi}(D), \sigma_{\phi}(D))$,
with  mean and covariance matrix defined through a
 deep NN conditioned on the observed data $D$.
 Note though that 
other distributions can be adopted as long as they 
 are easily sampled and their log-density and entropy evaluated. 
A problem is that these approximations often 
underestimate the uncertainty. Some developments
partly mitigate this  issue
by  enriching the variational family include normalizing flows \parencite{rezende2015variational} or the use of implicit distributions \parencite{huszar2017variational}.

In Chapter 2 we will further review these approaches and propose two frameworks to further scale up Bayesian inference in challenging settings. The first is based on augmenting an SG-MCMC sampler with multiple, parallel chains, incorporating interactions between particles so they can explore the posterior more efficiently. The second is a novel variational approximation that serves to automatically adapt the hyperparameters of any SG-MCMC sampler.


\subsection{Challenge 2. Security of machine learning}
\label{sec:chall2}

As described, over the last decade
an increasing number of processes are being automated through 
deep NN algorithms, being 
essential that these are robust and reliable
if we are to
trust operations based on their output. State-of-the-art
algorithms, as those described above, perform extraordinarily well on standard data,  but have been
shown to be vulnerable to adversarial examples, data instances targeted at
fooling them \parencite{goodfellow2014explaining}.
The presence of adaptive adversaries has
been pointed out in areas such as spam detection \parencite{zeager2017adversarial}
and computer vision \parencite{goodfellow2014explaining}, among many others. 
In those contexts, algorithms should acknowledge the presence of possible adversaries
to protect from their data manipulations.
As a fundamental underlying hypothesis, NN
based systems rely on using 
independent and identically distributed (iid) data for both training and operations. However, security aspects in deep
learning, part of the emergent field of
adversarial machine learning (AML),
question such hypothesis, given the
presence of adaptive adversaries ready to  intervene in the problem 
to modify the data and obtain a benefit. In addition, time series models, such as the introduced in Section \ref{sec:dlms} can also be subject to adversarial attacks \parencite{9063523,10.5555/3016100.3016102}, so it is of great interest to develop model-agnostic defences.

As a motivating example, vision algorithms (Section
\ref{kkvision}) are at the core of many AI 
applications such as autonomous driving systems (ADSs) \parencite{rumanos}. 
The simplest and most notorious attack examples to
such algorithms  
consist of modifications of images in such a way that the alteration becomes insignificant to the human eye, yet drives a model trained on millions of images to misclassify the modified ones,
with potentially relevant security consequences.
With a relatively simple CNN model, we are able to accurately predict 99\% of the handwritten digits in the MNIST data set.
However, if we attack those data with the fast gradient sign method \parencite{szegedy2013intriguin},
accuracy gets reduced to 62\%. Fig.~\ref{fig:digits} provides an example of an original MNIST image and a perturbed one: to our eyes both images look like a 2, but the classifier rightly identifies a 2 in the first case, whereas it suggests a 7 after the perturbation. 
%
%
%
\begin{figure}[hbt]
\centering
  \includegraphics[width=.6\linewidth]{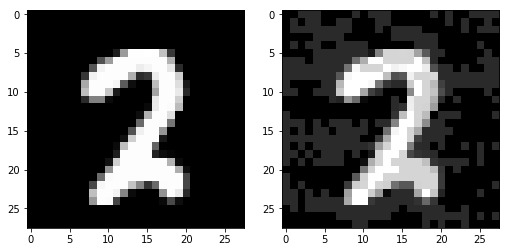}
  \caption{Left: original image, correctly classified as a 2. Right: slightly perturbed image, wrongly classified as a 7.}
  \label{fig:digits}
\end{figure}
Stemming from the pioneering work in adversarial classification 
\parencite{dalvi2004adversarial}, the prevailing paradigm in AML models
the confrontation between learning-based systems and adversaries through game theory. 
This entails common knowledge assumptions 
\parencite{gameTheoryACriticalIntroduction2004} which are 
questionable in security 
applications as adversaries try to conceal information. 
As \parencite{fan2019selective} points out, there is a need for a  framework that guarantees robustness of ML against adversarial manipulations in a principled manner. 

The usual approach for robustifying models against these examples is {\em adversarial training} (AT) \parencite{madry2018towards} and its 
variants, based on solving a 
bi-level optimisation problem whose objective function is the empirical risk of a model under worst case data perturbations. 
However, recent pointers urge modellers to depart from using 
norm based approaches \parencite{carlini2019evaluating} and develop more realistic attack models.

AML is a  difficult area which rapidly evolves and leads to an 
arms race in which the community alternates cycles of proposing attacks and  of implementing defences that deal with them. However, as mentioned, it is 
based on game theoretic ideas and strong common
knowledge conditions. The challenge is thus to develop defence mechanisms that are sufficiently scalable to the increasingly complex ML models using in real life.
Chapter 3 examines these issues, providing a scalable defence procedure inspired by ARA \parencite{adversarialRiskAnalysis2009,banks2015adversarial}.


\subsection{Challenge 3. Large scale competitive decision making}

In real life scenarios, decision makers rarely have to take a single action. Instead, multiple decisions must be made, with the former actions making a causal impact on the later ones. For example, in the case study from Section \ref{sec:dlms}, the decision maker might desire to plan ahead for a year the budget allocations of all the advertising channels. Thus, we have to delve into the realm of sequential decision making \parencite{french,DIEDERICH200113917}. There are many paradigms to tackle this problem, such as optimal control theory \parencite{kirk2004optimal}. Since this thesis is also focused on Machine Learning, we adopt the framework of Reinforcement Learning (RL) to deal with sequential decisions while making the decision maker able to learn from her experiences \parencite{sutton2012reinforcement,kaelbling1996reinforcement}. RL differs from supervised learning in not needing labelled input/output pairs be presented, and in not needing sub-optimal actions to be explicitly corrected. Instead the focus is on finding a balance between exploration (of uncharted experiences) and exploitation (of current knowledge). Of course, RL can also benefit from recent developments in deep, neural architectures to further enhance its efficiency in complex and highly dimensional environments \parencite{mnih2015human}.

However, RL typically only deals with one agent (the decision maker), taking actions against her environment, which is assumed to be stationary. Realistic settings have to take into account the presence of other rational agents. These can be potential collaborators to cooperate, but also could be adversaries. For instance, in the case study from Section \ref{sec:dlms} we could have modeled competitors also taking decisions that in a way, also affect the expected sales of the supported franchise. Under this dynamic, non-stationary environment, traditional RL techniques dramatically fail, and it is a necessity to develop novel frameworks that acknowledge the presence of other players. This is the field of multi-agent Reinforcement Learning (MARL), which usually takes grounding in game-theoretical approaches \parencite{marl_over,lanctot2017unified}. 

Game Theory has also its own drawbacks, whereas Adversarial Risk Analysis (ARA) offers a more realistic view, leveraging Bayesian ideas \parencite{Banks}. The challenge is thus how to adapt the ARA methodology into the sequential learning nature of RL, with low computational requirements in order to allow for scalability.
Chapter 4 deals with some of these issues.

\section{Thesis structure}

The three challenges explained above are the content of the following chapters. Chapter 2 describes two novel approaches for large scale Bayesian inference in complex models. Chapter 3 surveys the state of adversarial classification and presents an original approach to robustifying classifiers inspired by Adversarial Risk Analysis. Chapter 4 takes a deep dive into reinforcement learning, proposing a framework to support a decision maker against adversaries and then studying a realistic application in data sharing markets.

Finally, Chapter 5 ends up with several conclusions and avenues for further work. Figure \ref{fig:thesis} shows how this thesis is structured.

\begin{figure}[h]
\centering
\includegraphics[scale=0.8]{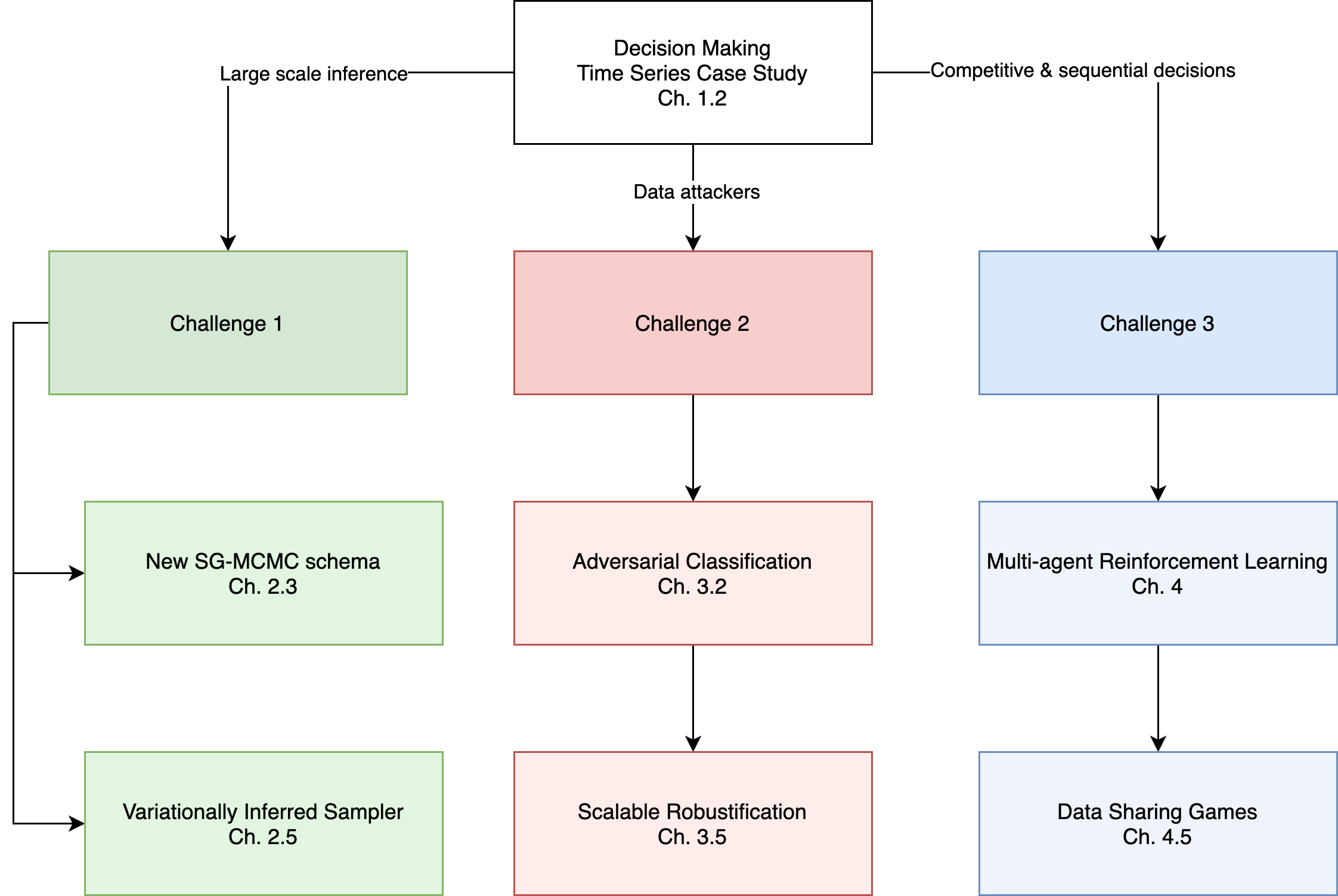}
\caption{The structure of this thesis}\label{fig:thesis}
\end{figure}

\begin{subappendices}

\section{MCMC Convergence Diagnostics}\label{app:GR}

In order to asses the convergence of the MCMC scheme described in Section \ref{s:MCMC}, we used the Gelman-Rubin convergence statistic $\hat{R}$. We report its value for each latent dimension of our best performing model, the RF variant, in Table \ref{tab:GR}. All values are under 1.1, confirming correct convergence. In addition, we display the trace plots for each variable in Figure \ref{fig:traces}.

\begin{table}[H]
\centering
{\footnotesize 
\begin{tabular}{|l|c|c|c|c|c|c|c|}
\hline
coefficient & \texttt{y\_AR} & \texttt{OOH} & \texttt{TV} & \texttt{ONLINE} & \texttt{SEARCH} & \texttt{RADIO} & \texttt{Hols} \\
\hline 
$\hat{R}$ & 0.9997814 & 0.9997621 & 1.00001 & 0.9998523 & 1.000033 & 1.000522 & 0.9999129 \\
\hline
coefficient & \texttt{AVG\_Temp} & \texttt{AVG\_Rain} & \texttt{Unemp\_Ix} & \texttt{CC\_IX} &  \texttt{Price\_IX}& \texttt{Sport\_EV}  & \\
\hline
$\hat{R}$ & 0.999768 & 0.9997925 & 1.030499 & 1.039316 & 1.005669 & 1.000532 & \\
\hline
\end{tabular}\caption{Gelman-Rubin statistic results}\label{tab:GR}
}
\end{table}

\begin{figure}[h]
\centering
\includegraphics[scale=0.6]{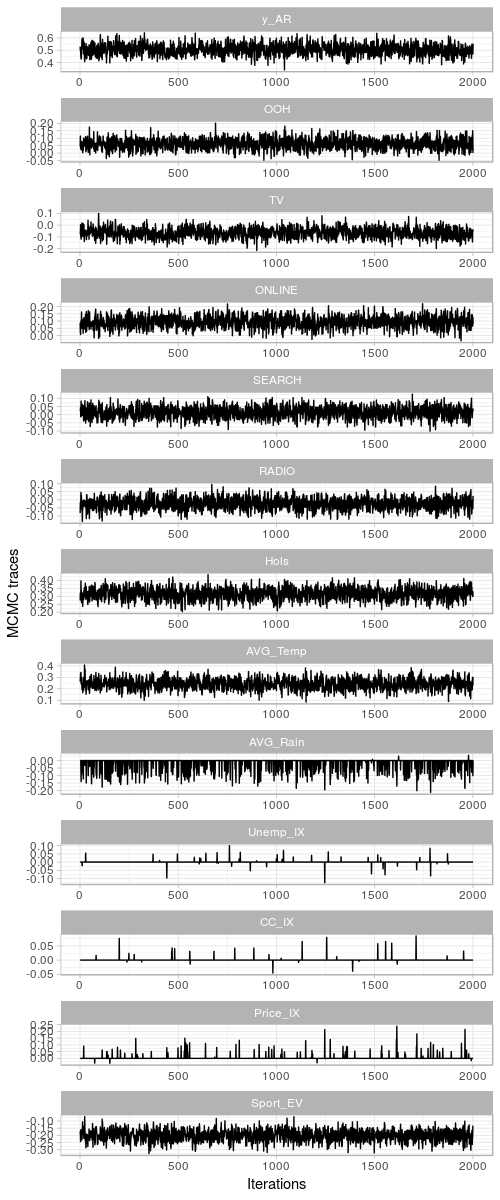}
\caption{MCMC traces after a burn-in of 2000 iterations.}\label{fig:traces}
\end{figure}

\end{subappendices}

\chapter{Large Scale Bayesian Inference}\label{cha:bayes}
\section{Introduction}

This chapter presents two developments at the intersection of approximate Bayesian inference and complex machine learning models, such as neural networks. The first one proposes a unifying perspective of two different Bayesian inference algorithms, Stochastic Gradient Markov Chain Monte Carlo (SG-MCMC) and Stein Variational Gradient Descent (SVGD), leading to improved and efficient novel sampling schemes. We also show that SVGD combined with a noise term can be framed as a multiple chain SG-MCMC method. Sections \ref{sec:back} to \ref{sec:experiments_sgmcmc} are devoted to this study.

On the other hand, the second part of this chapter introduces a complementary framework to boost the efficiency of Bayesian inference in probabilistic models  by embedding a Markov chain sampler within a variational posterior approximation. We call this framework “variationally inferred sampling” (VIS). Its strengths are its ease of implementation and the automatic tuning of sampler parameters, leading to a faster mixing time through automatic differentiation.
Sections \ref{sec:main} to \ref{sec:exps_vis} tackle this problem.

We start by giving a brief introduction to both approaches. Then, a discussion follows with the details of each of the two. 

\subsection{SG-MCMC with repulsive forces: motivation }

Bayesian computation lies at the heart of many machine learning models in both academia and industry \parencite{bishop2006pattern}. Thus, it is of major importance to develop efficient approximation techniques that tackle the intractable integrals that arise in large scale Bayesian inference and prediction problems \parencite{gelman2013bayesian}, as stated earlier in challenge 1. 

Recent developments in this area include variational based approaches such as \emph{Automatic Differentiation Variational Inference} (ADVI) \parencite{blei2017variational}, and \emph{Stein Variational Gradient Descent} (SVGD)  \parencite{liu2016stein} or sampling approaches such as \emph{Stochastic Gradient Markov Chain Monte Carlo} (SG-MCMC) \parencite{ma2015complete}. While variational techniques enjoy faster computations, they rely on optimizing a family of posterior approximates that may not contain the actual posterior distribution, potentially leading to severe bias and underestimation of uncertainty, \textcite{pmlr-v80-yao18a} or \textcite{riquelme2018failure}. SG-MCMC methods have been used in a wide range of real-world tasks, such as in computer vision settings \parencite{7780980} or in recommendation systems via matrix factorization \parencite{7952555}.

There has been recent interest in bridging the gap between variational Bayes and MCMC techniques, see e.g. \parencite{zhang2018}, to develop new scalable approaches for Bayesian inference, e.g., \parencite{carbonetto2012}.
In sections 2.2-2.4, we draw on a similitude between the SG-MCMC and SVGD approaches to propose a novel family of very efficient sampling algorithms. 

Any competing MCMC approach should verify the following list of properties, as our proposal will do:
\begin{itemize}
    \item \emph{scalability}. For this, we resort to SG-MCMC methods since at each iteration they may be approximated to just require a minibatch of the dataset,
    \item \emph{convergence to the actual posterior}, and
    \item \emph{flexibility}. Since we provide a parametric formulation of the transition kernel, it is possible to adapt other methods such as \emph{Hamiltonian Monte Carlo} \parencite{neal2011mcmc} or the Nos\'e-Hoover thermostat method \parencite{ding2014bayesian}.
\end{itemize}
There has been recent interest in developing new dynamics for SG-MCMC samplers with the aim of exploring the target distribution more efficiently. \textcite{chen2014stochastic} proposed a stochastic gradient version of HMC, whereas \textcite{ding2014bayesian} did the same, leveraging for the Nos\'e-Hoover thermostat dynamics. \textcite{chen2016bridging} adapted ideas from stochastic gradient optimization by proposing the analogue sampler to the Adam optimizer. A relativistic variant of Hamiltonian dynamics was introduced by \textcite{abbati2018adageo}. Two of the most recent derivations of SG-MCMC samplers are \textcite{zhang2019cyclical}, in which the authors propose a policy controlling the learning rate which serves for better preconditioning; and \textcite{gong2019meta}, in which a meta-learning approach is proposed to learn an efficient SG-MCMC transition kernel.
While the previous works propose new kernels which can empirically work well, they focus on the case of a single chain. We instead consider the case of several chains in parallel, and the bulk of our contribution focuses on how to develop transition kernels which exploit interactions between parallel chains. Thus, our framework can be seen as an orthogonal development to the previous listed approaches (and could be actually combined with them).

On the theoretical side, \textcite{chen2018unified} study another connection between MCMC and deterministic flows. In particular, they explore the correspondence between Langevin dynamics and Wasserstein gradient flows.
We instead formulate the dynamics of SVGD as a particular kind of MCMC dynamics, which enables the use of the Fokker-Planck equation to show in a straightforward manner how our samplers are valid. \textcite{liu2017stein} started to consider similitudes between SG-MCMC and SVGD, though in this work we propose the first hybrid scheme between both methods.

Our contributions are summarized as follows. First, we provide a unifying hybrid scheme of SG-MCMC and SVGD algorithms satisfying the previous list of requirements; second, and based on the previous connection, 
we develop new SG-MCMC schema that include repulsive forces between particles and momentum acceleration. Usually, when using multiple parallel chains in SG-MCMC, the chains are treated independently. We instead introduce an interaction between particles via a repulsive force, in order to make the particles do not collapse into the same point of the posterior.
Finally, we show how momentum-based extensions of SGD can be incorporated into our framework, leading to new samplers which benefit from both momentum acceleration and repulsion between particles for better exploration of complex posterior landscapes, therefore providing a very competitive scheme. 

Figure \ref{fig:diagram} depicts our proposal. Starting with SGD, one can add a carefully crafted noise term to arrive at the simplest SG-MCMC sampler, SGLD. Then, one can add repulsion between particles to speed-up the mixing time, leading to SGLD+R, our first contribution (Section \ref{sec:framework}). On top of that, we can further augment the latent space with momentum variables, leading to the Adam+NR sampler, our second proposed sampler (\ref{sec:momentum}).

After an overview of posterior approximation methods in Section \ref{sec:back}, we propose our framework in Section \ref{sec:framework}.
Section \ref{sec:experiments_sgmcmc} discusses relevant experiments showcasing the benefits of our proposal. 

\begin{figure}[!h]
    \centering
\includegraphics[width=0.4\textwidth]{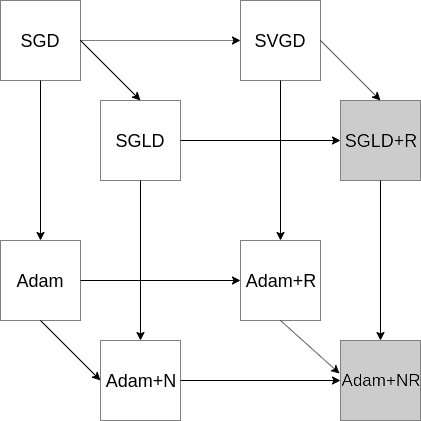}
    \caption{Relationships between samplers. In light gray, our proposed samplers.}\label{fig:diagram}
\end{figure}

\subsection{A motivation for variationally inferred samplers (VIS)}
Bayesian inference and prediction in large, complex models, such as in 
deep neural networks or stochastic processes, remains an elusive problem \parencite{blei2017variational,10.1214/17-BA1082,insua2012bayesian,alquier2020approximate}.
%
%
%
Variational approximations (e.g., automatic differentiation variational inference (ADVI) \parencite{kucukelbir2017automatic}) tend to be biased and 
underestimate uncertainty \parencite{riquelme2018failure}. 
On the other hand, depending on the target distribution,
Markov Chain Monte Carlo (MCMC) \parencite{andrieu2010particle} 
methods, such 
as Hamiltonian Monte Carlo (HMC) \parencite{neal2011mcmc}),   
tend to be exceedingly slow  \parencite{van2018simple} {in large scale settings with large amounts of data points and/or parameters}. For this reason, in recent years, there has been increasing interest in developing more efficient posterior approximations \parencite{nalisnick2016approximate,salimans2015markov,tran2015variational} and inference techniques that aim to be as general and flexible as possible {so that they can be easily used with any probabilistic model} \parencite{wood2014new, ge2018t}.

It is well known that the performance of a sampling method depends
heavily on the parameterization used \parencite{papaspiliopoulos2007general}. This work proposes
a framework to automatically tune the parameters of a 
MCMC sampler {with the aim of adapting the shape of the posterior, thus~boosting  Bayesian inference efficiency. 
 We deal with a case in which the latent variables or parameters are continuous. Our framework can also be regarded as a principled way to enhance the flexibility of  variational posterior approximation in search of an optimally tuned MCMC sampler;} thus the proposed name of our framework is {the variationally inferred sampler} (VIS).

{

The idea of preconditioning the posterior distribution to speed up the mixing time of a MCMC sampler has been explored recently in \parencite{hoffmanneutra,PhysRevLett.121.260601}, where a parameterization was learned before sampling via HMC. Both papers extend seminal work in \parencite{parno2014transport} by learning an efficient and expressive deep, non-linear transformation instead of a polynomial regression. However, they do not account for tuning the parameters of the sampler, as introduced in Section \ref{sec:main}, where a fully, end-to-end differentiable sampling scheme is proposed.

The work presented in \parencite{rezende2015variational} introduced a general
framework for constructing more flexible variational distributions, called normalizing flows. These transformations are one of the main techniques used to improve the flexibility of current variational inference (VI) approaches and have recently pervaded the approximate Bayesian inference
literature with developments such as continuous-time normalizing flows \parencite{chen2018continuoustime} (which extend an initial simple variational posterior with a discretization of Langevin dynamics) or Householder flows for mixtures of Gaussian distributions \parencite{LIU201943}. However, they require a generative adversarial network (GAN) \parencite{goodfellow2014generative} to learn the posterior,
which can be unstable in high-dimensional spaces. We overcome this problem with our novel formulation; moreover, our framework is also compatible with different optimizers, rather than only those derived from Langevin dynamics \parencite{mandt2017stochastic}. Other recent proposals create more flexible variational posteriors based on implicit approaches typically requiring a GAN, as presented in \parencite{huszar2017variational} and including unbiased implicit variational inference (UIVI)
\parencite{pmlr-v89-titsias19a} or semi-implicit variational inference (SIVI) \parencite{yin2018semi}. Our variational approximation is also implicit but uses a sampling algorithm to drive the evolution of the density, combined with a Dirac delta approximation to derive an efficient variational approximation, as reported through extensive experiments in Section \ref{sec:exps_vis}.

Closely related to our framework is the work presented in \parencite{hoffman2017learning}, where a variational autoencoder (VAE) is learned using HMC. We~use a similar compound distribution as the variational approximation, yet 
our approach allows  any stochastic gradient MCMC to be embedded, 
 as well as facilitating the tuning of sampler parameters via gradient descent.
Our work also relates to the recent idea of 
sampler amortization \parencite{feng2017learning}. 
A common problem with these approaches is that they incur in an additional error---the amortization gap \parencite{cremer2018inference}---which we alleviate by evolving a set of particles through a stochastic process in the latent space after learning a good initial distribution,
meaning that 
the initial approximation bias can be significantly reduced.
A recent related article was presented in~\parencite{pmlr-v97-ruiz19a},
which also defined a compound distribution. 
However, our focus is on efficient approximation using the reverse KL 
divergence, 
which allows sampler parameters to be tuned and achieves superior results. Apart from optimizing this kind of divergence, 
the~main point is that we can compute the gradients of sampler parameters 
(Section \ref{sec:tuning}), whereas in~\parencite{pmlr-v97-ruiz19a} the authors only consider a parameterless sampler: 
thus, our framework allows for greater flexibility, helping the user
to tune sampler hyperparameters.
In the Coupled Variational Bayes (CVB)~\parencite{dai2018coupled} approach,
optimization is in the dual space, whereas
we optimize the standard
evidence lower bound (ELBO). Note that even if the optimization was exact, the solutions would coincide, and it is not clear yet what happens in the truncated optimization case,
other than performing empirical experiments on given datasets. We thus feel that there is room for implicit methods that perform optimization in the primal space 
(besides this, they are easier to implement). Moreover,
the previous dual optimization approach requires the use of an additional neural network (see the paper on the Coupled Variational Bayes (CVB) approach
 or \parencite{fang2019implicit}). This adds a large number of parameters and requires another architecture decision. With VIS, we do not need to introduce an auxiliary network, since we perform a ``non-parametric'' approach by back-propagating instead through 
several iterations of SGLD. 
Moreover, the lack of an auxiliary network simplifies the design choices.
}


{ Thus, our contributions include a flexible and consistent variational approximation to the posterior,
embedding an initial variational approximation within a stochastic process;} 
an analysis of its key properties; 
    the provision of several strategies for ELBO
    optimization using the previous 
    approximation; and finally, 
    an illustration of its power through relevant complex examples.


\section{Background}\label{sec:back}

Consider a probabilistic model $p({x}| {z})$ and a prior distribution $p({z})$ where ${x}$ denotes an observation and ${z} = ( z_1, \ldots, z_d) $ an unobserved $d-$dimensional latent variable or parameter, depending on the context. We are interested in performing inference regarding the unobserved variable ${z}$, by approximating its posterior distribution
$$
p({z} | {x}) = \frac{ p({z})p({x}| {z}) }{ \int p({z})p({x}| {z}) d{z} } = \frac{ p({z})p({x}| {z}) }{ p({x}) } = \frac{p({z},{x})}{p({x})}.
$$
Except for reduced classes of distributions like \textit{conjugate priors}, \textcite{raiffa1961applied}, the integral $p({x}) = \int p({z})p({x}| {z}) d{z}$ is analytically intractable; no general explicit expressions of the posterior are available. Thus, several techniques have been proposed to perform approximate posterior inference.

\subsection{Inference as sampling}\label{sec:infassamp}

Hamiltonian Monte Carlo (HMC, \textcite{neal2011mcmc}) is an effective sampling method for models whose probability is point-wise computable and differentiable. 
HMC requires the exact simulation of a certain dynamical system which can be cumbersome in high-dimensional or large data settings. When this is an issue, \textcite{welling2011bayesian} proposed a formulation of a continuous-time Markov process that converges to a target distribution $p({z} | {x})$. It is based on the Euler-Maruyama discretization of Langevin dynamics
\begin{eqnarray}\label{eq:sgld}
{z}_{t+1} \leftarrow {z}_{t} + \epsilon_t \nabla \log p({z}_t,{x})  + \mathcal{N}({0}, 2\epsilon_t I),
\end{eqnarray}
where $\epsilon_t$ is the step size. The previous iteration uses the gradient evaluated at one data point ${x}$, but we can use the full dataset or mini-batches. Several extensions of the original Langevin sampler have been proposed to increase mixing speed, see for instance \parencite{li2016preconditioned,li2016high,li2019communication}.

\textcite{ma2015complete} proposed a general formulation of a continuous-time Markov process that converges to a target distribution $\pi({z}) \propto \exp (-H({z}))$. It is based on the Euler-Maruyama discretization of the generalized Langevin dynamics:
\begin{align}\label{eq:sgmcmc}
\begin{split}
&{z}_{t+1} \leftarrow {z}_{t}  -\epsilon_t \left[ ({D}({z}_t) + {Q}({z}_t)) \nabla H({z}_t) + {\Gamma}({z}_t) \right] + \mathcal{N}({0}, 2\epsilon_t {D}({z}_t)),
\end{split}
\end{align}
with ${D}({z})$ being a diffusion matrix; ${Q}({z})$, a curl matrix; and ${\Gamma}({z})_i = \sum_{j=1}^d \frac{\partial}{\partial {z}_j} ({D}_{ij}({z}) + {Q}_{ij}({z})) $ is a correction term which amends the bias.

To obtain a valid SG-MCMC algorithm, we simply have to choose the dimensionality of ${z}$ (e.g., if we augment the space with auxiliary variables as in HMC), and the matrices ${D}$ and ${Q}$. For instance, the popular Stochastic Gradient Langevin Dynamics (SGLD) algorithm, the first SG-MCMC scheme in (\ref{eq:sgld}), is obtained when ${D} = {I}$ and ${Q} = {0}$. In addition, the Hamiltonian variant can be recovered if we augment the state space with a $d-$dimensional momentum term ${m}$, leading to an augmented latent space $ \bar{{z}} = ({z}, {m})$. Then, we set ${D} = {0}$ and ${Q} = \begin{pmatrix}
{0} & -{I} \\
{I} & {0}
\end{pmatrix}$.

\subsection{Inference as optimization}\label{sec:iasopt}


Variational inference \parencite{kucukelbir2017automatic} tackles the problem of approximating the posterior $p({z} | {x})$ with a tractable parameterized distribution $q_{{\lambda}}({z}|{x})$. The goal is to find parameters ${\lambda}$ so that the variational distribution $q_{{\lambda}}({z}|{x})$ (also referred to as the variational guide or variational approximation) is as close as possible to the actual posterior. Closeness is typically measured through the
Kullback-Leibler 
divergence $KL(q || p)$, which is reformulated into the \emph{evidence lower bound} (ELBO) 
\begin{equation}\label{eq:elbo_bayes}
\mbox{ELBO}(q) = \mathbb{E}_{q_{{\lambda}}({z}|{x})} \left[ \log p({x},{z}) - \log q_{{\lambda}}({z}|{x})\right].
\end{equation}
To allow 
for greater flexibility, typically a deep non-linear model conditioned on observation ${x}$ defines the mean $\mu_{{\lambda}}({x})$ and covariance matrix $\sigma_{{\lambda}}({x})$ of a
Gaussian distribution $q_{{\lambda}}({z}|{x}) \sim \mathcal{N}(\mu_{{\lambda}}({x}), \sigma_{{\lambda}}({x}))$.
Inference is then performed using a gradient-based maximization routine, leading to variational parameter updates
\begin{align*}
{\lambda}_{t+1} = {\lambda}_t +  \epsilon_t \nabla_{{\lambda}} \mbox{ELBO}(q),
\end{align*}
where $\epsilon_t$ is the learning rate. Stochastic gradient descent (SGD) \parencite{hoffman2013stochastic}, or some variant of it, such as Adam \parencite{kingma2014adam}, are used as optimization algorithms.

On the other hand, SVGD \parencite{liu2016stein}, frames posterior sampling as an optimization process, in which a set of $L$ particles $\lbrace {z}_i \rbrace_{i=1}^L$ is evolved iteratively via a velocity field ${z}_{i, t+1} \leftarrow {z}_{i, t} + \epsilon_t \phi({z}_{i,t})$, where $\phi : \mathbb{R}^d \rightarrow \mathbb{R}^d$ is a smooth function characterizing the perturbation of the latent space.
Let $q$ be the particle distribution at iteration $t$ and $q_{\left[ \epsilon \phi \right]}$ the distribution after update ($t+1$). Then, the optimal choice of the velocity field $\phi$ can be framed through the optimization problem $\phi^* = \arg\max_{\phi \in \mathcal{F}} \lbrace -\frac{d}{d\epsilon} \mbox{KL}( q_{\left[ \epsilon \phi \right]} \| p) \rbrace$, with $\phi$ chosen to maximize the decreasing rate on the Kullback-Leibler (KL) divergence between the particle distribution and the target, and $\mathcal{F}$, some proper function space. When $\mathcal{F}$ is a \emph{reproducing kernel Hilbert space} (RKHS), \textcite{liu2016stein} showed that the optimal velocity field leads to
\begin{align}\label{eq:svgd}
\begin{split}
&{z}_{i, t+1} \leftarrow {z}_{i, t} - \epsilon_t \frac{1}{L}  \sum_{j=1}^L \left[ k({z}_{j, t}, {z}_{i, t})\nabla H({z}_{j,t}) + \nabla_{{z}_{j, t}}  k({z}_{j, t}, {z}_{i, t})  \right],
\end{split}
\end{align}
where the RBF kernel $k({z}, {z}') = \exp (-\frac{1}{h}  \| {z} - {z}' \|^2 )$ is typically adopted. Observe that in (\ref{eq:svgd}) the gradient term $\nabla_{{z}_{j, t}}  k({z}_{j, t}, {z}_{i, t})$ acts as a repulsive force that prevents particles from collapsing. 



\subsection{The Fokker-Planck equation}\label{sec:fp}
Consider a stochastic differential equation (SDE) of the form $d{z} = \mu_t({z})dt + \sqrt{2 D_t({z})}dB_t$. The distribution $q_t({z})$ of a population of particles evolving according to the previous SDE from some initial distribution $q_0({z})$ is governed by the Fokker-Planck partial differential equation (PDE) \parencite{risken-fpe-1989},
$$
\frac{\partial}{\partial t} q_t({z}) = -\frac{\partial}{\partial {z}} \left[ \mu_t({z}) q_t({z})\right] + \frac{\partial^2}{\partial {z}^2} \left[ D_t({z})q_t({z})\right].
$$
Deriving the Fokker-Planck equation from the SDE of a potential SG-MCMC sampler is of great interest since we can check whether the target distribution is a stationary solution of the PDE (and thus, the sampler is consistent); and to compare if two a priori different SG-MCMC samplers result in the same trajectories.

\section{Stochastic Gradient MCMC with Repulsive Forces}\label{sec:framework}

We use the framework from \textcite{ma2015complete} in an augmented state space ${z} = \left( {z}_{1}, {z}_{2}, \ldots,{z}_{L}\right) \in \mathbb{R}^{Ld}$ to obtain a valid posterior sampler that runs multiple ($L$) Markov chains with interactions. This version of SG-MCMC is given by the equation
\begin{equation}\label{eq:general}
{z}_{t+1} \leftarrow {z}_t -\epsilon_t \left[ ({{D_K}} + {Q_K}){\nabla} + {\Gamma_K} \right] + {\eta}_t,
\end{equation}
with ${\eta}_t \sim \mathcal{N}({0}, 2\epsilon_t {{D_K}})$.
Now, ${z}_t = \left({z}_{1,t} \ldots  {z}_{L,t} \right)^\top$ is an $Ld$-dimensional vector defined by the concatenation of $L$ particles; ${\nabla} \in \mathbb{R}^{L \times d \times 1}$ so that $({\nabla})_{i,:} = \nabla H({z}_{i,t})$\footnote{Though ${\nabla} \in \mathbb{R}^{L d \times 1}$ to allow multiplication by ${D_K} + {Q_K}$, we reshape it as ${\nabla} \in \mathbb{R}^{L\times d \times 1}$ to better illustrate how it is defined.}; ${D_K} \in \mathbb{R}^{Ld\times Ld}$ is an expansion of the diffusion matrix $D$, accounting for the distance between the particles; ${Q_K} \in \mathbb{R}^{Ld\times Ld}$ is the curl matrix, which is skew-symmetric and might be used if a Hamiltonian variant is adopted; and ${\Gamma_K}$ is the correction term from the framework of \textcite{ma2015complete}. Note that ${D_K}$, ${Q_K}$ and ${\Gamma_K}$ can depend on the state ${z}_t$, but we do not make it explicit to simplify notation.

In matrix form, the update rule (\ref{eq:svgd}) for SVGD can be expressed as
\begin{equation}\label{eq:svgd_mat}
\overline{{z}}_{t+1} \leftarrow \overline{{z}}_t -\frac{\epsilon_t}{L}\left( \overline{{K}} \overline{{\nabla}} + \overline{{\Gamma}} \right)
\end{equation}
where $\overline{{K}} \in \mathbb{R}^{L \times L}$ so that $(\overline{{K}})_{ij} = k({z}_i, {z}_j)$, $\overline{{\nabla}} \in \mathbb{R}^{L\times d}$ and $\overline{{z}}_t \in \mathbb{R}^{L\times d}$. Casting the later matrix as a tensor ${\nabla} \in \mathbb{R}^{L \times d \times 1}$ and the former one as a tensor ${K} \in \mathbb{R}^{(L \times d) \times (L \times d)}$ by broadcasting along the second and fourth axes, we may associate ${K}$ with the SG-MCMC's diffusion matrix ${D}$ over an $Ld-$dimensional space. 

The big matrix ${D_K}$ in Eq. (\ref{eq:general}) is defined as a permuted block-diagonal matrix consisting of $d$ repeated kernel matrices $\overline{{K}}$:
$$
{D_K} = \left[
\begin{array}{cccc}
\overline{{K}} &  &  &  \\
 & \overline{{K}} &  &  \\
 &  & \ddots &   \\
 &  &  & \overline{{K}}  \\
\end{array}
\right] {P},
$$
with ${P}$ being the $Ld \times Ld$ permutation matrix
$$
{P} = \left[
\begin{array}{c|c|c|c}
\begin{matrix}
1 & & &\\
 & & &\\
 & & &\\
  & & &\\
\end{matrix} &\begin{matrix}
 & & & \\
1 & & &\\
 & & & \\
  & & & \\
\end{matrix}  & \ddots & \begin{matrix}
 & & &\\
 & & &\\
 & & &\\
1 & & &\\
\end{matrix} \\
\hline
\begin{matrix}
 &1 & &\\
 & & &\\
 & & &\\
  & & &\\
\end{matrix} &\begin{matrix}
 & & & \\
 & 1& &\\
 & & & \\
  & & & \\
\end{matrix}  & \ddots & \begin{matrix}
 & & &\\
 & & &\\
 & & &\\
 &1 & &\\
\end{matrix} \\
\hline
\ddots &\ddots & \ddots & \ddots \\
\hline
\begin{matrix}
 & & &1\\
 & & &\\
 & & &\\
  & & &\\
\end{matrix} &\begin{matrix}
 & & & \\
 & & &1\\
 & & & \\
  & & & \\
\end{matrix}  & \ddots & \begin{matrix}
 & & &\\
 & & &\\
 & & &\\
& & &1\\
\end{matrix} \\
\end{array}
\right].
$$
The permutation matrix ${P}$ rearranges the block-diagonal kernel matrix to match with the dimension ordering of the state space ${z}_t = \left({z}_{1,t} \ldots  {z}_{L,t} \right)^\top$.
With this convention, ${D_K} {\nabla}$ is equivalent to $\overline{{K}} \overline{{\nabla}}$, only differing in the shape of the resulting matrix. This allows us to frame SVGD plus the noise term as a valid scheme within the SG-MCMC framework of \parencite{ma2015complete}, using the $Ld$-dimensional augmented state space.



From this perspective, \eqref{eq:svgd_mat} can be seen as a special case of \eqref{eq:general} with curl matrix ${Q_K} = {0}$ and no noise term. We refer to this perturbed variant of SVGD as \emph{Parallel SGLD plus repulsion} (SGLD+R):
\begin{equation}\label{eq:psvgd_mat}
{z}_{t+1} \leftarrow {z}_t -\frac{\epsilon_t}{L}\left( {D_K} {\nabla} + {\Gamma_K} \right) + {\eta}_t, \qquad {\eta}_t \sim \mathcal{N}({0}, 2\epsilon_t {D_K}/L).
\end{equation}
Since ${D_K}$ is a definite positive matrix (constructed from the RBF kernel), we may use Theorem 1 from \textcite{ma2015complete} to derive the following:

\begin{proposition}
SGLD+R (or its general form, Eq. \eqref{eq:general}) has $\pi({z}) = \prod_{l=1}^L \pi({z}_l)$ as stationary distribution, and the proposed discretization is asymptotically exact as $\epsilon_t \rightarrow 0$.
\end{proposition}
\noindent Having shown that SVGD plus a noise term can be framed as an SG-MCMC method, we now propose a particular sampler. 
Algorithm \ref{alg:alg1} shows how to set it up. The step sizes $\epsilon_t$ decrease to $0$ using the \textcite{robbins} conditions, given by $\sum_{t=1}^\infty \epsilon_t = \infty, \sum_{t=1}^\infty \epsilon_t^2 < \infty$. However, in practical situations we can consider a small and constant step size, see Section \ref{sec:experiments_sgmcmc}. 

\begin{algorithm}[!h] %
\caption{Bayesian Inference via SGLD+R}  
\label{alg:alg1}
\begin{algorithmic}[1]
\State {\bf Input:} A target distribution with density function $\pi({z}) \propto \exp (-H({z}))$; a prior distribution $p({z})$. 
\State {\bf Output:} A set of particles $\{{z}_i\}_{i=1}^{ML}$ that approximates the target distribution.  
\State Sample initial set of particles from prior: ${z}_1^0, {z}_2^0, \ldots {z}_L^0 \sim p({z})$.
\For{each iteration $t$}
\State 
\begin{align} \label{eq:psvgd_alg}
\begin{split}
&{z}_i^{t+1}  \gets  {z}_i^t  - \epsilon_t \frac{1}{L}\sum_{l=1}^L\big[  k({z}_l^t, {z}_i^t)  \nabla_{{z}_l^t} H({z}_l^t) + \nabla_{{z}_l^t} k({z}_l^t, {z}_i^t)\big] + {\eta}_i^t
\end{split}
\end{align}
\State where ${\eta}_i^t$ is the noise for particle $i$ defined as in Eq (\ref{eq:psvgd_mat}).
\State After  burn-in period, start collecting particles: $ \{{z}_i\}_{i=1}^{NL} \gets \{{z}_i\}_{i=1}^{(N-1)L} \cup \{ {z}_1^{t+1}, \ldots,  {z}_L^{t+1} \} $ 
\EndFor
\end{algorithmic}
\end{algorithm}

\paragraph{Complexity}  Our proposed method is amenable to sub-sampling, as the mini-batch setting from SG-MCMC can be adopted: the main computational bottleneck lies in the evaluation of the gradient $\nabla_{{z}} H(z)$, which can be troublesome in a big data setting as $- H(z) = \log p(z) + \sum_{i=1}^N \log p({x}_i | {z}) + \mbox{constant terms wrt} z$.
We may then  approximate the true gradient with an unbiased estimator taken along a minibatch of datapoints $\Omega \subset \lbrace 1, 2, \ldots, N \rbrace$ in the usual way
$$
-\nabla_{{z}} H(z) \approx \nabla_{{z}} \log p({z}) + \frac{N}{| \Omega |} \sum_{i \in \Omega} \nabla_{{z}} \log p({x}_i | {z}).
$$
\noindent As with the original SVGD algorithm, the complexity of the update rule (\ref{eq:svgd_mat}) is $\mathcal{O}(L^2)$, with $L$ being the number of particles, since we need to evaluate kernels of signature $k(z_i, z_j)$. Using current state-of-the-art automatic differentiation frameworks, such as \texttt{jax}, \textcite{jax2018github}, we can straightforwardly compile kernels using \emph{just-in-time} compilation, \textcite{frostig2018compiling}, at the cost of a negligible overhead compared to parallel SGLD for moderate values of $L \sim 50$ particles.

If many more particles are  used, one could approximate the expectation in (\ref{eq:svgd_mat}) using subsampling at each iteration, as proposed by the authors of SVGD, or by using more sophisticated approaches from the molecular dynamics literature, such as the \textcite{barnes1986hierarchical} algorithm, to arrive at an efficient $\mathcal{O}(L \log L)$ computational burden at a negligible approximation error. 

\subsection{Relationship with SVGD}\label{sec:relationship}

We study in detail the behaviour of SVGD and SGLD+R. To do so, we first derive the Fokker-Planck equation for the SGLD+R sampler.
\begin{proposition}
The distribution $q_t({z})$ of a population of particles evolving according to SGLD+R is governed by the PDE
$$
\frac{\partial}{\partial t} q_t({z}) = -\frac{\partial}{\partial {z}} \left[ (D_K \nabla \log \pi({z}) + \Gamma_K) q_t({z})\right] + \frac{\partial^2}{\partial {z}^2} \left[ D_K q_t({z})  \right].
$$
The target distribution $\pi({z})$ is a stationary solution of the previous PDE.
\end{proposition}
\begin{proof}
The first part is a straightforward application of the Fokker-Planck equation from Section \ref{sec:fp}. For the last part, we need to show that
\begin{align*}
\frac{\partial}{\partial t} q_t({z}) = 0 = &-\frac{\partial}{\partial {z}} \left[ ({D_K} \nabla \log \pi({z}) + {\Gamma_K}) \pi({z})\right] +  \frac{\partial^2}{\partial {z}^2} \left[ {D_K} \pi({z})  \right].
\end{align*}
To see so, we expand each term in the rhs:
\begin{align*}
& \frac{\partial}{\partial {z}} \left[ ({D_K} \nabla \log \pi({z}) + {\Gamma_K}) \pi({z})\right] = \\
&= \pi({z})\frac{\partial}{\partial {z}} \left[ {D_K} \nabla \log \pi({z}) + {\Gamma_K} \right] +  \left[ {D_K} \nabla \log \pi({z}) + {\Gamma_K} \right] \frac{\partial}{\partial {z}} \pi({z}) = \\
&= \pi({z})\frac{\partial}{\partial {z}}{D_K} \nabla \log \pi({z})+ \pi({z}){D_K} \nabla^2 \log \pi({z}) + \pi({z}) \frac{\partial}{\partial {z}} {\Gamma_K} + {D_K} \nabla \log \pi({z}) \nabla \pi({z})  + {\Gamma_K} \nabla \pi({z}) = \\
&= {\Gamma_K} \nabla \pi({z}) + \pi({z})\frac{\partial}{\partial {z}}{D_K}\nabla \log \pi({z}) +  
\pi({z}) \frac{\partial}{\partial {z}} {\Gamma_K} +
{D_K}(\pi({z})\nabla^2 \log \pi({z}) + \nabla \log \pi({z}) \nabla \pi({z})).
\end{align*}
The other term expands to
\begin{align*}
&\frac{\partial^2}{\partial {z}^2} \left[ {D_K} \pi({z})  \right] = \\
&= \frac{\partial}{\partial {z}} \left[ \frac{\partial}{\partial {z}} {D_K} \pi({z}) + {D_K} \nabla \pi({z}) \right] = \\
&= \frac{\partial^2}{\partial {z}^2}{D_K} \pi({z}) + \pi({z}) \nabla \log \pi({z}) \frac{\partial}{\partial {z}}{D_K} + \frac{\partial}{\partial {z}}  {D_K} \nabla \pi({z}) + {D_K} \nabla^2 \pi({z}) = \\
&= \frac{\partial}{\partial {z}}  {D_K} \nabla \pi({z}) +
\pi({z}) \nabla \log \pi({z}) \frac{\partial}{\partial {z}}{D_K} + \frac{\partial^2}{\partial {z}^2}{D_K} \pi({z})+ 
{D_K}(\pi({z})\nabla^2 \log \pi({z}) + \nabla \log \pi({z}) \nabla \pi({z})).
\end{align*}
Taking into account that by the definition of the correction term, $\Gamma_K = \frac{\partial}{\partial {z}} D_K$, the previous expansions are equal so they cancel each other in the rhs of the PDE.
\end{proof}
This last result is an alternative proof of our Proposition 1, without having to resort to the framework of \textcite{ma2015complete} as was the case there. It is of independent interest for us here, since we can establish a complementary result for the case of SVGD as follows.
\begin{proposition}
The distribution $q_t({z})$ of a population of particles evolving according to SVGD is governed by
$$
\frac{\partial}{\partial t} q_t({z}) = -\frac{\partial}{\partial {z}} \left[ (D_K \nabla \log \pi({z}) + \Gamma_K) q_t({z})\right] .
$$
In general, the target distribution $\pi({z})$ is not a stationary solution of the previous PDE in general.
\end{proposition}
\begin{proof}
As before, the first part is a straightforward application of the Fokker-Planck equation from Section \ref{sec:fp}. For the last part, note that the difference with Proposition 2 is that the term $\frac{\partial^2}{\partial {z}^2} \left[ D_K q_t({z})  \right]$ is absent now in the PDE, which prevents $\pi({z})$ from being a stationary solution in general.
\end{proof}
\noindent The term $\dfrac{\partial^2}{\partial {z}^2} \left[ D_K q_t({z})  \right]$ encourages the entropy in the distribution $q_t({z})$. By ignoring it, the SVGD flow achieves stationary solutions that underestimate the variance of the target distribution. On the other hand, SGLD+R, performs a correction, leading to the desired target distribution. The next example highlights this fact in a relatively simple setting.

\paragraph{Example.} Consider a standard bi-dimensional Gaussian target, $\pi({z}) \sim \mathcal{N}(0, I)$. The initial distribution of particles is $p({z}) = q_0({z}) \sim \mathcal{N}([3,3], \mbox{diag}([0.25, 0.25]))$. We let both samplers run for $T = 200$ iterations using $L = 6$ particles, and plot their trajectories in Figure \ref{fig:comp}. Note that since SGLD+R is a valid sampler it explores a greater region of the target distribution, in comparison with SVGD, which underestimates the extension of the actual target. This phenomenon was predicted by Propositions 2 and 3. We also attach a table reporting estimates of the target mean, $\mu$, and marginal standard deviations, $\sigma_x$ and $\sigma_y$, respectively. Notice how SGLD+R estimates are closer to the ground truth values for the standard Gaussian target in this example.

\begin{figure}[!ht]
    \centering
    \includegraphics[width=7cm]{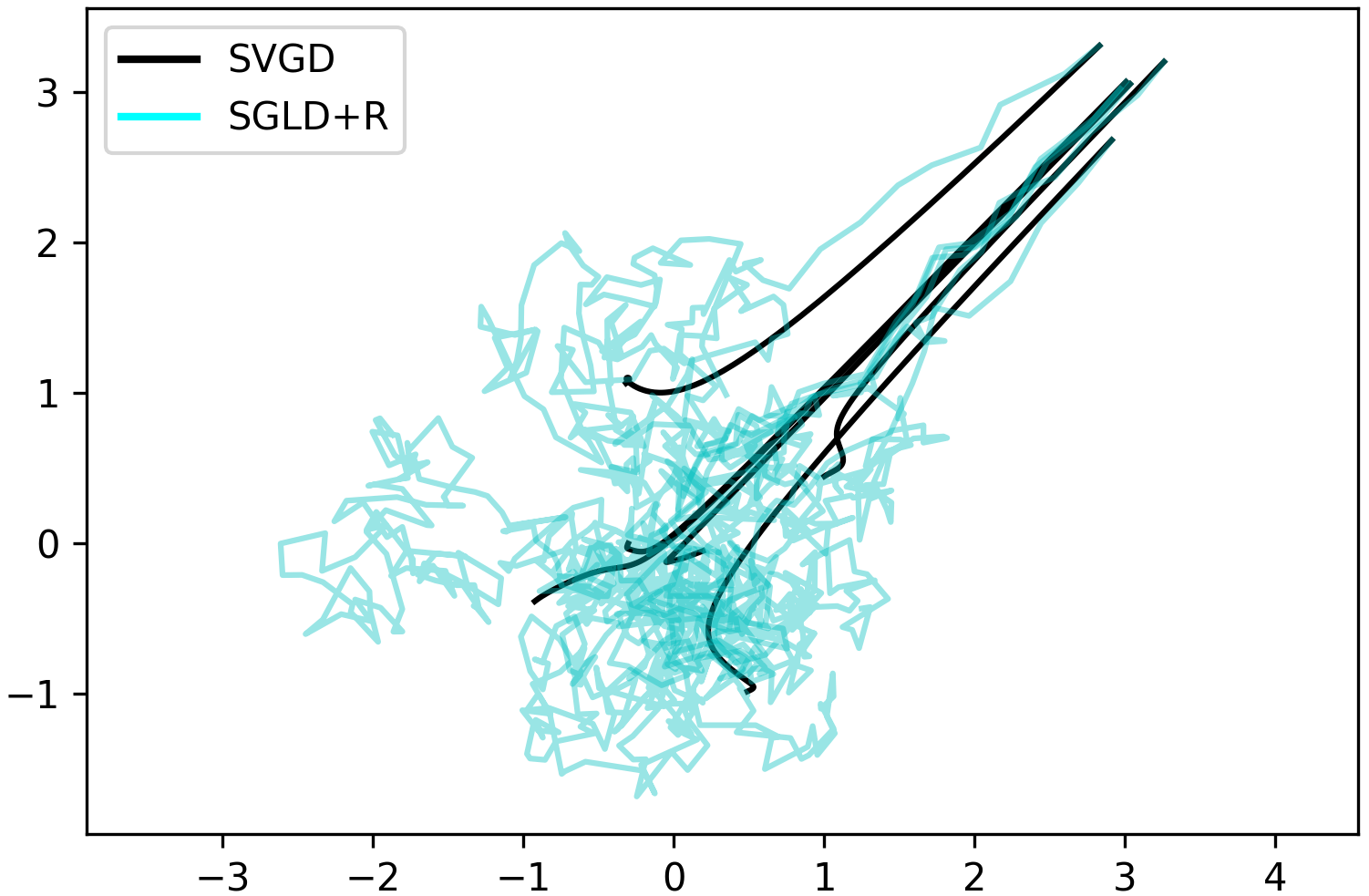}
    \qquad
    \begin{tabular}[b]{cccc}\hline
      \textbf{Estimates} & \textbf{SVGD} & \textbf{SGLD+R} & \textbf{Target} \\ \hline
      $\mu$ &  $0.18$ &$\textbf{0.08}$  & 0\\
      $\sigma_{x}$ & $0.70$ & $\textbf{0.90}$ & 1\\
      $\sigma_{y}$  & $0.74$ & $\textbf{0.87}$ & 1\\
    \hline
    \end{tabular}
    \captionlistentry[table]{A table beside a figure}
    \caption{Trajectories of the compared samplers. The table depicts estimates of quantities of interest for the standard Gaussian target.}
    \label{fig:comp}
  \end{figure}

\subsection{Momentum-based extensions of SG-MCMC samplers with repulsion}\label{sec:momentum}

SGLD+R can be seen as an extension of SGD to incorporate repulsion between particles and the noise term. It is possible to adapt recent developments from the stochastic optimization literature such as Adam, \textcite{kingma2014adam}, and propose their equivalent samplers with repulsion.

\paragraph{Momentum} This extension of SGD can be cast as an analogy to the \emph{momentum} in physics, keeping track of previous gradient to prevent oscillations which could slow down learning, \textcite{qian1999momentum}. We can frame SGD with momentum as
\begin{align*}
    {z}_{t+1} &= {z}_t - \epsilon_t {m}_t \\
    {m}_{t+1} &= {m}_t - \epsilon_t \nabla \log p({z}_t),
\end{align*}
where the ${m}_t$ are auxiliary variables.
The previous learning rule can be adapted to our framework by considering an augmented space $\bar{{z}} = ({z}, {m})$. Then, assuming the distribution of ${m}_t$ to be a standard Gaussian, the gradient of the log-density is equal to $-{m}_t$. Next, we set ${D} = {0}$ and ${Q} = \begin{pmatrix}
{0} & -{I} \\
{I} & {0}
\end{pmatrix}$ in (\ref{eq:sgmcmc}), arriving at the HMC sampler discussed in Section \ref{sec:infassamp}. 

From this point, we can further augment the latent space, considering $L$ particles $(\bar{z}_1, \ldots, \bar{z}_L)$, to arrive at the SGDm+R sampler. If we rearrange the latent space as $({z}_1, \ldots, {z}_L, {m}_1, \ldots, {m}_L)$, we may consider as ${Q_K}$ matrix the following one,
$$
{Q_K} = \left[
\begin{array}{c|c}
\begin{matrix}
{0} \\
\end{matrix}  & \begin{matrix}
 -\textbf{K} \\
\end{matrix} \\
\hline
\begin{matrix}
\textbf{K}  \\
\end{matrix} & \begin{matrix}
 {0} \\
\end{matrix}
\end{array}
\right],
$$
where $\textbf{K} = \begin{bmatrix}
k({z}_1, {z}_1) & \ldots & k({z}_1, {z}_L) \\
\ldots & \ldots & \ldots \\
k({z}_L, {z}_1) & \ldots & k({z}_L, {z}_L) \\
\end{bmatrix}$ is the kernel matrix from SVGD. It is straightforward to see that ${Q_K}$ is skew-symmetric. Thus we can apply again Proposition 1 to show that SGDm+R is a valid SG-MCMC sampler. The scheme is described in Algorithm \ref{alg:alg2}.

\begin{algorithm}[ht] %
\caption{Bayesian Inference via SGDm+R}  
\label{alg:alg2}
\begin{algorithmic}[1]
\State {\bf Input:} A target distribution with density function $\pi({z}) \propto \exp (-H({z}))$ and priors $p({z})$ and $p({m})$.
\State {\bf Output:} A set of particles $\{{z}_i\}_{i=1}^{ML}$ that approximates the target distribution.  
\State Sample initial set of particles from prior: ${z}_1^0, {z}_2^0, \ldots {z}_L^0 \sim p({z})$.
\State Sample initial set of moments from prior: ${m}_1^0, {m}_2^0, \ldots {m}_L^0 \sim p({m})$.
\For{each iteration $t$}
\State 
\begin{align*} 
{z}_i^{t+1}  &\gets  {z}_i^t - \epsilon_t \frac{1}{L}\sum_{l=1}^L\big[  k({z}_l^t, {z}_i^t)  {m}_l^t + \nabla_{{z}_l^t} k({z}_l^t, {z}_i^t)\big] \\
{m}_i^{t+1}  &\gets  {m}_i^t - \epsilon_t \frac{1}{L}\sum_{l=1}^L\big[  k({z}_l^t, {z}_i^t)  \nabla_{{z}_l^t} H({z}_l^t) + \nabla_{{z}_l^t} k({z}_l^t, {z}_i^t)\big] 
\end{align*}
\State After a burn-in period, start collecting particles: $ \{{z}_i\}_{i=1}^{NL} \gets \{{z}_i\}_{i=1}^{(N-1)L} \cup \{ {z}_1^{t+1}, \ldots,  {z}_L^{t+1} \} $ 
\EndFor
\end{algorithmic}
\end{algorithm}
\noindent Next, we show that a similar augmentation can be used to lift more complex optimization schemes to SG-MCMC samplers, one of our contributions.

\paragraph{Adam} The Adam stochastic optimization algorithm has become a \emph{de facto} scheme for optimizing complex non-linear models. In addition to keeping an estimate of the average of past gradients, as in momentum, Adam also keeps track of its variances.
To frame this method in our setting, note that the averages for the gradient can be expressed as
$$
{m}_t = \beta_1 {m}_{t-1} + (1 - \beta_1) \nabla \log p ({z}_t),
$$
where $\beta_1, \beta_2 \in (0, 1)$ are hyperparameters; and, similarly, for the gradient variances,
$$
{v}_t = \beta_2 {v}_{t-1} + (1 - \beta_2) \mbox{diag}(\nabla \log p ({z}_t)_i^2).
$$
Then, the latent state ${z}$ evolves according to
$$
{z}_{t+1} = {z}_t - \epsilon_t {m}_t / \sqrt{{v}_t}.
$$
To frame it in our scheme, with the benefits of interaction between particles,  consider an augmented space as in the momentum case, $({z}_1, \ldots, {z}_L, {m}_1, \ldots, {m}_L)$, and adopt ${Q_K}$ as before. The difference is that the log-density is now given by $H({z}, {m}) = \prod_{l=1} ^L \left[ \log p({z}_l) + {m}_l^{'} {M_l}^{-1} {m}_l \right]$, where ${M_l}$ is the mass matrix for chain $l$, in our case given by
$$
{M_l} = 
\begin{bmatrix}
\sqrt{v_{1,l}} & & \\
 & \ddots & \\
  & & \sqrt{v_{d,l}} 
\end{bmatrix}.
$$
Since $\nabla_{{m}_l} H({z},{m}) = {m}_l / \sqrt{{v}_l}$, we achieve the same effect of the original Adam, but on a per-chain basis. We call this sampler \emph{Adam plus noise and repulsion}, Adam+NR.

\subsection{Experiments}\label{sec:experiments_sgmcmc}

This Section describes the experiments developed to empirically test the proposed scheme. The simple example in Section \ref{sec:relationship} compared SVGD and SGLD+R. Here, we focus on confronting SGLD+R with the non-repulsive variant. First, we deal with two synthetic distributions, which offer a moderate account of complexity in the form of multimodality. In our second group of experiments, we explore a more challenging setting, testing a deep Bayesian model over several benchmark real data sets. 

Code for the different samplers is open sourced at \url{https://github.com/vicgalle/sgmcmc-force}. We rely on the library \texttt{jax} \parencite{jax2018github} as the main package, since it provides convenient automatic differentiation features with \emph{just-in-time} compilation, which is extremely useful in our case for the efficient implementation of SG-MCMC transition kernels.

\paragraph{Synthetic distributions.} 
The goal of this experiment is to see how well the samples generated through our framework approximate some quantities of interest, which can be analytically computed since the distributions are known. We thus test our proposed scheme with:

\begin{itemize}
\setlength\itemsep{-0.2em}
\item \textbf{Mixture of Exponentials (MoE)}. Two exponential distributions with different scale parameters $\lambda_1 = 1.5, \lambda_2=0.5$ and mixture proportions $\pi_1 = 1/3, \pi_2 = 2/3$. The pdf is
$$ 
p(z) = \sum_{i=1}^2 \pi_{i}\lambda_i \exp(-\lambda_i z).
$$
The exact value of the first and second moments can be computed using the change of variables formula
\begin{equation*}
\mathbb{E} \left[z^n \right] = \sum_{i=1}^2 \pi_{i}\frac{n!}{\lambda_i^n},
\end{equation*}
with $n\in \mathbb{N}$. 
Since $z>0$, to use the proposed scheme, we reparameterize using the $\log$ function. The pdf of  $y = \log(z)$ can be computed using
\begin{equation*}
p(y) = p(\log^{-1}(y))\left|\dfrac{\partial}{\partial y} \log^{-1}(y))\right|.
\end{equation*}
\item \textbf{Mixture of 2D Gaussians (MoG)}. A grid of $3 \times 3$ equally distributed isotropic 2D Gaussians, see Figure \ref{fig:mog}(d) for its density plot.  We set $\Sigma = \mbox{diag} (0.1, 0.1)$ and place the nine Gaussians centered at the following points:
\begin{align*}
\lbrace (-2,-2), (-2, 0), (-2, 2), (0, -2), \\(0, 0), (0, 2), (2, -2), (2, 0), (2, 2)  \rbrace.
\end{align*}
\end{itemize}
We compare two sampling methods, SGLD with $L$ parallel chains and our proposed scheme, SGLD+R. Note that the main difference between these two sampling algorithms is that for the former ${D_K} = {I}$, whereas the latter accounts for repulsion between particles and, therefore, $D_K$ is as  in Eq. (\ref{eq:psvgd_mat}). Tables \ref{tab:ess} and \ref{tab:ess2} report the effective sampling size metrics \parencite{kass1998markov} for each method using $L=10$ particles. Note that while ESS/s are similar, the repulsive forces in SGLD+R make for a more efficient exploration, resulting in much lower estimation errors. Figures \ref{fig:moe} and \ref{fig:mog} confirm this fact. In addition, even when increasing the number of particles $L$, SGLD+R achieves lower errors than SGLD (see Fig. \ref{fig:moe100}).

\begin{table}[H]
\centering{\small
\scalebox{0.95}{
\begin{tabular}{l|cc|cc}
\hline
& \multicolumn{2} {c} {ESS} & \multicolumn{2} {|c}{ESS/s}  \\
\textbf{Distribution} & {\bfseries SGLD} & {\bfseries SGLD+R} & {\bfseries SGLD}& {\bfseries SGLD+R}  \\
\hline
MoE& $ 44.3 $ & $ \pmb{59.1} $ & $ 51.5 $ & $ \pmb{61.0} $ \\
MoG& $151.3$ &  $ \pmb{169.5}$ &  $\pmb{36.3}$ &  $32.5$  \\
\hline
\end{tabular}
}
}
\caption{Effective sample size results for the two synthetic distributions task}\label{tab:ess}
\end{table}

\begin{table}[H]
\centering{\small
\scalebox{0.95}{
\begin{tabular}{l|cc}
\hline
& \multicolumn{2} {|c}{Error of $\mathbb{E}\left[ X \right]$} \\
\textbf{Distribution} & {\bfseries SGLD}& {\bfseries SGLD+R}  \\
\hline
MoE&  $0.39$  & $\pmb{0.14}$\\
MoG&  $1.42$ & $ \pmb{1.19} $ \\
\hline
\end{tabular}
}
}
\caption{Error results for the two synthetic distributions task}\label{tab:ess2}
\end{table}

For the computation of the error of $\mathbb{E}\left[ X \right]$ in Table \ref{tab:ess}, we sample for 500 iterations after discarding the first 500 iterations as burn-in, and we collect samples every 10 iterations to reduce correlation between samples. For the MoE case, we used 10 particles, whereas for the MoG task, we used 20 particles given the bigger number of modes.

\begin{figure}[ht]
\begin{center}
\minipage{0.45\textwidth}
\hspace{-0.5em}
  \includegraphics[width=\linewidth]{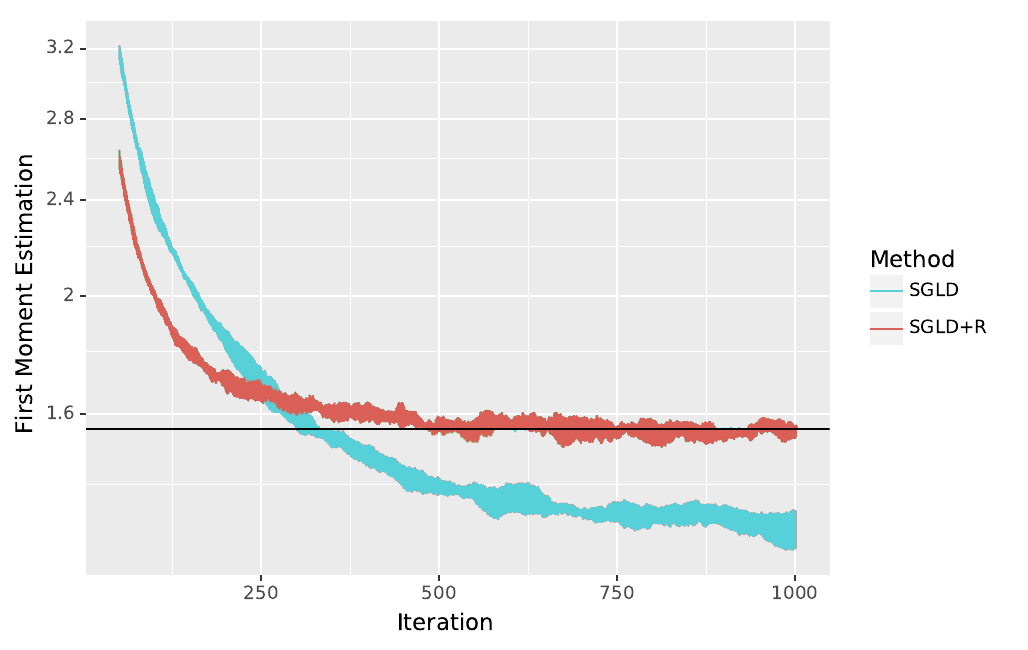}
\endminipage\hfill
\minipage{0.45\textwidth}
  \includegraphics[width=\linewidth]{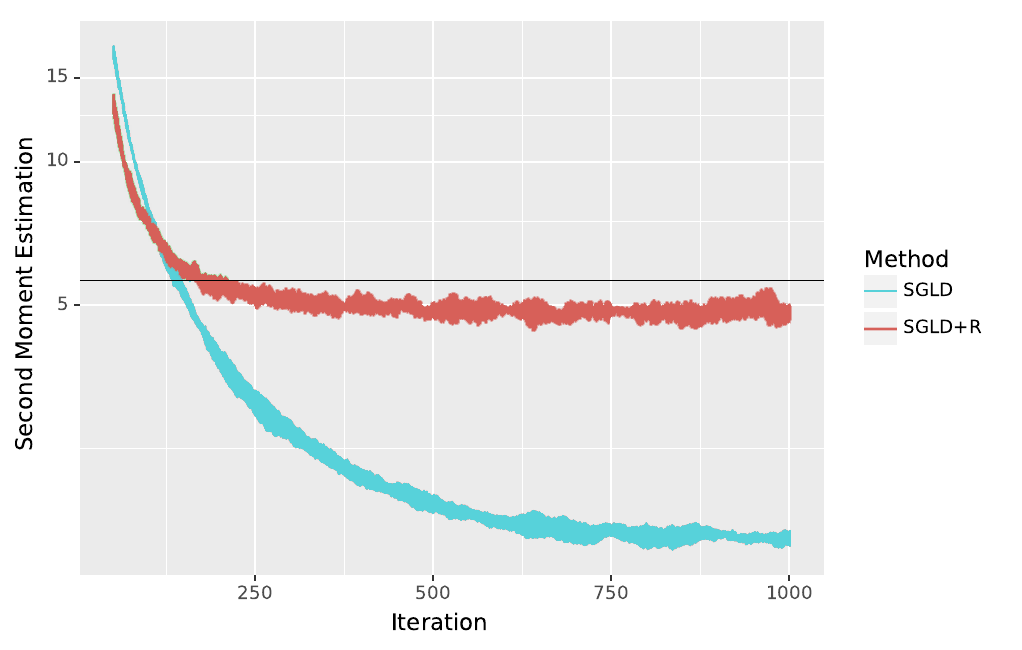}
\endminipage
\end{center}
\caption{Evolution of estimation, MoE experiment. Curves plotted for 5 simulations. 10 particles used at each simulation. Black line depicts exact value to be estimated. Left: Estimation of $\mathbb{E}\left[X\right]$. Right: Estimation of $\mathbb{E}\left[X^2\right]$.} \label{fig:moe}
\end{figure}

\begin{figure}[ht]
\begin{center}
\minipage{0.23\textwidth}
\hspace{-0.5em}
  \includegraphics[width=\linewidth]{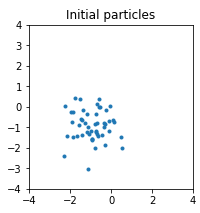}
\endminipage\hfill
\minipage{0.23\textwidth}
  \includegraphics[width=\linewidth]{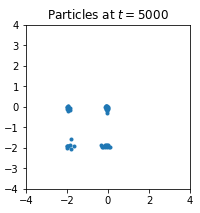}
\endminipage\hfill
\minipage{0.23\textwidth}
  \includegraphics[width=\linewidth]{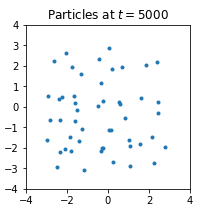}
\endminipage\hfill
\minipage{0.23\textwidth}
  \includegraphics[width=\linewidth]{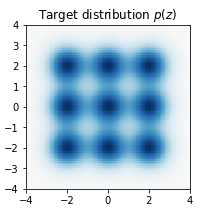}\label{fig:3x3gauss}
\endminipage
\end{center}
\caption{Evolution of  particles during the MoG experiment. (a) Prior particles. (b) SGLD dyn. (c) SGLD+R dyn. (d) MoG density.}\label{fig:mog}
\end{figure}

\begin{figure}[ht]
\begin{center}
\minipage{0.45\textwidth}
\hspace{-0.5em}
  \includegraphics[width=\linewidth]{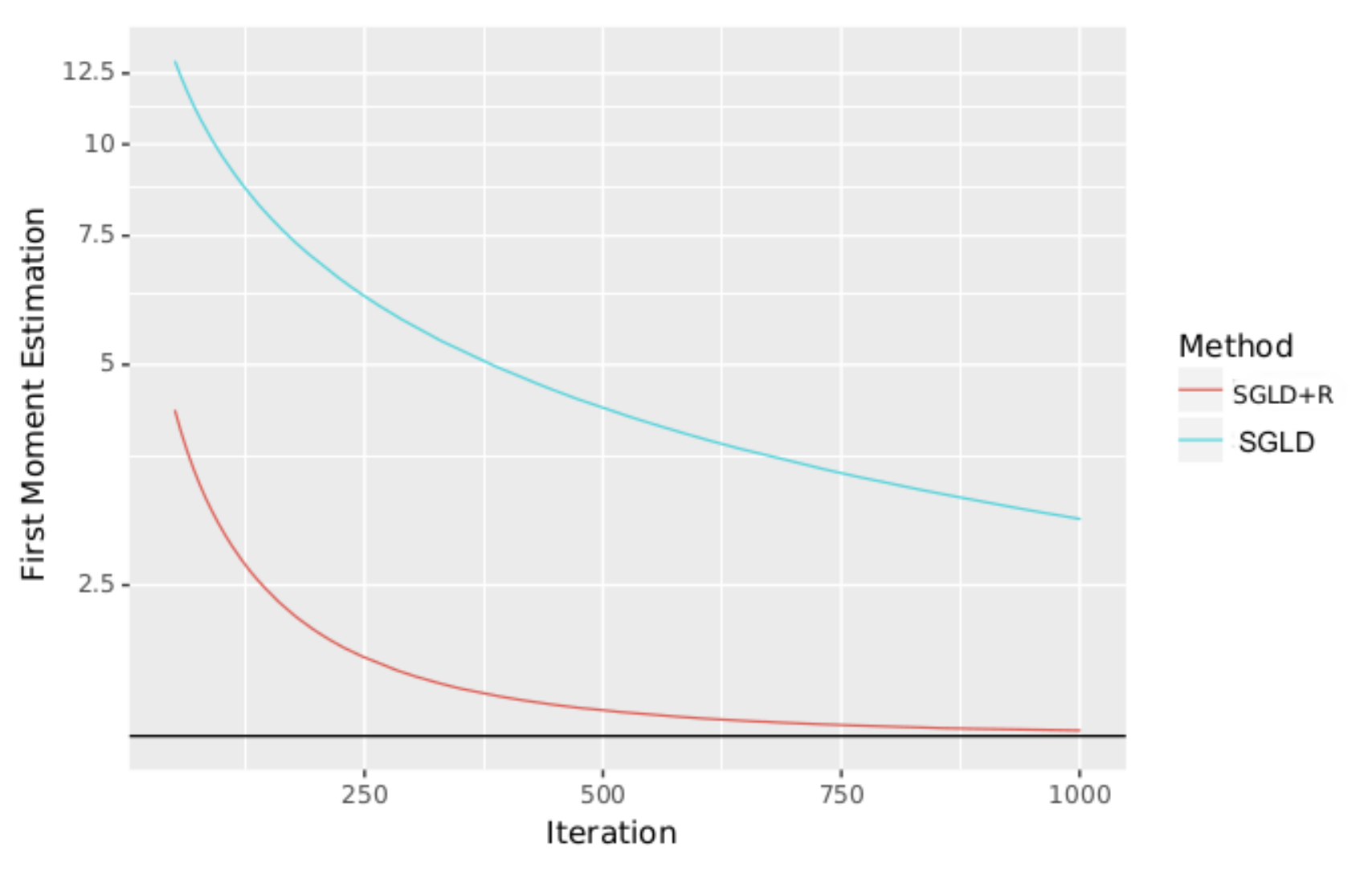}
\endminipage\hfill
\minipage{0.45\textwidth}
  \includegraphics[width=\linewidth]{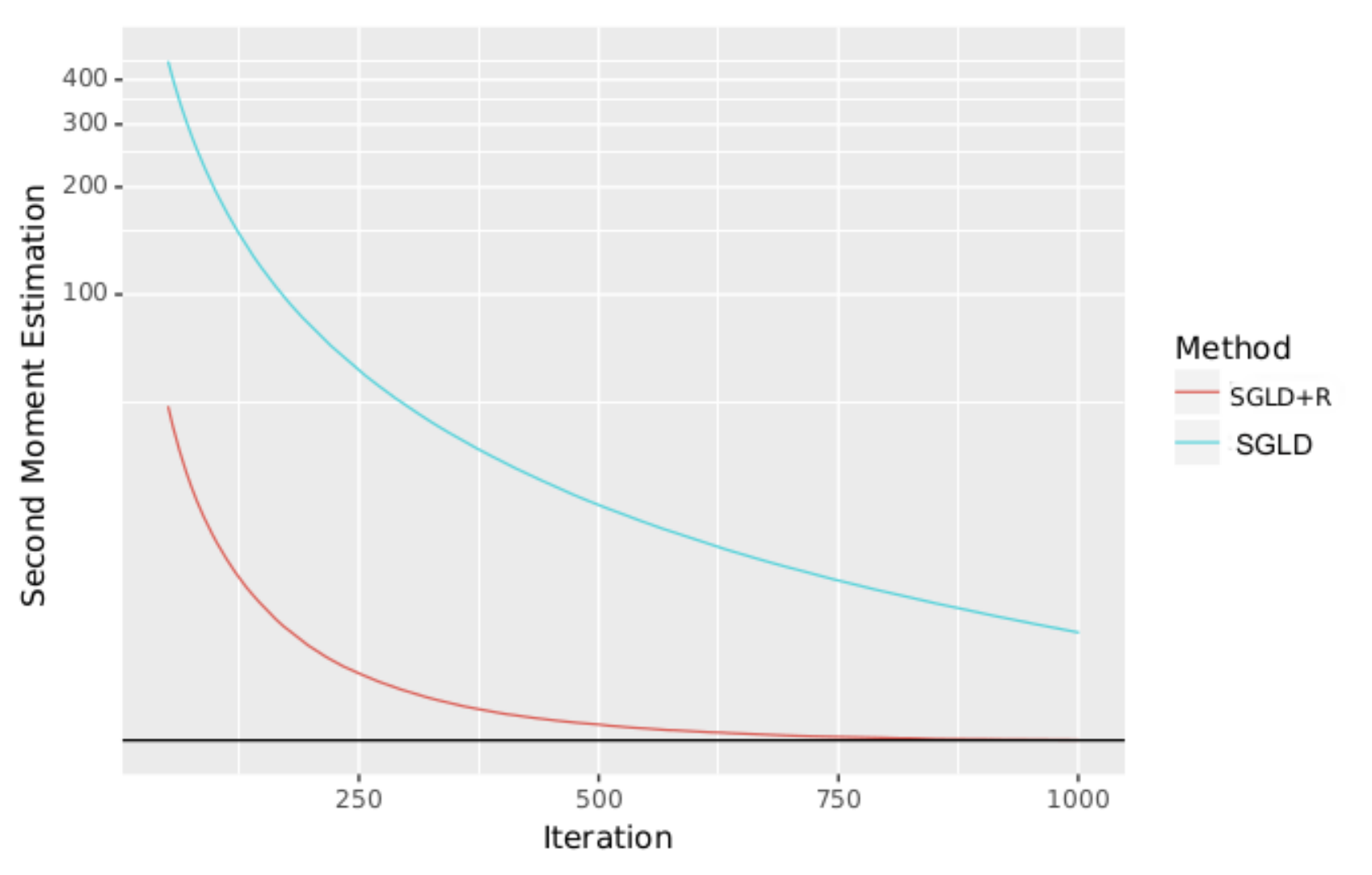}
\endminipage
\end{center}
\caption{Evolution of estimation during the MoE experiment. 100 particles are used. Black line depicts the exact value to be estimated. Left: Estimation of $\mathbb{E}\left[X\right]$. Right: Estimation of $\mathbb{E}\left[X^2\right]$.}\label{fig:moe100}
\end{figure}

\paragraph{Bayesian Neural Network.} 
We test the proposed scheme in a suite of regression tasks using a feed-forward neural network with one hidden layer with 50 units and ReLU activation functions. The goal of this experiment is to check that the proposed samplers scale well to real-data settings and complex models such as Bayesian neural networks. The datasets are taken from the UCI repository \parencite{Lichman:2013}. We use minibatches of size 100. 
As before, we compare SGLD and SGLD+R, reporting the average root mean squared error and log-likelihood over a test set in Tables \ref{tab:bnn} and  \ref{tab:bnn2}. We observe that SGLD+R typically outperforms SGLD. During the experiments, we noted that, in order to reduce computation time, during the last half of training we could disable the repulsion between particles without incurring in performance cost.

The learning rate $\epsilon$ was chosen from a grid $\{1e-5, \ldots, 1e-3 \}$ validated on another fold. The number of iterations was set to 2000 in every experiment. As before, to make predictions we collect samples every 10 iterations after a burn-in period. 20 particles were used for each of the tested datasets.

\begin{table}[h]
\centering{\small
\scalebox{0.95}{
\begin{tabular}{l|cc}
\hline
& \multicolumn{2} {c} {Avg. Test LL}  \\
\textbf{Dataset} & {\bfseries SGLD} & {\bfseries SGLD+R}  \\
\hline
Boston&  $ -2.551\pm 0.018$ & $ -2.575 \pm 0.007$ \\
Kin8nm&  $0.826 \pm 0.005$ & $0.831 \pm 0.006$  \\
Naval&  $3.379\pm 0.011$ & $ \pmb{3.428 \pm 0.019} $  \\
Protein&   $-2.991 \pm 0.000 $ & $\pmb{-2.987 \pm 0.001} $ \\
Wine&  $ -0.765 \pm 0.008 $ & $ \pmb{-0.750 \pm 0.007}$   \\
Yacht&  $-1.211 \pm 0.020   $ & $-1.172 \pm 0.026 $  \\
\hline
\end{tabular} 
}
}
\caption{Log-Likelihood results for the BNN experiments}\label{tab:bnn}
\end{table}

\begin{table}[H]
\centering{\small
\scalebox{0.95}{
\begin{tabular}{l|cc}
\hline
& \multicolumn{2} {c} {Avg. Test RMSE}  \\
\textbf{Dataset} & {\bfseries SGLD} & {\bfseries SGLD+R}  \\
\hline
Boston& $ 2.392 \pm 0.018$ & $ \pmb{2.295 \pm 0.017}$  \\
Kin8nm& $0.104 \pm  0.001  $ & $0.104 \pm 0.001$  \\
Naval& $0.008 \pm 0.000$ & $0.008 \pm 0.000$  \\
Protein& $4.810 \pm 0.003$ & $\pmb{4.794 \pm 0.003} $  \\
Wine& $0.522 \pm 0.004$ & $\pmb{0.514 \pm 0.004}$  \\
Yacht& $0.942 \pm 0.015$ & \pmb{$0.894 \pm 0.029 $}   \\
\hline
\end{tabular} 
}
}
\caption{Root Mean Squared Error results for the BNN experiments}\label{tab:bnn2}
\end{table}


\paragraph{Adversarial robustness.}

The last set of experiments aims to compare the different momentum-based approaches from Section \ref{sec:momentum}. We study a slightly more complex task than the ones in the previous experiment, by studying the robustness of a deep neural network against adversarial examples \parencite{goodfellow2014explaining} in the MNIST digit recognition dataset \parencite{MNIST}. Note that the Chapter \ref{cha:adv} is devoted to adversarial robustness in greater depth.

We follow a similar setting to \parencite{li2017dropout}, as we hypothesize that the uncertainty from Bayesian neural networks helps against adversarial attacks. First, we generate an attacked dataset using the \emph{fast gradient sign method} (FGSM) from \textcite{goodfellow2014explaining}. This attack generates an adversarial perturbation through
$$
x' = x - \eta\,\mbox{sign}(\nabla_x \max_y \log p(y|x)),
$$
for a step size $\eta$ that measures attack strength.

The model used for attack generation is a fully-connected network with three hidden layers and 1000 units each, with ReLU activations. Then, we test these perturbations on the same network, but trained using SGLD, SGLD+R and Adam+NR, respectively. Figure \ref{fig:attacks} depicts the attack evaluation curves for these inference techniques. As can be seen, Adam+NR in clearly superior in the low attack strength regime compared to the non-momentum counterpart. However, when the attack strength increases, the accuracy of both methods decay to similar levels.

\begin{figure}[h]
    \centering
\includegraphics[width=0.5\textwidth]{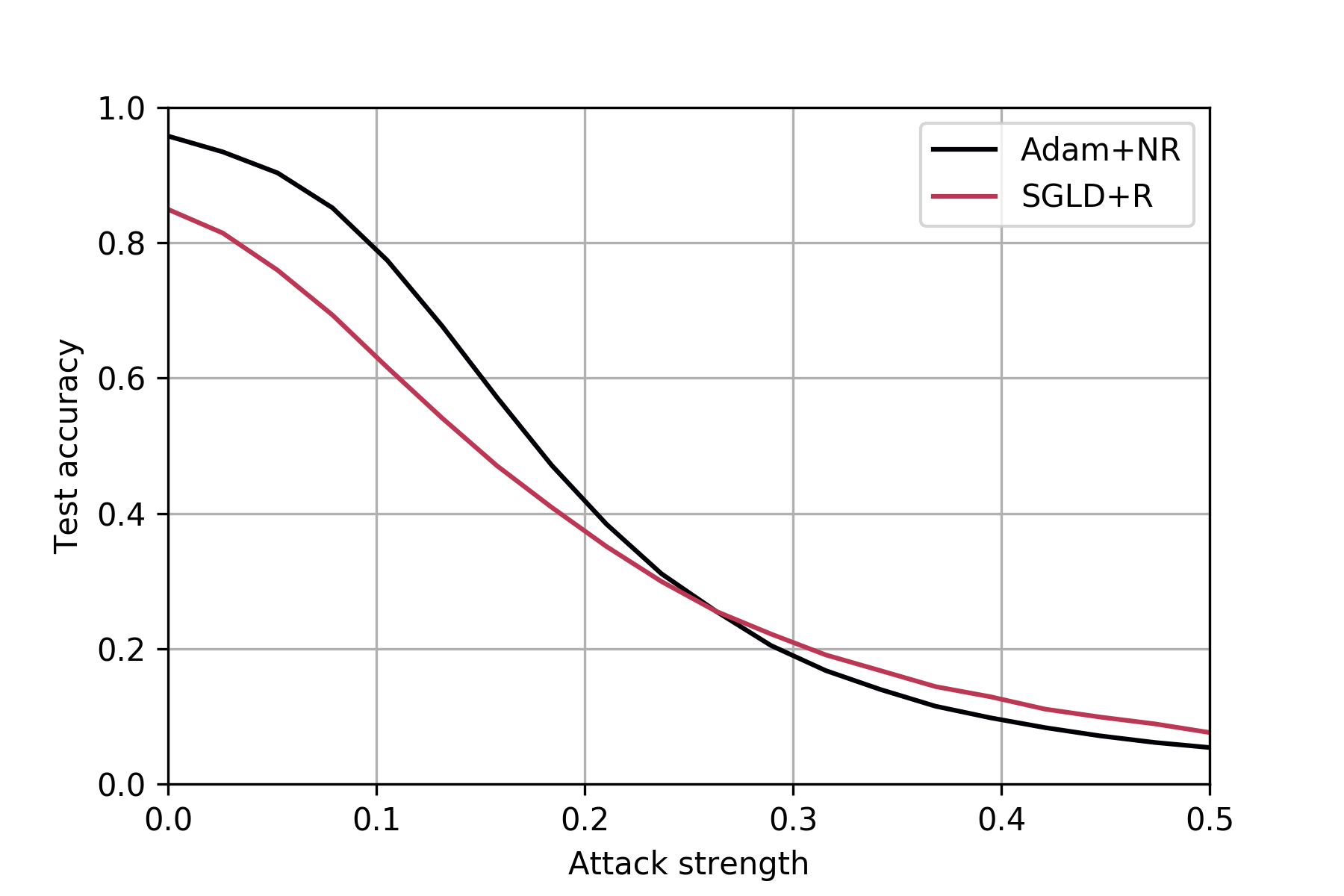}
    \caption{Accuracy for various values of attack strengths under the FGSM attack.}\label{fig:attacks}
\end{figure}



\section{A Variationally Inferred Sampling Framework}\label{sec:main}

Having reviewed the theory behind SG-MCMC methods, and proposed a new  transition kernel involving multiple chains with repulsion in Section \ref{sec:framework}, we now turn to our attention to variational inference (VI), proposing a method to further accelerate any SG-MCMC sampler.

In standard VI, the variational approximation 
is analytically tractable and typically chosen as a factorized Gaussian, as mentioned above. {However, it is important to note that 
other distributions can be adopted as long as they 
are easily sampled and their log-density and entropy values
computed. Moreover, in the rest of this section, we focus on the Gaussian case, the usual choice in the Bayesian deep learning community.}
Stemming from {this variational approximation}, we introduce several elements to construct the VIS.

Our first major modification of standard VI proposes the use of a more
flexible distribution, approximating the posterior by embedding a sampler through
\begin{equation}\label{eq:q}
q_{\phi,\eta}(z|x) = \int Q_{\eta, T}(z|z_0)q_{0,\phi}(z_0|x)dz_0,
\end{equation}
where $q_{0,\phi} (z | x)$ is the initial and tractable density
$q_{\phi} (z | x)$
(i.e., the starting state for the sampler). 
We designate this as the \emph{refined variational approximation}.
The conditional distribution $Q_{\eta, T}(z|z_0)$ refers
to a stochastic process parameterized by $\eta$ and 
used to evolve the original density $q_{0,\phi}(z|x)$
for $T$ periods, so as to achieve greater flexibility. Specific 
forms for $Q_{\eta, T}(z|z_0)$
will be described later in this section.
{Observe that when $T=0$, no refinement steps are performed and the refined variational approximation coincides with the original one; on the other hand, as 
 $T$ increases, the approximation will be closer to the exact posterior, assuming that $Q_{\eta, T}$ is a valid MCMC sampler
 in the sense of \parencite{ma2015complete}}.

We next maximize a refined ELBO objective, replacing in Equation (\ref{eq:elbo_bayes}) the 
original $q_{\phi }$ 
by $q_{\phi, \eta}$
\begin{equation}\label{eq:elbo_bayes2}
\mbox{ELBO}(q_{\phi,\eta}) = \mathbb{E}_{q_{\phi, \eta}(z|x)} \left[ \log p(x,z) - \log q_{\phi, \eta}(z|x)\right]
\end{equation}
\textls[-5]{This is done to optimize the divergence $KL(q_{\phi,\eta}(z|x) ||  p(z|x))$. {{The first term of Equation (\ref{eq:elbo_bayes2})}}
requires only being able to sample from $q_{\phi,\eta}(z|x)$; however, the second
term, which is the entropy
\linebreak $-\mathbb{E}_{q_{\phi,\eta}(z|x)} \left[ \log q_{\phi,\eta}(z | x) \right]$, also requires the evaluation of the evolving, implicit density.
\linebreak As a consequence, performing~variational inference with the refined variational approximation can be regarded as using the original variational guide while optimizing an alternative, tighter ELBO, as~Section~\ref{sec:rewriting}~shows. }

The above facilitates a framework for learning the sampler parameters $\phi, \eta$ using gradient-based optimization, with the help of automatic differentiation \parencite{baydin2017automatic}.
For this, the~approach operates in two phases.
First, in a refinement phase, the sampler parameters are learned in an optimization loop that maximizes the ELBO with the new posterior. After~several iterations, the second phase,
focused on inference, starts.
We allow the tuned sampler to run for
sufficient iterations, as in SG-MCMC samplers.
This is expressed algorithmically in Algorithm \ref{alg:vis}.

\begin{algorithm}[!ht]
\begin{algorithmic}[1]
\State \textbf{Refinement phase}:
\hspace{0.5cm}\While{not convergence}
  \State Sample an initial set of particles, $z_0 \sim q_{0,\phi}(z|x)$.
\State Refine the particles through the sampler, $z_T \sim Q_{\eta, T}(z|z_0)$.
\State  Compute the ELBO objective from Equation (\ref{eq:elbo_bayes2}).
\State   Perform automatic differentiation on the objective wrt parameters $\phi, \eta$ to update~them.
 \EndWhile
 \State 
 \State \textbf{Inference phase}:
 \State \hspace{0.cm}Once good sampler parameters $\phi^*, \eta^*$ are learned:
 \State \hspace{0.6cm} Sample an initial set of particles, $z_0 \sim q_{0,\phi^*}(z|x)$.
 \State \hspace{0.6cm} Use the MCMC sampler $z_T \sim Q_{\eta^*, T}(z|z_0)$ as $T \rightarrow \infty$.
\end{algorithmic}
 \caption{Variationally inferred sampler}\label{alg:vis}
\end{algorithm}

\noindent Since the sampler can be run for a different number of steps depending on the phase, we use the following notation when necessary: VIS-$X$-$Y$ denotes $T = X$ iterations during the refining phase and $T=Y$ iterations during the inference phase.


Let us specify now the key elements.

\subsection*{The Sampler $Q_{ \eta, T}(Z|Z_0)$ } \label{sec:grad}

{ As the latent variables $z$ are continuous}, 
we evolve the original density $q_{0,\phi}(z|x)$ through a stochastic diffusion process \parencite{pavliotis2014stochastic}. To make it tractable, we discretize the Langevin dynamics using the Euler--Maruyama scheme, arriving at the stochastic gradient Langevin dynamics (SGLD) sampler (2). 
We then follow the process $Q_{\eta,T} (z | z_0)$,
which represents $T$ iterations of the MCMC sampler. 

As an example, for the SGLD sampler $z_t = z_{t-1} + \eta \nabla \log p(x, z_{t-1}) + \xi_{t},$ where $t$ iterates from 1 to $T$. In this case, the only parameter
is the learning rate $\eta$ and the noise is $\xi_t \sim \mathcal{N}(0, 2\eta I)$. 
The initial variational distribution $q_{0, \phi}(z|x)$ is a Gaussian parameterized by a deep neural network (NN). Then, after $T$ iterations of the sampler $Q$ are parameterized by $\eta$, we arrive at $q_{\phi, \eta}$. 

An alternative arises by ignoring the noise $\xi$ \parencite{mandt2017stochastic}, thus refining the initial variational approximation using only the stochastic gradient descent (SGD).
 Moreover, we can use Stein variational gradient descent (SVGD) \parencite{liu2016stein} or a stochastic version \parencite{gallego2018stochastic} to apply repulsion between particles and promote more extensive explorations of the latent space, such as any of the samplers with repulsion developed in Section \ref{sec:framework}.

\subsection*{Approximating the Entropy Term}\label{sec:approx}

We propose four approaches for the ELBO optimization 
which take structural advantage of the refined variational approximation.

    \paragraph{Particle Approximation (VIS-P).} 
    
    In this approach, we approximate the posterior $q_{\phi,\eta}(z|x)$ by a mixture of Dirac deltas (i.e., we approximate it with a finite set of particles), by sampling $z^{(1)}, \ldots, z^{(M)} \sim q_{\phi,\eta}(z|x)$ and setting 
    $$
    q_{\phi,\eta}(z|x) = \frac{1}{M} \sum_{m=1}^M \delta(z - z^{(m)}).
    $$
    
    { In this approximation, the entropy term in Eq. (\ref{eq:elbo_bayes2}) 
    is set to zero. Consequently,  the sample converges to the 
    maximum posterior (MAP).} This may be undesirable when training generative models, as the generated samples usually have little diversity. Thus, in subsequent computations, we add to the refined ELBO the entropy of the initial variational approximation, $\mathbb{E}_{q_{0,\phi}(z|x)} \left[ \log q_{0,\phi}(z | x) \right]$, which
    serves as a regularizer alleviating the previous problem. When using SGD as the sampler, the resulting ELBO is tighter than that without refinement,
    as will be shown in Section \ref{sec:rewriting}.

    
    \paragraph{MC Approximation (VIS-MC).}
     Instead of performing the full marginalization in Equation (\ref{eq:q}), we approximate it with  { $q_{\phi,\eta}(z_T,\ldots, z_0|x) = \prod_{t=1}^T q_\eta(z_t | z_{t-1}) q_{0,\phi}(z_0|x)$; i.e., we consider the joint distribution for the refinement. However, in inference we only keep the $z_T$ values}. The entropy for each factor in this 
    approximation is straightforward to compute. For 
    example, for the SGLD case, we have
    {\bf
    $$ 
    z_t = z_{t-1} + \eta \nabla \log p(x, z_{t-1}) + \mathcal{N}(0, 2\eta I),\qquad  t=1, ..., T.
    $$
    }
This approximation tracks a better estimate of the entropy than 
    VIS-P, as we are not completely discarding it; rather, for each $t$, we marginalize out the corresponding $z_t$ using one sample.
          \paragraph{Gaussian Approximation (VIS-G).} 
          This approach is targeted at settings in which it could be helpful to have a posterior approximation that places density over the whole
        $z$ space. In the specific case of using SGD as the inner kernel, we have
\begin{align*}
z_0 &\sim q_{0,\phi}(z_0|x) = \mathcal{N}(z_0 | \mu_\phi(x), \sigma_\phi(x))\\
z_t &= z_{t-1} + \eta \nabla \log p(x, z_{t-1}), \qquad t=1,\ldots,T.
\end{align*}
By treating the gradient terms as points, the refined variational approximation can be computed as
$ q_{\phi,\eta}(z|x) = \mathcal{N}(z | z_T, \sigma_\phi(x))$. Observe 
that there is an implicit dependence on $\eta$ through $z_T$.

      \paragraph{Fokker--Planck Approximation (VIS-FP).} 
      Using the Fokker--Planck equation, we derive 
    a deterministic sampler via iterations of the form
\begin{equation}\label{eq:fppp}
z_{t} = z_{t-1} + \eta (\nabla \log p(x, z_{t-1}) - \nabla \log q_t (z_{t-1})),\qquad  t=1, ..., T{.}
\end{equation}
Then, we approximate the density $q_{\phi,\eta}(z|x)$ using a mixture of Dirac deltas. A detailed derivation of this approximation is given in Appendix \ref{app:fp}.

\subsection*{Back-Propagating through the Sampler}\label{sec:tuning}

In standard VI, the variational approximation $q(z|x;\phi)$ is parameterized by $\phi$. The~parameters are learned employing SGD, or variants such as Adam \parencite{kingma2014adam}, using the gradient $\nabla_{\phi} \mbox{ELBO}(q)$. We have shown how to embed a sampler inside the variational guide. 
It~is therefore also possible to compute a gradient of the objective with respect to the sampler parameters $\eta$ (see Section \ref{sec:grad}). For instance, we can compute a gradient
$\nabla_{\eta} \mbox{ELBO}(q)$
with respect to the learning rate $\eta$ from the SGLD or SGD processes to search for an optimal step size at every VI iteration. This is an additional step apart from using the gradient $\nabla_{\phi} \mbox{ELBO}(q)$ which is used to learn a good initial sampling distribution.

\subsection{Analysis of VIS}

Below, we highlight key properties of the proposed framework.

\subsubsection*{Consistency}

The VIS framework is geared towards SG-MCMC samplers, where we can compute the gradients of sampler hyperparameters to speed up mixing time (a common major drawback in MCMC \parencite{graves2011automatic}).
After back-propagating for a few iterations through the SG-MCMC sampler and learning a good initial distribution, one can resort to the learned sampler in the second phase, so standard consistency results from SG-MCMC apply as $T \rightarrow \infty$ \parencite{brooks2011handbook}.

\subsubsection*{Refinement of ELBO}\label{sec:rewriting}

 Note that, for a refined guide using the VIS-P approximation and $M=1$ samples, the refined objective function can be written as 
$$
 \mathbb{E}_{q(z_0|x)} \left[ \log p(x, z_0 + \eta \nabla \log p(x,z_0) ) - \log q(z_0 | x)\right]
$$
noting that $z = z_0 + \eta \nabla \log p(x,z_0)$ when using SGD for $T=1$ iterations.
This is equivalent to the refined ELBO in Eq. (\ref{eq:elbo_bayes2}). Since we are perturbing the latent variables in the steepest direction, we show easily that, for a moderate $\eta$, the previous bound is tighter than
$\mathbb{E}_{q(z_0|x)} \left[ \log p(x, z_0  ) - \log q(z_0 | x)\right]$, the one for the original variational guide~$q(z_0 | x)$. This~reformulation of ELBO is also convenient since it provides a clear way of implementing our refined variational inference framework in any probabilistic 
programming language (PPL) supporting algorithmic differentiation.

Respectively, for the VIS-FP case, we find that its 
deterministic flow follows the same trajectories as SGLD: 
based on standard results of MCMC samplers \parencite{murray2008notes}, we have 
$$
KL(q_{\phi,\eta}(z|x) ||  p(z|x)) \leq KL(q_{0, \phi}(z|x) ||  p(z|x)).
$$

A similar reasoning applies to the VIS-MC approximation; however, it does not hold for VIS-G since it assumes that the posterior is Gaussian.

\subsubsection*{Taylor Expansion}\label{sec:taylor}

This analysis applies only to VIS-P and VIS-FP.
As stated in Section \ref{sec:rewriting},  within the VIS framework, optimizing the ELBO resorts to the performance of $\max_z \log p(x, z + \Delta z)$, where~$\Delta z$ is one iteration of the sampler; i.e., $\Delta z = \eta \nabla \log p(x, z)$ in the SGD case (VIS-P), \linebreak or  $\Delta z = \eta \nabla (\log p(x, z) - \log q(z))$ in the VIS-FP case. 
For notational clarity, we consider the case $T=1$, 
although a similar analysis
follows in a straightforward manner if more refinement steps are performed.

Consider a first-order Taylor expansion of the refined objective 
$$
\log p(x, z + \Delta z) \approx \log p(x, z) + (\Delta z)^\intercal \nabla \log p(x, z).
$$

Taking gradients 
with respect to the latent variables $z$, we arrive at
$$
\nabla_z \log p(x, z + \Delta z) \approx \nabla_z \log p(x,z) + \eta \nabla_z \log p(x,z)^\intercal \nabla_z^2 \log p(x,z),
$$
where we have not computed the gradient through the $\Delta z$ term (i.e., we treated it as a constant for simplification). Then, the refined gradient can be deemed to be the original gradient plus a second order correction. Instead of being modulated by a constant learning rate, this correction is adapted by the chosen sampler. The experiments in Section \ref{sec:exp} 
show that this is beneficial for the optimization as it 
typically takes fewer iterations than the original variant to achieve lower losses. 

By further taking gradients through the $\Delta z$ term, we may tune the sampler parameters such as the learning rate as 
presented in Section \ref{sec:tuning}. Consequently, the next subsection describes two 
differentiation modes.


\subsubsection*{Two Automatic Differentiation Modes for the Refined ELBO Optimization}\label{sec:AD}

For the first variant, 
remember that the original variant can be 
rewritten (which we term Full AD)~as
\begin{equation}
     \mathbb{E}_q \left[ \log p(x, z + \Delta z) - \log q(z + \Delta z | x) \right].
\end{equation}

We now define a stop gradient operator $\bot$ (which corresponds to \texttt{detach} in Pytorch or \linebreak \texttt{stop\_gradient}~in~tensorflow)  that sets the gradient of its operand to zero---i.e.,~$\nabla_x \bot (x) = 0$---whereas in a forward pass, it acts as the identity function---that is, $\bot (x) = x$. With this, a variant  of the ELBO objective (which we term Fast AD) is
\begin{equation}
    \mathbb{E}_q \left[ \log p(x, z + \bot (\Delta z)) - \log q(z + \bot(\Delta z) | x) \right].
\end{equation}

 Full AD ELBO enables a gradient to be computed
 with respect to the sampler parameters inside $\Delta z$ at the cost of a slight increase in computational burden.
 On the other hand, the~Fast AD variant may be useful in numerous scenarios, as illustrated in the experiments.

\paragraph{Complexity.} 
Since we need to back propagate through $T$ iterations of an SG-MCMC scheme, using~standard results of meta-learning and automatic differentiation \parencite{franceschi2017forward}, the time complexity of our more intensive approach (Full-AD) is $\mathcal{O}(mT)$, where $m$ is the dimension of the hyperparameters (the learning rate of SG-MCMC and the latent dimension). Since~for most use cases, the hyperparameters lie in a low-dimensional space, the approach is
therefore~scalable.







\subsection{Experiments}\label{sec:exps_vis}


{The following experiments showcase the power the VIS framework as well as illustrating the the impact of various parameters on its performance, guiding their 
choice in practice. We 
also present a comparison with standard VIS
and other recent variants, showing that the increased computational complexity of computing
gradients through sampling steps is worth the gains in flexibility.
Moreover, the proposed framework is compatible with other structured inference techniques, such as the sum--product algorithm, as well as serving to support other tasks such as  classification}.

Within the spirit of reproducible research, 
the code for VIS has been released at 
\url{https://github.com/vicgalle/vis}. 
The VIS framework is implemented with Pytorch \parencite{paszke2017automatic}, although we have also released a notebook for the first experiment using Jax to highlight the simple implementation of VIS.
In any case, we emphasize that the approach facilitates 
rapid iterations over a large class of models. 

\subsubsection*{Funnel Density}

We first tested the framework on a synthetic yet complex target distribution. This~experiment assessed whether VIS is
suitable for modeling complex distributions. The~target bi-dimensional density was defined through
\begin{align*}
    z_1 &\sim \mathcal{N}(0, 1.35) \\
    z_2 &\sim \mathcal{N}(0, \exp(z_1)).
\end{align*}

\textls[-5]{We adopted the usual diagonal Gaussian distribution
as the variational approximation. 
For~VIS, we used the VIS-P approximation and refined it for $T = 1$ steps using SGLD. Figure~\ref{fig:funnel} {top} shows the trajectories of the lower bound for up to 50 iterations of variational optimization with Adam: our refined version achieved a tighter bound. The {bottom} figures present  contour curves of the learned variational approximations. Observe that the VIS variant was placed closer to the mean of the true distribution and was more disperse than the original variational approximation, illustrating the fact that the refinement step helps in attaining more flexible posterior approximations.}

\begin{figure}[ht]
\begin{center}
\minipage{0.45\textwidth}
\hspace{-0.5em}
  \includegraphics[width=\linewidth]{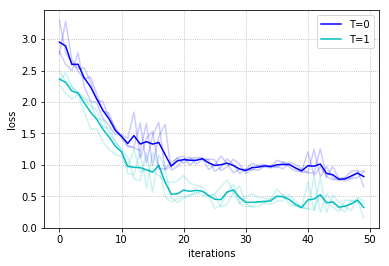}
\endminipage\hfill
\minipage{0.45\textwidth}
  \includegraphics[width=\linewidth]{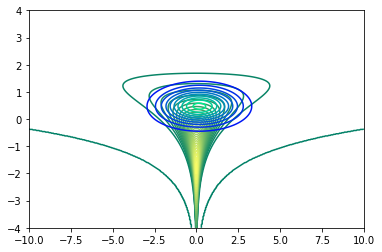}
\endminipage
\minipage{0.45\textwidth}%
  \includegraphics[width=\linewidth]{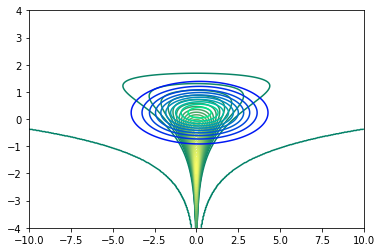}
\endminipage
\end{center}
\caption{{Top: Evolution of the negative evidence lower bound (ELBO) loss objective over  50 iterations. Darker lines depict means along different seeds (lighter lines). Bottom left: Contour curves (blue--turquoise) of the variational approximation with no refinement ($T=0$) at iteration 30 (loss of $1.011$). Bottom right: Contour curves (blue--turquoise) of  refined variational approximation ($T=1$) at iteration 30 (loss of $0.667$). Green--yellow curves denote target density.}} \label{fig:funnel}
\end{figure}

\subsubsection*{State-Space  Markov Models}

We tested our variational approximation on two state-space models: one for discrete data and another for continuous observations. These experiments also demonstrated that the framework is compatible with standard inference techniques such as the sum--product scheme from the Baum--Welch algorithm or Kalman filter.
In both models, we performed inference on their 
parameters $\theta$.
All the experiments in this subsection used the Fast AD version (Section \ref{sec:AD}) as it was not necessary to further tune the sampler parameters to obtain competitive results. Full model implementations can be found in Appendix \ref{app:ss}, based on \texttt{funsor} \parencite{obermeyer2019functional}, a PPL on top of the \texttt{Pytorch} autodiff framework.

\emph{Hidden Markov Model} (HMM): The model equations are
\begin{equation}\label{snow}
p(x_{1:\tau} , z_{1:\tau}, \theta) = \prod_{t=1}^\tau p(x_t|z_t,\theta_{em})p(z_t|z_{t-1},\theta_{tr})p(\theta),
\end{equation}
where each conditional is a categorical distribution taking 
five different classes.  The prior is $p(\theta) = p(\theta_{em})p(\theta_{tr})$ based on two Dirichlet distributions that sample the observation and state transition probabilities, respectively. 

\emph{Dynamic Linear Model} (DLM): The model
equations are as in (\ref{snow}), 
although the conditional distributions are now Gaussian and the parameters $\theta$ refer to the observation and transition variances. 

 For each model, we generated a synthetic dataset and used the refined variational approximation with $T = 0, 1, 2$. For the original variational approximation to the parameters $\theta$, we used a Dirac delta. Performing VI with this approximation corresponded to MAP estimation using 
 the Baum--Welch algorithm in the HMM case \parencite{rabiner1989tutorial} and
 the Kalman filter in the DLM case \parencite{zarchan2013fundamentals},
  as we marginalized out the latent variables $z_{1:\tau}$. We used the VIS-P variant since it was sufficient  to show performance gains in this case.
 
 Figure \ref{fig:ss} shows the results. The first row reports the experiments related to the HMM, the~second row those for the DLM. We report the evolution of the log-likelihood during inference  in all graphs; the first column reports the number of ELBO iterations, and the second column portrays 
 clock times as the optimization takes place. They confirm that VIS ($T>0$) achieved better results than standard VI ($T=0$) for a comparable amount of time. {Note also  that there was not as much gain when changing from $T=1$ to $T=2$ as there is from $T=0$ to $T=1$, suggesting the need to carefully 
 monitor this parameter. Finally, the~top-right graph for the case $T=0$ is shorter as it requires less clock time.}

\begin{figure}[h]

\minipage{0.48\textwidth}
  \includegraphics[width=\linewidth]{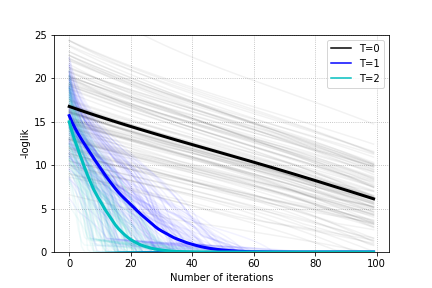}
\endminipage
\minipage{0.48\textwidth}
  \includegraphics[width=\linewidth]{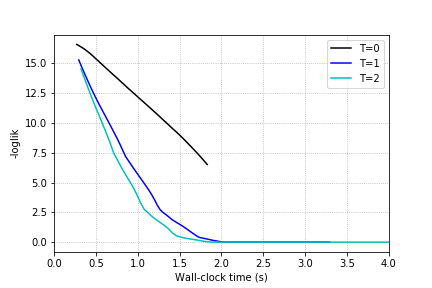}
\endminipage\hfill

\minipage{0.48\textwidth}%
  \includegraphics[width=\linewidth]{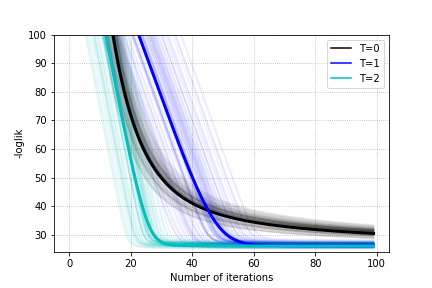}
 \endminipage
 \minipage{0.48\textwidth}%
  \includegraphics[width=\linewidth]{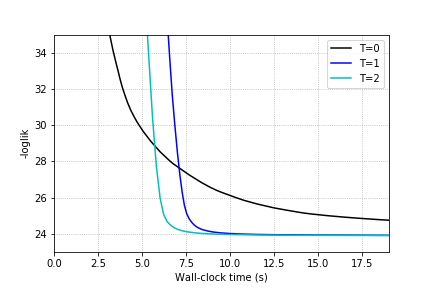}
\endminipage
\caption{Results of ELBO optimization for state-space models. Top-left (Hidden Markov Model (HMM)): Log-likelihood against the number of ELBO gradient iterations. Top-right (HMM): Log-likelihood against clock time. Bottom-left (Dynamic Linear Model (DLM)): Log-likelihood against number of ELBO gradient iterations. Bottom-right (DLM): Log-likelihood against against clock time.}\label{fig:ss}%
\end{figure}

\paragraph{Prediction with an HMM.} 
With the aim of assessing whether ELBO optimization helps in attaining better auxiliary scores, results in a prediction task are also reported. We generated a synthetic time series of alternating values of 0 and 1 for $\tau=105$ timesteps. We trained the previous HMM model on the first 100 points and report in Table \ref{tbl:preds} the accuracy of the predictive distribution $p(y_t)$ averaged over the final five time-steps. We also report the predictive entropy as it helps in assessing the confidence of the model in its predictions, as a strictly proper scoring rule \parencite{gneiting2007strictly}. To guarantee the same computational budget time and a fair comparison, the model without refinement was run for 50 epochs (an epoch was a full iteration over the training dataset), whereas the model with refinement was run for 20 epochs. It can be observed that the refined model achieved higher accuracy than its counterpart. In addition,
it was more correctly confident in its predictions.
\begin{table}[!h]

\caption{Prediction metrics for the HMM.}\label{tbl:preds}
\begin{tabular}{c@{\hskip 1.2in}c@{\hskip 1.1in}c@{\hskip 1in}c}
\toprule
   & ${T=0}$                             & ${T=1}$   \\ 
 \midrule
    accuracy          & $0.40$ &  $0.84$ \\
    predictive entropy          & $1.414$ &  $1.056$ \\
    logarithmic score   & $-1.044$ & $-0.682$ \\
 \bottomrule
\end{tabular}
\end{table}

\paragraph{Prediction with a DLM.}

We tested the VIS framework on Mauna Loa monthly $CO_2$ time-series data \parencite{keeling2005atmospheric}. We used the first 10 years as a training set, and we tested over the next two years. We used a DLM composed of a local linear trend plus a seasonal block of periodicity 12. 
Data were standardized 
to {a mean of zero and standard deviation of one}. To guarantee the same computational budget time, the model without refining was run for 10 epochs, whereas the model with refinement was run for 4 epochs.  Table \ref{tbl:preds_dlm}
reports the mean absolute error (MAE) and predictive entropy. 
In addition, we computed the interval score 
 \parencite{gneiting2007strictly}, as a strictly proper scoring rule. As can be seen, for similar clock times, the refined model not only achieved a lower MAE, but also its predictive intervals were narrower than the non-refined counterpart. 

\begin{table}[!h]
\centering
\caption{Prediction metrics for the DLM.}\label{tbl:preds_dlm}
\begin{tabular}{c@{\hskip 1.1in}c@{\hskip 1in}c@{\hskip 1in}c}
\toprule
   & ${T=0}$                             & ${T=1}$   \\ 
 \midrule
    MAE          & $0.270$ &  $0.239$ \\

    predictive entropy          & $2.537$ &  $2.401$ \\
    interval score ($\alpha=0.05$) & $15.247$ & $13.461$\\
 \bottomrule
\end{tabular}
\end{table}

\subsubsection*{Variational Autoencoder}

The third batch of experiments showed that VIS 
was competitive with respect to other algorithms from the recent literature, including unbiased implicit variational inference (UIVI~\parencite{pmlr-v89-titsias19a}), semi-implicit variational inference (SIVI~\parencite{yin2018semi}),  variational contrastive divergence (VCD~\parencite{pmlr-v97-ruiz19a}), 
and~the HMC variant from~\parencite{hoffman2017learning}, showing that our framework can outperform those approaches in similar experimental settings. 

 To this end, we tested the approach with a variational autoencoder (VAE) model \parencite{kingma2013auto}. 
The VAE defines a conditional distribution $p_{\theta}(x | z)$, generating an observation $x$ from a latent variable $z$ using {parameters $\theta$}. For this task, our interest 
was in modeling the $28 \times 28$ image distributions 
underlying the MNIST 
\parencite{MNIST} and the fashion-MNIST \parencite{xiao2017/online} datasets. To perform inference (i.e., to learn
the parameters $\theta$) the VAE introduces a variational approximation $q_{\phi}(z | x)$. In the standard setting, this distribution is Gaussian; we instead used the refined variational approximation comparing various values of $T$. We used the VIS-MC approximation (although we achieved similar results
with VIS-G) with the Full AD variant given in Section \ref{sec:AD}.

For the experimental setup, we reproduced the setting in \parencite{pmlr-v89-titsias19a}. For $p_{\theta}(x | z)$, we used a factorized Bernoulli distribution parameterized by a two layer feed-forward network with 200~units in each layer and relu activations, except for a final sigmoid activation. As a variational approximation $q_{\phi}(z | x)$, we used a Gaussian with mean and (diagonal) covariance matrix parameterized by
two distinct neural networks with the same structure as previously used, except for sigmoid activation for the mean and a softplus activation for the covariance matrix.

Results are reported in Table \ref{tbl:vae}. To guarantee 
 fair comparison, we trained the VIS-5-10 variant for 10 epochs, whereas all the other variants were trained for 15 epochs (fMNIST) or 20 epochs (MNIST), so that the VAE's performance was comparable to that reported in \parencite{pmlr-v89-titsias19a}. Although VIS was trained for fewer epochs, by increasing the number $T$ of MCMC iterations, we dramatically improved the test log-likelihood. In terms of computational complexity, the~average time per epoch using $T=5$ was 10.46 s, whereas with no refinement ($T=0$), the time was 6.10 s (which was the reason behind our decision to train the refined variant for fewer epochs): a moderate increase in computing time may be worth the dramatic increase in log-likelihood while not introducing new parameters into the model, except for the learning rate $\eta$.

 \begin{table}[!ht]
\centering

\caption{Test log-likelihood on binarized MNIST and fMNIST. Bold numbers indicate the best results. UIVI: unbiased implicit variational inference; SIVI: semi-implicit variational inference; VAE: variational autoencoder; VCD: variational contrastive divergence; HMC-DLGM: Hamiltonian Monte Carlo for Deep Latent Gaussian Models; VIS: variationally inferred sampler.}\label{tbl:vae} %
\begin{tabular}{c@{\hskip 0.9in}c@{\hskip 0.8in}c@{\hskip 0.9in}c}
\toprule
\textbf{Method}   & \textbf{MNIST}                             & \textbf{fMNIST}   \\ \midrule
 \multicolumn{3}{c}{Results from \parencite{pmlr-v89-titsias19a}}       \\
 \midrule
    UIVI          & $-94.09$ &  $-110.72$ \\
    SIVI          & $-97.77$ &  $-121.53$ \\
    VAE          & $-98.29$ &  $-126.73$ \\
\midrule
 \multicolumn{3}{c}{ Results from \parencite{pmlr-v97-ruiz19a}}       \\
 \midrule
    VCD          & $-95.86$ &  $-117.65$ \\
    HMC-DLGM & $-96.23$ & $-117.74$ \\ 
\midrule
    \multicolumn{3}{c}{ This paper}       \\
    \midrule
    VIS-5-10     & {${-82.74 \pm {0.19}}$} & ${-105.08 \pm 0.34}$  \\
    VIS-0-10     & $-96.16 \pm 0.17$ & $-120.53 \pm 0.59$  \\
    VAE (VIS-0-0)              & $-100.91 \pm 0.16$ & $-125.57 \pm 0.63$ \\
 \bottomrule
\end{tabular}
\end{table}
Finally, as a visual inspection of the VAE reconstruction 
quality trained with VIS, Figures \ref{fig:reco} and \ref{fig:reco2}, respectively, display 10 random samples of each dataset.
\begin{figure}[h]

\includegraphics[width=\linewidth]{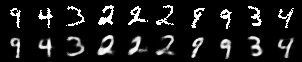}

  \caption{Top: original images from MNIST. Bottom: reconstructed images using VIS-5-10 at 10~epochs.}\label{fig:reco}
\end{figure}
\unskip
\begin{figure}[h]

\includegraphics[width=\linewidth]{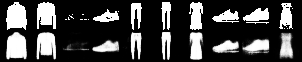}

  \caption{Top: original images from fMNIST. Bottom: reconstructed images using VIS-5-10 at 10~epochs.}\label{fig:reco2}
\end{figure}

\subsubsection*{Variational Autoencoder as a Deep Bayes Classifier}\label{sec:exp}
In the final experiments, we investigated whether VIS can deal with more general probabilistic graphical models and also perform well in other inference tasks such as classification.
We explored the flexibility of the proposed scheme to solve inference problems in an experiment with a classification task in a high-dimensional setting 
with the MNIST dataset.
More concretely, we extended the VAE model, conditioning it on a discrete variable $y \in \mathcal{Y} = \lbrace 0, 1, \ldots, 9 \rbrace$, leading to a conditional VAE (cVAE). The cVAE defined a decoder distribution $p_\theta(x | z, y)$ on an input space $x \in \mathbb{R}^D$ given a class label $y \in \mathcal{Y}$, latent variables $z \in \mathbb{R}^d$ 
\textls[-15]{and parameters $\theta$. Figure \ref{fig:deep_bayes} depicts the corresponding {probabilistic graphic model}. Additional details regarding the model architecture and hyperparameters are given in Appendix~\ref{sec:detail}.}

\begin{figure}[h]
\centering
\includegraphics[width=0.25\linewidth]{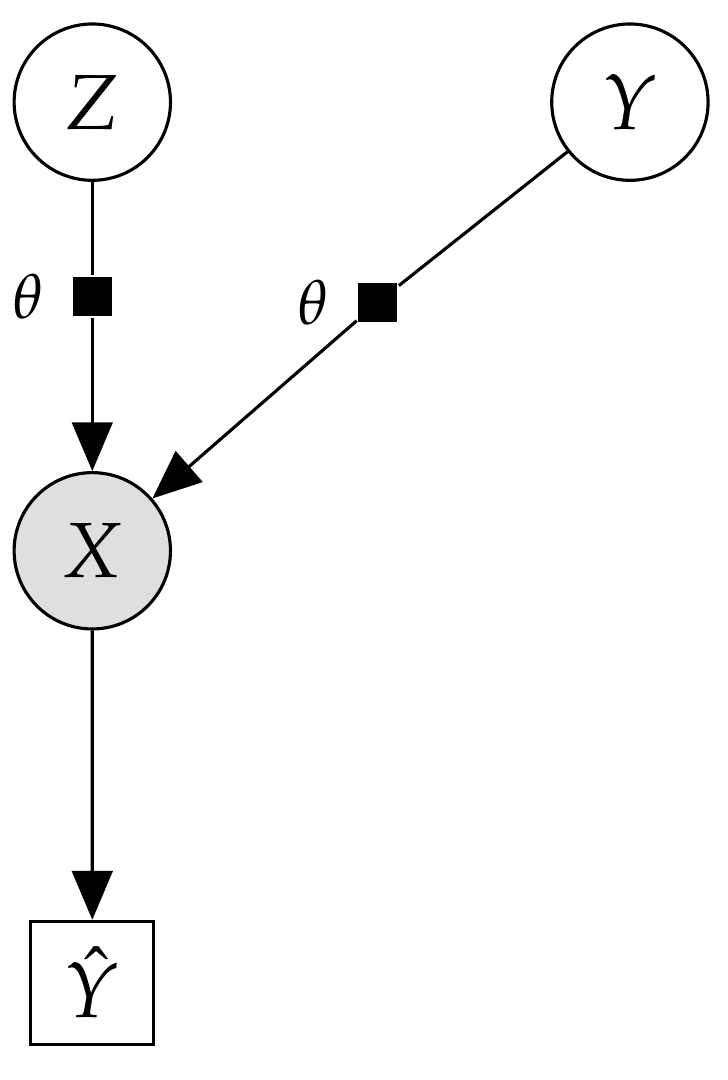}
  \caption{{Probabilistic graphical model} for the deep Bayes classifier.}\label{fig:deep_bayes}
\end{figure}

To perform inference, a variational posterior was learned as an encoder $q_\phi(z|x,y)$ from a prior $p(z) \sim \mathcal{N}(0, I)$.
Leveraging the conditional structure on $y$, we used the generative model as a classifier using the Bayes rule,
\begin{equation}\label{eq:mc_cvae}
p(y|x) \propto p(y)p(x|y) = p(y) \int p_\theta(x|z,y)q_\phi(z|x,y)dz  \approx  \frac{1}{M} \sum_{m=1}^M p_\theta (x | z^{(m)}, y)p(y), 
\end{equation}
where we used $M$ Monte Carlo samples $z^{(m)} \sim q_\phi(z|x,y)$. In the experiments, we set $M = 5$. Given a test sample $x$, the label $\hat{y}$ with the highest probability $p(y|x)$ is predicted.

For comparison, we performed several experiments changing $T$ in 
the transition distribution $Q_{\eta, T}$ of 
the refined variational approximation. 
The results are given in Table \ref{tab1}, which reports
the test accuracy at 
end of {the refinement phase}. Note that we are comparing different values of $T$ depending on their use in {refinement or inference} phases (in the latter, the model and variational parameters were kept frozen). The model with $T_{ref} = 5$ was trained for 10~epochs, whereas the other settings were for 15 epochs, to give all settings a similar training time.  Results were averaged over three runs with different random seeds. In all settings, we used the VIS-MC approximation for the entropy term. From the results, it is clear that the effect of using the refined variational approximation (the cases when $T > 0$) is crucially beneficial to achieve higher accuracy. The effect of learning a good initial distribution and inner learning rate by using the gradients $\nabla_{\phi} \mbox{ELBO}(q)$ and $\nabla_{\eta} \mbox{ELBO}(q)$ has a highly positive impact in the accuracy obtained.

On a final note, we have not included the case of only using an SGD or an SGLD sampler (i.e., without learning an initial distribution $q_{0, \phi} (z|x)$) since the results were much worse than those in Table \ref{tab1} for a comparable computational budget. This strongly suggests that, for inference in high-dimensional, continuous latent spaces, learning a good initial distribution through VIS may accelerate mixing time
dramatically.

\begin{table}[H]
\centering

\caption{Results on  digit classification task using a deep Bayes classifier.}\label{tab1}

\begin{tabular}{c@{\hskip 1.4in}c@{\hskip 1.3in}c@{\hskip 1.4in}c}
\toprule
${T_{ref}}$ &  ${T_{inf}}$ & \textbf{Acc. (Test)} \\
\midrule
0 & 0  & $96.5 \pm 0.5$ \% \\ 
0 & 10 &  $97.7 \pm 0.7$ \%\\
5 & 10 & $ {99.8 \pm 0.2}$ \% \\
\bottomrule
\end{tabular}
\end{table}

\section{Summary}

In this chapter, we have delved into current approaches for scalable Bayesian inference. Two families of methods compound the state of the art: SG-MCMC and VI approaches. We extended SG-MCMC by proposing a new sampler that takes several parallel chains, but not independently; rather, we added interaction between them, in the form of a repulsive force, in order to make the particles do not collapse into the same point of the posterior. Regarding variational approaches, we proposed a new variational approximation in the form of an SG-MCMC sampler, whose hyperparameters can be tuned using gradient optimization techniques. Experiments confirm that both approaches are effective and realizable. In addition, these algorithmic contributions are also usable with state space models, as confirmed in the corresponding experiments. Thus, we believe both contributions will be welcomed by Bayesian practitioners and researchers that want to work with contemporary, complex models.

\begin{subappendices}

\section{Fokker-Planck Approximation (VIS-FP)}\label{app:fp}
  
  The Fokker--Planck equation is a PDE 
  that describes the temporal evolution of the density of a random variable under a (stochastic) gradient flow \parencite{pavliotis2014stochastic}. For a given SDE 
  
$$
dz = \mu(z,t)dt + \sigma(z,t)dB_t,
$$
the corresponding Fokker--Planck equation is
$$
\frac{\partial}{\partial t} q_t(z) = -\frac{\partial}{\partial z}\left[ \mu(z,t)q_t(z)\right] + \frac{\partial^2}{\partial z^2} \left[ \frac{\sigma^2(z,t)}{2} q_t(z) \right].
$$
We are interested in converting the SGLD dynamics to a deterministic gradient flow. 

\begin{Proposition}
The SGLD dynamics, given by the SDE
$$
dz = \nabla \log p(z)dt + \sqrt{2}dB_t,
$$
have an equivalent deterministic flow, written as the ODE 
$$
dz = (\nabla \log p(z) - \nabla \log q_t (z))dt.
$$
\end{Proposition}
\begin{proof}
Let us write the Fokker--Planck equation for the respective flows. For the Langevin SDE, it is 
$$
\frac{\partial}{\partial t} q_t(z) = - \frac{\partial}{\partial z} \bigg[ \nabla \log p(z) q_t(z) \bigg] + \frac{\partial^2}{\partial z^2} \bigg[ q_t(z) \bigg].
$$
On the other hand, the Fokker--Planck equation for the deterministic gradient flow is given by
$$
\frac{\partial}{\partial t} q_t(z) = - \frac{\partial}{\partial z} \bigg[ \nabla \log p(z) q_t(z)\bigg] + \frac{\partial}{\partial z} \bigg[ \nabla \log q_t(z) q_t(z)\bigg].
$$
The result immediately follows, since $ \frac{\partial}{\partial z} \left[ \nabla \log q_t(z) q_t(z)\right] = \frac{\partial^2}{\partial z^2} \left[ q_t(z) \right]$.
\end{proof}
Given that both flows are equivalent, we restrict our attention to the deterministic flow. Its discretization leads to iterations of the form
\begin{equation}\label{eq:deterministic_flow}
z_{t} = z_{t-1} + \eta (\nabla \log p(z_{t-1}) - \nabla \log q_{t-1} (z_{t-1})).
\end{equation}
In order to tackle the last term, we make the following particle approximation. Using a variational formulation, we have 
\begin{align*}
    - \nabla \log q(z) = \nabla \left( - \frac{\delta}{\delta q} \mathbb{E}_q \left[ \log q\right] \right).
\end{align*}

Then, we smooth the true density $q$ convolving it with a kernel $K$, typically the rbf 
kernel, $K(z, z') = \exp \lbrace - \gamma \| z - z' \|^2 \rbrace$, where $\gamma$ is the bandwidth hyperparameter, leading to
\begin{align*}
    \nabla \left( - \frac{\delta}{\delta q} \mathbb{E}_q \left[ \log q\right] \right) &\approx
     \nabla \left( - \frac{\delta}{\delta q} \mathbb{E}_q \left[ \log (q\ast K )\right] \right) \\
     &= \nabla \log (q \ast K) - \nabla \left( \frac{q}{(q \ast K)} \ast K \right).
\end{align*}

If we consider a mixture of Dirac deltas, $q(z) = \frac{1}{M} \sum_{m=1}^M \delta(z - z_m)$, then the approximation is given 
by
$$
- \nabla \log q(z) \approx - \frac{\sum_k \nabla_{z_m} K(z_m, z_n)}{\sum_n K(z_m, z_n)}
- \sum_l \frac{\nabla_{z_m} K(z_m, z_l)}{\sum_n K(z_n, z_l)},
$$
which can be inserted into Equation (\ref{eq:fppp}). Finally, note that
it is possible to back-propagate through this equation; i.e., the gradients of $K$ can be explicitly computed.

\section{Experiment Details}\label{sec:detail}

\subsection{State-Space Models}\label{app:ss}

\paragraph{Initial experiments.}\label{app:hmm}
For the HMM, both the observation and transition probabilities are categorical distributions, taking values in the domain $\lbrace 0, 1, 2, 3, 4 \rbrace$.

The equations of the DLM are 
\begin{align*}
    z_{t+1} &\sim \mathcal{N}(0.5z_t + 1.0,\sigma_{tr}) \\
    x_{t} &\sim \mathcal{N}(3.0z_t + 0.5, \sigma_{em}).
\end{align*}
with $z_0  = 0.0$.


\paragraph{Prediction task in a DLM.}
The DLM model comprises a linear trend component plus a seasonal block with a period of 12. The trend is specified as
\begin{align*}
x_t &= z_{level,t} + \epsilon_t \qquad \epsilon_t \sim \mathcal{N}(0, \sigma_{obs}) \\
z_{level,t} &= z_{level,t-1} + z_{slope,t-1} + \epsilon'_t \qquad \epsilon'_t \sim \mathcal{N}(0, \sigma_{level}) \\
z_{slope,t} &= z_{slope,t-1} + \epsilon''_t \qquad \epsilon''_t \sim \mathcal{N}(0, \sigma_{slope}).
\end{align*}

With respect to the seasonal component,
we specify it through
\begin{align*}
x_t &= F z_t + v_t \qquad v_t \sim \mathcal{N}(0, \sigma_{obs})\\
z_t &= G z_{t-1} + w_t \qquad w_t \sim \mathcal{N}(0, \sigma_{seas})
\end{align*}
where $F$ is a $12$-dimensional vector
$( 1,0,\ldots, 0,0)$ 
and $G$ is the $12\times 12$ matrix 
\begin{equation*}
G = \begin{bmatrix}
0 & 0 & \ldots & 0 & 1 \\
1 &	0 & & 0 & 0 \\
0 & 1 & & 0 & 0 \\
 & & \ddots & & \\
 0 & 0 & & 1 & 0
\end{bmatrix}.
\end{equation*}
\\

Further details are in \textcite{west1998bayesian}.

\subsection{VAE}

\paragraph{Model details.}

The prior distribution $p(z)$ for the latent variables $z \in \mathbb{R}^{10}$ is a standard factorized Gaussian. The decoder distribution $p_\theta(x|z)$ and the encoder distribution (initial variational approximation) $q_{0,\phi}(z|x)$ are parameterized by two feed-forward neural networks, as 
detailed in Figure \ref{fig:arch_vae}.

\paragraph{Hyperparameter settings.}
The optimizer Adam is used in all experiments, with 
la earning rate of $\lambda=0.001$. We~also set $\eta = 0.001$. We train for 15 epochs (fMNIST) and 20 epochs (MNIST) to achieve a performance 
similar to the VAE in \parencite{pmlr-v89-titsias19a}. For the VIS-5-10 setting, we train only for 10 epochs to allow a fair computational comparison in terms of similar computing times.

\subsection{cVAE}


\paragraph{Model details.}

The prior distribution $p(z)$ for the latent variables $z \in \mathbb{R}^{10}$ is a standard factorized Gaussian. The decoder distribution $p_\theta(x|y,z)$ and the encoder distribution (initial variational approximation) $q_{0,\phi}(z|x,y)$ are parameterized by two feed-forward neural networks whose details can be found in Figure \ref{fig:arch}.
Equation (\ref{eq:mc_cvae}) is approximated with one MC sample from the variational approximation in all experimental settings, {as it allowed fast inference times while offering better results.}

\begin{figure}[h]
\RecustomVerbatimEnvironment{Verbatim}{BVerbatim}{}
\begin{minted}{python}
class VAE(nn.Module):
    def __init__(self):
        super(VAE, self).__init__()

        self.z_d = 10
        self.h_d = 200
        self.x_d = 28*28

        self.fc1_mu = nn.Linear(self.x_d, self.h_d)
        self.fc1_cov = nn.Linear(self.x_d, self.h_d)
        self.fc12_mu = nn.Linear(self.h_d, self.h_d)
        self.fc12_cov = nn.Linear(self.h_d, self.h_d)
        self.fc2_mu = nn.Linear(self.h_d, self.z_d)
        self.fc2_cov = nn.Linear(self.h_d, self.z_d)

        self.fc3 = nn.Linear(self.z_d, self.h_d)
        self.fc32 = nn.Linear(self.h_d, self.h_d)
        self.fc4 = nn.Linear(self.h_d, self.x_d)

    def encode(self, x):
        h1_mu = F.relu(self.fc1_mu(x))
        h1_cov = F.relu(self.fc1_cov(x))
        h1_mu = F.relu(self.fc12_mu(h1_mu))
        h1_cov = F.relu(self.fc12_cov(h1_cov))
        # we work in the logvar-domain
        return self.fc2_mu(h1_mu),
        torch.log(F.softplus(self.fc2_cov(h1_cov)))

    def decode(self, z):
        h3 = F.relu(self.fc3(z))
        h3 = F.relu(self.fc32(h3))
        return torch.sigmoid(self.fc4(h3))
\end{minted}
\caption{Model architecture for the VAE.}
\label{fig:arch_vae}
\end{figure}
\unskip
\begin{figure}[h]

\RecustomVerbatimEnvironment{Verbatim}{BVerbatim}{}
\begin{minted}{python}
class cVAE(nn.Module):
    def __init__(self):
        super(cVAE, self).__init__()

        self.z_d = 10
        self.h_d = 200
        self.x_d = 28*28
        num_classes = 10

        self.fc1_mu = nn.Linear(self.x_d + num_classes, self.h_d)
        self.fc1_cov = nn.Linear(self.x_d + num_classes, self.h_d)
        self.fc12_mu = nn.Linear(self.h_d, self.h_d)
        self.fc12_cov = nn.Linear(self.h_d, self.h_d)
        self.fc2_mu = nn.Linear(self.h_d, self.z_d)
        self.fc2_cov = nn.Linear(self.h_d, self.z_d)

        self.fc3 = nn.Linear(self.z_d + num_classes, self.h_d)
        self.fc32 = nn.Linear(self.h_d, self.h_d)
        self.fc4 = nn.Linear(self.h_d, self.x_d)

    def encode(self, x, y):
        h1_mu = F.relu(self.fc1_mu(torch.cat([x, y], dim=-1)))
        h1_cov = F.relu(self.fc1_cov(torch.cat([x, y], dim=-1)))
        h1_mu = F.relu(self.fc12_mu(h1_mu))
        h1_cov = F.relu(self.fc12_cov(h1_cov))
        # we work in the logvar-domain
        return self.fc2_mu(h1_mu),
        torch.log(F.softplus(self.fc2_cov(h1_cov)))

    def decode(self, z, y):
        h3 = F.relu(self.fc3(torch.cat([z, y], dim=-1)))
        h3 = F.relu(self.fc32(h3))
        return torch.sigmoid(self.fc4(h3))
\end{minted}
\caption{Model architecture for the cVAE.}
\label{fig:arch}
\end{figure}

\paragraph{Hyperparameter settings.}
The optimizer Adam was used in all experiments, with a learning rate of $\lambda=0.01$. We set the initial $\eta = 5 \times 10^{-5}$.

\end{subappendices}

\chapter{Adversarial Classification}\label{cha:adv}

\section{Introduction}



Classification is a major research area in machine learning with important applications  in security and cybersecurity, including fraud detection \parencite{bolton2002statistical}: phishing detection \parencite{rakesh}, terrorism \parencite{terror} or cargo screening \parencite{cargo}. An increasing number of processes are being automated through classification algorithms, being essential that these are robust to trust key operations based on their output. State-of-the-art classifiers perform extraordinarily well on standard data, but they have been shown to be vulnerable to adversarial examples, that is, data instances targeted at fooling the underlying algorithms. \textcite{comiter} provides an excellent introduction from a policy perspective, pointing out the potentially enormous security impacts that such attacks may have over systems for filter content, predictive policing or autonomous driving, to name but a few. 

Most research in classification has focused on obtaining more accurate algorithms, largely ignoring the eventual presence of adversaries who actively manipulate data to fool the classifier in pursue of a benefit. Consider 
spam detection: as classification algorithms are incorporated to such task,
spammers learn how to evade them. Thus, rather than sending their spam messages in standard language, they 
 slightly modify spam words (frequent in spam messages but not 
so much in legitimate ones), misspell them or 
change them with synonyms; or they add good words (frequent in legitimate emails but not in spam ones) to fool the detection system. 

Consequently, classification algorithms in critical AI-based systems must be robust against adversarial data manipulations. To this end, they have to take into account possible modifications of input data
due to adversaries.~The subfield of
classification that seeks for algorithms with robust behaviour against adversarial perturbations is known as adversarial classification (AC)
and  was pioneered by \textcite{dalvi2004adversarial}.
Stemming from their work, the prevailing paradigm 
when modelling the confrontation between classification systems and adversaries has been game theory, see recent reviews by \textcite{BIGGIO2018317} and \textcite{doi:10.1002/widm.1259}. This entails well-known common knowledge hypothesis \parencite{Antos,gameTheoryACriticalIntroduction2004} according to which agents share information about their beliefs and preferences. From a fundamental point of view, this is 
not sustainable in  application areas such as security or cybersecurity,
as participants try to
hide and conceal information.  

After reviewing key developments in game-theoretic approaches to AC in Sections \ref{sec:ac_acr} and \ref{sec:ac_ac}, we cover novel techniques based on adversarial risk analysis (ARA, \textcite{AMLARA}) in Sections  \ref{sec:ac_acra} and \ref{sec:scalable}. Their key advantage is that they do not assume strong common knowledge hypothesis concerning belief and preference sharing, as with standard game theoretic approaches to AML. In this, we unify, expand
and improve upon earlier work in \textcite{naveiro2018adversarial} and
\textcite{gallego2020protecting}.
Our focus is on binary classification
 problems in face only of {\em exploratory attacks}, defined to have influence over operational data but not over training ones. In addition, we restrict our attention to attacks affecting only malicious instances, the so-called \textit{integrity-violation attacks}, the usual context in most security scenarios. We assume that the attacker will, at least, try to modify every malicious instance before the
classifier actually observes it. 
Moreover, attacks will be assumed to be {\em deterministic},
in that we can predict for sure the results of their 
application over a given instance.
\textcite{AdversarialMachineLearning2011} and \textcite{Barreno2006} provide taxonomies of attacks against classifiers. 

We first consider approaches in which learning about the adversary is performed in the
operational phase, studying how to robustify generative classifiers against attacks. In certain applications, these  could be very demanding from a computational perspective; for those cases, we present in Section \ref{sec:scalable} an  
approach in which adversarial aspects are incorporated in the training phase.
Section \ref{sec:conEx} illustrates the proposed framework
with spam detection and image processing examples.

\section{Attacking Classification Algorithms}\label{sec:ac_acr}

\subsection{Binary Classification Algorithms}\label{sec2.1}

In binary classification settings, an agent that we call classifier ($C$, she)
may receive instances belonging to one of two possible  classes denoted, in our context, as malicious ($y=y_1$) or innocent ($y=y_2$).  Instances have  features $x \in \mathbb{R}^d$ whose distribution informs about their class $y$.  Most classification approaches can be typically broken down in two separate stages, the training and operational phases \parencite{bishop2006pattern}.
%
%

The first one is used to learn the distribution $p_C(y|x)$, modelling the classifier's beliefs about the  instance class $y$ given its features $x$. 
Frequently, a distinction is introduced 
between {\em generative} and {\em discriminative} models. 
   In the first case, models $p_C(x|y)$ and $p_C (y)$ are learnt from training data; based on them,
$p_C(y|x)$ is deduced through Bayes theorem. Typical examples include 
Naive Bayes \parencite{rish2001empirical} and (conditional) variational autoencoders 
 \parencite{kingma2014semi}.
%
In discriminative cases, $p_C(y|x)$ is directly learnt from data.
Within these, an important group of methods uses  
a parameterised function
$f_\beta : \mathbb{R}^d \rightarrow \mathbb{R}^2 $
so that the prediction is given through 
$p_C(y|x,\beta) = \mbox{softmax} (f_\beta (x))[y]$: when $f_\beta (x) = \beta' x$, we recover the logistic regression
model 
\parencite{mccullagh1989generalized}; 
if $f_\beta$ is  a sequence of linear transformations alternating certain nonlinear activation functions,   
 we obtain a feed-forward neural network \parencite{bishop2006pattern}. 
 Learning depends then on the underlying
methodology adopted.
\begin{itemize}
\item In frequentist approaches, training data $\mathcal{D}$ is 
typically used to construct a (possibly regularised) maximum likelihood estimate $\hat{\beta}$, and $p_C(y | \hat{ \beta }, x )$ is employed for classification. Parametric
differentiable models are amenable to training with \emph{stochastic gradient descent}
(SGD) \parencite{bottou2008tradeoffs}  
using a minibatch 
of samples at each iteration. 
This facilitates, e.g.,  training 
deep neural networks with large amounts of
high-dimensional data as 
with images or text data \parencite{10.5555/3086952}.
\item In Bayesian approaches, a prior $p_C (\beta)$ is used to compute the posterior $p_C (\beta | \mathcal{D} )$, given the data $\mathcal{D}$, and the predictive distribution
\begin{equation}\label{MIERDA}
p_C (y | x, \mathcal{D})= \int p_C(y | \beta , x ) p_C (\beta | \mathcal{D} ) d\beta 
\end{equation} 
is used to classify. Given current technology,
in complex environments  
we are sometimes only able to approximate the posterior mode $\hat {\beta}$,
then using $p_C(y | \hat {\beta} , x )$.
\end{itemize}
In any case, and whatever the learning approach adopted, we shall use the notation $p_C (y |x)$. 

The second stage is {\em operational}.
The agent makes class assignment decisions based on $p_C (y |x)$. 
This may be formulated through the influence diagram (ID)
\parencite{evaluatingInfluenceDiagrams1986} in Figure \ref{fig:classification}. {In such a diagram, square nodes describe decisions; circle nodes, uncertainties; hexagonal nodes refer to the associated utilities. Arcs pointing to decision nodes are dashed and represent information available when the corresponding decisions are made. Arcs pointing to chance and value nodes suggest conditional dependence.}
\begin{figure}[H]
\centering
\includegraphics[scale=1]{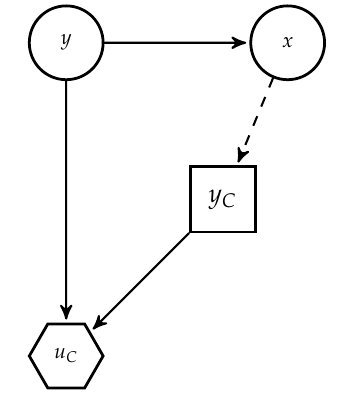}
\caption{Classification as an influence diagram.} \label{fig:classification}
\end{figure}
 The classifier guess $y_c$ for an observed instance $x$ provides her with a utility $u_C (y_c, y_i)$ when the
 actual class is $y_i$ and she aims at maximising expected utility
 through  
\begin{equation}\label{amarinha}
 \arg\max_{y_c} \sum_{i=1}^2
u_C (y_c , y_i ) p _C (y_i | x ).
\end{equation}
An important example of utility functions is the 0-1 utility: $u_C (y_c, y_i) = \mathbb{I}(y_c = y_i)$, where $\mathbb{I}$ is the indicator function. This leads to deciding based on maximising the
 predictive probability of correct~classification
\begin{equation}\label{seen}
\arg\max_{y_c} \sum_{i=1}^2 \mathbb{I}(y_c = y_i) p_C(y_i|x) = \arg\max_{y_c} p_C(y_c|x).
\end{equation}
In general, the 
utility is characterised as
a $2 \times 2$ matrix, whose $ij$-th entry 
represents the utility that the classifier
receives for classifying an instance of true label $j$ as of being of class $i$. This essentially balances the relative importance of false positives and negatives. Consider the case of autonomous driving systems. 
In this case, proper classification of identified objects is of
major 
importance. Misclassification errors increase the 
likelihood of incorrect forecasts of an object’s behaviour and an inaccurate assessment of the environment. To minimise the risk of accidents, the system should be very sensitive and react safely when there is uncertainty about the situation. Regrettably, such systems are prone to false-positive emergency identifications that lead to unnecessary reactions. 


Recall, anyway, that there are classification techniques, such as those based on SVMs, that rather than breaking classification in training and operational stages, directly learn a function that maps features $x$ into labels $y$. Should we like to apply the methodologies described below to this type of classifiers, we would need to produce estimates of $p_C(y|x)$ using their outputs. This can be achieved using calibration techniques such as 
\parencite{platt1999probabilistic} scaling.

\subsection{Attacks to Binary Classification Algorithms} \label{sec:att_class}

Consider now another agent called adversary ($A$, he).
He aims at fooling $C$ and make her err in classifying instances to gain some benefit. 
$A$ applies an attack $a$ to the features $x$ leading to $x'=a(x)$, the actual observation received by $C$, which does 
not observe the originating instance $x$.
For notational convenience, we sometimes write $x=a^{-1} (x')$.
Upon observing $x'$, $C$ needs to determine the instance  class. 
As we next illustrate, an adversary unaware classifier
may incur in gross mistakes if she classifies based on features $x'$, instead of the original ones.

\paragraph{Example.} Attacks to spam detection systems will be illustrated with experiments carried out with the UCI Spam Data Set \parencite{spambase1999}.
This set contains data from 4601 emails, out of which $39.4 \%$ are spam. For classification purposes, we represent each email through 54 binary variables indicating the presence (1) or absence (0) of 54 designated words in a dictionary.

Table \ref{tab:cleanVSattack} presents the performance
of four standard classifiers ({\em  Naive Bayes, logistic regression, neural net} and {\em  random forest}) based on a 0--1 utility function, against tainted and untainted data. The neural network model is a two layer one. The logistic regression is applied with L1 regularisation; {this  is equivalent to performing maximum a posteriori estimation in a logistic regression model with a Laplace prior \parencite{park2008bayesian}}. 
Means and standard deviations of accuracies are estimated via repeated hold-out validation over ten repetitions~\parencite{kim2009estimating}.
%

\begin{table}[H]
\caption{Accuracy comparison (with precision) 
 	of four classifiers on clean
 	(untainted) and attacked (tainted) data.}
	\centering
	\begin{tabular}{ccccc}
		\toprule
		\textbf{Classifier} & \textbf{ Acc. Unt.} & \textbf{Acc. Taint.}  \\
		\midrule
		Naive Bayes   & $0.891 \pm 0.003$ & $0.774\pm 0.026$  \\
		Logistic Reg.  & $0.928 \pm 0.004$ & $0.681 \pm 0.009$    \\  
		Neural Net & $0.905 \pm 0.003$ &      $0.764 \pm 0.007$       \\
		Random Forest  & $0.946 \pm 0.002$  &    $0.663 \pm 0.006$       \\
		\bottomrule
	\end{tabular}%
	
	\begin{tabular}{@{}c@{}} 
\multicolumn{1}{p{\textwidth -.88in}}{\footnotesize Observe the important loss in accuracy of the four classifiers, showcasing a major degradation in performance of adversary unaware classifiers when facing attacks.}
\end{tabular}
	\label{tab:cleanVSattack}%
\end{table}


\section{Adversarial Classification: Game-Theoretic Approaches}\label{sec:ac_ac}

As exemplified, an adversary unaware classifier may be fooled into issuing wrong classifications leading to severe performance deterioration. Strategies to mitigate this problem are thus needed. These may be based on 
building models of the attacks likely to be undertaken by the adversaries and enhancing classification algorithms to be robust against such attacks.

For this, the ID describing the classification problem (Figure \ref{fig:classification}) is augmented to incorporate adversarial decisions, leading to a biagent influence diagram (BAID) \parencite{Banks},
 Figure \ref{fig:jointProblem}. { In it, grey nodes refer to elements solely affecting $A$'s decision; white nodes to issues solely pertaining to $C$'s decision; striped nodes affect both agents' 
decisions}. 
We only describe the new elements. 
First, the adversary decision is represented through node $a$ (the chosen attack). The impact of the data transformation over $x$ implemented by $A$ is described through node $x'$, the data actually observed by the classifier; the 
corresponding node is deterministic 
(double circle) as we assume deterministic attacks.
Finally, the utility of $A$ is represented with node $u_A$,
with form  $u_A(y_c, y)$, when $C$ says $y_c$ and the actual label is $y$.
 We assume that attack implementation has negligible
costs.
As before, $C$ aims at maximising her 
expected utility; $A$ also aims at maximising his expected utility 
trying to confuse the classifier (and, consequently, reducing her 
expected utility). 
\begin{figure}[H]
\centering
\includegraphics[scale=1]{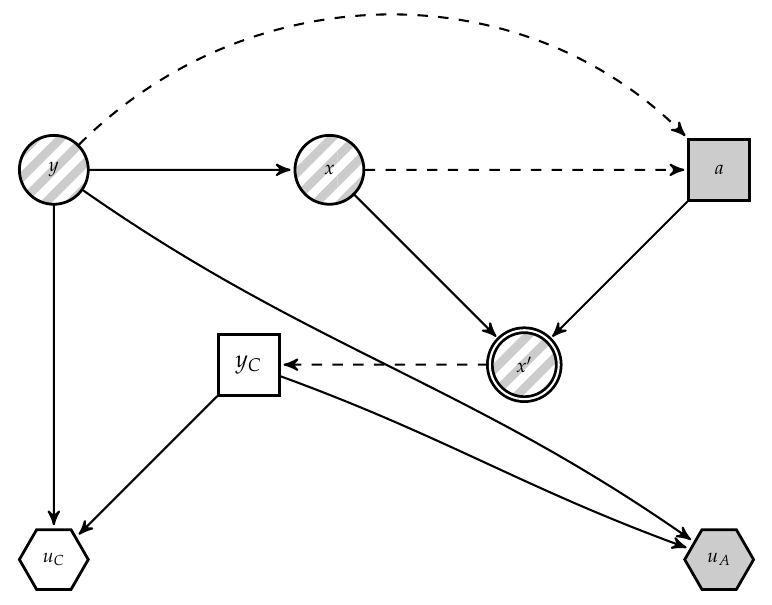}
\caption{Adversarial classification as a biagent influence diagram.}\label{fig:jointProblem}
\end{figure}

\subsection{Adversarial Classification. The Pioneering Model.
}
\textcite{dalvi2004adversarial} provided a pioneering approach to enhance classification algorithms when an adversary is present, calling it adversarial classification (AC).
 Because of its importance,
we briefly review it, using our notation. 
The authors view the problem as a game between a classifier $C$ and an adversary $A$,
using the following forward myopic proposal.

\begin{enumerate}
\item \textit{C} first assumes that data is untainted and computes her optimal classifier through (\ref{amarinha}).
\textcite{dalvi2004adversarial} focuses on  a utility sensitive Naive Bayes algorithm \parencite{elkan2001foundations}.
\item Then, assuming that \textit{A} has complete information about the classifier's elements (a common knowledge
assumption)
and that $C$ is not aware of his presence, the authors 
compute $A$'s optimal attack. 
To that end, they propose solving the integer programming  problem
\begin{align} \label{dalviA}
    \begin{split}
        \min & \left \lbrace \sum_{X_i \in \mathcal{X}_C} \sum_{x'_i \in \mathcal{X}_i} \mathcal{C}(x_i, x'_i) \delta _{i, x'_i} \right \rbrace\qquad {\rm s.t.} \\
        & \sum_{X_i \in \mathcal{X}_C} \sum_{x'_i \in \mathcal{X}_i} \left ( \log \frac{p_C(x_i|y_1)}{p_C(x_i|y_2)} - \log \frac{p_C(x'_i|y_1)}{p_C(x'_i|y_2)} \right ) \delta _{i, x'_i} \geq gap(x) .
    \end{split}
\end{align}
where $\mathcal{X}_C$ is the set of features the classifier uses for making her decision and $X_i$, the $i$-th feature, with original value $x_i \in \mathcal{X}_i$, assumed to be discrete. 
 $\delta_{i, x'_i}$ is a binary variable adopting value 1 when feature $X_i$ is changed from $x_i$ to $x'_i$, being $\mathcal{C}(x_i, x'_i)$ the cost of such change. 
It is easy to see that problem \eqref{dalviA} reflects the fact that the adversary tries to minimise the cost of modifying an instance, provided that such modification induces a change in the classification decision.

%
%

\item Subsequently, the classifier, assuming that $A$ implements the previous attack (again a common knowledge assumption) and that the training data is untainted, deploys her optimal classifier against it:
she chooses $y_c$ maximising $\sum_{i=1}^2 u_C(y_c, y_i) p_C(y_i |x')$, her posterior expected utility given that she observes the possibly modified instance $x'$. This is equivalent to optimising 
\begin{eqnarray}\label{dalviCK}
u_C (y_c, y_1) p_C(x' |y_1) p_C(y_1) + u_C (y_c, y_2) p_C(x' |y_2) p_C(y_2).
\end{eqnarray}
Estimating all these elements is straightforward, except for $p_C(x' \vert y_1)$. Again, appealing to a common knowledge assumption, the authors assume that the classifier, who knows all the adversary's elements, can solve 
\eqref{dalviA} and compute $x' = a(x)$, for each $x$ the adversary may receive.~Thus
\begin{eqnarray*} 
p_C(x' |y_1) = \sum_{x \in \mathcal{X}'} p_C (x \vert y_1) p_C (x' \vert x, y_1)
\end{eqnarray*}
where $\mathcal{X}'$ is the set of possible instances leading to the observed one and $p_C(x' \vert x, y_1) = 1$ if $a(x) = x'$ and 0 otherwise.
%
%
\end{enumerate}
The procedure could continue for more stages.
However, \textcite{dalvi2004adversarial} considers sufficient to use these three.

As presented (and the authors actually stress this
in their paper), very strong common knowledge assumptions are made: all parameters of both players are known to each other. Although standard in game theory, such  assumption is unrealistic in the security scenarios  
typical of AC.


\subsection{Other Adversarial Classification Game-Theoretic Developments}\label{sec:other_AC}

In spite of this, stemming from 
\textcite{dalvi2004adversarial}, AC has been predated by game-theoretic approaches, as reviewed in \textcite{biggio2014security} or \textcite{li2014feature}.
Subsequent attempts have focused on analysing attacks over classification algorithms and assessing their robustness against them, 
under various assumptions about the adversary. Regarding attacks, these have been classified as {\em white box}, when
the adversary knows every aspect of the defender's system
like the data used, the algorithms and the entire feature space; 
{\em black box}, that assume limited capabilities for the adversary, e.g., he is able to send membership queries to the classification system as in
\textcite{adversarialLearning2005}; and, finally,
{\em gray box}, which are in between the previous ones, as in \textcite{zhou2012adversarial}
where the adversary, who has no knowledge about the data and the algorithm used
seeks to push his malicious instances
as innocent 
ones, thus assuming that he is able to estimate such instances and has knowledge about the feature space. 

Of special importance in the AC field, mainly 
within the deep learning community, are the so called {\em adversarial 
examples} \textcite{goodfellow2014explaining} which may be formulated in game-theoretic terms as 
optimal attacks to a deployed classifier, requiring, in principle,
precise knowledge about the model used by the classifier.
To create such examples, $A$ finds the best attack which 
leads to perturbed data instances obtained from solving
problem   
$
\min_{\| \delta \| \leq \epsilon} \widehat{c}_A(h_{\theta} (a(x)), y),
$
 with $a(x) = x + \delta$, a  perturbation of the original data instance $x$; 
$h_{\theta} (x)$, the output of a predictive model with parameters $\theta$;
and
$\widehat{c}_A(h_{\theta} (x), y)$ the adversary's cost when instance $x$ of class $y$ is classified as of being of  class $h_\theta (x)$. This cost is usually taken to be $-\widehat{c}_D(h_{\theta} (x), y)$, where $c_D$ is the defender's cost.
The Fast Gradient Signed Method  (FGSM, \textcite{goodfellow2014explaining}) and related attacks in the literature \parencite{vorobeichikantar} assume that the attacker
has precise knowledge of the underlying model and parameters of the involved classifier, 
  debatable in most security settings.

A few methods have been proposed to robustify classification algorithms in adversarial settings. Most of them have focused on application-specific domains, as \textcite{Kocz2009FeatureWF} on spam detection. \textcite{Vorobeychik:2014:ORC:2615731.2615811} study the impact of randomisation schemes over classifiers against adversarial attacks proposing an optimal randomisation scheme as best defence.
 To date, 
\emph{adversarial training} (AT) \parencite{madry2018towards}
is one of the most promising defence techniques:
 it trains the defender model using attacked samples,
solving the~problem
\begin{equation*}
    \min_{\theta} \mathbb{E}_{(x,y) \sim \mathcal{D}} \left[ \max_{\| \delta_x \| \leq \epsilon} \widehat{c}_D(h_{\theta} (a(x)), y) \right],
\end{equation*}
thus minimising the empirical risk of the model under worst case perturbations of the data $\mathcal{D}$.
AT can be formulated as a zero-sum game. The inner maximisation problem is solved through project gradient descent (PGD) with iterations 
$
x_{t+1} = \Pi_{B(x)} (x_t - \alpha \nabla_x  \widehat{c}_A(h_{\theta} (x_t), y)),
$
where $\Pi$ is a projection operator ensuring that the perturbed input falls within an acceptable boundary $B(x)$,  and $\alpha$ is an intensity hyperparameter referring to the attack strength. After $T$ PGD iterations, set $a(x) = x_T$ and optimise with respect to $\theta$.   Other attacks that use gradient information are deepfool \textcite{moosavi2016deepfool}, yet \textcite{madry2018towards} argue that the PGD attack is the strongest one based only on gradient information from the target model. However, there is evidence that it is not sufficient for full defence in neural models since it is possible to perform attacks using global optimisation routines, such as the \emph{one pixel attack} from \textcite{su2019one} or \textcite{gowal2018ibp}.


Other approaches have focused on improving the game theoretic model in \textcite{dalvi2004adversarial}. However, to our knowledge, none has been able to overcome the above mentioned unrealistic common knowledge assumptions, as may be seen in recent reviews by \textcite{BIGGIO2018317} and \textcite{doi:10.1002/widm.1259}, who point out the importance of this issue. As an example, \textcite{kantarciouglu2011classifier} use a Stackelberg game in which both players know each other's payoff functions. Only \textcite{grosshans2013bayesian} have attempted to relax common knowledge assumptions in adversarial regression settings, reformulating the corresponding problem as a Bayesian game.

\section{Adversarial Classification: Adversarial Risk Analysis Approaches}\label{sec:ac_acra}

Given the above mentioned issue, we provide ARA solutions to AC.
We focus first on modelling the adversary's problem in
the operation phase. We present the classification problem faced by $C$ as a Bayesian decision analysis problem in Figure \ref{fig:classifierProblem}, derived from Figure \ref{fig:jointProblem}. In it, $A$'s decision appears as random to the classifier, since she does not know how the adversary will attack
the data. {For notational convenience, when necessary we distinguish between random variables and realisations using upper and lower case letters, respectively; in particular, we denote by $X$ the random variable referring to the original instance (before the attack) and $X'$ that referring to the possibly attacked instance. 
$\hat{z}$ will indicate an estimate of $z$.}
\begin{figure}[H]
\centering
\includegraphics[scale=1]{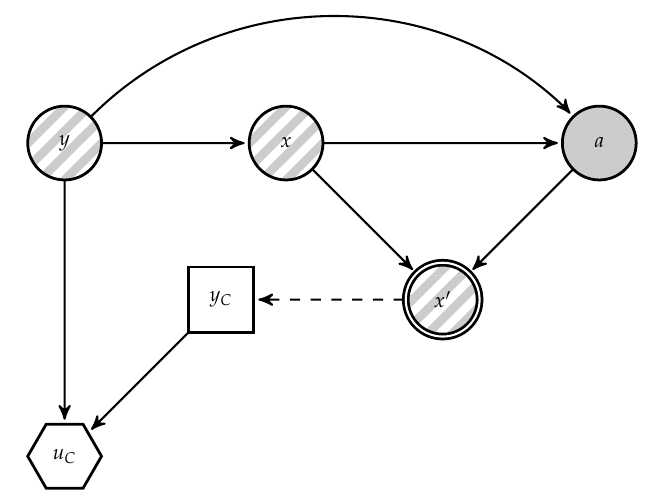}
\caption{Classifier problem.}
\label{fig:classifierProblem}
\end{figure}

An adversary unaware classifier would classify the observed instance $x'$ based on
\begin{equation}\label{valdokk}
\arg\max_{y_c} \sum_{i=1}^2  u_C (y_c , y_i ) p_C (y_i | X = x ').
\end{equation}
This leads to performance degradation as reflected in Section \ref{sec:att_class}.
In contrast, an adversary aware classifier would use
\begin{equation}\label{valdosurf}
\arg\max_{y_c} \sum_{i=1}^2  u_C (y_c , y_i ) p_C (y_i | X' = x ').
\end{equation}
In the adversary unaware case, $p_C (y_i | X = x ')$ is easily estimated using the training data. However, estimating 
the probabilities    $p_C (y_i | X' = x ')$ is harder, as it entails modelling how the adversary will modify the original instance $x$, which is unobserved. Moreover,  recall that common knowledge is not available, so we actually lack $A$'s beliefs and preferences. ARA helps us in modelling our uncertainty about them. In doing so, 
 robustness is typically enhanced.
We discuss two strategies depending on whether we use generative or discriminative classifiers as base models. 



\subsection{The Case of Generative Classifiers}

Suppose first that a generative classifier is required. As training data is clean by assumption, we can estimate $p_C (y)$ (modelling the classifier's beliefs about the {class} distribution) and $p_C (X=x|y)$ (modelling her beliefs about the feature distribution given the {class} when $A$ is not present). In addition, assume that when $C$ observes $X' = x'$, she can estimate the set $\mathcal{X}'$ of original instances $x$ potentially leading to the observed $x'$.  As
later discussed,
in most applications this will typically be a very large set.
When the feature space is endowed with a metric $d$,
an approach to approximate ${\cal X}'$ would be to consider
$
{\cal X}'= \{ x : d(x,x')<\rho \}
$ 
 for a certain threshold $\rho$. We will now survey the approach presented in \textcite{naveiro2018adversarial}.



Given the above, when observing $x'$ the classifier should choose the class with maximum posterior expected utility (\ref{valdosurf}).
%
%
Applying Bayes formula, and ignoring the denominator, 
which is irrelevant for optimisation purposes, 
she must find the class 
\begin{eqnarray}\label{fene}
   y^*_c(x') &=& \argmax_{y_c} \sum_{i=1}^2 u_C(y_c, y) p_C (y_i) p_C (X'=x'|y_i) \nonumber \\
    &=& \argmax_{y_c} \sum_{i=1}^2  u_C(y_c, y_i) p_C (y_i) 
    \left[ \sum_{x \in \mathcal{X}'} p_C (X' = x'\vert X=x, y_i) p_C (X=x \vert y_i)\right] .   
\end{eqnarray}
In such a way, $A$'s modifications are taken into account through the probabilities $p_C(X' = x'\vert X=x, y)$.  
At this point, recall that the focus is restricted to integrity violation attacks.
Then, $p_C(X' = x'\vert X=x, y_2) = \delta(x'-x)$ and problem
 (\ref{fene}) becomes 
\begin{eqnarray} \label{pis}
   &\argmax_{y_c} \bigg[ u_C(y_c, y_1) p_C(y_1)  \sum_{x \in \mathcal{X}'} p_C(X' = x'\vert X=x, y_1) p_C(X=x \vert y_1) \nonumber \\
   & + u_C(y_c, y_2) p_C(y_2) p_C(X=x' \vert y_2)\bigg].
\end{eqnarray} 
%
%

Note that should we assume full common knowledge, we would know $A$'s beliefs and preferences and, therefore, we would be able to 
solve his problem exactly: when $A$ receives an instance $x$ from class $y_1$, we could compute the transformed instance. In this case, $p_C(X'|X=x,y_1)$ would be 1 just for the $x$ whose transformed instance coincides with
that observed by the classifier and 0, otherwise. Inserting this 
 in \eqref{pis}, we would recover Dalvi's formulation \eqref{dalviCK}.
However, common knowledge about beliefs and preferences does not hold. 
Thus, when solving $A$'s problem we have to take into account our uncertainty about his elements and, given that he receives an instance $x$ with label $y_1$, we will not be certain about the attacked output $x'$.
This will be reflected in our estimate $p_C(x' |x,y_1)$ which will not be 0 or 1 as in Dalvi's approach (stage 3). With this estimate, we would solve problem \eqref{pis}, summing $p_C(x|y_1)$ over all possible originating instances, with each element weighted by $p_C(x' |x,y_1)$.

To estimate these last distributions, we resort to $A$'s problem, assuming that this agent aims at modifying $x$ to maximise his expected utility by making $C$ classify malicious instances as innocent. The decision problem faced by $A$ is presented in Figure \ref{fig:adversaryProblem}, derived from Figure \ref{fig:jointProblem}. In it, $C$'s decision appears as an uncertainty to $A$. 
\begin{figure}[H]
\centering
\includegraphics[scale=1]{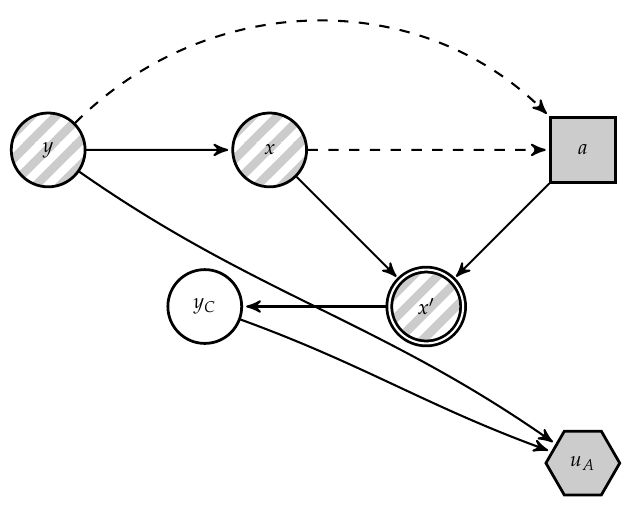}
\caption{Adversary problem.}\label{fig:adversaryProblem}\label{id}
\end{figure}
To solve the problem, we need $p_A(y^*_c(x')|x')$, which models $A$'s beliefs about $C$'s decision when she observes $x'$.  Let
$p$ be the probability $p_A( y_c^* (a (x)) = y_1 | a (x) )$
that $A$ concedes to $C$ saying that the instance is malicious when she observes $x' = a(x)$. 
Since $A$ will have uncertainty about it, let us model its density using $f_A (p | x' = a(x) )$ with expectation $p_{x' = a(x)}^A $. Then, upon observing an instance $x$ of class $y_1$, $A$ would choose the data transformation maximising his expected utility:

%

\begin{eqnarray}\label{att_prob_n}
x'(x,y_1) &\equiv& a^*(x, y_1) = \argmax_{z} \int \bigg[ u_A( y_c^*(z) = y_1,y_1) \cdot p \nonumber\\
&+& u_A(y_c^*(z) = y_2, y_1) \cdot (1-p)\, \bigg] f _A (p | z = a(x) ) \dd p \nonumber \\
&=& \argmax_{z}
\left[ u_A(y_1,y_1) - u_A(y_2, y_1) \right] p_{z = a(x)}^A   +  u_A(y_2, y_1),
\end{eqnarray}
where $u_A(y_i, y_j)$ is the attacker's utility when the defender classifies an instance of class $y_j$ as one of  class $y_i$. 


However, the classifier does not know the involved utilities $u_A$ and probabilities $ p_{z=a(x)} ^A $ from the adversary.
Let us model such uncertainty through a random utility function $U _A$ and a random expectation $ P_{z=a(x)}^A $. Then, we could solve for the random attack, optimising the random expected~utility
\begin{eqnarray*}
 X'(x,y_1) \equiv A^*(x, y_1) = \argmax_{z} \bigg(  \left[ U_A( y_1, y_1) - U_A( y_2, y_1) \right] P_{z=a(x)}^A   +    U_A( y_2, y_1) \bigg).
\end{eqnarray*}
We then use such distribution and make ({assuming that the set of attacks is discrete,
and similarly in the continuous case}) 
$p_C(x'|x,y_1) =Pr(X'(x,y_1) = x')$ which was the missing
ingredient in 
problem \eqref{pis}. Observe that it could be the case that $Pr(X'(x,y_1)=x)>0$,
i.e., the attacker does not modify the~instance.

%

Now, without loss of generality, we can 
associate utility $0$ with the worst consequence and $1$ with the best one, having the other consequences intermediate utilities. In $A$'s problem, his best consequence holds when the classifier accepts a malicious instance as innocent (he has opportunities to continue with his operations) while the worst consequence appears when the defender stops an instance (he has wasted effort in a lost opportunity),
 other consequences being intermediate. Therefore, we 
adopt $U_A(y_1, y_1) \sim \delta_{0}$ and $U_A(y_2, y_1) \sim \delta_{1}$. Then, the Attacker's random optimal attack would be
\begin{equation}\label{random_att_prob_n}
X'(x,y_1) \equiv A^*(x, y_1) = \argmax_{z} \bigg[\Big( 0 - 1 \Big) P_{z=a(x)}^A + 1 \bigg] 
     =  \argmin_{z}  P_{z=a(x)}^A .
\end{equation}
Modelling $P^A _{z=a(x)}$ is more delicate. 
It entails strategic thinking and 
could lead to a hierarchy of decision making problems, described in \textcite{rios2012adversarial} in a simpler context. 
A heuristic to assess it 
is based on using the probability $r=Pr _C (y_c^* (z) = y_1|z)$ that $C$ assigns to the instance received being malicious assuming that she observed $z$, with some uncertainty around it. 
As it is a probability, $r$ ranges in $[0,1]$ and we could make $P_{z=a(x)}^A \sim \beta e (\delta _1, \delta _2 )$, with mean $\delta _1 / (\delta _1 + \delta _ 2) = r$ and variance $ (\delta _1 \delta _2) / [(\delta _1 + \delta _2 )^2 (\delta _1 + \delta _2 + 1) ]=var $ as perceived. $var$ has to be tuned depending on the amount of knowledge $C$ has about $A$. Details on how to estimate $r$ are problem dependent.
%

In general, to approximate $p_C(x'|x,y_1)$ we use
Monte Carlo (MC) simulation drawing $K$ samples $\bigl(P_{z} ^{A,k} \bigr)$, $k = 1,\dots,K\, $
from  $P_{z} ^{A}$, finding
$ X'_k (x, y_1)=\argmin_{z}  P_{z} ^{A,k}
$ %
  and estimating $p_C(x'|x,y_1)$ using the proportion of times in which the result of the random optimal attack coincides with the instance actually observed by the defender:
\begin{eqnarray}
\widehat{p}_C ( x'\,|\,x, y_1) =  \frac {\# \{X'(x, y_1) =  x'\}} {K}. \label{mc_estimate}
\end{eqnarray}
It is easy to prove, using arguments in \textcite{10.5555/3172929}, that \eqref{mc_estimate} converges almost surely to $p_C (x'|x,y_1)$.
{In this, and other MC approximations considered, 
recall that the sample sizes  are essentially dictated by the required
    precision. 
    Based on the Central Limit Theorem \parencite{Chung},
    MC sums approximate integrals with probabilistic bounds
    of the order $\sqrt{\frac{var}{N}}$ where $N$ is the MC sum size.
    To obtain a variance estimate, we run a few iterations and estimate 
    the variance, then choose the  required size based on such bounds.}

Once we have 
an  approach to estimate the required probabilities, 
we implement the scheme described through
Algorithm \ref{alg:uno}, which reflects an initial training 
phase to estimate the classifier and an operational phase
which performs the above once a (possibly perturbated)
instance $x'$ is received by the classifier.

\begin{algorithm}[ht] %
\caption{General adversarial risk analysis (ARA) procedure for AC. Generative}  
\label{alg:uno}
\begin{algorithmic}[1]
\State {\bf Input:} Training data $\mathcal{D}$, test instance $x'$.
\State {\bf Output:} A classification decision $y_c^*(x')$.
\Train
\State Train a generative classifier to estimate $p_C(y)$ and $p_C(x|y)$
\EndTrain
\Operation
\State Read $x'$.
\State Estimate $p_C(x' |x,y_1)$ for all $x \in \mathcal{X}'$.
\State Solve
\begin{eqnarray*}
y_c^*(x') &=& \argmax_{y_C} \bigg[ u_C(y_C, y_1) \widehat{p}_C(y_1) \sum_{x\in \mathcal{X}'} \widehat{p}_C(x' |x,y_1) \widehat{p}_C(x|y_1) 
\\ &+& u_C(y_C, y_2) \widehat{p}_C(x'|y_2)\widehat{p}_C(y_2)\bigg].
\end{eqnarray*}
\State Output $y_c^*(x')$.
\EndOperation
\end{algorithmic}
\end{algorithm}

\section{Scalable Adversarial Classifiers}\label{sec:scalable}
The approach in Section \ref{sec:ac_acra} performs all relevant inference about the adversary during operations, and is only suitable for generative models, such as Naive Bayes. This could be too expensive computationally, especially in applications that require fast predictions
based on large scale deep models as motivated by the following image
processing problem.


\paragraph{Attacks to neural-based classifiers.}
Section \ref{sec:other_AC} discussed adversarial examples.  This kind of attack may harm intensely neural
network performance, such as that used in image classification tasks \parencite{szegedy2013intriguin}. It has been shown that simple one-pixel attacks can seriously affect performance \parencite{su2019one}.
As an example, we continue the discussion started in the introduction of this thesis, in Section \ref{sec:chall2}. Consider 
a relatively simple deep neural network (a multilayer perceptron model) \parencite{10.5555/3086952}, trained to predict the handwritten digits in the MNIST dataset \parencite{MNIST}. In particular, we used a 2 layer feed-forward neural network with relu activations and a final softmax layer to compute the predictions over the 10 classes. This requires simple extensions from binary to multi-class classification. This network accurately predicts 99\% of the digits. Figure \ref{fig:10samples} provides ten MNIST original 
samples (top row) and the corresponding images (bottom row) perturbed through FGSM,
which are misclassified. For example, the original 0 
(first column) is classified
as such; however, the perturbed one
is not classified as a 0 (more specifically, as an 8) even if it looks as such to the human~eye.


\begin{figure}[H]
\centering
  \includegraphics[scale=1.]{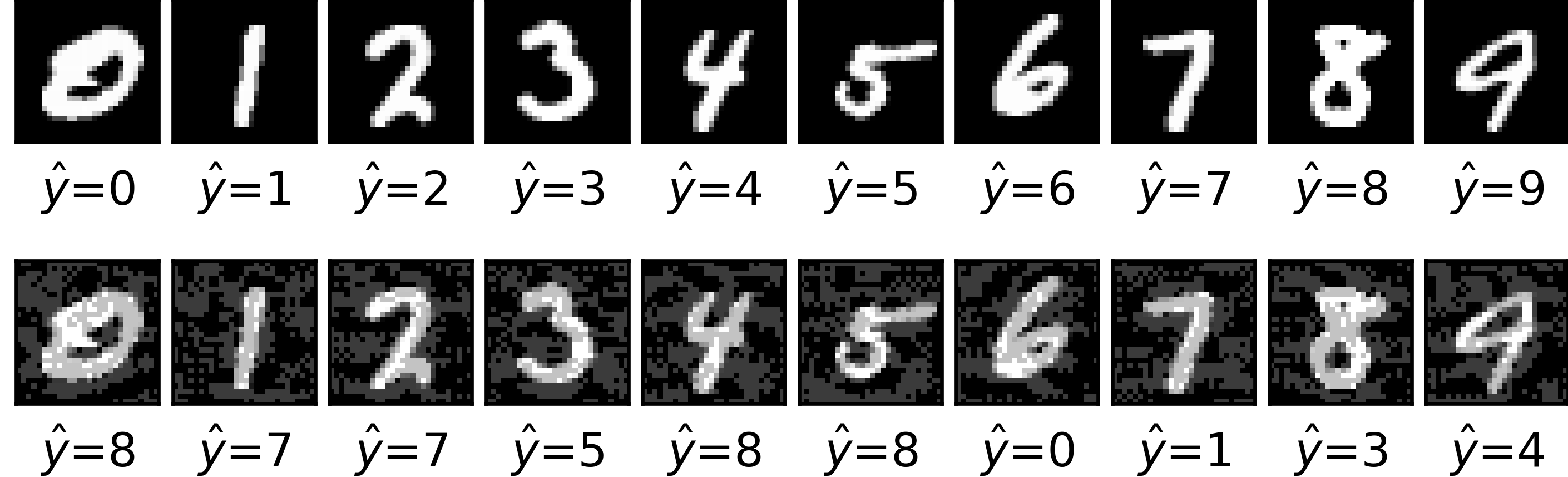}
  \caption{Ten MNIST examples (top) and their perturbations (bottom). Predicted class shown for each~example.}
  \label{fig:10samples}
\end{figure}
Globally, accuracy gets reduced to around 50\% (see Figure \ref{fig:comparison_mnist}a, curve NONE far right). This suggests a very important performance degradation due to the attack. Section \ref{sec:exp_scalabl} continues this example to evaluate the robustification procedure we introduce in Section  \ref{sec:uncs}. \\

The computational difficulties entailed by 
the approach in Section \ref{sec:ac_acra} stem from the following issue:
iterating over the set $\mathcal{X}'$ of potentially
originating instances.  
 If no assumptions about attacks are made, this set grows rapidly.
 For instance, in the spam detection example in Section \ref{sec:att_class}, 
 let $n$ be the number of words in the dictionary considered by
 $C$ to undertake the classification (54 in our spam case). If 
 we assume that the attacker modifies at most one word,
 the size of $\mathcal{X}'$ is $\mathcal{O}(n)$; if he modifies at most two words, it is $\mathcal{O}(n^2)$, $\dots$, and 
 if he modifies at most $n$ words, 
 the number of possible adversarial manipulations (and thus the size of $\mathcal{X}'$) is $2^n$.
 
   Even more extremely, in the 
 image processing example we would have to deal with
 very high-dimensional data (the MNIST images above consist of $28 \times 28$ pixels, each taking 256 possible values depending on the gray level). 
 This deems any enumeration over the set $\mathcal{X}'$ totally unfeasible. In order to tackle this issue,
 constraints over the attacks could be adopted, for example 
 based on distances.

 Therefore, as the key adversary modelling steps are taken during operations,
 the approach could be inefficient in applications requiring fast predictions.

 \subsection{Protecting Differentiable Classifiers}\label{sec:uncs}
 
Alternatively, learning about the adversary 
could be undertaken during the training phase as now 
presented. This provides faster predictions during operations and avoids the expensive step of sampling from $p_C(x \vert x')$.
A relatively weak assumption is made to achieve this: 
the model is probabilistic and can be differentiable in the $\beta$ parameters.
By this we understand classifiers with structural form 
$p_C(y|\beta,x)$ differentiable in $\beta$. 
A specially relevant form 
is 
\begin{eqnarray}\label{TGF}
p(y|\beta, x) = \mbox{softmax} (f_\beta (x))[y], & \text{where} & \mbox{softmax}(x)[j] = \frac{\exp{x_j}}{\sum_{i=1}^k \exp{x_i} }. 
\end{eqnarray}
This covers a large class of models. For example,
if $f_\beta$ is linear in inputs, 
we recover multinomial regression 
\parencite{mccullagh1989generalized}; if we take $f_\beta$ to be a sequence of linear transformations alternating non-linear activation functions, such as Rectified Linear Units (ReLU), we obtain a feed-forward neural network \parencite{10.5555/3086952}. 
These models have the benefit of being amenable to training using \emph{stochastic gradient descent} (SGD) \parencite{bottou2010large},
or any of its recent variants, such as Adam \parencite{kingma2014adam}, 
allowing for scaling to both wide and tall datasets. Then, from SGD we can obtain the posterior distribution of the model by adding a suitable noise term as in SG-MCMC samplers like stochastic gradient Langevin Dynamics (SGLD) \parencite{welling2011bayesian} or accelerated variants such as the ones introduced in chapter 2 of this thesis.

 We require  
 only sampling attacked instances from $p_C(x'|x)$, an attacker model. Depending on the type of data, this model can
 come through a discrete optimisation problem
 (as in the attacks of Section \ref{sec:att_class}) or a continuous optimisation problem
 (in which we typically resort to gradient information to obtain the most harmful perturbation 
 as in FGSM).
 These attacks require white-box access to the defender model, 
which, as mentioned, is usually unrealistic in security.
With continuous data, adversarial perturbations $x'$ 
are typically computed through the optimization problem
$$
x'= \arg\min_{x' \in B(x)} \log p(y|x',\beta),
$$ 
where $B(x)$ is some neighborhood of $x$ 
over which the attacker has influence on.
Exact solution of this problem is intractable in high-dimensional data.
Thus, attacks in the literature resort to approximations using gradient information. One of the most popular ones is FGSM,  
given by $x' = x -\epsilon\, \mbox{sign} \nabla_x \log p(y|x,\beta)$ where $\epsilon$ is a step size 
 reflecting attack intensity. Other attack examples 
include the Projected Gradient Descent (PGD) \parencite{madry2018towards} or the
\textcite{carlini2017towards} attack. 
These assume that the attacker has full knowledge of the target model,
which is unrealistic in many scenarios. 
AT using FGSM would correspond to sampling from a Dirac delta distribution centered at the FGSM update, 
that is, $p(x'|x) = \delta (x' - (x -\epsilon \mbox{ sign} \nabla_x \log p(y|x,\beta)))$. 

More realistically, based on ARA, we 
apportion two sources of uncertainty.



\paragraph{Defender uncertainty over the attacker model $p_C(x'|x)$.}

 The attacker modifies data in the operation phase.
 The defender has access only to training data $\mathcal{D}$; therefore, 
 she will have to simulate attacker's actions using such training set. Now, uncertainty can also come from the adversarial perturbation chosen. If the model is also differentiable wrt the input $x$ (as 
 with continuous data such as images or audio), instead of computing a single, optimal and deterministic perturbation, as in AT, we use SGLD to sample adversarial examples from regions of high adversarial loss, adding a noise term to generate uncertainty. Thus, we employ iterates of the form
\begin{equation}\label{eq:unc_attack}
x_{t+1} = x_t -\epsilon\, \mbox{sign} \nabla_x \log p_C(y|x_t,\beta) + \mathcal{N}(0, 2\epsilon),
\end{equation}
for $t=1,\ldots,T$, where $\epsilon$ is a step size. We can also consider uncertainty over the hyperparameters $\epsilon$ (say, from a rescaled Beta
distribution, since it is unreasonable to consider too high or too low learning rates) and the number $T$ of iterations (say, from a Poisson distribution). In addition, we can consider mixtures of different attacks, for instance by sampling a Bernoulli random variable and, then, choosing the gradient corresponding to either FGSM or another 
attack such as Carlini and Wagner's.

\paragraph{Attacker uncertainty over the defender model $p_C(y|x,\beta)$.}

It is reasonable to assume that the specific model architecture and parameters 
are unknown by the attacker. 
To reflect his uncertainty, he will instead perform attacks over
a model $p_C(y|x,\beta)$ with uncertainty  
over the values of the model parameters $\beta$, with continuous
support. 
This can be implemented through scalable Bayesian approaches
in deep models: the defended model is trained using SGLD, obtaining posterior samples via the iteration
\begin{equation}\label{eq:unc_model}
\beta_{t+1} = \beta_t + \eta \nabla_\beta (\log p(y|x, \beta) + \log p(\beta)) + \mathcal{N}(0, 2\eta I),
\end{equation}
with $\eta$ a learning rate, and $x$ sampled either from
the set $\mathcal{D}$ (untainted) or using an attacker model as in the previous point. We sample maintaining
a proportion 1:1 of clean and attacked data.  To see that this sampling method actually samples
from the  posterior $p(\beta| \mathcal{D})$,
where 
$\mathcal{D}$ designates a mixture of the clear and attacked dataset, just note that the Langevin SG-MCMC sampler for the posterior can be written as
$$
\beta_{t+1} = \beta_{t} + \eta \nabla \log p(\beta | \mathcal{D})  + \mathcal{N}(0, 2\eta I).
$$
Noting that $\log p(\beta | \mathcal{D}) = \log p(\mathcal{D} | \beta ) + \log p(\beta) - \log p(\mathcal{D})$ and taking gradients wrt to $\beta$ we have that $\nabla \log p(\beta | \mathcal{D}) = \nabla \log p(\mathcal{D} | \beta ) + \nabla \log p(\beta)$. An unbiased estimator of the gradient is obtained by replacing the whole dataset with just a single sample, or a minibatch, leading to the sampler in Eq. (\ref{eq:unc_model})
with $x \sim \mathcal{D}$. \\



The previous approaches incorporate some uncertainty about the attacker's elements to sample from $p_C(x' \vert x)$. A full ARA sampling from this distribution could be performed as well. 
Algorithm \ref{alg:large_ara} describes how to generate samples incorporating both types of uncertainty previously described. On the whole, it uses the first source to generate perturbations to robustly train the defender's model based on the second source.

\begin{algorithm}[!ht] %
\caption{Large scale ARA-robust training for AC}  
\label{alg:large_ara}
\begin{algorithmic}[1]
\State {\bf Input:} Defender model $p_C(y|x,\beta)$, attacker model $p_C(x'|x)$. 
\State {\bf Output:} A set of $k$ particles $\{\beta_i\}_{i=1}^{K}$  approximating the posterior distribution of the defender model learnt using ARA training.  
\For{$t=1$ to $T$}
\State Sample $x_1, \ldots,  x_K \sim p_C(x' | x)$ with 
(\ref{eq:unc_attack}).
\State $\beta_{i,t+1} = \beta_{i,t} + \eta \nabla_\beta (\log p_C(y|x, \beta) + \log p(\beta)) + \mathcal{N}(0, 2\eta )$ for each $i$ (SGLD)
\EndFor
\State Output $(\beta_{1,T},...,\beta_{K,T}))$
\end{algorithmic}
\end{algorithm}

Very importantly, its outcome tends to be more robust  than
the one we could achieve with just AT protection, since we incorporate some level of adversarial uncertainty. In the end, we collect $K$ posterior samples, $\lbrace \beta_k \rbrace_{k=1}^K$, and compute predictive probabilities for a new sample $x$ via marginalisation
through
$p_C(y|x) = \dfrac{1}{K} \sum_{k=1}^K p_C(y|x, \beta_k)$,
using the previous predictive probability to robustly classify the
received instance.


\section{Case Studies}
\label{sec:conEx}

We use the spam classification dataset from Section \ref{sec:att_class} 
and the MNIST dataset to illustrate the methods in Section \ref{sec:scalable}. As shown in Section \ref{sec:att_class}, simple attacks such as good/bad word insertions are sufficient to
critically affect the performance of spam detection algorithms. The small perturbations from Section \ref{sec:scalable} prove that unprotected image classification systems can easily be fooled by an adversary. 

\subsection{Robustified Classifiers in Spam Detection Problems}\label{sec:exp_scalable}

For the first batch of experiments, we use the same classifiers as in
 Section \ref{sec:scalable}. As a dataset, we use again the UCI Spam Data Set from that section. Once the models to be defended are trained, we perform attacks
over the instances in the test set, solving 
problem \eqref{random_att_prob_n} for each test spam email, removing the uncertainty that is not present from the adversary's point of view. 
We next evaluate the scalable approach in Section \ref{sec:scalable} under the two differentiable models
among the previous ones: logistic regression and neural network (with two hidden layers). 
As required, both models can be trained using SGD plus noise methods to obtain uncertainty estimates from the posterior as in (\ref{eq:unc_model}). 
Next, we attack the clean test set using the procedure 
in Section \ref{sec:att_class} and evaluate the performance of 
our robustification proposal. 
Since we are dealing with discrete attacks, we cannot use the uncertainty over attacks as in (\ref{eq:unc_attack}) and just resort to adding the attacker's uncertainty over the defender model as in (\ref{eq:unc_model}). To perform classification with this model, we evaluate the Bayesian predictive distribution using $5$ posterior samples obtained after SGLD iterations.
Results are in Table \ref{tab:rob_exps} which include as 
entries:
Acc.\ Unt. (accuracy over untainted data); 
Acc.\ Tai. (it.\ over tainted data);
Acc.\ Rob.\ Taint.\ (it.\ over tainted data after
our robustification adding uncertainties).
Note that the first two columns do not coincide with those
in Table \ref{tab:cleanVSattack}, as we have changed 
the optimisers to variants of SGD to be amenable to  
the robustified procedure in Section \ref{sec:uncs}.


\begin{table}[H]
\caption{Accuracy 
	of two classifiers on clean
	(untainted), and attacked (tainted) data, with and without robustification.}
	\centering
	\begin{tabular}{cccc}
		\toprule
		\textbf{Classifier }& \textbf{Acc. Unt.} & \textbf{Acc. Tai.}  & \textbf{Acc. Rob. Taint.} \\
		\midrule
		Logistic Reg.  & $0.931 \pm 0.007$    & $0.705 \pm 0.009$ &    $0.946 \pm 0.003$  \\  
		Neural Net  &      $0.937 \pm 0.005$     &     $0.636 \pm 0.009$    &     $0.960 \pm 0.002$  \\

		\bottomrule
	\end{tabular}%
	\label{tab:rob_exps}%
\end{table}%

Observe that the proposed robustification process
 indeed protects differentiable classifiers, 
 recovering from the degraded performance under attacked
 data. Moreover, in this example, the robustified classifiers
  achieve even higher accuracies than those attained by the original classifier over clean data. This is likely due to the fact that the presence of an adversary has a regularizing effect, being able to improve the original accuracy of the base algorithm and making it more robust. 

\subsection{Robustified Classifiers in Image Classification Problems}\label{sec:exp_scalabl}

Next, we show the performance of the scalable approach continuing with the digit recognition example from Section \ref{sec:scalable}. This batch of experiments aims to show that the ARA-inspired defence  can also scale to high-dimensional feature spaces and multiclass problems. The 
network architecture is shown in Section \ref{sec:scalable}. It is trained using SGD with momentum ($0.5$) for 5 epochs, with learning rate of $0.01$ and batch size of 32. The training set includes 50,000 digit images, and we report results over a 10,000 digit test set. As for uncertainties from Section \ref{sec:uncs}, we use both kinds.

\begin{figure}[!htb]
\begin{center}
\minipage{0.45\textwidth}
\hspace{-0.5em}
  \includegraphics[width=\textwidth]{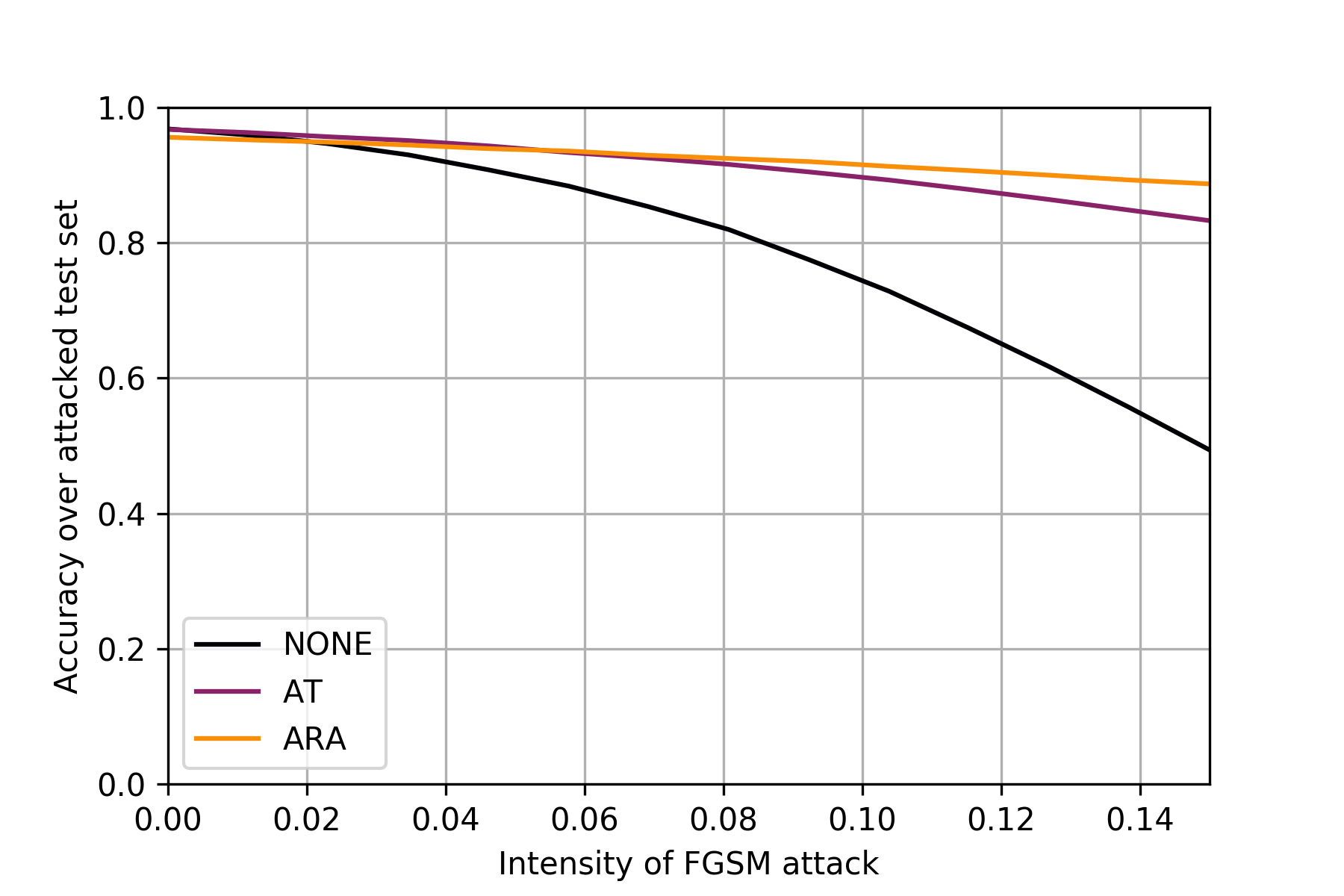}
   \caption{FGSM attack.}
\endminipage\hfill
\minipage{0.45\textwidth}
  \includegraphics[width=\textwidth]{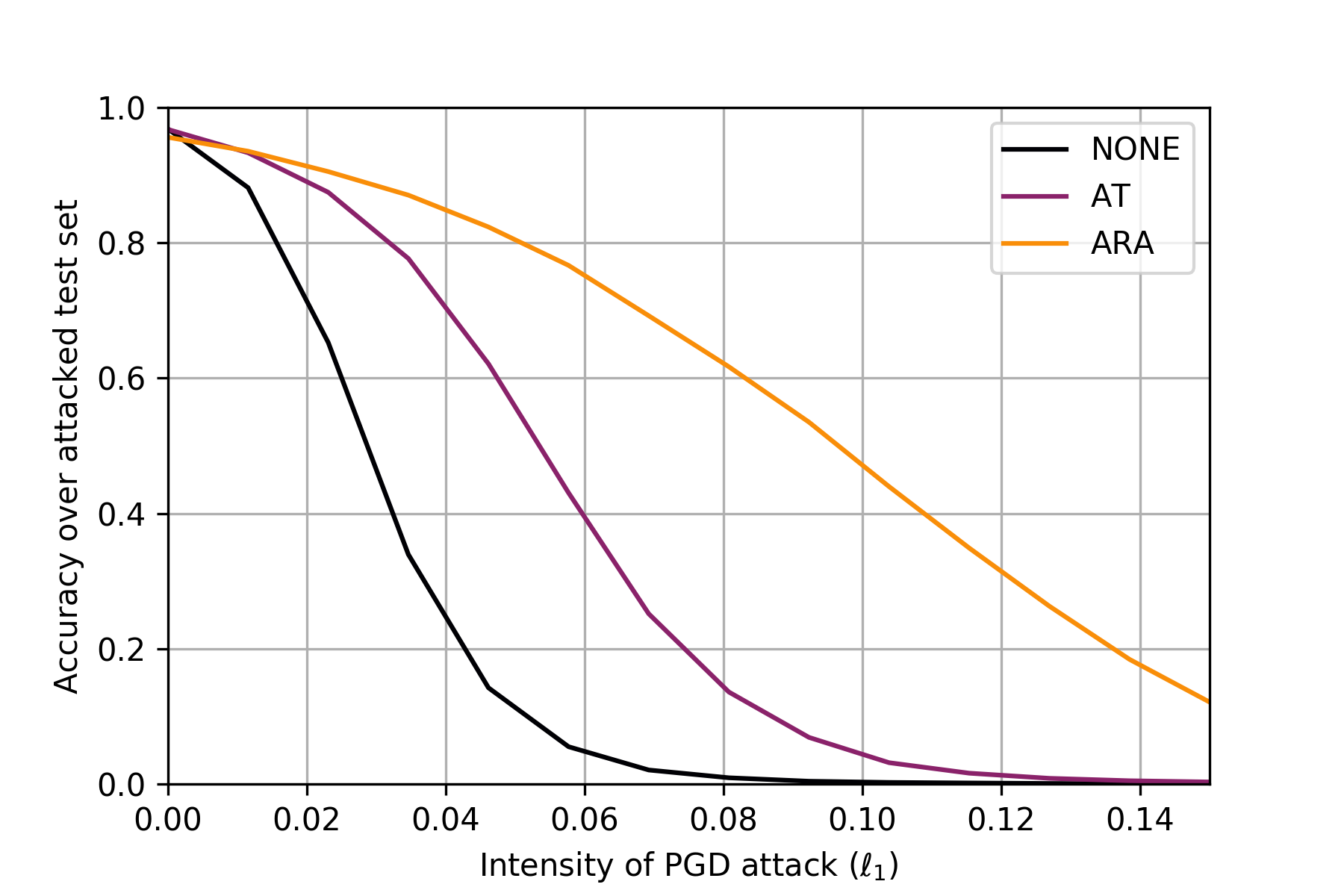}
   \caption{PGD attack under $\ell_1$ norm.}
\endminipage
\end{center}
\caption{Robustness of a deep network for MNIST under three defence mechanisms 
(none, AT, ARA). (a) depicts the security evaluation curves under the FGSM attack. (b) depicts the security evaluation curves under the PGD attack.}\label{fig:comparison_mnist}
\end{figure}

Figure \ref{fig:comparison_mnist} shows the \emph{security evaluation curves} \parencite{BIGGIO2018317} for three different defences (none, AT and ARA), using two attacks at test time: FGSM and PGD. Such curves depict the accuracy of the defender model at this task ($y$-axis), under different attack intensities $\alpha$ ($x$-axis). Note how the uncertainties provided by the ARA training method substantially improve the robustness of the neural network under both attacks. From the graphics, we observe that the ARA approach provides a greater degree of robustness as the attack intensities increase. This suggests that the proposed robustification framework can scale well to both tall and wide datasets and that the introduction of uncertainties by the ARA approach is highly beneficial to increase the robustness of the defended models.

\section{Summary}

In this chapter, we studied the important problem of developing defences that protect ML models against malicious and intentional attacks and increase their robustness. We have surveyed how game theoretical approaches provide a framework to develop defence mechanisms, however they are not realistic since they are pervaded by common knowledge assumptions. We adopted ideas from Adversarial Risk Analysis to model several sources of uncertainty that the defender model may face, and proposed an enhanced robustification method. Experiments in spam detection and image classification show the benefits of the introduced framework.

\chapter{Issues in Multi-agent Reinforcement Learning}\label{cha:ararl}
\section{Introduction}

This chapter presents a framework at the intersection of machine learning (reinforcement learning, RL, in particular) and game theory, with applications to security and the problem of data sharing. 
The first part introduces Threatened Markov Decision Processes (TMDPs) 
as  a framework to support an agent against potential opponents
in a RL context as well as schemes 
resulting in a novel learning approach to deal with TMDPs. Sections \ref{sec:background} to \ref{sec:exps_ararl} are dedicated to this new multiagent RL framework.
The second part formalizes the interactions between data producers and consumers as a dynamic game. We call it the data sharing game, leveraging the theory about the Iterated Prisoner's Dilemma and RL. Sections \ref{sec:ds} is dedicated to this issue.

\subsection{A motivation for adversarial RL}

The AML literature has predominantly focused on
the supervised setting \parencite{BIGGIO2018317}.
Our focus will be in reinforcement learning (RL) \parencite{sutton2012reinforcement}.
In it, an agent takes actions
sequentially to maximize a cumulative reward (utility), learning from
interactions with the environment \parencite{kaelbling1996reinforcement}. With the advent of deep learning, deep RL
has faced an incredible growth
\parencite{mnih2015human,silver2017mastering,chinorros}. 
However, such systems may be also targets of 
adversarial attacks \parencite{huang2017adversarial,lin2017tactics} 
and robust learning methods are thus
needed. 
A related field of interest is multi-agent RL \parencite{marl_over,leelee};
in it, multiple 
agents try to learn to compete or cooperate. Single-agent RL methods fail in
these settings, since they do not take into account 
the non-stationarity 
stemming from the actions of the other agents. 

The major contribution of this chapter is
to demonstrate how ARA facilitates dealing with 
secure RL by developing
a framework to model adversaries that interfere with the reward
generating processes. 
Unlike earlier work in multi-agent RL that focus on partial aspects
of learning, we present a general framework and provide extensive empirical evidence of its efficiency, flexibility and robustness.
In particular, we provide
various strategies to learn an opponent's policy, including a fictitious play approach, a level-$k$ thinking scheme and a model averaging procedure to update the most likely adversary. Moreover, we extend 
the approach to deep RL scenarios.
A variety of security scenarios serve to showcase the 
generality of our proposal, covering issues like robustness
against model misspecification, multiple opponents and extensions
to deep RL settings.



\section{Background in Reinforcement Learning}\label{sec:background}

Our focus in this chapter is on RL, widely studied as an efficient computational approach to deal with \textit{Markov decision processes} (MDP) \parencite{howard:dp}. These model a single agent (the decision maker, DM, she)
making decisions 
while interacting within an environment.  
They consist of a
tuple $\left( \mathcal{S}, \mathcal{A}, \mathcal{T}, R\right)$
where $\mathcal{S}$ is the state space with states $s$; $\mathcal{A}$,   the set of
actions $a$ available to the DM;
$\mathcal{T}: \mathcal{S} \times \mathcal{A} \rightarrow \Delta (\mathcal{S})$, the transition distribution, where $\Delta(X)$ denotes the set of
all distributions over a set $X$;
and, finally, $R : \mathcal{S} \times \mathcal{A}  \rightarrow \Delta(\mathbb{R}) $, the reward distribution modelling the utility that the agent perceives
from state $s$ and action $a$. 
The DM chooses her actions according to a policy $\pi: \mathcal{S} \rightarrow \Delta(\mathcal{A})$ 
with the aim of maximizing 
her long term discounted expected utility
$ 
 \mathbb{E}_{\tau} \left[  \sum_{t=0}^\infty \gamma^t R(a_t, s_t) \right] 
 $ 
where $\gamma \in (0,1)$ is a discount factor and $\tau = (s_0, a_0, s_1, a_1, \ldots)$ is a trajectory of states and actions.
  An efficient approach to solving MDPs is $Q$-learning \parencite{sutton2012reinforcement}; with it, the DM maintains a
 function
$Q : \mathcal{S} \times \mathcal{A} \rightarrow \mathbb{R}$ that estimates
her expected cumulative reward. This function is updated according to
\begin{align}\label{eq:ql}
Q(s,a) &:= (1 - \alpha) Q(s, a)  +  \alpha \left(r(s,a) + \gamma\max_{a'} Q(s', a') \right),
\end{align}
where $\alpha$ is a learning rate hyperparameter
and $s'$ designates the state 
after having chosen action $a$ in state $s$ and received the reward $r(s,a)$.

Our interest is in security settings in which other  agents
interfere with the DM's reward process. This renders the environment
non stationary making  $Q$-learning suboptimal \parencite{marl_over}. 
Thus, to support the agent,
we must be able to reason about and forecast the adversaries' behaviour.
Several opponent modelling principles have been proposed in the  literature,
as reviewed in \textcite{Albrecht2018AutonomousAM}. 
Since our concern is on the ability of our agent to predict 
her adversary's actions, we focus on methods
with this goal encompassing three approaches: policy reconstruction, type-based reasoning and recursive learning.

The first group of methods fully reconstruct the adversary's decision
making problem, typically assuming a parametric model fitted after 
after observing the adversary's behaviour. A dominant approach,
known as fictitious play  (FP) \parencite{brown1951iterative} models 
the other agents computing their frequencies of choosing various actions.
Secondly, 
type-based reasoning assume that the modelled agent belongs
to one of several fully specified types learning 
a distribution over such types, 
without explicitly including the ability of
other agents to reason about their opponents' decision making. 
Finally, explicit representations of the other agents' beliefs about their
opponents lead to an infinite hierarchy of decision making problems,
as illustrated in \textcite{rios2012adversarial} in a  simpler class of problems. Level-$k$ thinking  \parencite{stahl1994experimental} typically stops this potentially infinite regress at a level in which no more information is available, fixing the action prediction process at that depth with a non-informative probability distribution.

These modelling tools have been used in analytics 
research but, to the best
of our knowledge, their application to $Q$-learning in multi-agent settings remains 
largely unexplored. Relevant extensions 
have rather focused on modelling the whole system 
through Markov games. The three best-known solutions include minimax-$Q$ learning \parencite{littman1994markov}, where at each iteration a minimax 
problem is solved; Nash-$Q$ learning \parencite{hu2003nash}, 
which generalizes the previous algorithm to the non-zero sum case;
and friend-or-foe-$Q$ learning \parencite{littman2001friend}, in which the DM knows in advance whether her opponent is an adversary or a collaborator. Within the bandit literature, \textcite{auer1995gambling} introduced a non-stationary setting in which the reward process is affected by an adversary. Bayesian estimation of Q-values was studied in \textcite{RePEc:wly:apsmbi:v:27:y:2011:i:2:p:151-163} using Monte Carlo samples, but they do not take into account the presence of other agents. Our approach departs fundamentally from the above
work in that we consider the problem from a single DM's point of view and 
explicitly model the opponent using several strategies.

Our work relates to 
\textcite{lanctot2017unified}. Its authors propose a deep
cognitive hierarchy as an approximation to a response oracle 
to obtain policies that can exploit their adversaries. Ours
instead draws on the
level-$k$ thinking tradition, building opponent models that help the DM predict 
their behaviour. In addition, we provide a solution to 
the important issue
of choosing between different opponent cognition levels. 
\textcite{he2016opponent} also addressed opponent modelling 
in deep RL scenarios. However, the authors rely on using a particular neural network (NN) architecture over the deep $Q$-network model. Instead, we tackle the problem without assuming a NN architecture for the agent policies, though our proposed scheme can be adapted to that setting, Section 3.4. \textcite{foerster2018learning}
adopt a similar experimental setting, but their methods apply only
when both players get to know exactly their opponents' policy parameters or a maximum-likelihood estimator, a form of common knowledge. Our work
instead builds upon estimating the opponents' $Q$-function, not requiring direct knowledge of the opponent's internal parameters.

There are also links with iPOMDPs as in \textcite{gmytrasiewicz2005framework}, though the authors only address the planning problem, 
not addressing the case in which the DM also has to learn from the environment, corresponding to the RL scenario. In addition, we introduce a more simplified algorithmic apparatus (both in terms of implementation and complexity)
that performs well in examples and cover the case of mixtures of opponent types.

 \textcite{caballero2021identifying} takes grounding on level-$k$ thinking to compute optimal strategies in normal form games, 
 relying on solving exactly stochastic optimization programs. 
 As our framework also tackles MDPs, their approach would be infeasible in this case due to the large number of states. Therefore, we resort to a modified Q-learning scheme.

\textcite{pinto2017robust} proposes a method for deep RL based on policy gradients. However, it only works for zero-sum games. In addition, their experiments focuses on single-agent RL settings, that do not account for an adversary. 
 \textcite{wen2019probabilistic} uses level-$k$ thinking in deep RL, yet they only go up to level 2 and their framework 
do not extend easily to levels, unlike ours which 
is generic. Even at the same cognitive level, their framework is much more expensive computationally, since they rely on 
variational inference to estimate several latent variables. 

In summary, all of the proposed multi-agent $Q$-learning extensions are inspired by game theory
with its entailed common knowledge assumptions \parencite{gameTheoryACriticalIntroduction2004}, which are not realistic in the
security domains of interest to us. To mitigate this assumption, we consider the problem of
prescribing decisions to a single agent versus her opponents,
augmenting the MDP to account for potential adversaries  conveniently 
re-defining the $Q$-learning rule. This enables us to apply some of the previously mentioned modelling techniques
to the $Q$-learning setting, explicitly accounting
for the possible lack of information about the modelled opponent.
In particular, we propose to extend $Q$-learning from an ARA \parencite{adversarialRiskAnalysis2009} perspective. 

Our focus is on the case of a DM (agent $A$, she) facing a single opponent ($B$, he), though we provide an extension to a setting with multiple adversaries in Section \ref{sec:mul}.

\section{Threatened Markov Decision Processes}\label{sec:tmdps}

We propose an augmentation of a MDP
to account for the presence of adversaries who perform their
actions modifying state and reward dynamics, thus making the environment non-stationary.
\begin{definition}
A \emph{Threatened Markov Decision Process} (TMDP) is a tuple
$\left( \mathcal{S}, \mathcal{A}, \mathcal{B}, \mathcal{T}, R, p_A \right) $
in which $\mathcal{S}$ is the state space; $\mathcal{A}$, the set of
actions $a$ available to the supported agent $A$; $\mathcal{B}$, the set of
actions $b$ 
available to the adversary $B$, or threat actions; 
$\mathcal{T}: \mathcal{S} \times \mathcal{A} 
\times \mathcal{B} \rightarrow \Delta(\mathcal{S})$, the 
transition distribution; 
$R : \mathcal{S} \times \mathcal{A} \times \mathcal{B} \rightarrow
\Delta(\mathbb{R}) $, the reward distribution, the utility that 
the agent perceives from a given state $s$ and a pair
$(a,b)$ of actions; and $p_A (b | s)$ the
 distribution over the threats modeling 
 the DM's beliefs about
her opponent's move, for each state $s \in \mathcal{S}$.
\end{definition}
\noindent 
 To deal with TMDPs, we modify the standard $Q$-learning update rule (\ref{eq:ql}) by averaging over
 the likely actions of the adversary. This way the DM may anticipate potential threats within her
 decision making process and enhance the robustness of her decision making policy. Formally, replace (\ref{eq:ql})  by 
\begin{align}\label{eq:lr}
Q(s, a, b) := (1 - \alpha)Q(s, a, b) +  \alpha \left( r(s,a,b)
+ \gamma \max_{a'} \mathbb{E}_{p_A(b|s')} \left[ Q(s',a',b)  \right]  \right)
\end{align}
where $s'$ is the state reached after the DM and her adversary, respectively, adopt actions $a$ and $b$ when at state  $s$. Then, compute the 
expectation of $Q(s,a,b)$ taking into account the
uncertainty about  the opponent's action 
\begin{align}\label{eq:lr2}
Q(s,a) := \mathbb{E}_{p_A(b|s)} \left[ Q(s,a,b) \right], 
\end{align}
We use it to compute an $\epsilon-$greedy policy for the DM when the system is at state $s$: with probability $(1-\epsilon)$ choose action
$a^* = \argmax_a  \left[ Q(s,a)  \right] $; with probability $\epsilon$,
choose a random action uniformly. Appendix \ref{sec:p} proves the convergence
of the rule. Note that although in the experiments we focus on $\epsilon-$greedy strategies, other sampling methods can be straightforwardly used, such as a softmax policy to learn mixed strategies as described in Section 4.1.1.

Since common knowledge is not imposed,
we need to model the agent's uncertainty regarding the
adversary's policy through  $p_A (b | s)$.
 For this, we make the 
 assumption, standard in multi-agent RL, that both agents observe their opponent's actions and rewards
 after they have committed to them, proposing three approaches to forecast the opponent's policy $p_A (b | s)$:  
When the adversary is considered non-strategic
(he acts without awareness of the DM) we assess it based 
on FP; then, we provide a level-$k$ scheme;
finally, a method to combine different opponent models
is outlined, allowing us to deal with mixed behaviors.

\subsection{Non-strategic (level-0) opponents}\label{sec:non}

Consider first a stateless setting. In such case, 
the $Q$-function (\ref{eq:lr}) is written
  $Q(a_i,b_j)$, with $a_i \in \mathcal{A}$ the action chosen by
 the DM and $b_j \in \mathcal{B}$ the action chosen by the adversary, assuming that $\mathcal{A}$ and $\mathcal{B}$ are discrete action spaces.
 Then, the DM  computes the expected utility of action $a_i$ using the stateless version of \eqref{eq:lr2}
\[ \psi(a_i) = \mathbb{E}_{p_A(b)} [Q(a,b)] = \sum_{b_j \in \mathcal{B}} Q(a_i, b_j) p_{A}(b_j), \]
where $p_A (b_j)$ reflects $A$'s beliefs about her opponent's actions,
 and chooses the action $a_i \in \mathcal{A}$
maximizing $\psi(a_i)$.
 
 She needs to predict the action
 $b_j$ chosen by her opponent. A first option is to model her adversary
 using an approach inspired by FP, 
 estimating the probabilities with the empirical frequencies of the opponent past plays
 with $Q(a_i, b_j)$ updated according to the stateless version of Eq. (\eqref{eq:lr}).
We refer to this variant as
FP$Q$-learning.\footnote{Observe that although inspired by FP,
it is not the same scheme, since only one of the players (the DM) uses it to 
model her opponent, whereas in the standard FP algorithm \parencite{brown1951iterative} all players 
are assumed to adopt such scheme.}
We can re-frame it from a Bayesian perspective,
favouring its convergence if we have available relevant  
prior information about the adversary. 
Let $p_j = p_A(b_j)$ be the probability with which
the opponent chooses action $b_j$. 
If we adopt a Dirichlet prior
$(p_1 , \ldots, p_n) \sim \mathcal{D}(\alpha_1,\ldots,\alpha_n)$,
where $n$ is the number of actions available to the opponent, 
the posterior will be  
$\mathcal{D}(\alpha_1 + h_1,\ldots,\alpha_n + h_n)$, with  $h_i$ 
being the count for action $b_i$, $i=1,...,n$. 
If we denote its density as $f(p|h)$, the DM would choose the action $a_i$ maximizing her expected utility,
adopting the form
\begin{eqnarray*}
& \psi(a_i)  = \int \left[\sum_{b_j \in \mathcal{B}}
Q(a_i, b_j) p_j\right] f(p|h) dp 
= \sum_{b_j \in \mathcal{B}} Q(a_i, b_j) \mathbb{E}_{p|h}[p_j]
\propto  \sum_{b_j \in \mathcal{B}} Q(a_i, b_j) (\alpha_j + h_j).
\end{eqnarray*}

Generalizing this approach to account for states is straightforward
conceptually.
The $Q$-function adopts now the form $Q(s, a_i, b_j)$. The DM needs to assess the probabilities $p_A(b_j | s)$,
since it is natural to expect that her opponent behaves differently depending
on the state.
However, as $\mathcal{S}$ may be huge, even continuous, keeping track of $p_A(b_j|s)$ 
incurs in potentially prohibitive memory costs. 
We mitigate the 
problem by using 
Bayes rule, $p_{A}(b_j| s) \,\, \propto \,\, p(s| b_j)p(b_j)$, 
and the supported DM will choose her action at state $s$ by maximizing
\[ \psi_s(a_i) = \sum_{b_j \in \mathcal{B}} Q(s, a_i, b_j) p_{A}(b_j|s)
\propto 
\sum_{b_j \in \mathcal{B}} Q(s, a_i, b_j) p_{A}(s | b_j) p (b_j ).  \]
  \textcite{tang2017exploration} suggested an efficient method for keeping track of $p(s| b_j)$ using a hash table or bloom filters to maintain a count of the number of times that an agent visits each state, within single-agent RL to support a better exploration of the environment. We propose keeping 
  track of $n$ bloom filters, one for each
distribution $p(s|b_j)$, for tractable computation of the opponent's intentions
in the TMDP setting. This is transparently integrated
within the Bayesian paradigm, as we only need to store an additional array with the Dirichlet prior parameters $\alpha_i$, $i=1,\ldots, n$ for the $p(b_j)$ part. 

As a final comment, if we assume that the opponent has memory of the previous stage actions, we could straightforwardly extend the above scheme. However,
to mitigate memory requirements we use the concept of
mixtures of Markov chains \parencite{raftery1985model},
thus avoiding an exponential growth in 
the number of required parameters and linearly controlling 
model complexity.
For example,
in case the opponent belief model
is $p_{A}(b_t | a_{t-1}, b_{t-1}, s_t)$, so that the adversary recalls 
 the previous actions $a_{t-1}$ and $b_{t-1}$, we 
 factor it through a mixture 
\[
p_{A}(b_t | a_{t-1}, b_{t-1}, s_t) = w_1 p_{A}(b_t | a_{t-1})  + w_2 p_{A}(b_t | b_{t-1})
 + w_3 p_{A}(b_t | s_t),
 \] 
 with $\sum_i w_i = 1, w_i \geq 0, i=1,2,3$.
 
\subsection{Level-$k$ opponents}\label{sec:k}
When the opponent is
 strategic, he may model our supported DM as a level-0 thinker, thus making
 him a level-1 thinker. This chain can go up to infinity, so we will
 have to deal with modelling the opponent as a level-$k$ thinker, with $k$
 bounded by the computational or cognitive resources of the DM.
For this, a hierarchy of TMDPs is introduced with  
 $\emph{TMDP}_{i}^k$ referring
 to the TMDP that agent $i$ (the DM or the 
 adversary) needs to optimize,
 while considering its rival as a level-$(k-1)$ thinker,
 so that: 

\begin{itemize}
\item If the supported DM is a level-1 thinker, she optimizes  $ \emph{TMDP}_{A}^1 $. She models $B$ as a level-0 thinker
(e.g. as in Section \ref{sec:non}).
\item If she is a level-2 thinker, the DM optimizes 
$ \emph{TMDP}_{A}^2 $ modelling $B$ as a level-1 thinker:
this ``modelled" $B$ optimizes $ \emph{TMDP}_{B}^1 $, and while doing so,
he models the DM as level-0.
\item In general, we have a chain of TMDPs:
$$ \emph{TMDP}_{A}^k \rightarrow \emph{TMDP}_{B}^{k-1}
\rightarrow \emph{TMDP}_{A}^{k-2}  \rightarrow \cdots $$
\end{itemize}
Exploiting the fact that TMDPs correspond to repeated interaction settings
(and, by assumption, both agents observe all past
decisions and rewards), each agent may  estimate their
counterpart's $Q$-function, $\hat{Q}^{k-1}$: 
if the DM is optimizing $\emph{TMDP}_A^k$, she will keep her own 
$Q$-function (call it $Q_k$), and also an estimate
$\hat{Q}_{k-1}$ of her opponent's $Q$-function. This estimate may be
computed optimizing $\emph{TMDP}_B^{k-1}$ and so on until $k=1$.
Finally, the top level DM's policy is 
\[
\argmax_{a_{i_k}} Q_k(s, a_{i_k}, b_{j_{k-1}}),
\]
where $b_{j_{k-1}}$ is given by 
$
\argmax_{b_{j_{k-1}}} \hat{Q}_{k-1}(s, a_{i_{k-2}}, b_{j_{k-1}}),
$ and so on, until the induction basis (level-1) is reached 
in which the opponent may be modelled using the FP$Q$-learning approach 
in Section \ref{sec:non}.
Algorithm \ref{alg:l2ur} specifies the approach 
for a level-2 DM. 
The algorithm thus accounts for $Q_2$,  her $Q$-function,  and $\hat{Q}_1$,
that of her
opponent (who will be level-1). Figure
\ref{fig:lev2_scheme} provides a schematic view of the dependencies.

\begin{algorithm*}[!ht]
\begin{algorithmic}[1]
\Require $Q_2$, $\hat{Q}_1$, $\alpha_2, \alpha_1$ (DM and opponent $Q$-functions
and learning rates, respectively).
\State Observe  transition elements $(s, a, b, r_A, r_B, s')$ 
\State $\hat{Q}_1(s,b,a) := (1 - \alpha_1)\hat{Q}_1(s,b,a)  + \alpha_1 (r_B + \gamma \max_{b'} \mathbb{E}_{p_B(a'|s')} \left[ \hat{Q}_1(s',b', a') \right] )$ 
\State Compute $B$'s estimated $\epsilon-$greedy policy $p_A(b|s')$ from $\hat{Q}_1(s,b,a)$
\State $Q_2(s,a,b) := (1 - \alpha_2)Q_2(s,a,b) + \alpha_2 (r_A + \gamma \max_{a'} \mathbb{E}_{p_A(b'|s')} \left[ Q_2(s',a',b') \right]) $ 
\end{algorithmic}
\caption{Level-2 thinking update rule}
\label{alg:l2ur}
\end{algorithm*}

\begin{figure}[!ht]
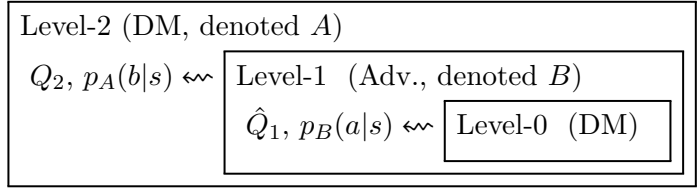

\centering
\stackinset{c}{.5in}{t}{.73in}{%
  \fboxrule=0pt\relax\framebox[2in][t]{%
  }}{\fboxrule=.75pt%
  \fbox{\stackunder{Level-2%
   \hspace*{\fill} (DM, denoted $A$) }%
    {
    $Q_2$,
    $p_A(b | s) \leftsquigarrow$
    \fbox{\stackunder{Level-1 \hspace*{\fill} (Adv., denoted $B$) }%
      {
      $\hat{Q}_1$,
    $p_B(a | s) \leftsquigarrow$
      \fbox{\stackunder{Level-0 \hspace*{\fill} (DM) }%
        {}
        }}}}}
}
\caption{Level-$k$ thinking scheme, with $k=2$}\label{fig:lev2_scheme}
\end{figure}

Note that in the previous hierarchy the decisions are obtained 
in a greedy manner, by maximizing the lower level
$\hat{Q}$ estimate. We may gain insight in a Bayesian fashion by adding uncertainty to the policy at each level. For instance, 
when considering $\epsilon-$greedy policies,
we could impose distributions $p_k(\epsilon)$ at each level $k$ of the hierarchy,
with the mean of $p_k(\epsilon)$ being an increasing function with respect to 
the level $k$ to account for the fact that uncertainty is higher
at the upper thinking levels. 

\subsection{Combining Opponent Models}\label{sec:com}

We have discussed hierarchies of opponent models. 
In most situations the DM will not know which type of particular opponent she is actually facing.
To deal with this, she may place a prior $p(M_i)$ denoting her beliefs that her opponent is using a model $M_i$, for $i = 1, \ldots, m$, the range of models that might describe her adversary's behavior,
with $\sum_{i=0}^m p(M_i) = 1$ and $p(M_i) > 0$,
$i=1,...,m$. 

As an example, she might place a Dirichlet prior on the levels of the $k-$level hierarchy. Then,
at each iteration, after having observed her opponent's action, she may
update her belief $p(M_i)$ by increasing the count $n_i$ of model $M_i$ which
caused that action, as in the standard Dirichlet-Categorical Bayesian
update rule (Algorithm \ref{alg:update_averaging}). 
\begin{algorithm*}[!ht]
\begin{algorithmic}[1]
\Require $p(M | H)\, \propto\, (n_1, n_2, \ldots, n_m)$ (counts for each model), with 
$H$ the sequence $(b_0, b_1, \ldots, b_{t-1})$ of past opponent actions.
\State Observe  transition $(s_{t}, a_t, b_t, r_{A,t}, r_{B,t}, s_{t+1})$ at iteration $t$. 
\State For each opponent model $M_i$, set  $b^i$ to be the predicted action by model $M_i$.
\State If $b^i = b_t$, then update posterior:
$$
p(M | (H || b_t) ) \, \propto\,  (n_1, \ldots, n_i + 1, \ldots, n_m) 
$$
\end{algorithmic}
\caption{Opponent average updating}
\label{alg:update_averaging}
\end{algorithm*}
\noindent This is possible since 
the DM maintains an estimate $p_{M_i}(b|s)$ of the opponent's policy for each
opponent model $M_i$. Should none of these have 
predicted the observed $b_t$, then we may not perform an 
update (as stated in Algorithm \ref{alg:update_averaging} and done
in our experiments) or we could increase the count for all possible models.

This model averaging scheme subsumes the framework of cognitive hierarchies \parencite{camerer2004cognitive}, though the distribution is placed over the different hierarchy levels. Our scheme is more flexible,
though, since more kinds of opponents could be taken into account, 
say a minimax agent as 
in \textcite{rios1}.

\subsection{Approximate Q-learning with function approximation}\label{sec:approx_rl}

The tabular $Q$-learning in Algorithm \ref{alg:l2ur} does not scale with the size of the state ${\cal S}$ or action spaces
${\cal A, B}$.
To solve this issue, we expand the framework to the case when  $Q$-functions are represented using a function approximator, typically a linear regression or
a deep learning $Q$-network \parencite{mnih2015human}. 

Algorithm \ref{alg:l2urdeep} shows the details, where $\phi_A$ and $\phi_B$ designate the parameters of the corresponding functions approximating the $Q$-values. Note that in this setting,
the framework is compatible with continuous action spaces.

\begin{algorithm*}[!ht]
\begin{algorithmic}[1]
\Require DM's $Q$-function, $Q_{\phi_A}$ and estimate of the opponent's $Q$-function, $\hat{Q}_{\phi_B}$, $\alpha_2, \alpha_1$ (learning rates).
\State Observe  transition $(s, a, b, r_A, r_B, s')$.
\State $\phi_B := \phi_B - \alpha_1 \frac{\partial \hat{Q}_{\phi_B}}{\partial \phi_B}(s, b,a)\left[ \hat{Q}_{\phi_B}(s,b,a) - (r_B + \gamma \max_{b'}\mathbb{E}_{p_B(a'|s')} \hat{Q}_{\phi_B} (s', b',a') ) \right]  $
\State Compute $B$'s estimated $\epsilon-$greedy (or softmax) policy $p_A(b|s')$ from $\hat{Q}_{\phi_B}(s,b,a)$.
\State $\phi_A := \phi_A - \alpha_2 \frac{\partial Q_{\phi_A}}{\partial \phi_A}(s, a,b) \left[ Q_{\phi_A}(s,a,b) - (r_A + \gamma \max_{a'}\mathbb{E}_{p_A(b'|s')} Q_{\phi_A} (s', a',b') ) \right]  $ 
\end{algorithmic}
\caption{Level-2 thinking update rule using function approximators.}
\label{alg:l2urdeep}
\end{algorithm*}

\subsection{Facing multiple opponents}\label{sec:mul}
TMDPs may be extended to the case of a DM facing more than one adversary. Then, the DM would have uncertainty about all of her opponents and she would need to average her $Q$-function over their
likely actions. Let $p_A(b^1, \dots, b^M \vert s)$ represent the DM's beliefs about her $M$ adversaries' actions. The extension of the TMDP framework to multiple adversaries will require to account for all possible opponents in the DM's $Q$ function which adopts now the form $Q(s,a,b^1,\dots,b^M)$. Finally, the DM would average this $Q$ over $b^1, \dots, b^M$ in \eqref{eq:lr}, proceeding as in \eqref{eq:lr2}.

When the DM is facing non-strategic opponents, she could learn $p_A(b^1, \dots,$ $ b^M \vert s)$ in a Bayesian way, as explained in Section \ref{sec:non}. This would entail placing a Dirichlet prior on the $n^M$ dimensional vector of joint actions of all adversaries. However, keeping track of those probabilities may be unfeasible as the dimension scales exponentially with the number of opponents. The case  of conditionally independent adversaries turns out to be much simpler as
$p_A(b^1, \dots, b^M \vert s) = p_A(b^1 \vert s) \dots p_A(b^M \vert s)$. In this case, we could learn each $p_A(b^i \vert s)$ for $i=1, \dots, M$ separately, as above combining the forecasts  
multiplicatively.

\subsection{Computational complexity}\label{sec:cc}

As outlined in Section \ref{sec:k}, a level-$k$ $Q$-learner has to estimate the $Q$ function
of a level-$(k-1)$ $Q$-learner, and so on. Assuming that the original $Q$-learning 
update rule has time complexity $\mathcal{O}(T(|\mathcal{A}|))$, with $T$ being a factor that depends on the number of actions of the DM, the update rule from 
Algorithm \ref{alg:l2ur} has time complexity $\mathcal{O}(kT(\max \lbrace |\mathcal{A}|,  |\mathcal{B}|\rbrace))$, i.e.,
it is linear in the level of the hierarchy. 
Regarding space complexity, the overhead is also linear in such level $k$ since the DM only needs to store $k$ $Q$-functions, leading to a memory complexity $\mathcal{O}(kM(|\mathcal{S}|,|\mathcal{A}| \cdot |\mathcal{B}|))$ with $M(|\mathcal{S}|,|\mathcal{A}| \cdot |\mathcal{B}|)$ accounting for the memory needed to store the $Q$-function in tabular form. This quantity depends on the number of states and pairs of actions for the DM and her opponent.

\section{Experiments}\label{sec:exps_ararl}

We illustrate key modelling and computational concepts 
about the TMDP reasoning framework with three
sets of experiments:
an adversarial security environment proposed in \parencite{leike2017ai}
used to illustrate robustness issues; a Blotto game security resource allocation problem
used to illustrate handling multiple opponents; and a 
gridworld game showing that our framework is compatible with $Q$-values estimated with parametric functions, as in deep RL. 
All the code and experimental setup details
are released at \url{https://github.com/vicgalle/ARAMARL}
 for reproducibility.

\subsection{The advantages of modelling adversaries}\label{s:sg}

We demonstrate first that  by
explicitly modelling the behaviour of adversaries, 
our framework improves upon $Q$-learning methods. 
For this, we consider a suite of recent RL safety benchmarks
 introduced in \parencite{leike2017ai}.
Our focus is 
on the safety \emph{friend or foe} environment:
the supported DM
needs to travel through a room and choose 
between two identical boxes, respectively hiding a positive
and a negative reward (+50 and -50, respectively), controlled by
an adaptive opponent.
This may be interpreted as a spatial Stackelberg game in which
the adversary is planning to attack one of two targets; the defender
will obtain a positive reward if she travels to the chosen target. Otherwise,
she will miss the attacker and incur in a loss.

As \parencite{leike2017ai} shows, a \emph{deep $Q$-network}
(and, similarly, the independent tabular $Q$-learner as we show) fails to
achieve optimal results since the reward process is controlled by the adversary.

Figure \ref{fig:friendorfoe} shows the initial set up. 
Cells 1 and 2 depict the adversary's targets, who decides which 
one will 
hide the positive reward.

\begin{figure}
   \centering
   \includegraphics[scale=0.5]{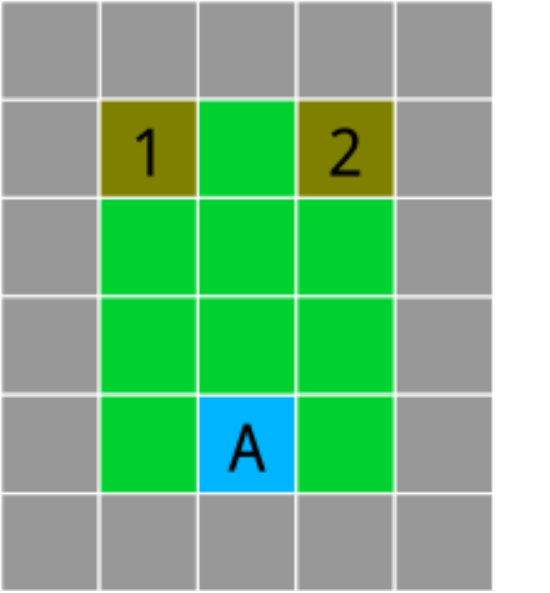}
   \caption{\emph{Friend or foe} environment from the AI Safety Gridworlds benchmark \parencite{leike2017ai}.
   The blue cell 
represents the DM's initial state, gray cells represent the walls of the room.   } \label{fig:friendorfoe}
 \end{figure}

\paragraph{Stateless Variant.}\label{sec:statv}

Consider a simplified initial setting with a single state and
two actions. As in  \parencite{leike2017ai}, the adaptive
opponent estimates the DM's actions using an exponential smoother.
Let $\bm{p} = (p_1, p_2)$ be the probabilities
with which the DM will, respectively, choose targets 1 or 2
as estimated by the adversary. 
Initial estimates 
are
$\bm{p} = (0.5, 0.5)$.
After each iteration, the update is 
$
\bm{p} := \beta \bm{p} + (1 - \beta ) \bm{a},
$
with $0 < \beta < 1$ 
a learning rate, unknown from the DM's point of view,
and $\bm{a} \in \lbrace (1, 0), (0, 1) \rbrace$ is a one-hot encoded
vector respectively indicating whether the DM  chose targets 1 or 2. 
Assume an adversary which places the positive reward in target
$t = \argmin_i (\bm{p})_i$.

Since the DM has to deal with a strategic adversary, 
consider a variant of FP$Q$-learning (section 4.3) giving
more relevance to more recent actions.
Algorithm \ref{alg:duwff} provides a modified update 
scheme, based on the the property that the Dirichlet distribution is conjugate of the Categorical distribution: 
 instead of weighting all observations equally,
 we essentially account for just the last $\frac{1}{1 - \lambda}$ 
 opponent actions.  
 \begin{algorithm}
\begin{algorithmic}[1]
\State Initialize pseudocounts $ \bm{\alpha^0} = (\alpha^0_1, \ldots, \alpha^0_n)$
\For{$t = 1, \ldots, T $}
\State $\bm{\alpha^t} = \lambda \bm{\alpha^{t-1}}$ \Comment Reweight with factor $0 < \lambda < 1$
\State Observe opponent action $b^t_i, i \in \lbrace b_1, \ldots, b_n \rbrace$
\State $\alpha^t_i = \alpha^{t-1}_i + 1$ \Comment Update posterior
\State $\alpha^t_{-i} = \alpha^{t-1}_{-i}$
\EndFor
\end{algorithmic}
\caption{Dirichlet updating with forget factor}
\label{alg:duwff}
\end{algorithm}

For a level-2 defender, as we do not know 
the actual rewards of the adversary (modelled as a level-1 learner), 
 we  model it as in a zero-sum scenario  ($r_B = -r_A$) making this case similar to the Matching Pennies game. The adopted discount factor is $\gamma = 0.8$;
 there are $5000$ episodes; the initial exploration parameter
 is $\epsilon = 0.1$ and the learning rate, $\alpha = 0.1$. The 
 assumed 
 forget factor is $\lambda = 0.8$.
\begin{figure}%
\centering
\includegraphics[scale=0.5]{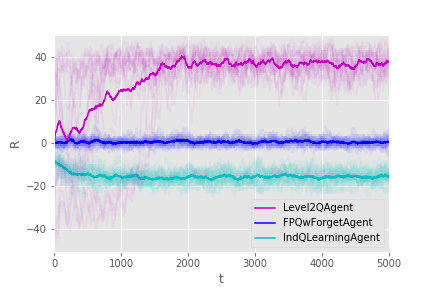}%
\caption{Rewards for the DM against the adversarial opponent}\label{fig:4C_adv}
\end{figure}
 Figure \ref{fig:4C_adv} displays results. We consider three defenders:
an opponent unaware $Q$-learner; a level-1 DM with forget
using Algorithm \ref{alg:duwff}; and a
level-2 agent. The first defender is clearly is exploited by the adversary 
achieving suboptimal results (rewards close to -20).
In contrast, the level-1 DM with forget effectively
learns a stationary optimal policy (reward 0). Finally, the
level-2 agent is capable of exploiting 
the adaptive adversary actually achieving positive rewards
(around 40, close to the upper bound of 50 due to the value of the positive reward). 

Note that the actual adversary behaves differently from how the DM models him,
as he is not a level-1 $Q$-learner. Even so, modelling him as 
such, gives the DM sufficient advantage in this case.
This robustness against opponent misspecification emerges as an
advantage of our proposed framework.

We next perform a similar experiment replacing the $\epsilon-$greedy
policy with a softmax one: actions at state $s$ are taken with probabilities 
proportional to $Q(s, a)$. Figure \ref{fig:softmax}
provides several simulation runs of a level-2 $Q$-learner versus the adversary,
showing that, indeed, changing the policy sampling scheme does not worsen
the DM with respect to the $\epsilon-$greedy alternative.

\begin{figure}[h!]
\centering
\includegraphics[scale=0.5]{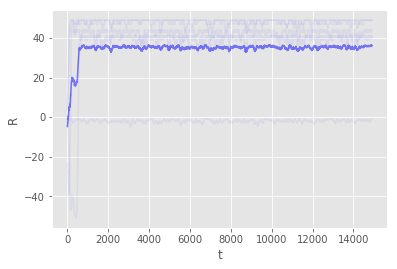}%
\caption{Rewards for the DM against the adversarial opponent, using a softmax policy.}\label{fig:softmax}
\end{figure}

Lastly, we tried other models for the opponent's rewards $r_B$. Instead of observing the opponent reward as $r_B = -r_A$ in a zero-sum setting, where $r_B \in \lbrace -50, 50 \rbrace$,  we tried two additional reward scalings $r_B \in \lbrace -1, 1 \rbrace$ and $r_B \in \lbrace 0, 1 \rbrace $. Figure \ref{fig:4C_rs} displays the result, portraying
similar results qualitatively.

 \begin{figure}%
 \centering
 \subfigure[Rewards $+1$ and $0$ for the adversary]{%
 \label{fig:4C_binary_adv}%
 \includegraphics[scale=0.5]{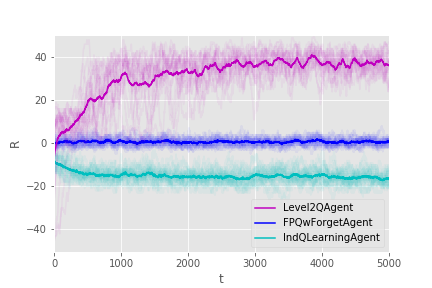}}%
 \subfigure[Rewards $+1$ and $-1$ for the adversary]{%
 \label{fig:4C_binary_1m1}%
 \includegraphics[scale=0.5]{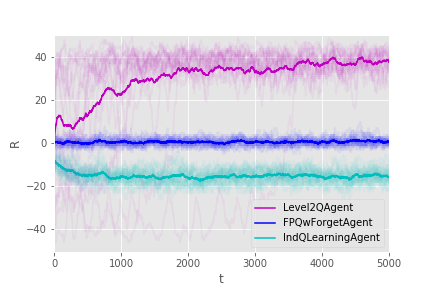}}%
 \caption{Rewards against same adversary (exp. smoother) using different reward scalings.}\label{fig:4C_rs}
 \end{figure}

\paragraph{Facing more powerful adversaries.}

So far, the DM has interacted against an exponential smoother
adversary. This may be exploited if the DM is a level-2 agent.
We study now the outcome of the process
when we consider more powerful adversaries.

First, we parameterize our opponent as a level-2 $Q$-learner.
To do so, we specify the rewards he receives as $r_B = -r_A$
(for simplicity we consider a zero-sum game, although our framework allows dealing with the general-sum case). Figure \ref{fig:L2vsL2} depicts the rewards for both the DM (blue)
and her adversary (red). We have computed the frequency of choosing each of the actions, and both players select either action 
with probability $0.5 \pm 0.002$ based on 10 different random seeds. Both agents achieve the Nash equilibrium, consisting of
choosing between both
actions with equal probability, leading to an expected cumulative reward of 0, as shown in the graph.

\begin{figure*}%
\centering
\subfigure[L2$Q$-learner (blue) vs L2$Q$-learner (red)]{%
  \label{fig:L2vsL2}%
  \includegraphics[height=1.4in]{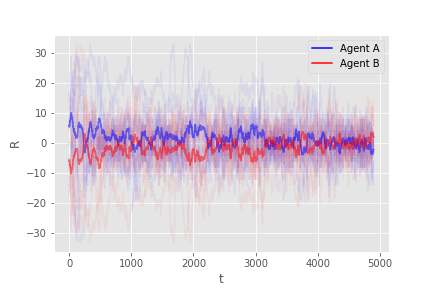}}%
  \subfigure[L3$Q$-learner (blue) vs L2$Q$-learner (red)]{%
  \label{fig:L3vsL2}%
  \includegraphics[height=1.4in]{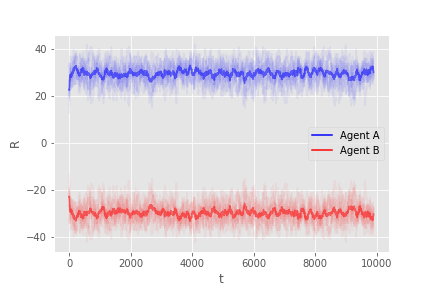}}%
  
  \subfigure[L3$Q$-learner (blue) vs L1$Q$-learner (red)]{%
  \label{fig:L3vsL1}%
  \includegraphics[height=1.4in]{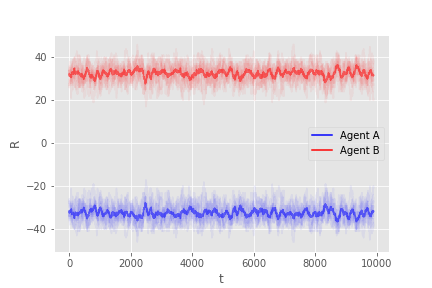}}%
  \subfigure[L3$Q$-learner with opponent averaging (blue) vs L1$Q$-learner (red)]{%
  \label{fig:L3DirvsL1}%
  \includegraphics[height=1.4in]{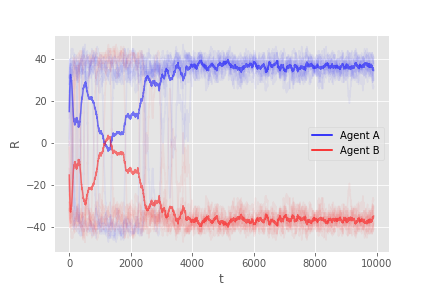}}%
  
  \subfigure[Estimate of  $P_{L1Q}$: DM's belief that her opponent is a level-1 $Q$-learner]{%
  \label{fig:probas}%
  \includegraphics[height=1.4in]{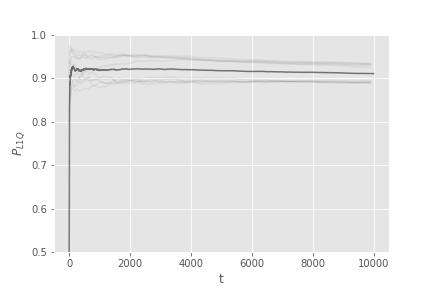}}%

  \caption{Rewards obtained against the exponential smoother adversary. }
\end{figure*}

Increasing the level of our DM to make her level-3, allows her to exploit a level-2 adversary, Fig. \ref{fig:L3vsL2}. However, this DM fails to exploit a level-1 opponent, i.e., a FP$Q$-learner, Fig. \ref{fig:L3vsL1}. The explanation to this apparent paradox is that the DM is modelling her opponent as a more powerful agent
than he actually is: her model is inaccurate and leads to poor performance.
However, this failure suggests a potential solution to the problem
using type-based reasoning (Section \ref{sec:com}). Figure \ref{fig:L3DirvsL1} depicts the rewards of a DM that keeps track of both level-1 and level-2 opponent models and learns, in a Bayesian manner, which one is she 
actually facing. The DM keeps estimates of the probabilities $P_{L1Q}$ and $P_{L2Q}$ that her opponent is acting as if he was a level-1 or a level-2 $Q$-learner, respectively. Figure \ref{fig:probas} depicts the evolution of $P_{L1Q}$: we observe that it places most of the probability in the correct opponent type.









\paragraph{Spatial Variant.}

We now compare the independent $Q$-learner and a level-$2$ $Q$-learner against the
same adaptive exponential smoother opponent in the spatial gridworld domain
in Fig. \ref{fig:friendorfoe}. The DM has four actions to choose, one for each direction in which the DM is allowed to move for one step. Target rewards 
are delayed until the DM arrives at one of the pertinent locations, 
obtaining $\pm 50$ depending on the target chosen by the adversary.
Each step is penalized with -1 for the DM. Episodes end at a maximum of 50 steps or when the agent arrives first at targets 1 or 2. 
The discount factor is $\gamma = 0.8$ and $15000$ episodes are 
considered. For the level-2 agent, we consider different exploration hyperparameters, $\epsilon_A$ and $\epsilon_B$ for the DM's policy and her correponding model for her opponent;
initially,  $\epsilon_A = \epsilon_B = 0.99$ with decaying rules $\epsilon_A := 0.995\epsilon_A$ and $\epsilon_B := 0.9\epsilon_B$ every $10$ episodes and learning rates $\alpha_2 = \alpha_1 = 0.05$. For the independent $Q$-learner, we set similar initial hyperparameters.

Results are displayed
in Figure \ref{fig:4C_gridworld}. Again, an independent $Q$-learner is
exploited by the adversary, obtaining even more negative results  
than in 
Figure \ref{fig:4C_adv} due to the penalty at each step. In contrast,
a level-2 agent approximately estimates adversarial
behavior, modelling him as a level-1 agent, obtaining positive rewards. Figure \ref{fig:L3Dir_spatial} depicts the rewards of a DM that maintains 
opponent models for both level-1 and level-2 $Q$-learners. Although the adversary is of neither class, the DM achieves positive rewards,
suggesting that the framework is capable of generalizing between opponent behaviours not exactly reflected by the DM's opponent model.
This shows that our level-$k$ thinking scheme it is sufficiently robust to opponent misspecification.

\begin{figure*}[h]
\centering
\subfigure[Rewards for various DM models]{%
  \label{fig:4C_gridworld}%
  \includegraphics[scale=0.5]{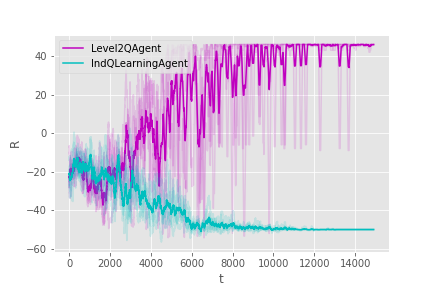}}%
  \subfigure[Rewards for a DM with opponent models for a L1 $Q$-learner and a L2 $Q$-learner (red)]{%
  \label{fig:L3Dir_spatial}%
  \includegraphics[scale=0.5]{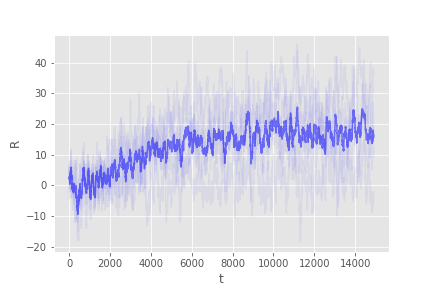}}%
  \caption{Rewards against the exponential smoother opponent in the spatial environment. }
\end{figure*}

Lastly, we also perform several experiments in which we try different values of the hyperparameters,  to further highlight the robustness of the framework. Table \ref{tab:rob} displays mean rewards (and standard deviations) for five different random seeds, over different hyperparameters of Algorithm \ref{alg:l2ur}. Except in the case where the initial exploration rate $\epsilon_0$ is set to a high value (0.5, which makes the DM achieve a positive mean reward), the other settings showcase that the framework (for the level-2 case) is robust to different learning rates.

\begingroup
\renewcommand{\arraystretch}{0.7}
\begin{table}[h]
\small
\caption{Hyperparameter robustness of Algorithm \ref{alg:l2ur} on the spatial gridworld.}\label{tab:rob}
\centering
\begin{tabular}{lllr}
\hline
$\alpha_2$ &  $\alpha_1$ & $\epsilon_0$ & Mean Reward \\
\hline
0.01 & 0.005 & 0.5 & $15.46 \pm 47.21$  \\
0.01 & 0.005 & 0.1 & $40.77 \pm 27.48$  \\
0.01 & 0.005 & 0.01 & $46.32 \pm 16.15$  \\\hline
0.01 & 0.02 & 0.5 & $15.58 \pm 47.17$  \\
0.01 & 0.02 & 0.1 & $43.05 \pm 23.65$  \\
0.01 & 0.02 & 0.01 & $47.81 \pm 10.83$  \\\hline
0.1 & 0.05 & 0.5 & $15.30 \pm 47.27$  \\
0.1 & 0.05 & 0.1 & $42.82 \pm 24.08$  \\
0.1 & 0.05 & 0.01 & $48.34 \pm 8.10$  \\\hline
0.1 & 0.2 & 0.5 & $15.97 \pm 47.03$  \\
0.1 & 0.2 & 0.1 & $43.05 \pm 23.66$  \\
0.1 & 0.2 & 0.01 & $48.51 \pm 6.96$  \\\hline
 0.5 & 0.25 & 0.5 & $15.95 \pm 47.04$  \\
 0.5 & 0.25 & 0.1 & $43.06 \pm 23.64$  \\
 0.5 & 0.25 & 0.01 & $48.41 \pm 7.68$  \\\hline
 0.5 & 1.0 & 0.5 & $15.19 \pm 47.31$  \\
 0.5 & 1.0 & 0.1 & $42.98 \pm 23.71$    \\
 0.5 & 1.0 & 0.01 & $48.53 \pm 6.82$  \\\hline
\end{tabular}
\end{table} 
\endgroup
\subsection{Facing multiple opponents}

We illustrate the multiple opponent concepts from Section \ref{sec:mul} introducing a novel suite of  resource allocation experiments 
relevant in security settings. They are based on a modified version of Blotto games \parencite{hart2008discrete}: the DM needs to distribute limited
resources over several positions susceptible of being attacked. Each of the attackers has to choose different positions
to deploy their attacks. Associated with each of the attacked positions there is a positive (resp. negative) reward with value 1 (-1). If the DM 
deploys more resources than the attackers in a particular position, she wins the positive reward and the negative one will be equally split  between
the attackers that chose to attack such position. If the DM deploys less resources, she will receive
the negative reward and the positive one will be equally split between the corresponding attackers.
In case of a draw, no player receives any reward.  


We compare the performance of a FP$Q$-learning agent vs a standard $Q$-learning agent, when facing two conditionally independent opponents both using 
exponential smoothing to estimate the probability of the DM placing a resource at each position, and implementing the attack where
this probability is the smallest (obviously both opponents perform exactly the same attacks).

Consider
a problem of defending three different positions; the DM needs to allocate two resources among such positions. For both the $Q$-learning and the FP$Q$-learning agents the discount factor will be $\gamma = 0.96$, $\epsilon = 0.1$ and the learning rate $\alpha = 0.1$. 
As Fig. \ref{fig:2expsmoothers} shows, FP$Q$-learning is able to learn the opponents strategies and thus is less exploitable than standard $Q$-learning.
This experiment showcases the suitability of the framework to deal with multiple, independent adversaries, by straightforwardly extending the level-$k$ thinking scheme as discussed (Section \ref{sec:mul}).
\begin{figure}%
\centering
\includegraphics[scale=0.5]{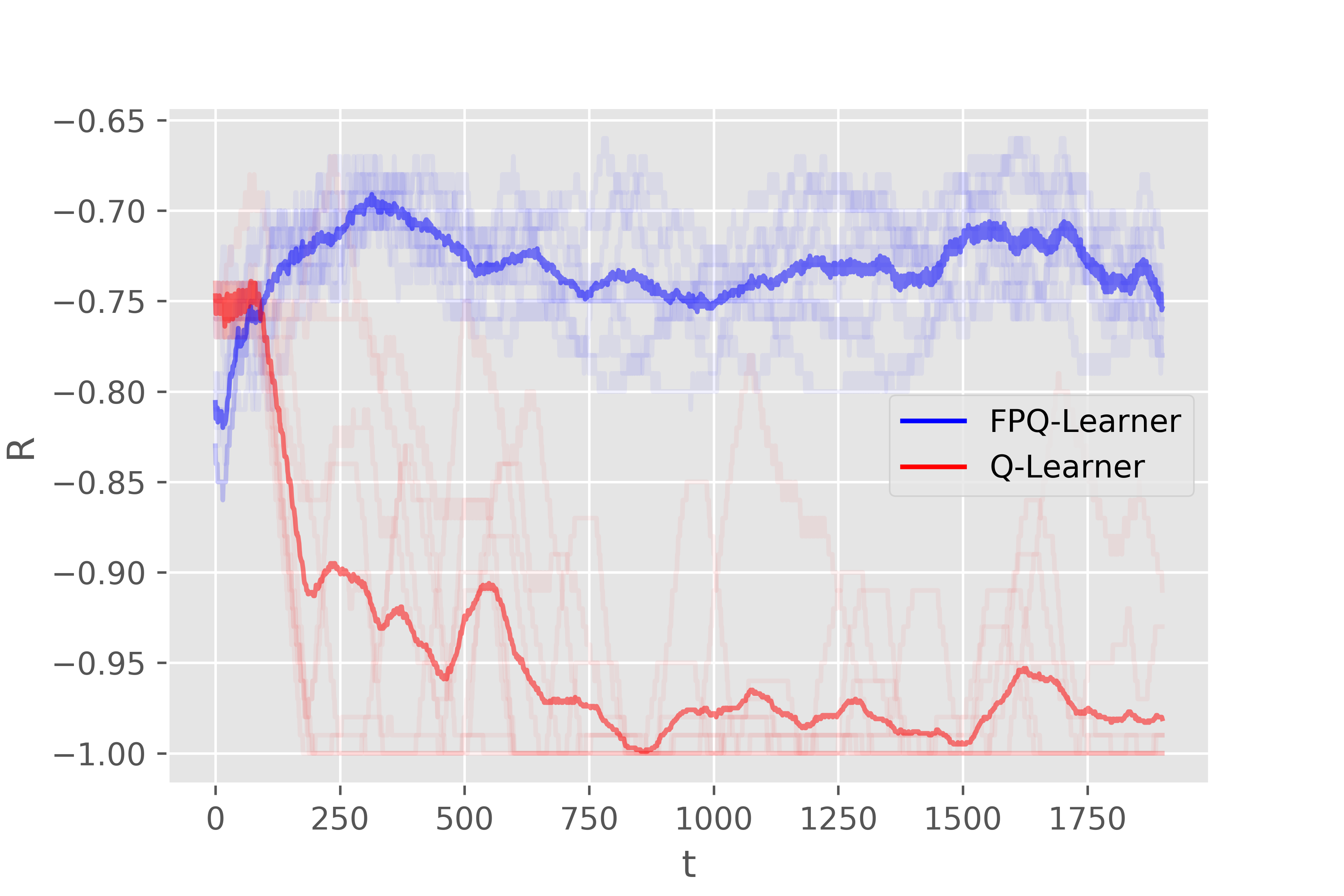}%
\caption{Rewards for the DM against the opponent}\label{fig:2expsmoothers}
\end{figure}
%







\subsection{Experiments with parametric Q-values using function approximation}\label{sec:cg}

We run a battery of experiments to showcase 
how our framework is actually compatible with variants in which $Q$-values are approximated with a parametric function, in particular a deep learning model, as in Section \ref{sec:approx_rl}.

\begin{figure}[h]
\centering
\includegraphics[scale=0.3]{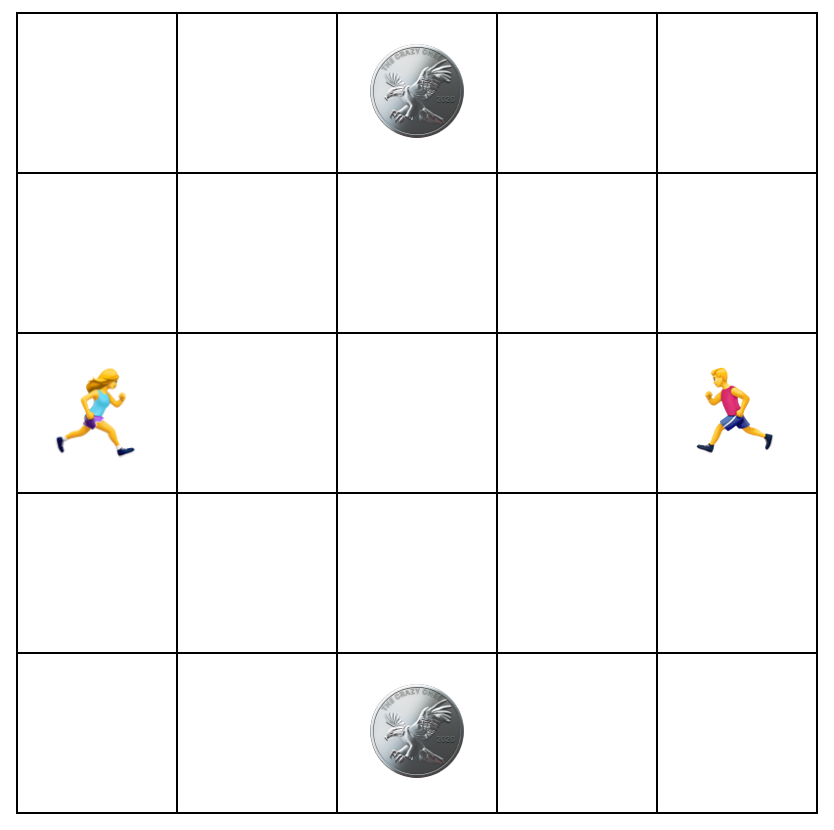}%
\caption{An initial state for the Coin Game}\label{fig:coingame}
\end{figure}

For this, we create the gridworld game in Figure \ref{fig:coingame}. It is 
is played over an $N \times N$ grid ($N=5$ in the Figure). Each turn, both players (blue and red) move in either of the four directions (unless they try to cross the boundary of the grid), with the objective of chasing two different coins in the grid. For each coin the agent picks, he receives a reward of 1, unless the other agent arrives also at the same location,
in which case the perceived reward is only 0.5. 
Thus, this gridworld game can be cast as a coordination game.
The episode ends after 12 steps (so that each agent has ample time to pick both coins). 

The space of states is given as a $N \times N \times 4$ array, in which each of the four last slices denote the position of each player and coin as a one-hot encoded $N \times N$ matrix. A straightforward application of $Q$-learning would require a table $Q(s, a)$ with $(N \times N)^4 \times 4$ entries. Instead, we parameterize the $Q$-values to reduce the number of parameters: the state $s$ is flattened from a tensor in $\mathbb{R}
^{N \times N \times 4}$ to a vector in $\mathbb{R}
^{4N^2}$, and we project it to $\mathbb{R}^4$ using linear regression to obtain one $Q$-value for each of the four possible actions: 
given any state $s \in \mathbb{R}^{4N^2}$, we obtain the corresponding $Q$-value $Q(s, a)$ via $Q(s, a) = w_a^{\intercal}s $ with $w_a \in \mathbb{R}^{4N^2}$ being the weight vector for action $a$. Thus, we have reduced the total number of parameters to $N
^2\times 4 \times 4$. The extension to the $Q$-function approximator for the level-2 case is straightforward: instead of projecting to $\mathbb{R}^4$ we project to $\mathbb{R}^{4\times 4}$ to consider each pair of actions $(a, b)$, in order to approximate the corresponding $Q$-function $Q(s,a,b)$. 

Figures \ref{fig:coin1} and \ref{fig:coin2} depict the results of two games. In light color we depict the smoothed rewards along 10000 iterations for five different random seeds, and in darker colors the three averages of the $3 \times 5$ previous runs:
even in this approximate regime, a DM exhibiting higher rationality than its adversary can make an advantage from him. Thus, we have shown that our opponent modelling framework is compatible with approximate $Q$-values, being capable of harnessing all the benefits from the approximate regime, such as the reduced parameter count or generalization capabilities.
\begin{figure*}%
\centering
\subfigure[Rewards of two independent $Q$-learners against each other]{%
  \label{fig:coin1}%
  \includegraphics[scale=0.5]{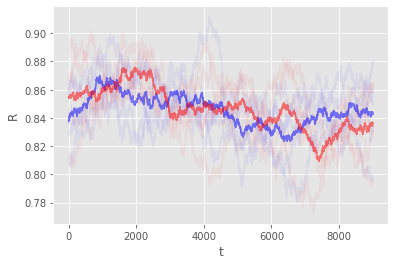}}%
  \subfigure[Rewards of a L2 $Q$-learner (blue) against an independent $Q$-learner (L0) (red)]{%
  \label{fig:coin2}%
  \hspace{0.2cm}
  \includegraphics[scale=0.5]{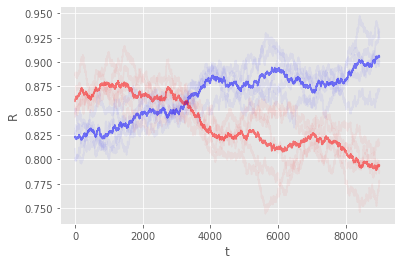}}%
  \caption{Rewards in the Coin Game using function approximators. }
\end{figure*}

Note we also used a multi-layer NN to parameterize the $Q$-functions, but we did not achieve a significant improvement in the results, so we just report those of the simplest model.


\section{An application to data sharing}\label{sec:ds}
Next, we  move to the announced application concerning data sharing games.

\subsection{A motivation for data sharing games}

As recently discussed \parencite{wolfram}, {\em data, as the intangible asset par excellence in the 21st century, is the most disputed raw material at global scale}. Ours is a data-driven society and economy,
with  data guiding most business actions and decisions. 
This is becoming even more important as
many business processes are articulated through a 
cycle of sensing-processing-acting. Indeed, Big Data is the
consequence of a digitized world where people, objects and operations are fully instrumented and interconnected, producing all sorts of data, both machine-readable (numbers and labels, known as
structured) and human-readable (text, audio or video, known as unstructured).
As data acquisition grows at sub-second speed, the capability to monetize
them arises through the ability to derive new synthetic data. 
Thus, considered as an asset, data create markets and enhance 
competition. Unfortunately, they are creating bad practices as well.
See \parencite{share} for an early discussion 
as well as the 
recent European directives and legislative initiatives to 
promote public-private data partnerships, e.g. \textcite{europe1}.

This is the main reason for analyzing {\em data sharing 
games} with mechanisms that could foster cooperation to guarantee and promote social progress.
%
Data sharing problems have been the object of several contributions and studied
from different perspectives. For example,
 \textcite{kamhoua2012game} proposes a game theoretic approach to help users determine their optimal policy in terms of sharing data in online social networks,
 based on a confrontation between a user (aimed at sharing certain information and hiding the rest) and an attacker (aimed at exposing the user's PI
 or concealing the information the user is willing to share). This
  is modelled through a zero-sum Markov game; a  Markov equilibrium is computed and the corresponding Markov strategies are used to give advice to users.
 The work of \textcite{figueiredo2017data} reviews the impact of data sharing in science and society and presents guidelines to improve the efficiency of 
 data sharing processes,
 quoting \textcite{pronk}, who provide a game theoretical analysis suggesting 
 that sharing data with the community can be the most profitable and stable strategy. Similarly, \textcite{dehez2013data} consider a setting in which 
 a group of firms must decide whether to cooperate in 
 a project that requires the combination of data held by several of them; the authors  address the question of how to compensate the firms for the data they
 contribute with,
 framing the problem as a transferable utility game and characterizing its Shapley value as a compensation mechanism. 
  
Our approach models interactions between 
citizens and data owners  
inspired by the iterated prisoner's dilemma (IPD) \parencite{axelrod84}. This is an elegant incarnation of the problem of how to achieve agents` cooperation in competitive settings. Other authors have used similar models in other socio-technical problems as in politics 
\parencite{brams},  security \parencite{kunreuther}, and cybersecurity \parencite{etesami2019dynamic}, among others. Our approach to model agent's behavior is different and relies on multi-agent reinforcement learning (MARL) arguments \parencite{gallego2019opponent,gallego2019reinforcement}. Reinforcement learning (RL) has been successfully applied to games that are repeated over time, thus making it possible for agents to optimize their strategies \parencite{lanctot2017unified,chasparis2012distributed}. 
Through RL we are able 
to identify relevant mechanisms to promote cooperation.


\subsection{Context}
Before modeling interactions between data consumers and producers, it is convenient to understand the data categories available. 
Even though admittedly with a blurry frontier, 
from a legal standpoint, there are two main types:

\begin{itemize}
\item {\em Data that should not be bought/sold}. This refers to  
 personal information (PI), as e.g.\ the data preserved in the European
 Union through 
 the General 
 Data Protection Regulation (GDPR) \parencite{gdpr} and other citizen defense frameworks
 aimed at guaranteeing civic liberties. PI includes 
 data categories such as 
{\em internal information} (like knowledge and beliefs, 
 and health 
 data); 
{\em financial information} (like accounts or 
 credit data);
{\em social information} (like 
criminal records 
or communication data); or,
{\em tracking information} (like computer device; 
or location data).
\item {\em Data that might be purchased}. Citizen’s data is a property,
there being a need to guarantee a fair and transparent compensation.
Accountability mitigates market frictions. For
traceability and transparency reasons,
blockchain-based platforms are being implemented at the moment within this domain.
\end{itemize}
  A characterization of what type of data 
  belongs to each category will depend on the context and is, most of the times, subjective.

In any case, in the last decades, modern data analytics techniques and strategies are enabling the generation of new types of data:
\begin{itemize}
\item {\em Data that might be estimated/derived.}  Currently available analytics technologies have the ability of estimating efficiently citizen behavior and other characteristics by deeply analyzing Big Data. For instance, platforms such as IBM Personality Insights \parencite{ibm} estimate personality traits of a
given individual using his/her tweets, thus facilitating marketing activities.
As a result, the originating data becomes a new asset for a company
willing to undertake its analysis.
\end{itemize}

Having mapped the available data, there is a need to understand the 
knowledge actually available and how is it uncovered.	
Within the above scenario, 
we consider two players in  a data 
sharing game: the data providers 
(Citizen, she) and the Dominant Data Owner (DDO, he).
A DDO could be a private
company, e.g. GAFA (Google, Apple, Facebook, Amazon) or Telefonica,
or a public institution (Government). 
Inspired by the classic Johari window \parencite{johari},
we inter-relate now what a Citizen knows, or does not,
with what a DDO knows, or does not,
to obtain these scenarios:
\begin{enumerate}
\item Citizen knows what DDO does.
The citizen has created 
a data asset which she sells to a DDO. 
Sellable data create a market which could 
evolve in a sustainable manner if accountability and transparency are somehow guaranteed.
\item 	Citizen knows what DDO does not.
This is the PI realm.
Citizens would want legal frameworks like the 
GDPR or data standards preserving  
citizen rights, mainly 
ARCO-PL (access, rectification, cancellation, objection, portability
and limitation) so that PI is respected.
\item Citizen does not know what DDO does. 
The DDO has unveiled
citizen’s PI through deep analysis of
Big Data.\footnote{As in the famous Target pregnant
teenager case \parencite{target}} 
This analysis may be acceptable if data are dealt just as a target.
Data protection frameworks should guarantee civil rights and liberties in such activities.
\end{enumerate}
Note that we could also think of a fourth scenario in which 
neither the citizen knows, nor the DDO does, although this is clearly unreachable. 

Once explained how knowledge is shared, we analyze how 
knowledge creation can be fostered to stimulate social progress, studying 
cooperation  scenarios between Citizen and DDO. 
We simplify by considering two strategies
for both players, respectively designated {\em Cooperate}  (C)
and {\em Defect} (D), leading to the four scenarios 
in Table
\ref{kaka}.

\begin{table}[htbp]
	\centering
	\scalebox{0.8}{
	\begin{tabular}{c|c|c}
			        &  DDO cooperates & DDO defects  \\
		\hline  
Citizen cooperates &  Citizen sells data, &Citizen taken for a ride\\
	     	&  demands data protection &  selling data, while DDO                 \\
   	        &  DDO 	purchases and  &    does not pay Citizen         \\
   	        &  respect Citizen data.    &   data with services            \\ \hline
Citizen defects  &    DDO  taken for a ride & Citizen sells wrong/noisy  \\ 
		  &    purchasing. Citizen     & data does not pay for DDO  \\
		  &  selling   wrong/noisy  &  services, who does not pay data  \\
		  &  data becomes free rider.         & with  services.  \\ 		
			\end{tabular}%
			}
	\caption{Scenarios in the data sharing game.}
	\label{kaka}%
\end{table}

Reflecting about them, the only one 
that ultimately fosters knowledge creation and, therefore, stimulates social progress, is mutual cooperation. It is the best scenario and produces mutual value. Cooperation begs for a FATE (fair, accountable, transparent, ethical) technology like blockchain. In such scenario, data (Big Data), algorithms and processing technology would boost knowledge. Mutual cooperation is underpinned by decency and indulgence values 
such as being
{\em nice} (cooperate when the other party does); 
{\em provokable} (punish non cooperation);
{\em forgiving} (after punishing, immediately cooperate
and reset credit);
and {\em clear} (the other party easily understands and realises that 
the best next move is to cooperate).

Mutual defection is the worst scenario 
in societal terms: it produces a data market failure, stagnating social progress. As there is no respect from both sides, no valuable data trade will happen, 
and even a noisy data vs.\ unveiled data war will take place. Loss of freedom may arise as a result. 

The scenario (Citizen cooperates, DDO defects) is the worst
for the citizen, leading to data power abuses, as with the
 UK ``ghost" plan. 
It would generate asymmetric information,
adverse selection, and moral hazard problems, in turn producing 
data market failures.
The DDO behaves incorrectly, there being a need to punish unethical
and illegal behaviour. As an example, the GDPR sets the right to receive explanations for
algorithmic decisions. There is also a need for mitigating
systematic cognitive biases in algorithms.
Citizens may respond by 
sending noisy data,
rejecting data services, imposing standards over data services or setting prices 
according to success. 

Finally, the scenario (Citizen defects, DDO cooperates) is the worst for the DDO. It leads to data market failures and shrinks knowledge. This  
stems from  a behavior of not paying for
public/private services that can be obtained anyway. 
In the long run, this erodes public and private services quality and creativity. This misbehavior should be punished to restore cooperation and a fair price should be demanded for services.

\subsection{A model for the data sharing game}\label{sec:models}

We model interactions between citizens and DDOs over 
time from the perspective of the IPD. 
Table \ref{tab:payoffIPD} shows its reward bimatrix. 
The row player will be the Citizen, for whom {\em cooperate} means that she
wishes to sell and protect her data, whereas {\em defect} means she either sells wrong data or decides not to contribute. The DDO will be the column player
for whom {\em cooperate} means that he purchases and protects data,
whereas {\em defect} means that he is not going to pay for the collected data or will not protect it.
Payoffs satisfy the usual conditions in the IPD, that is $T>R>P>S$ and $2R > T+S$.
When numerics are necessary,
we adopt the choice $T= 6 $, $R= 5 $, $ P = 1 $,  and $S= 0 $.

\begin{table}[]
\begin{center}
\begin{tabular}{cl|lll}
\multicolumn{1}{l}{}                                   &     & \multicolumn{3}{l}{\textbf{DDO}} \\ \cline{3-5} 
\multicolumn{1}{l}{}                                   &     & $C$         &       & $D$        \\ \hline
\multicolumn{1}{c|}{\textbf{Citizen}} & $C$ & $R,R$       &       & $S,T$      \\
\multicolumn{1}{c|}{}                                  &     &             &       &            \\
\multicolumn{1}{c|}{}                                  & $D$ & $T,S$       &       & $P,P$     
\end{tabular} 
\end{center}
\caption{Payoffs in the data sharing game}
\label{tab:payoffIPD}
\vspace{-2ex}
\end{table}

It is well-known that in the one-shot version of the IPD game, the unique Nash equilibrium is $(D,D)$, leading to the social dilemma described above: the selfish rational point of view of both players leads to an inferior societal position. Similarly, if the game is played $N$ times,
and this is known by the players, these have no incentive to cooperate, as we may reason by backwards induction \parencite{axelrod84}. 
However, in realistic scenarios, players are not sure about 
whether they will meet 
again in future and, consequently, they cannot be sure when the last interaction will be taking place \parencite{axelrod84}. Thus, it seems reasonable to assume that players will interact an indefinite number of times or that there is a positive probability of meeting again. This possibility that players might interact again is precisely what makes cooperation emerge.





The framework that we adopt to deal with this 
dynamic game is MARL 
\parencite{marl_over}. Each agent $a \in \lbrace C, DDO \rbrace $ maintains its policy $\pi_a(d_a|o_a, \theta_a)$ used to select a decision $d_a$ under some observed state of the game $o_a$ (for example, the previous pair of decisions) and parameterised through 
 $\theta_a$. Each agent learns
how to make decisions by optimizing his policy under the expected sum of discounted utilities
$$
\max_{\theta_a} \mathbb{E}_{\pi_a} \left[ \sum_{t=0}^\infty \gamma^t r_{a, t} \right],
$$
where $\gamma \in (0,1)$ is a discount factor and $r_{a,t}$ is the reward that
agent $a$ attains at time $t$.
The previous optimization can be performed through 
Q-learning or policy gradient methods \parencite{sutton2012reinforcement}. 
The main limitation with this approach in the multi-agent setting is that if the
agents are unaware of each other, they are shown to fail to cooperate 
\parencite{gallego2019opponent}, leading to defection every time, which is undesirable in the data sharing  game.

As an alternative,  we propose three approaches, depending on the degree of decentralization and incentivisation sought for when trying 
to foster collaboration. 
\begin{itemize}
        \item In a (totally) decentralized case, C and DDO are alone and we resort to opponent modelling strategies,
        as we will showcase in Section \ref{sec:decentralized}. However, 
        this approach may fail under severe misspecification in the opponent's model. Ideally, we would
        like to encourage collaboration without making strong assumptions about learning algorithms used by each player. 
      
        \item Alternatively, a third-party could become a regulator of the data market: C and DDO use it and  
        the regulator introduces taxes, as showcased 
        in Section \ref{sec:regulator}. 
        The benefit of this approach is that the regulator only needs to observe the actions adopted by the agents, not needing to make any assumption about their models or motivations and optimizing their behavior based on whatever social metric is considered.
        
        \item Finally, in Section \ref{sec:incentives} we augment the capabilities of the previous regulator to enable it to  incentivise the agents, leading to further increases in the social metric considered.
    \end{itemize}

\noindent To fix ideas, we focus on a social utility (SU) metric
defined as the agents' average utility  
\begin{equation}\label{eq:su}
SU_t = \frac{r_{C,t} + r_{DDO,t}}{2}.
\end{equation}
This requires adopting a notion of transferable utility,
serving as a common medium of exchange that can be transferred between agents, see e.g. \textcite{aumann1960}.


\subsection{Three solutions via Reinforcement Learning}\label{sec:sols}

\paragraph{The decentralized case.}\label{sec:decentralized}
Our first approach models the interaction between both agents as an IPD.  We first fix the strategy of the DDO, assume that the citizen models the DDO behaviour and simulate interactions between
both agents 
finally assessing social utility.\footnote{Code for all the simulations performed can be found at \url{https://github.com/vicgalle/data-sharing}}
Through this, we assess the impact of different 
DDO strategies over social utility.

We  model the Citizen as a Fictitious Play Q-learner (FPQ) in the spirit of Section \ref{sec:tmdps}.
She chooses her action $d_a \in \lbrace C, D \rbrace$
to maximize her expected utility $\psi(d_a)$
defined through 
\[ \psi(d_a) = \mathbb{E}_{p_{FP}(d_b)} [Q(d_a,d_b)] = \sum_{d_b \in \lbrace C, D \rbrace } Q(d_a, d_b) p_{FP}(d_b), \]
 where $p_{FP} (d_b)$ reflects the Citizen's beliefs about her opponent's
 actions $d_b \in \lbrace C, D \rbrace$
 and 
 $Q(d_a,d_b)$ is the augmented Q-function from the threatened Markov decision processes  \parencite{gallego2019opponent}, 
 an estimate of the expected utility obtained by the Citizen if both
 players were to commit to actions $d_a, d_b$.
 
 We estimate the probabilities $p_{FP} (d_b)$
  using the empirical frequencies of the opponent's past
  plays as in Fictitious Play 
 \parencite{brown1951iterative}. To further favor learning, the Citizen could 
place a Beta prior over $p_C \sim \mathcal{B}(\alpha, \beta)$,
the probability of the DDO cooperating, 
with probability $p_D = 1-p_C$ of defecting.
 Then, if the opponent chooses,
for instance, {\em cooperate}, the citizen updates her
beliefs to the posterior $p_C \sim \mathcal{B}(\alpha + 1, \beta)$, 
and so on. 

We may also augment the Citizen model to have memory
of the previous opponent's action. This can be
straightforwardly done replacing $Q(d_a,d_b)$ with $Q(s,d_a,d_b)$ and $p_{FP}(d_b)$ with $p_{FP}(d_b|s)$
where $s \in \lbrace C, D \rbrace \times \lbrace C, D \rbrace$ is
the previous pair of actions both players took. 
For this, we  
 keep track of four Beta distributions, one for each
value of $s$. This FPQ agent with
memory will be called FPM. 
Clearly, this approach could be expanded to account for 
longer memories over the action sequences. However,  \textcite{press2012iterated} shows that agents with a 
good memory-1 strategy can effectively force the iterated 
game to be played as memory-1, even if the opponent has a
longer memory.

\paragraph{Experiments for the decentralized setting.}

We simulate the previous IPD under different strategies 
for the DDO and measure the impact over social utility. For each scheme, we display the social utility  attained over time by the agents. In all experiments, 
the citizen is modelled as an FPM agent (with memory-1). The discount factor was set to 0.96. 

\subparagraph{Selfish DDO.} 
When we assume a DDO playing always defect, our simulation confirms that this strategy will force 
the citizen to play defect and sell wrong data, not having incentives to abandon
such strategy. 
Even when 
citizens have strong prior beliefs that the DDO will cooperate, after a
few iterations they will learn that the DDO is 
always defecting and thus choose also to defect, as shown in Figure \ref{fig:nash_ut}.

\begin{figure*}[h!]
\centering
\subfigure[Agents' utilities.]{%
  \label{fig:nash_ut}%
  \includegraphics[height=1.8in]{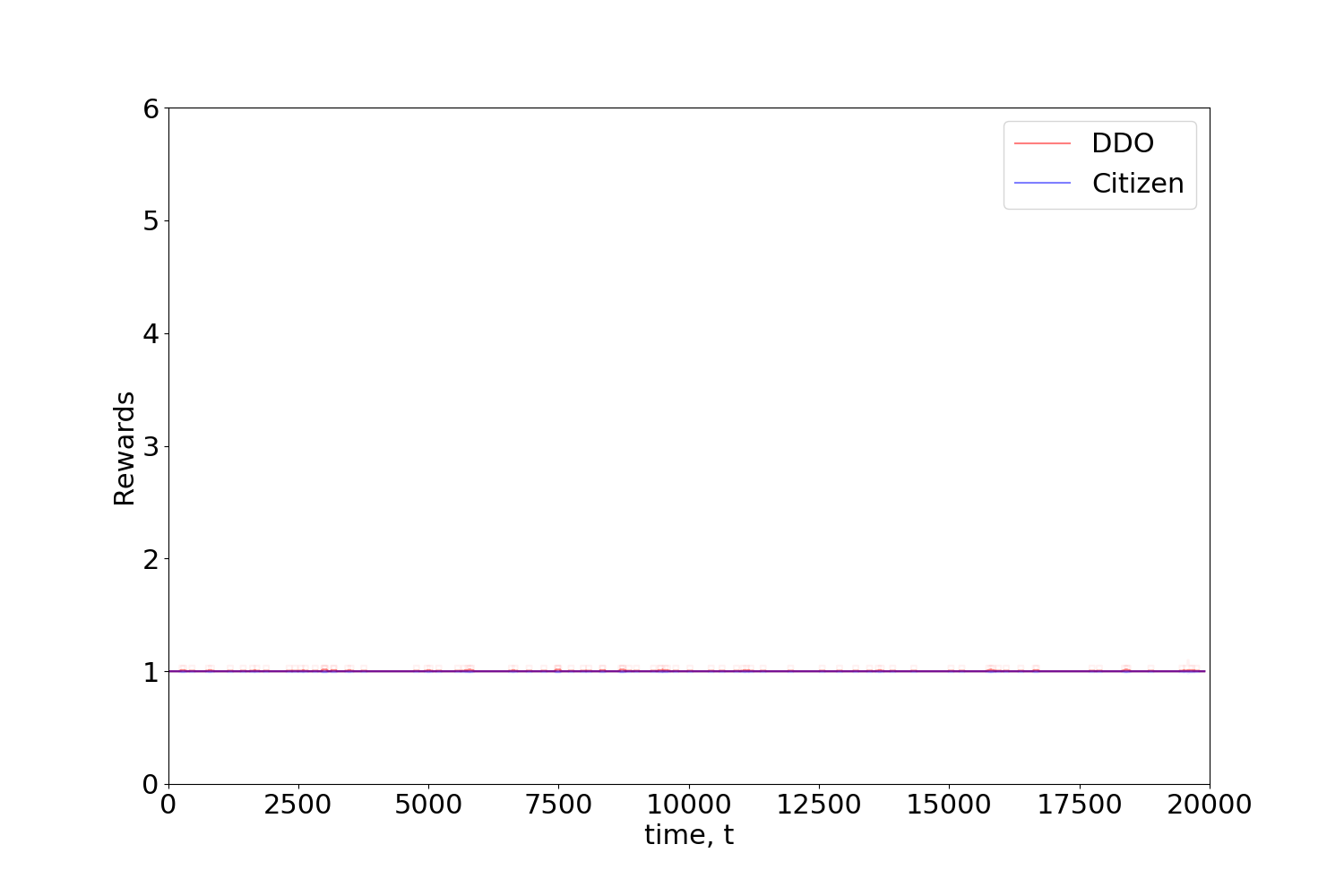}}%
  \subfigure[Social utility.]{%
  \label{fig:nash_sut}%
  \includegraphics[height=1.8in]{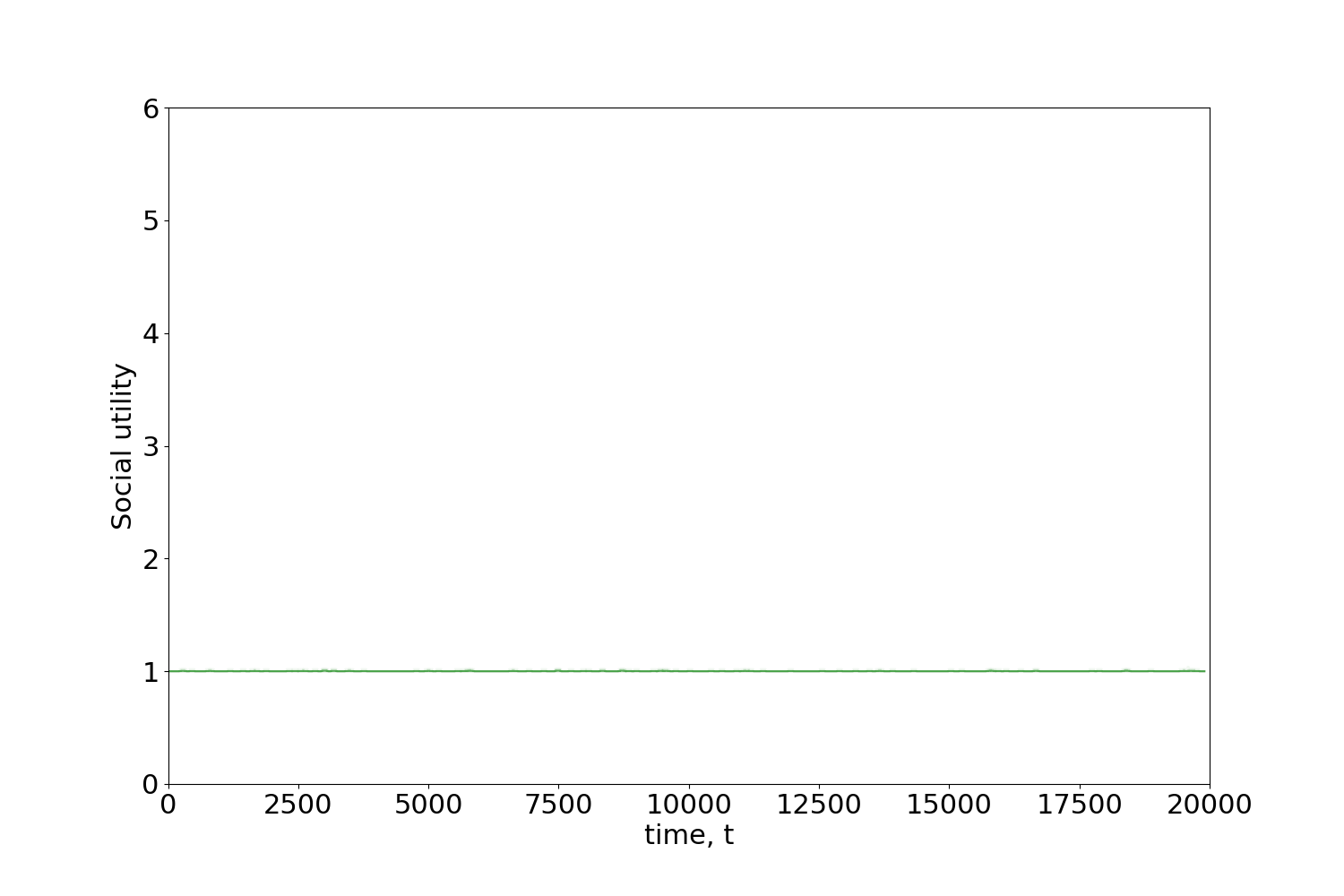}}%
  \caption{Agents' utilities and social utilities in case of DDO always defecting.}
\end{figure*}

\noindent Figure \ref{fig:nash_sut} shows that under the defecting strategy,
the social utility achieves its minimum value. 

\subparagraph{A Tit for Tat DDO.}
We next model the DDO as a player using the Tit for Tat (TfT) strategy
(it will first cooperate and, then, subsequently replicate the opponent's previous action: if the opponent was previously cooperative, the agent is cooperative; if not, it defects).
This policy has been widely used in the IPD, because of its simplicity and effectiveness \parencite{axelrod84}. A recent experimental study 
\parencite{dal2019strategy}
 tested real-life people's behaviour in IPD scenarios,
 showing that TfT was one of the most widely strategies.
 Figure \ref{fig:FPMvsTfT} shows that under TfT, the 
 social utility achieves its maximum value: mutual cooperation is achieved, thus leading to the optimal social utility.
It is important to mention though that 
if the citizen had no memory about previous actions, the policy of the DDO could not be learnt and mutual cooperation would not be achieved.

\begin{figure}[h!]
\centering
\includegraphics[width=0.6\linewidth]{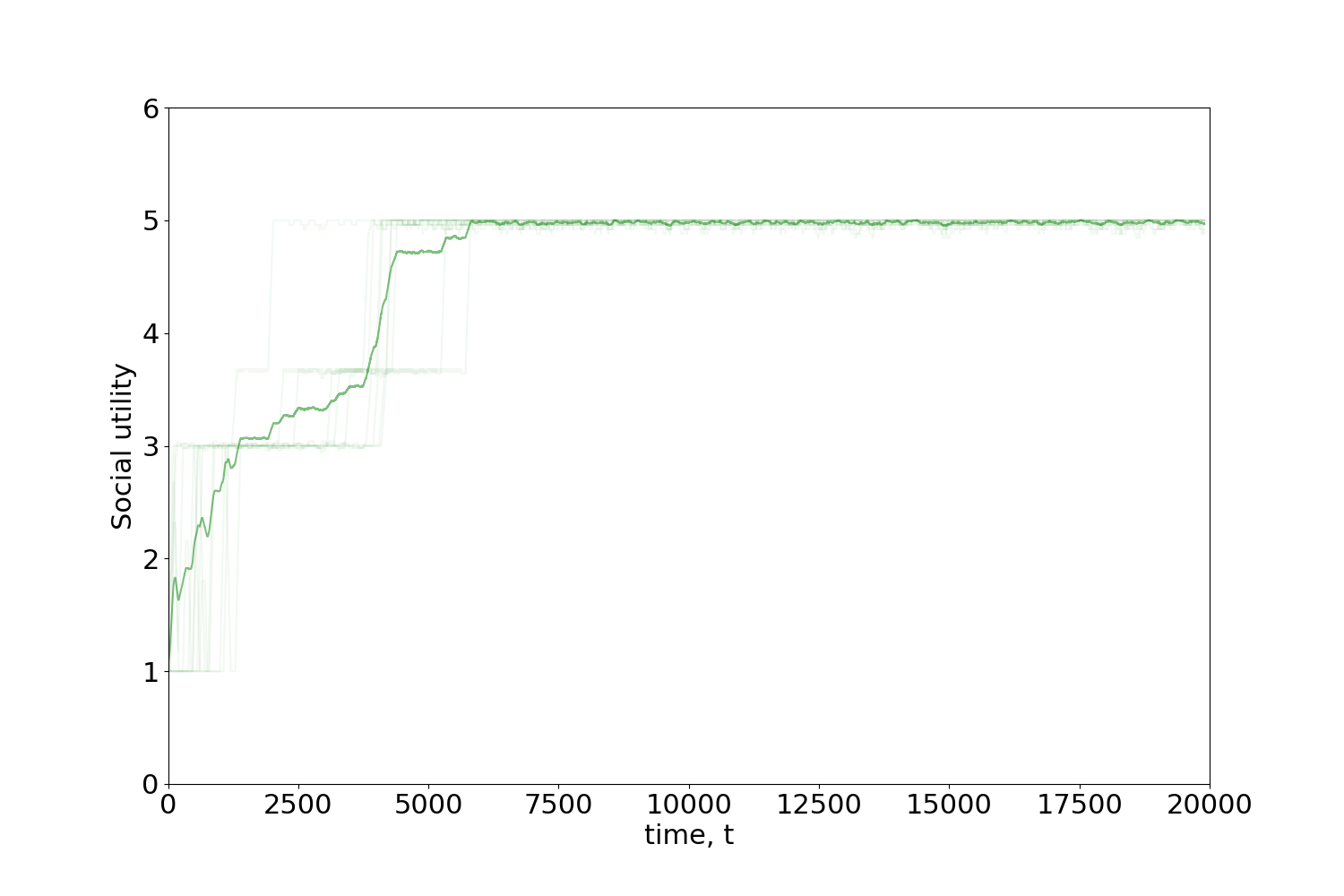}%
\caption{Social utility of a FPM citizen against a TfT DDO.}\label{fig:FPMvsTfT}
\end{figure}

\subparagraph{Random behaviour among citizens.}
 Previously,  all citizens were assumed to
 act according to the FPM model.
 However, 
 it is unrealistic to assume
 that the whole population will behave following such complex strategies. A more reasonable
 assumption is to consider having a  subpopulation of citizens that acts randomly. To simulate this, we modify the FP/FPM model drawing a random action with probability $0 < \epsilon < 1$ at each turn. As Figure \ref{fig:25} shows, where we set $\epsilon = 0.7$, this entails an important decrease in social utility. 


\subparagraph{A forgiving DDO.}
A possible solution for this decrease in social utility consists of forcing the DDO to eventually forgive the Citizen and play cooperate, regardless of
her previous actions. We model this as follows: with probability $p$ the DDO will cooperate, whereas with probability $1-p$ he will play TFT. 
%
%

\begin{figure}[h!]
\centering
\includegraphics[width=0.6\linewidth]{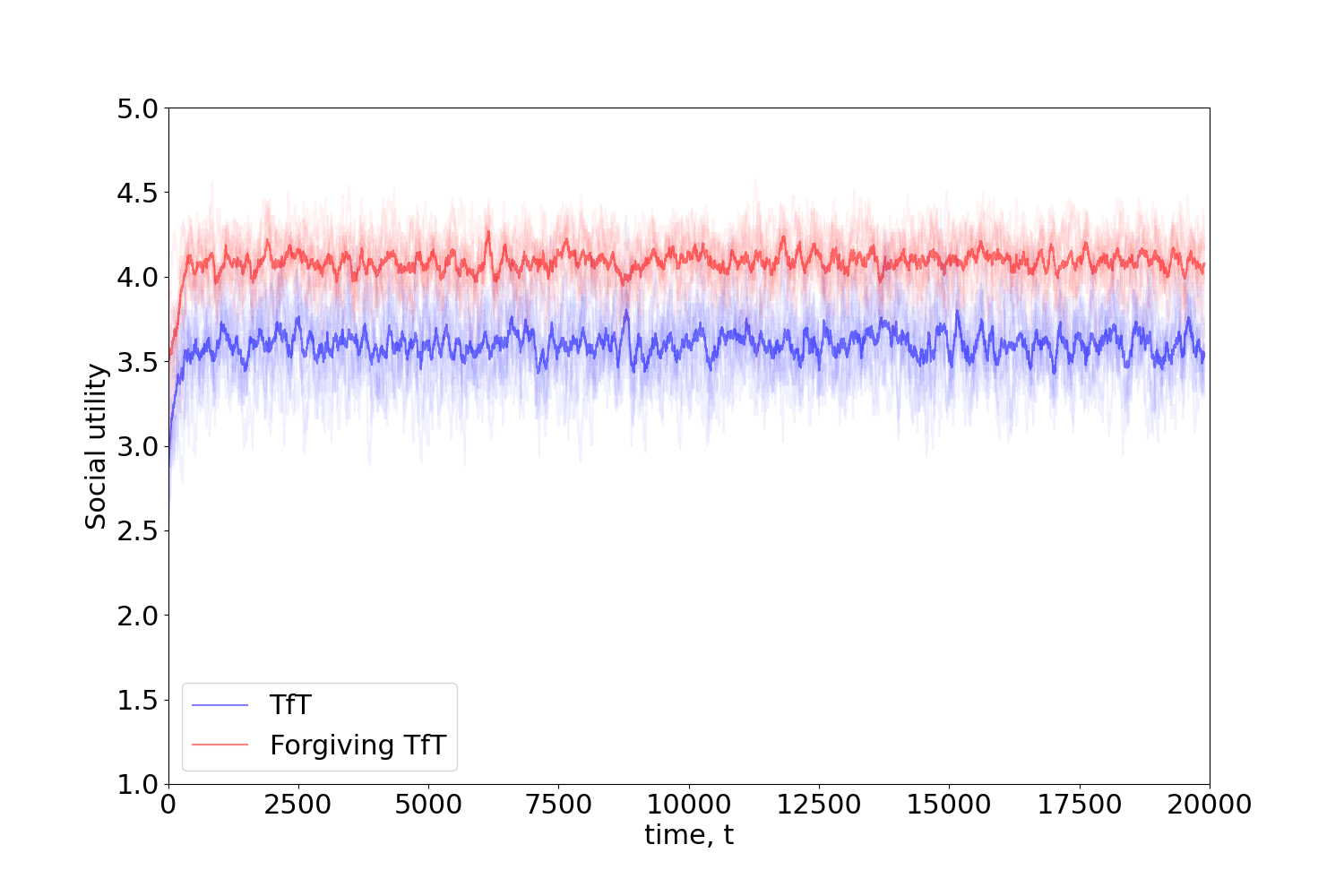}%
\caption{Social utility when citizens act randomly $70$ \% of the time against a TfT and a forgiving TfT DDO.}\label{fig:25}
\end{figure}

To assess what proportion of time should the DDO forgive, we evaluated a grid of values from 0 to 100, and chose the one that produced the highest increase in social utility. The optimal value was forgiving $70 \%$ of time.
As Figure \ref{fig:25} shows, this produces an increase of approximately half a unit in the average social utility with respect to the case of never forgiving.

Note, though, that there exists a limit value for the forgiving rate such that, if surpassed, the social utility will decrease to around 3. The reason for this is that, in this regime, when not acting randomly, the Citizen will learn that the DDO cooperates most of the time, and thus her  optimal strategy will be to defect. Thus, in most iterations the actions chosen will be $(C, D)$. 



\paragraph{Taxation through a Regulator.}\label{sec:regulator}

Let us discuss an alternative solution to promote cooperation introducing a third player, a Regulator (R, it). Its objective is to nudge the behaviour of the other players through utility transfer, based on taxes.   Appendix \ref{one-shot}
discusses a one-shot version identifying its equilibria.
Our focus 
is on the iterated version of this game.

At each turn, the regulator will choose a tax policy for the agents 
$$
(tax_{C, t}, tax_{DDO,t}) \sim \pi_R(\cdot | o_R, \theta_R), 
$$
where $o_R$ is the observed state of the game
and $\theta_R$ are relevant parameters for the regulator. 
Then, the other two agents will receive their corresponding adjusted utility $\tilde{r}_{a,t}$ through 
$$
\tilde{r}_{a,t} = r_{a, t} - tax_{a,t} + \frac{1}{2} \sum_a tax_{a, t},
$$
where the first term is the original utility (Table \ref{tab:payoffIPD});
the second is the tax that the regulator collects 
from that
agent; and, finally, the third one is the (evenly) redistributed collected 
reward.  Note that
$$
SU_t = \frac{r_{C, t} + r_{DDO, t}}{2} = \frac{\tilde{r}_{C, t} + \tilde{r}_{DDO, t}}{2}. 
$$
Thus, under this new reward regime, utility is not created nor destroyed, only transferred between players.

Let us focus now on the issue of how does the Regulator learn its
tax policy. For this, we make it another RL agent that maximizes the social welfare function,
thus optimizing its policy by solving 
$$
\max_{\theta_R} \mathbb{E}_{\pi_R} \left[ \sum_{t=0}^\infty \gamma^t SU_t \right].
$$
Therefore, two nested RL problems are considered: first, 
the regulator selects a tax regime and, next, the other two players optimally adjust their behaviour to this regime. After a few steps, 
the regulator updates its
policy to further encourage cooperation (higher $SU_t$), and so on. At the end of this process, we would expect  both players' behaviours to have been nudged towards cooperation.
 
 We thus frame learning as a bi-level RL problem with two nested loops, 
 using policy gradient methods:
\begin{enumerate}
    \item \textbf{(Outer loop)} The regulator 
    has parameters $\theta_R$, imposing a certain tax policy.
    \begin{enumerate}
        \item \textbf{(Inner loop)} The agents learn under this tax policy for $T$ iterations:
        \item They update their parameters: $\theta_{a, t+1} = \theta_{a, t} + \eta \nabla \mathbb{E}_{\pi_a} \left[ \sum_{t=0}^\infty \gamma^t r_{a, t} \right] $.
    \end{enumerate}
    \item The regulator updates its parameters: $\theta_{R, t+1} = \theta_{R, t} + \eta \nabla \mathbb{E}_{\pi_R} \left[ \sum_{t=0}^\infty \gamma^t SU_t \right]  $.
\end{enumerate}


Let us highlight a few benefits of this approach.
First, the regulator makes no assumptions about the policy models of the other players (thus it does not matter whether they are just single-RL agents or are opponent-modelling). Moreover,
the framework is also agnostic to the \emph{social welfare function} to
be optimized; for simplicity, we just use expression (\ref{eq:su}).
     It is also scalable to more than two players: 
     the regulator only needs to collect taxes for each player, and then redistribute wealth.
   In presence of $k > 2$ agents, it would have to split the sum of taxes by $1/k$. 

\paragraph{Experiments for the Regulator setting.}

This experiment 
illustrates 
how the inclusion of a Regulator encourages the emergence of cooperative behavior.

Consider the interactions between a Citizen and a DDO. 
The parameter for 
each player is a vector $\theta_a \in \mathbb{R}^2$, with $a \in \lbrace C, DDO\rbrace $, representing the logits of choosing the actions,
i.e. the unnormalized probabilities of choosing each decision.
We consider two types of regulators.

The first one has a discrete action space defined through
$$
        tax_{a,t} =  \begin{cases} 0.00 & \text{if } a_R = 0\\
        0.15\cdot r_{a,t} & \text{if } a_R = 1\\
        0.30\cdot r_{a,t} & \text{if } a_R = 2\\
        0.50 \cdot r_{a,t} & \text{if } a_R = 3. \end{cases}
$$
For example, when $a_R=2$ the tax rate reaches 30\%.
In this case, $\theta_R \in \mathbb{R}^4$ represent the 
logits of a categorical 
random variable taking the previous values (0,1,2,3).

The second regulator adopts a Gaussian policy defined 
through $\pi_R(d_R | o_R, \theta_R) \sim \mathcal{N}(d_R | \theta_R, 0.05^2)$, with tax 
$$
tax_{a,t} = 0.5 \cdot sigmoid(d_R) \cdot r_{a,t},
$$
to allow 
for a continuous range in $\left[ 0, 0.5 \right]$.

Experiments run for $T=1000$ iterations.
After each iteration, both agents perform one update of
their policy parameter gradient.
The regulator updates its parameters using policy gradients every 50 iterations.
The decision of updating the regulator less frequently than the other agents is motivated to allow them to learn and adapt to the new tax regime and stabilise
 overall learning of the system.
Figure \ref{fig:su} displays results. 
 For each of the three variants (no intervention, discrete, continuous) we plot 5 different runs and their means in darker color.
 
\begin{figure}[!h]
\centering
\includegraphics[width=0.55\linewidth]{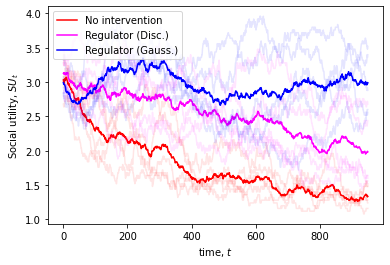}
\caption{Social utility under three different regulation scenarios.}\label{fig:su}
\end{figure}

 Clearly, under no intervention, both agents fail to learn to cooperate
converging to the static 
Nash equilibrium $(D,D)$. 
We also appreciate that the discrete policy is neither 
effective, also converging to $(D,D)$, albeit at a much slower pace.
On the other hand, the Gaussian regulator is 
 more efficient as it allows to avoid convergence to $(D,D)$
although it does not preclude convergence to $(C,C)$.
This regulator is more effective than its discrete counterpart,
because it can better exploit 
the policy gradient information. Because of this, 
in the next subsection we focus on this Gaussian regulator.

In summary, the addition of a Regulator can make a positive impact in the social utility attained in the market, preventing collapse into $(D, D)$. However, introducing taxes is not sufficient, since in Figure \ref{fig:su} the social utility converged towards a value of 3, far away from the optimal value of 5.

\paragraph{Introducing incentives.}\label{sec:incentives}

In order to further stimulate cooperative behavior, we introduce incentives to the players via the Regulator: if both players cooperate at a given turn, they will receive an extra amount $I$ of utility,
 a scalar that adds to their perceived rewards. Appendix  \ref{sec:oneshot_inc} shows that incentives complement well with the tax framework, so that mutual cooperation is possible in the one-shot version of this game. Note that, when $I>T-R$, instead of the Prisoner's Dilemma, we have an instance of the Stag Hunt game \parencite{skyrms2004stag}, in which both $(C, C)$ and $(D, D)$ are 
 pure Nash equilibria.\footnote{Achieving mutual cooperation is much simpler in this case.}

As before, we focus the discussion in the iterated version.
In this batch of experiments, players interact over $T=1000$ iterations, and the Regulator only provides incentives during the first 500 iterations. After that, it will only collect taxes from the players and redistribute them as in Section \ref{sec:regulator}. Figure \ref{fig:inc1} shows
results from 
several runs under different incentive values. A few comments are in order.

\begin{figure}[!h]
\centering
\includegraphics[width=0.6\linewidth]{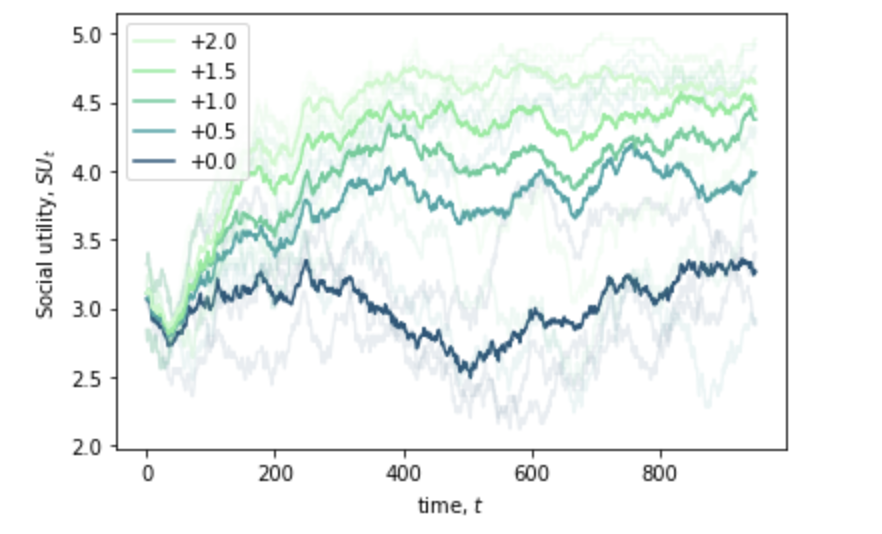}
\caption{Social utility under different incentives with tax collection.}\label{fig:inc1}
\end{figure}

\begin{figure}[!h]
\centering
\includegraphics[width=0.6\linewidth]{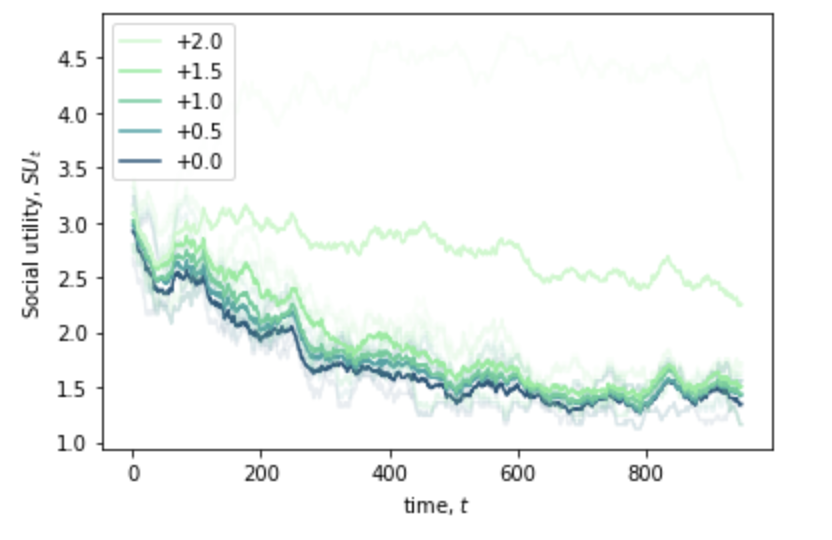}
\caption{Social utility under different incentives with no tax collection.}\label{fig:inc2}
\end{figure}

Firstly, note that as the incentive increases, also does the social utility.
For an incentive of 1, the maximum reward of $(C, C)$ and $(D, C)$ is the same (6) for the Citizen, and cooperation naturally emerges. Also note that since the policies for each player are stochastic, it is virtually impossible to maintain an exact convergence towards the optimal value of 5, since a small amount of time the agents are deviating from $(C,C)$ due to the stochasticity in their actions. Second,
observe that even when the Regulator stops incentives to players in the middle of the simulations, both players keep cooperating along time.

We hypothesize that the underlying tax system from Section \ref{sec:regulator} is necessary for players to learn to cooperate and maintain that behaviour even after the Regulator ends up incentives. To test this hypothesis, we repeat the experiments removing tax collection, \emph{ceteris paribus}. Results are shown in Figure \ref{fig:inc2}. Observe now that even under the presence of high incentives, both agents fail to cooperate, with social utility decaying over time. Thus,  the tax collection framework from \ref{sec:regulator} has a synergic effect with the incentives introduced in this Section.

\section{Summary}

In this chapter we have introduced a novel formalism, extending MDPs into TMDPs, in order to account for adversaries interfering with the reward generation process. We have also proposed an updated learning algorithm based on Q-learning and level-k thinking to solve TMDPs in an optimal fashion. This framework shows good performance and defence capabilities based on the different experimental settings we have evaluated it.

Lastly, we studied a concrete case of a data sharing game, modeled after the iterated prisoner's dilemma. We proposed several strategies based on Multi-agent reinforcement learning and analyzed how cooperation between the two players could be achieved.

\begin{subappendices}

\section{Convergence proof of update rule for TDMPs}\label{sec:p}

Consider an augmented state space so that transitions are of the form
$$
 (s,b) \xrightarrow[]{a} (s',b') \xrightarrow[]{a'} \ldots.
$$
In our setting, the DM does not observe the full state since she does
not know the action $b$ taken by her adversary. However, if she knows his policy $p(b|s)$, or has a good estimate of it, she can take advantage of this information.

Assume for now that we know the opponent's action $b$. The $Q$-function would satisfy the following recursive update \parencite{sutton2012reinforcement}, when the DM is acting under a policy $\pi$,

\begin{eqnarray*}
Q^\pi (s,a,b) &=  \sum_{s'} \sum_{b'} p(s', b' |s,a,b) \left[ R_{ss'}^{ab} +  \mathbb{E}_{\pi(a'|s',b')} \left[ Q^\pi(s',a',b') \right] \right],\\
\end{eqnarray*}
where we explicitly take into account the structure of the state space and 
use 
$$
R_{ss'}^{ab} = \mathbb{E}\left[ r_{t+1}|s_{t+1} = s', s_{t} = s, a_t = a, b_t = b \right].
$$ 
The next opponent action is conditionally independent of the 
current DM action, state and his own action. We thus write $p(s', b' |s,a,b) = p(b'|s')p(s'|s,a,b)$. Then
\begin{eqnarray*}
Q^\pi (s,a,b) &=  \sum_{s'} p(s'|s,a,b) \left[ R_{ss'}^{ab} +  \mathbb{E}_{p(b'|s')} \mathbb{E}_{\pi(a'|s',b')} \left[ Q^\pi(s',a',b') \right] \right],\\
\end{eqnarray*}
as $R_{ss'}^{ab}$ does not depend on the next opponent action $b'$. Finally, the optimal $Q$-function verifies
$$
Q^*(s,a,b) =  \sum_{s'} p(s'|s,a,b) \left[ R_{ss'}^{ab} + \gamma \max_{a'} \mathbb{E}_{p(b'|s')} \left[ Q^*(s',a',b') \right] \right],
$$
since in this case $\pi(a|s) = \argmax_a Q^*(s,a)$. Observe now that:

\begin{lemma}\label{lema:1}
Given $q: \mathcal{S} \times \mathcal{B} \times \mathcal{A} \rightarrow \mathbb{R}$,
the operator $\mathcal{H}$ 
\begin{flalign*}
(\mathcal{H}q) (s,b,a) = \sum_{s'} p(s'|s,b,a) \big[ r(s,b,a) +\gamma \max_{a'} \mathbb{E}_{p(b'|s')} q(s',b',a') \big].
\end{flalign*}
is a contraction mapping under the supremum norm.
\end{lemma}
\begin{proof}
We prove that $\| \mathcal{H}q_1 - \mathcal{H}q_2 \|_{\infty} \leq \gamma \| q_1 - q_2 \|_{\infty}$.
\begin{flalign*}
& \| \mathcal{H}q_1 - \mathcal{H}q_2 \|_{\infty} = \\
&= \max_{s,b,a}  \Big\vert \sum_{s'} p(s'|s,b,a) \big[ r(s,b,a) + \gamma \max_{a'} \mathbb{E}_{p(b'|s')} q_1(s',b',a') -r(s,b,a) - \gamma \max_{a'} \mathbb{E}_{p(b'|s')} q_2(s',b',a') \big]  \Big\vert = \\
&= \gamma \max_{s,b,a}   \Big\vert \sum_{s'} p(s'|s,b,a) \big[   \max_{a'} \mathbb{E}_{p(b'|s')} q_1(s',b',a')  -\max_{a'} \mathbb{E}_{p(b'|s')} q_2(s',b',a')\big]  \Big\vert \leq \\
&= \gamma \max_{s,b,a} \sum_{s'} p(s'|s,b,a)  \Big\vert  \max_{a'} \mathbb{E}_{p(b'|s')} q_1(s',b',a') - \max_{a'} \mathbb{E}_{p(b'|s')} q_2(s',b',a') \Big\vert \leq \\
&= \gamma \max_{s,b,a} \sum_{s'} p(s'|s,b,a)  \max_{a',z}  \Big\vert \mathbb{E}_{p(b'|z)} q_1(z,b',a') - \mathbb{E}_{p(b'|z)} q_2(z,b',a')  \Big\vert \leq \\
&= \gamma \max_{s,b,a} \sum_{s'} p(s'|s,b,a)  \max_{a',z,b'}  \Big\vert q_1(z,b',a') -  q_2(z,b',a')  \Big\vert =\\
&= \gamma \max_{s,b,a} \sum_{s'} p(s'|s,b,a)  \|  q_1 -  q_2 \|_{\infty} =  \gamma \|  q_1 -  q_2 \|_{\infty}.\hspace{7cm}\qedhere
\end{flalign*}
\end{proof}
Then, using the proposed learning rule (\ref{eq:lr}), we would converge
to the optimal $Q$ for each of the opponent actions. The proof follows directly from the standard $Q$-learning convergence proof, see e.g. \parencite{melo2001convergence}, and making use of Lemma 1.
However, at the time of making the decision, we do not know what action
he would take. Thus, we average over the possible opponent
actions, weighting each of the actions by $p(b|s)$, as in \eqref{eq:lr2}.

\section{One-shot game for the centralized case}\label{one-shot}

    


We model the one-shot version of the centralized case game as a
three-agent sequential game.
The regulator acts first choosing a tax policy; 
after observing it, the agents take their actions. 
Introducing a regulator can foster cooperation in the one shot game.

For simplicity, consider the following policy: the regulator will retain a percentage $x$ of the reward if the agent decides to defect, and 0 if it decides to cooperate. Then, the regulator will share evenly the amount collected between 
both agents. With this, given the regulator's action $x$, the payoff matrix 
is as in Table \ref{tab:payoffIPD2}, recalling that  $T > R > P >S$.
    
    \begin{table}[h!]
    \begin{center}
    \begin{tabular}{cl|lll}
    \multicolumn{1}{l}{}                                   &     & \multicolumn{3}{l}{\textbf{DDO}} \\ \cline{3-5} 
    \multicolumn{1}{l}{}                                   &     & $C$         &       & $D$        \\ \hline
    \multicolumn{1}{c|}{\textbf{Citizen}} & $C$ & $R,R$       &       & $S',T'$      \\
    \multicolumn{1}{c|}{}                                  &     &             &       &            \\
    \multicolumn{1}{c|}{}                                  & $D$ & $T',S'$       &       & $P,P$     
    \end{tabular} 
    \end{center}
    \caption{Utilities for the data sharing game}
    \label{tab:payoffIPD2}
    \vspace{-2ex}
    \end{table}
    
    Assume that if one agent defects and the other cooperates, the first one will receive a higher payoff, that is $T' > S'$, which means that $x < 1 - \frac{S}{T}$. Depending on $x$, three scenarios arise:
    \begin{enumerate}
        \item $T' > R > P > S' \iff x < 2 \left[ \frac{P-S}{T}\right]$. This is
        equivalent to the prisoner's dilemma. $(D,D)$ strictly dominates, thus being the unique Nash Equilibrium. 
        
        \item $R > T' > S' > P \iff x > 2 \left[ 1- \frac{R}{T}\right]$. In this case, $(C,C)$ strictly dominates, becoming the unique Nash Equilibrium. 
        
        \item $T' > R > S' > P \iff x \in \left( 2 \left[ \frac{P-S}{T}\right], 2 \left[ 1- \frac{R}{T}\right] \right) $. This is a coordination game. There are two possible Nash Equilibria with pure strategies $(C, D)$ and $(D, C)$.
    \end{enumerate}
    
    Moving backwards, consider the regulator's decision. Recall that
    R maximizes social utility. Again, three scenarios emerge:
   \begin{enumerate}
        \item $x < 2 \left[ \frac{P-S}{T}\right]$. The social utility is $P$.
        
        \item $x > 2 \left[ 1- \frac{R}{T}\right]$. The social utility is $R$.
        
        \item $x \in \left( 2 \left[ \frac{P-S}{T}\right], 2 \left[ 1- \frac{R}{T}\right] \right) $. The social utility is $\frac{S+T}{2}$.
    \end{enumerate}
    
    As $R > P$ and $R > \frac{T+S}{2}$ (as requested in the IPD),
    the regulator maximizes his payoff choosing $x > 2 \left[ 1- \frac{R}{T}\right]$. Therefore, $(x, C, C)$, with $x > 2 \left[ 1- \frac{R}{T}\right]$ is a subgame perfect equilibrium, and 
    we can foster cooperation in the one-shot version of the game.
    
\section{One-shot game for the centralized case plus incentives}\label{sec:oneshot_inc}

Under this scenario, we consider the reward bimatrix in Table \ref{tab:payoffIPD_inc}, 
where $I$ is the incentive introduced by the Regulator. 

 \begin{table}[h!]
    \begin{center}
    \begin{tabular}{cl|lll}
    \multicolumn{1}{l}{}                                   &     & \multicolumn{3}{l}{\textbf{DDO}} \\ \cline{3-5} 
    \multicolumn{1}{l}{}                                   &     & $C$         &       & $D$        \\ \hline
    \multicolumn{1}{c|}{\textbf{Citizen}} & $C$ & $R+I,R+I$       &       & $S,T$      \\
    \multicolumn{1}{c|}{}                                  &     &             &       &            \\
    \multicolumn{1}{c|}{}                                  & $D$ & $T,S$       &       & $P,P$     
    \end{tabular} 
    \end{center}
    \caption{Utilities for the data sharing game with incentives}
    \label{tab:payoffIPD_inc}
    \vspace{-2ex}
    \end{table}

Consider the case in which the agents take 
the $(C, D)$ pair of actions. In this case, they perceive rewards $(S, T)$.
After tax collection and distribution, it leads to
$(S - \frac{Sx}{2} +  \frac{Tx}{2}, T -  \frac{Tx}{2} +  \frac{Sx}{2})$, with $x$ being the tax rate collected by the Regulator. In order to ensure 
that $(C,C)$ is a Nash equilibrium, two conditions must hold:
\begin{itemize}
    \item $S - \frac{Sx}{2} +  \frac{Tx}{2} > P$, so that agents do not 
    switch from $(C,D)$ to $(D,D)$. This simplifies to $x > 2\frac{(P-S)}{T-S}$.
    \item $R + I > T -  \frac{Tx}{2} +  \frac{Sx}{2}$, so that 
    the agents do not switch from $(C, C)$ to $(C,D)$. This simplifies to $x > 2\frac{T - (R+I)}{T-S}$.
\end{itemize}

This shows that even if the gap between $T$ and $R$ is large, with the
aid of incentives both agents could reach mutual cooperation, also under a tax framework, since $I$ can grow arbitrarily to ignore the second restriction.

\end{subappendices}

\chapter{Conclusions}\label{cha:conclusions}
\section{Summary}

We end up this thesis by summarising results and suggesting a few challenges.
After several waves of popularity, NN models seem to 
have reached a definitive momentum because of the many relevant applications
based on them. Most work in NNs is based on the MLE tradition.
We have highlighted and illustrated the benefits of a Bayesian treatment of deep learning. Indeed, as we have described,
they provide improved uncertainty estimates;
they have enhanced generalization capabilities; 
they have enhanced robustness against adversarial attacks.
We also note that, although not studied in this thesis, 
they have improved capabilities for model calibration;
and the use of sparsity-inducing priors, could further induce 
improvements in learning. 
However, efficient Bayesian integration methods in 
deep NN are 
still to be found, this remaining a major challenge.
In particular their solution would facilitate the
development of probabilistic programming languages \parencite{gordon2014probabilistic,carpenter2017stan,wood2014new}
as the next step for differentiable programming,
leading to new tools for contemporary Bayesian 
statistics.

\subsection{Large Scale Bayesian Inference}\label{sec:conclusion_lsb}

Chapter 2 showed how to generate new SG-MCMC methods, such as SGLD+R and Adam+NR, consisting of multiple chains plus repulsion between  particles. Instead of a naive parallelization, in which a particle from a chain is agnostic to the others, we showed how it is possible to adapt another method from the literature, SVGD, to account for a better exploration of the space, avoiding between particle collapse. We also showed how momentum-accelerated extensions of SGD can be used as SG-MCMC samplers, while also being compatible with the extension to repulsive forces between parallel chains. Our 
experiments show that the proposed ideas improve efficiency when dealing 
with large scale inference and prediction problems in presence of many 
parameters and large data sets.

In the second part, we  proposed VIS, a flexible and efficient framework to perform
large-scale Bayesian inference in probabilistic models. The scheme benefits from useful properties and can be 
employed to efficiently perform inference with a wide class of models such as state-space time series, variational autoencoders {and variants such as the conditioned VAE for classification tasks}, defined through continuous, high-dimensional distributions.
The framework can be seen as a general 
approach to tuning MCMC sampler parameters, adapting the initial distributions and learning rate. 
Key to the success and applicability of  VIS  are the ELBO approximations based on the introduced refined variational approximation, which are computationally cheap but convenient. We note that both approaches developed in this Chapter could be combined, leading to en even better behaviour.

\subsection{Adversarial Classification}

Adversarial classification aims at enhancing classifiers to achieve robustness in presence of adversarial examples, as usually encountered in many security applications. The pioneering work of \textcite{dalvi2004adversarial} framed most later approaches to AC 
within the standard game theoretic paradigm, in spite of the unrealistic common knowledge assumptions about shared beliefs and preferences required, actually even questioned by those authors. After reviewing them, and analysing  their assumptions, we have presented two formal probabilistic approaches for AC based on ARA that mitigate 
such strong common knowledge assumptions.
They are general in the sense that application-specific assumptions are kept to a minimum. 
We have presented the framework in two different forms: in Section \ref{sec:ac_acra}, learning about the adversary is performed in the
operational phase, for generative classifiers. In Section \ref{sec:scalable}, adversarial aspects are incorporated in the training phase. Depending on the particular application, one of the frameworks could be preferred over the other. The first one allows us to make real time inference about the adversary, as it explicitly models his decision making process during operations; its adaptability is better as it does not need to be retrained every time we need to modify the adversary model. However, this comes at a high computational cost, and the harsh restriction of being only applicable to generative models. In applications in which there is a computational bottleneck, the second approach may be preferable, with possible changes in the adversary's behaviour incorporated via retraining.
This tension between the need to robustify algorithms against attacks (training phase, Section  \ref{sec:scalable}) and the fast adaptivity of attackers against defences (operational phase, Section \ref{sec:ac_acra}) is well exemplified in the phishing detection domain as discussed e.g. in \parencite{rakesh}.

\subsection{Adversarial aspects in Reinforcement Learning}\label{sec:final}
We have introduced TMDPs, a reformulation of MDPs to 
 support decision makers who confront opponents that interfere 
with the reward generating process in RL settings.
They have potential applications in security, cybersecurity and 
competitive marketing, to name but a few.
TMDPs aim at providing one-sided prescriptive support to a RL agent, maximizing her expected utility, taking into account potential negative actions adopted by an adversary. 
  The proposed learning rule is a contraction mapping and we may use RL convergence results, while gaining advantage from 
  opponent modelling within $Q$-learning. Indeed, we proposed a scheme to model adversarial behavior based on level-$k$ reasoning
and extended it by using type-based reasoning to account for uncertainty about the opponent's level. Finally, we have sketched how the framework could be extended to deep learning settings and to the multiple opponents case. Key features of our proposal are its generality
(it is model agnostic as it is compatible with tabular $Q$-learning or any function approximator for the $Q$-values) and its robustness, since it also offers 
protection against adversarial behaviour that is not exactly as described by the DM's opponent model. These significant benefits come at a reasonable cost, since the increase in complexity (both in time and space) is linear compared to unprotected, baseline vanilla $Q$-learners.

Empirical evidence is provided via extensive experiments, with encouraging results. In security settings, we see that by explicitly modelling a finite set of adversaries via the opponent averaging scheme, a supported DM can take advantage of her actual opponent, even when he is not explicitly modelled through a component from the finite mixture. This highlights the ability of our framework to generalize between different kinds of opponents. As take home lesson, we find that a level-2 $Q$-learner may effectively deal with a wide class of adversaries. However, maintaining a mixture of different adversaries is necessary if we consider a level-3 DM. As a rule of thumb, the supported DM may start at a low level in the hierarchy, and switch to a level-up temporarily, to check if the obtained rewards are higher. Otherwise, she would  continue on the initial, lower level.  
Indeed, the proposed scheme is
model agnostic, so we expect it to be usable in both shallow and deep multi-agent RL settings, such as the ones pioneered by \parencite{mnih2015human}. This is a desirable key property of our framework implying that it can be adopted in a wide array of relevant settings and configurations.



Regarding  data sharing games, it can be useful to recall that a defining trend in modern society is the abundance 
of data which opens up new opportunities, challenges
and threats. In the upcoming years, social progress will be essentially conditioned by the capacity of society to gather, analyze and understand data, as this will 
facilitate better and more informed decisions. 
Thus, to guarantee social progress,  efficient mechanisms
for data sharing are key. Obviously, such mechanisms
should not only facilitate the data sharing process, 
but must also guarantee the protection of the citizen's personal
information. As a consequence, the problem of data sharing not only
has importance from a socioeconomic perspective, but also from the
legislative point of view. This is well described in numerous recent legislative pieces from the 
EU, e.g.\ \parencite{europe1}, as well as in the concept of flourishing 
in a data-enabled society \parencite{allea}. 


We have studied the problem of data sharing
from a dynamic game theoretic perspective with two agents.
Within our setting,  mutual cooperation emerges as the strategy 
leading to the best social outcome, and it must be somehow promoted. We
have proposed modelling the confrontation between dominant data owners and citizens using two versions of the iterated prisoner dilemma via  
multi agent reinforcement learning: the decentralized case, in which both agents interact freely, and the centralized case, in which the interaction is regulated by an external agent/institution. In the first case, we have shown that there are strategies with which mutual cooperation is possible, and that a forgiving policy by the DDO can be beneficial in terms of social utility. In the centralized case, regulating the interaction between citizens and DDOs via an external agent could foster mutual cooperation through taxes and incentives.

\section{Further work}
Several avenues are open for further work. Here we discuss only several promising ones, following our three core chapters.

\subsection{Large Scale Bayesian Inference}
 First, with a very large particle regime (more than 100 particles) there is room to use approximating algorithms such as Barnes-Hutt to keep the computational cost tractable.
Secondly, we used the RBF kernel in all our experiments, but a natural 
issue to address would be to define a parameterized kernel $k_{\theta} (z_i, z_j)$ and learn the parameters $\theta$ on the go to optimize the ESS/s rate, using meta-learning approaches such as the one proposed in \textcite{gallego2019vis} for the SGLD case.

When dealing with shallow networks in comparatively small scale problems
\parencite{muller1998issues} dealing with acceptance Metropolis steps was crucial, for example,
when focusing on architecture selection; incorporating such steps
to the proposed approaches could be beneficial. 
If many more particles are  used, one could approximate the expectation in Eq.  (\ref{eq:svgd_mat}) using subsampling at each iteration, as proposed by the authors of SVGD, or by using more sophisticated approaches from the molecular dynamics literature, such as the \textcite{barnes1986hierarchical} algorithm, to arrive at an efficient $\mathcal{O}(L \log L)$ computational burden at a negligible approximation error.

Better estimates of the refined density and its gradient may be a fruitful line of research, such as the spectral estimator used in \textcite{shi2018spectral}. Another alternative is to use a deterministic flow (such as SGD or SVGD), keeping
track of the change in entropy at each iteration using the change of the variable formula, as in \textcite{duvenaud2016early}. However, this requires a costly Jacobian computation, making it unfeasible to combine with our  back-propagation through the sampler approach (Section \ref{sec:tuning}) for moderately complex problems. We leave this for future exploration. {Another interesting and useful line of further research would be to tackle the case in VIS in which the latent variables $z$ are discrete. This would entail adapting the automatic differentiation techniques to be able to back-propagate the gradients through the sequences of acceptance steps necessary in Metropolis--Hastings samplers.}

In order to deal with the implicit variational density in the VIS framework, it may be worthwhile to consider optimizing the Fenchel dual of the KL divergence, as
 in \parencite{fang2019implicit}. However, this requires the use of an auxiliary neural network, which may entail a large computational price compared with our simpler particle approximation.

Lastly, probabilistic programming offers powerful tools for Bayesian modeling.
A~PPL can be viewed as a programming language extended with random sampling and Bayesian conditioning capabilities, complemented with an inference engine that produces answers to inference, prediction and decision-making queries. Examples 
include WinBUGS~\parencite{lunn2000winbugs}, Stan \parencite{carpenter2017stan} or the recent Edward \parencite{tran2018simple} and Pyro \parencite{bingham2018pyro} languages. We plan to adapt VIS into several PPLs to facilitate the adoption of the framework.

\subsection{Adversarial Classification}
Our AC framework may be extended in several ways. First, we could adapt the proposed approach to situations in which there is repeated play of the AC game, introducing the possibility of learning the adversarial utilities and probabilities in a Bayesian way. Learning over opponent types has been explored with success in reinforcement learning scenarios, see Chapter 4. This could be extended to the classification setting. Besides exploratory ones, attacks over the training data \parencite{biggio2012poisoning} may be relevant in certain contexts. 
In addition, the extension to the case of attacks to innocent instances (not just integrity violation ones) seems feasible using the scalable framework. 
We have restricted our attention to deterministic attacks, that is, $a^*(x,y_1)$ will always lead to $x'$; extending our framework to deal with stochastic attacks would entail modelling $p(x' \vert a^*, x, y_1)$. 
Additional work should be undertaken concerning the
algorithmic aspects.  In our approach  we go through a simulation stage to forecast attacks and an optimisation stage to determine optimal classification. The whole process might be performed through a single stage, possibly based on augmented probability simulation \parencite{ekin2019augmented}.

We have also shown how the robustification procedure from Section \ref{sec:scalable} can be an efficient way to protect large-scale models, such as those trained using first-order methods. It is well-known that Bayesian marginalisation improves generalisation capabilities of flexible models since the ensemble helps in better exploring the posterior parameter space \parencite{wilson2020bayesian}. Our experiments suggest 
that this holds also in the domain of adversarial robustness. Thus, bridging the gap between large scale Bayesian methods and Game Theory, as  done in the ARA framework, suggests a powerful way to develop principled defences. To this end, strategies to more efficiently explore the highly complex, multimodal posterior distributions of neural models, such as the ones developed in Chapter 2, can be adopted.

Lastly, several application areas could benefit highly from
protecting their underlying ML models. Spam detectors were
the running example in Chapter 3. Malware and phishing detection are two
crucial cybersecurity problems in which the data distribution of computer programs is constantly changing, driven by attacker's interests in evading detectors. 

\subsection{Adversarial aspects in Reinforcement Learning}

Several lines of work are possible for further research. First of all, in the experiments, we have just considered DMs up to level-3,
though the extension to higher order adversaries is relevant. 
In addition, rather than trying to learn opponent $Q$-values, we could
use policy gradient methods \parencite{baxter2000direct} to expand our proposal. In the recent years, there has been a great deal of literature focusing on deep policy gradient methods, see e.g.  \textcite{pmlr-v48-mniha16,schulman2017ppo}, so it is natural to adapt our opponent modeling framework to account for these kinds of learning agents.
It also might be interesting to explore similar expansions to semi-MDPs, 
in order to perform hierarchical RL or allow for time-dependent rewards 
and transitions between states.

The developed framework could have several applications in the financial worlds, for instance, in mutual funds. Managing such a portfolio can be seen as making decisions over time, and with our framework we could take into account the presence of competing funds.

 Besides fostering cooperation,  the data sharing game may be seen as
  an instance of a two sided market \parencite{rochet2006two}. Therefore,
  the creation of intermediary platforms that facilitate the connection between dominant data owners and citizens to enable data sharing
  would be key to guarantee social progress.

\phantomsection
\addcontentsline{toc}{chapter}{\bibname}
\newrefcontext[sorting=nty]

\printbibliography

\end{document}